\documentclass{article}

\usepackage{arxiv}

\usepackage[utf8]{inputenc} % allow utf-8 input
\usepackage[T1]{fontenc}    % use 8-bit T1 fonts
\usepackage{hyperref}       % hyperlinks
\usepackage{url}            % simple URL typesetting
\usepackage{booktabs}       % professional-quality tables
\usepackage{amsfonts}       % blackboard math symbols
\usepackage{nicefrac}       % compact symbols for 1/2, etc.
\usepackage{microtype}      % microtypography
\usepackage{xcolor}         % colors

\usepackage{graphicx}
\usepackage{xcolor}
\usepackage{xspace}
\usepackage{doi}
\usepackage[ruled,vlined,linesnumbered]{algorithm2e}
\usepackage{svg}
\usepackage{comment}
\usepackage{caption}
\usepackage[normalem]{ulem}
\usepackage{amsmath,amsthm,mathtools,tensor,bm,cancel}
\DeclarePairedDelimiterX{\norm}[1]{\lVert}{\rVert}{#1}

\DeclareMathOperator*{\argmax}{argmax}
\usepackage{amssymb}
\usepackage{caption}
\usepackage{subcaption}
\usepackage[inline]{enumitem}
\usepackage{thm-restate}

\newcommand{\para}[1]{\noindent{\textbf{#1}}}

\newif\ifcomments

\newcommand{\new}[1]{#1}

\commentsfalse

\ifcomments
    \newcommand{\amit}[1]{\textcolor{red}{#1 --amit}}
    \newcommand{\chenhao}[1]{\textcolor{magenta}{#1 --chenhao}}
    \newcommand{\abhinav}[1]{\textcolor{blue}{#1 --abhinav}}
\else
    \newcommand{\amit}[1]{}
    \newcommand{\chenhao}[1]{}
    \newcommand{\abhinav}[1]{}
\fi
\title{%It's Obliiviate! 
Probing Classifiers are Unreliable for Concept Removal and Detection
}

\author{%
  Abhinav Kumar \\
  Microsoft Research\\
  \texttt{t-abkumar@microsoft.com} \\
   \And
   Chenhao Tan \\
   University of Chicago \\
   \texttt{chenhao@uchicago.edu} \\
   \And
   Amit Sharma \\
   Microsoft Research \\
   \texttt{amshar@microsoft.com} \\
}

\begin{document}
\maketitle
\newcommand{\refalg}[1]{Alg.~\ref{#1}}
\newcommand{\refeqn}[1]{Eq.~\ref{#1}}
\newcommand{\reffig}[1]{Fig.~\ref{#1}}
\newcommand{\reftbl}[1]{Table~\ref{#1}}
\newcommand{\refsec}[1]{\S\ref{#1}}
\newcommand{\refsubsec}[1]{\S\ref{#1}}
\newcommand{\method}[1]{\mbox{\textsc{#1}}}
\newcommand{\reflemma}[1]{Lemma~\ref{#1}}
\newcommand{\reftheorem}[1]{Theorem~\ref{#1}}
\newcommand{\refdef}[1]{Def~\ref{#1}}
\newcommand{\refassm}[1]{Assm~\ref{#1}}
\newcommand{\refalgo}[1]{Alg.~\ref{#1}}
\newcommand{\refexp}[1]{Ex.~\ref{#1}}
\newcommand{\refappendix}[1]{\S\ref{#1}}
\newcommand{\refline}[1]{Line~\ref{#1}}
\newcommand{\refprop}[1]{Prop~\ref{#1}}
\newcommand{\refpara}[1]{Para~\ref{#1}}

% \DeclarePairedDelimiterX{\norm}[1]{\lVert}{\rVert}{#1}
% \DeclareMathOperator*{\argmin}{argmin}

%\newcommand{\reminder}[1]{\textcolor{red}{[[  #1 ]]}\typeout{#1}}
\newcommand{\reminder}[1]{}

\newcommand{\system}{C++ }
\newcommand{\systemfull}{Checklist++ }
\newcommand{\INLP}{INLP\xspace}
\newcommand{\inlp}{Iterative-Null-Space-Removal }
\newcommand{\stageOne}{\emph{Stage 1} }
\newcommand{\stageTwo}{\emph{Stage 2} }
\newcommand{\algoOne}{DST }
\newcommand{\algoTwo}{DP }

\newcommand{\fullprop}{spurious feature\xspace}
\newcommand{\prop}{concept\xspace}
\newcommand{\Prop}{Concept\xspace}
\newcommand{\props}{concepts\xspace}
\newcommand{\feature}{SF\xspace}
\newcommand{\features}{SFs\xspace}
\newcommand{\sensitive}{sensitive\xspace}
\newcommand{\nsr}{NSR\xspace}
\newcommand{\ar}{AR\xspace}
\newcommand{\mnli}{MultiNLI\xspace}
\newcommand{\pan}{Twitter-PAN16\xspace}
\newcommand{\aae}{Twitter-AAE\xspace}
\newcommand{\syn}{Synthetic-Text\xspace}

\newtheoremstyle{tightstyle}
  {2pt} % Space above
  {2pt} % Space below
  {\itshape} % Body font
  {} % Indent amount
  {\bfseries} % Theorem head font
  {.} % Punctuation after theorem head
  {.5em} % Space after theorem head
  {} % Theorem head spec (can be left empty, meaning `normal')

%Theorem, Definition, Lemmas, etc configuration
\theoremstyle{tightstyle} \newtheorem{theorem}{Theorem}[section]
\newtheorem{claim}[theorem]{Claim}
\newtheorem{lemma}[theorem]{Lemma}
\newtheorem{corollary}{Corollary}[theorem]
\theoremstyle{tightstyle} \newtheorem{definition}{Definition}[section]
\theoremstyle{tightstyle} \newtheorem{assumption}{Assumption}[section]
\newtheorem{example}{Example}[section]
\newtheorem{proposition}{Proposition}[section]
\newtheorem*{remark}{Remark}

\begin{abstract}
      %Since large-scale natural language processing (NLP) models are hard to interpret, probing classifiers on the representation layer have emerged as a popular way to interpret what information a model has learnt and have been used in downstream methods to remove any unwanted correlations. 
      %However, real-world data  often exhibits correlation between the main task label for the NLP model and the property (e.g., a spurious or sensitive attribute) that is being probed. In such a scenario, past work has argued  that a probing classifier may learn some correlated features. However, we prove a stronger result:  Even when the feature alone provides 100\% accuracy, a probing classifier for the feature is likely to learn non-zero weights for other correlated features.\chenhao{this sentence is very hard to parse. I cannot really understand its meaning.}\amit{tried to simplify}
      Neural network models trained on text data have been found to encode undesirable linguistic or sensitive concepts in their representation. Removing such concepts is non-trivial because of a complex relationship between the concept, text input, and the learnt representation. Recent \new{work has} proposed post-hoc and adversarial methods to remove such unwanted concepts from a model's representation. 
      %In particular, it learns features corresponding to the main task.
      Through an extensive theoretical and empirical analysis, we show that these methods can be counter-productive: they are  unable to remove the concepts entirely, and in 
      \new{some cases may fail severely by destroying all task-relevant features.}
      % the worst case may end up destroying all task-relevant features. 
      The reason is the methods' reliance on a \textit{probing} classifier as a proxy for the concept. 
      Even under the most favorable conditions for learning a probing classifier when a concept's relevant features in representation space
      alone can provide 100\% accuracy, we prove that a probing classifier is likely to use non-concept features and thus post-hoc or adversarial methods will fail to  remove the concept correctly. 
    %   \abhinav{wrt to labels of probe task and only when the labels are binary}
      % As a result, even in this ideal setting, we show that two state-of-the-art feature removal techniques,  post-hoc removal and adversarial training, will either fail to remove the spurious feature fully, remove the the features relevant for the main task, or do both.
      %Our theoretical result implies that       removing spurious learnt correlations using probing classifiers can be counter-productive since they can remove legitimate features for the main task.
      These theoretical implications are confirmed by % \sout{empirical} 
      experiments on models trained on synthetic, Multi-NLI, and Twitter datasets. 
      For sensitive applications of concept removal such as fairness, we recommend caution against using these methods and propose a spuriousness metric to gauge the quality of the final classifier.
      %our results show that probing classifiers can be counter-productive. For practical usage, we introduce a spuriousness metric for a probing classifier that can help gauge its quality.
      
\end{abstract}

%Adding the main sections
\section{Introduction}
\label{sec:introduction}
% notes:ml classifiers bias, important problem to remove those feature,, can spurious or biased. two different ways of doing it. 
% in this paper, we show it can be counter-productive. both depend on creating a predictor for feature. we show predictor is bad. then method will also be bad. our work extends empirical work on adversarial removal, and connects to work in probing classifier literature in general. while they criticize simple use of probing that it does not what classifier uses, even amnesic CF probing is faulty
%Empirically, we find..

%related work: adversarial --accuracy is low empirically previously shown. we show theoretically.
% Neural network classifiers built using text data 
Neural models in text classification have been shown to learn spuriously correlated features~\cite{GururanganMNLINegationArtifact,McCoyMNLIArtifact} or embed sensitive attributes like gender or race~\cite{conneau-etal-2018-cram,DebiaWordEmbProjectionBolukbasi,aaeDataset} in their representation layer. Classifiers that use such sensitive or spurious concepts (henceforth \textit{\props}) raise concerns of model unfairness and 
% and affect 
out-of-distribution generalization failure~\cite{GroupDRO,IRM:2020,AdvDomAdapGanin}. Removing the influence of these \props is non-trivial because the classifiers are based on hard-to-interpret deep neural networks. Moreover, since many \props cannot be modified at the input tokens level, removal methods that work at the representation layer have been proposed: \textbf{1) }post-hoc removal~\cite{DebiaWordEmbProjectionBolukbasi,PostHocRemovalXu,PostHocDuvenaud} on a pre-trained model (e.g., null space projection~\cite{NullItOut:2020}), and \textbf{2)} adversarial removal~\cite{AdvDomAdapGanin,AdvRemNeubig,AdvRemYoav} by jointly training the main task classifier with an (adversarial) classifier for the \prop. 
%\amit{question:is `attribute' better or `property'? Property is nice since it could refer to complex concepts whereas attribute might be thought of as a singleton variable. Feature makes more sense as a representation feature (see ass. 3.1), that's why the change in title.}

In this paper, we theoretically show that both these classes of methods can be counter-productive in real-world settings  where the main task label is often correlated with the \prop. Examples include natural language inference (\textit{main} task) where the presence of negation 
% \chenhao{negation? will check back later} 
(spurious \textit{concept}) may be correlated with the ``contradicts'' label; or tweet sentiment classification (\textit{main} task) where the author's gender (sensitive \textit{concept}) maybe correlated with the sentiment label.    Our key result is based on the observation that both these methods internally use an auxiliary (or probing) classifier~\cite{ProbeYoavICLR17,shi-etal-2016-stringProbing} that aims to predict the spurious concept based on the  representation learnt by the main classifier. 

We show that an auxiliary classifier cannot be a reliable signal on whether the representation includes features that are causally derived from the \prop. %Even if the representation does not use the probing feature, it may not be surprising to observe that an auxiliary classifier can obtain better-than-random accuracy in the presence of correlated features. 
As previous work has argued~\cite{ProbingSurvey,PartialPredictiveInvDimitris,OverparamExcerbatesSpSagawa,IRM:2020}, if the representation features causally derived from the \prop are not predictive enough, the probing classifier for the \prop can be expected to rely on correlated features to obtain a higher accuracy. 
However, we show a stronger result: this behavior holds even when there is no potential accuracy gain and the \prop's features are easily learnable. 
Specifically, even when the \prop's causally-related features alone can provide 100\% accuracy and are linearly separable with respect to a binary probing task label,
% \abhinav{changed}
%from other features\abhinav{wrt to labels of probe task and only when the labels are binary} , 
the probing classifier may still learn non-zero weights for the correlated main-task relevant features.  Based on this result, under some simplifying assumptions,  we prove that both post-hoc and adversarial training methods can fail to remove the undesired \prop, remove useful task-relevant features in addition to the undesired \prop, or do both. 
%\abhinav{need to address comment}.
As a \new{severe failure mode}, we show that post-hoc removal methods can  lead to a \textit{random-guess} main-task classifier by removing all task-relevant information from the representation.

Empirical results on four datasets---natural language inference, sentiment analysis, tweet-mention detection, and a synthetic task---confirm our claims. %For benchmarking, we use implementations of the null space removal~\cite{NullItOut:2020} and adversarial removal~\cite{AdvRemYoav,AdvDomAdapGanin} methods for removing a given \prop from the classifier.
Across all datasets, as the correlation between  the main task and the \prop increases, post-hoc removal using  null space projection removes a higher amount of the main-task features, eventually leading to a random-guess classifier. In particular, for a pre-trained classifier that does not use the \prop at all, the method modifies the representation to yield a \new{classifier that either uses the \prop or has} lower main-task accuracy, irrespective of the correlation between the main task and the \prop. %, even when the data contains no correlation between main task and the property.
Similarly, for the adversarial removal method, we find that it does not remove all \prop-related features. For most datasets, the \prop features left within an classifier's representation are comparable to that for a standard main-task classifier.
% \sout{For practical usage of the removal methods, we propose a spuriousness metric 
% to gauge the quality of the final classifier.}

Our theoretical analysis  complements past empirical critiques of adversarial methods for concept removal~\cite{AdvRemYoav}. More generally, we extend the literature on probing classifiers and their unreliability~\cite{ProbingSurvey}.  %, in general, on model interpretability~\cite{todo}. 
Adding to known limitations of \new{explainability} methods~\cite{hewitt-liang-2019-control,ravichander-etal-2021-probing} based on the accuracy of a probing classifier, our results show that recent causally-inspired methods like amnesic probing~\cite{AmnesicProbing} are also flawed because they  depend on access to a good quality \prop classifier. Our contributions include:

\begin{itemize}[leftmargin=*,topsep=0pt]{}
    \item Theoretical analysis of null space and adversarial removal methods showing that they fail to remove an undesirable \prop from a model's representation, even under favorable  conditions.
    \item Empirical results on four datasets showing that the methods are unable to remove a spurious \prop's  features fully and end up unnecessarily removing task-relevant features. 
    \item A practical spuriousness score for evaluating the output of \new{\prop} removal methods.
\end{itemize}

\section{Concept removal: Background and problem statement}
\label{section:related_works}

For a classification task, let $(\bm{x}^{i},y^{i}_{m})_{i=1}^{n}$ be set of examples in the dataset $\mathcal{D}_{m}$, where $\bm{x}^{i}\in \bm{X}$ are the input features  and $y^{i}_{m}\in Y_{m}$ the label. We call this the \textit{main task} and label $y^{i}_{m}$ the main task label. The main task classifier can be written as $c_m(h(\bm{x}))$ where $h:\bm{X}\rightarrow \bm{Z}$ is an encoder mapping the input $\bm{\bm{x}}$ to a latent representation $\bm{z}:=h(\bm{x})$ and  $c_{m}:\bm{Z}\rightarrow Y_{m}$ is the classifier on top of  the representation~$\bm{Z}$. Additionally, we are given labels for a spurious or sensitive \prop,   $y_{p}\in Y_{p}$, i.e., $(\bm{x}^{i},y^{i}_{p})_{i=1}^{n'}$ in a dataset $\mathcal{D}_p$, and our goal is to ensure that the representation $h(\bm{x})$ learnt by the main classifier does not include features causally derived from the \prop. Below we define what it means to be ``causally derived'': the representation should not change under an intervention on \prop.
% \chenhao{the final sentence is not causally related?}\amit{made it more connected now}

\begin{definition}
\label{def:causally_derived_feature}
(\textbf{Concept-causal feature}) 
% \chenhao{I think something shorter might be more useful} 
A feature $Z_j \in \bm{Z}$ (jth dimension of $h(\bm{x})$) at the representation layer is defined to be causally derived from a \prop (\prop-causal for short) if upon changing the value of the \prop, the corresponding change in the input's value $\bm{x}$ will lead to a change in the feature's ($Z_j$) value. %\abhinav{This definition seems like we are assuming there is anti-causal relationship between the probing label and text}
\end{definition}

For simplicity, we assume that the non-concept-causal features are the \textit{main task} features. Often, the main task and the concept label are correlated; hence  the learnt representation $h(\bm{x})$ for the main task may include concept-causal features too.  A \prop removal algorithm is said to be successful if it produces a \textit{clean} representation $h'(\bm{x})$ to be used by the main-classifier that has no \prop-causal features and it does not corrupt or removes the main-task features. If the representation does not contain such features, the main classifier cannot use them \cite{AdvRemYoav}. In practice, it is okay if the \prop-causal features are not completely removed, but our key criterion is that the removal process should not remove the correlated main task features. 

\textbf{Existing \prop removal methods. }% Adversarial training and null space projection.}
When the text input can be changed based on changing the value of concept label, 
%sensitive property's value, 
methods like data augmentation~\cite{kaushik2021learningthediff,ryanCAD2019,cadSen2022} have been proposed for concept removal. However, for most sensitive or spurious concepts, it is not possible to know the correct change to apply at the input level corresponding to a change in the concept's value. %Hence, the above  approaches cannot be applied directly. % however it is useful as a conceptual definition.  

Instead, methods based on the representation layer have been proposed. To determine whether 
features in a representation are causally derived by the \prop, these methods train an auxiliary, probing classifier $c_{p}:\bm{Z}\rightarrow Y_{p}$ where $y_{p}\in Y_{p}$ is the label of the \prop we want to remove from the latent space $\bm{z}\in \bm{Z}$. The accuracy of the classifier indicates the predictive information about the \prop embedded in the representation. This probing classifier is then used to remove the sensitive \prop from the latent representation which will ensure that the main-task classifier cannot use them. Two kinds of feature removal methods have been proposed:
1) \textit{post-hoc} methods such as null space removal \cite{NullItOut:2020,AmnesicProbing,ruslanINLP,INLPinterpret}, with removal after the main-task classifier is trained; 2) \textit{adversarial} methods that jointly train the main task with the probing classifier as the adversary \cite{AdvDomAdapGanin,AdvRemNeubig,LinearConceptErasure,KernelConceptErasure}. 

For adversarial removal, recent empirical results cast doubt on the method's capability to fully remove the sensitive \prop from the model's representation~\cite{AdvRemYoav}. We extend those results with a rigorous theoretical analysis and provide experiments for both adversarial and post-hoc removal methods. %We also empirically investigate the null space method, and extend the analysis on adversarial methods by examining the effect of  correlation between the main task and \prop labels.

\vspace{-0.5em}
\section{Attribute removal using probing classifier can be counter-productive}
\label{sec:feature_removal_problem}

% For the main-task classification, we are given $\bm{x}^{i}\in \bm{X}$ with corresponding main label $y^{i}_{m}\in Y_{m}$ in the dataset $\mathcal{D}_{m}$. The main-task classifier $c_{m}(h(\bm{x}))$ uses an encoder $h(\bm{X}):\bm{X}\rightarrow \bm{Z}$ which maps the input from input $\bm{X}$ to latent representation $\bm{Z}$ and $c_{m}:\bm{Z}\rightarrow Y_{m}$ is the classifier on top of latent representation $\bm{Z}$. Next, we are given the probing dataset $\mathcal{D}_{p}$ which contains example $\bm{x}^{j}\in \bm{X}$ with corresponding probing label $y_{p}^{j}\in Y_{p}$ for \sensitive attribute. The goal of the removal methods like null-space and adversarial removal is to remove any feature from the latent space  which are \emph{causally-derived} (\refdef{def:causally_derived_feature})  from the \sensitive attribute. These methods  use an auxiliary probing classifier $c_{p}(\bm{z})$ trained from the latent representation $\bm{z}$ of the main-task classifier as a proxy for \sensitive feature in latent representation. 

%As mentioned above, representation-based attribute removal methods can be broadly divided into two categories \begin{enumerate*}
%    \item Post-Hoc methods like null-space removal where the removal of \sensitive \prop is done post main-classifier training by training a probing classifier and, 
%    \item Adversarial methods which jointly train the main task and also removes the \sensitive \prop by training a probing classifier.
%\end{enumerate*} 
As mentioned above, both removal methods internally use a probing classifier as a proxy for  the \prop's features. 
In \refsec{subsec:clean_classifier_problem}, we start off by showing that for any classification task be it probing or main-task classification, it is difficult to learn a \emph{clean} classifier which doesn't use any spuriously correlated feature (\reflemma{lemma:sufficient_condition} and \reflemma{lemma:sufficient_condition_main_task}). Hence the key assumption driving the use of predictive classifiers within both removal methods is incorrect.
%That is, probing classifiers for null-space are unreliable, so do both probing and main-classifiers for adversarial removal. 
Next in \refsec{subsec:null_space_problem} and \ref{subsec:adv_rem_problem}, we will show how these individual components' failure leads to the failure of both removal methods. %Thus, our key message is to be cautious when using these latent space removal methods for high-risk applications like ensuring fairness. %Also, when proposing another probing based method one need to take into account the failure mode we observe for these methods.
Finally, in \refsec{subsec:spurriousness_score}, we propose a practical \emph{spuriousness score} to assess the output classifier from any of the removal methods. Throughout this section, we assume that both the main task label $y_{m}$ and probing task label $y_{p}$ are binary ($\in \{-1,1\}$) and  there is a basic, fixed encoder $h$ converting the text input to features in the representation space (e.g., a pre-trained model like BERT \cite{devlin-etal-2019-bert}). 
% \chenhao{move definition first, and then present two assumptions, and then the lemma, put no English in between, and then summarize main implications in English.}

\subsection{Fundamental limits to learning a \emph{clean} classifier: Probing and Main Classifier}
\label{subsec:clean_classifier_problem}
Given $\bm{z}=h(\bm{x})$ and the \prop label $y_p$, the goal of the probing task is to learn a classifier $c_p(\bm{z})$ such that it only uses the \prop-causal features and the accuracy for $y_p$ is maximized. 
We assume that the main task and concept labels are correlated, so it can be beneficial to use main-task features to maximize accuracy for $y_p$. 
As argued in the probing literature~\cite{hewitt-liang-2019-control,ProbingSurvey}, if there are features in $\bm{z}$ outside \prop-causal that help improve the accuracy of the classifier, a classifier trained on standard losses such as cross-entropy or max-margin is expected to use those features too. Below we show a stronger result: even when there is no accuracy benefit of using non \prop-causal features, we find that a probing classifier may still use those features.

\para{Creating a favorable setup for the probing classifier.} Specifically, we create a setting that is the most favorable for a probing classifier to use only \prop-causal features: \textbf{1)} no accuracy gain on using features outside of \prop-causal because
\new{\prop-causal features are linearly separable for concept labels ,}
% concept labels are linearly separable using \prop-causal features,
and \textbf{2)} disentangled representation so that no further representation learning is required. Yet we find that a trained probing classifier would use non-\prop-causal features. % As defined in Section~\ref{rw}, given $z=h(\bm{x})$ and the main-task label $y_m$, the main-task is to learn a classifier $c_m(z')$ such that $z'$ does not have any \prop-causal features.
%Specifically, we make the following assumptions.

\begin{assumption}[Disentangled Latent Representation]
\label{assm:disentagled_latent}
The latent representation $\bm{z}$ is disentangled and is of form $[\bm{z}_{m},\bm{z}_{p}]$, where  $\bm{z}_{p}\in \mathbb{R}^{d_{p}}$  are the \prop-causal features
%are the features causally derived from the concept label
and $\bm{z}_{m}\in \mathbb{R}^{d_{m}}$ are the main task features. Here $d_{m}$ and $d_{p}$ are the dimensions of $\bm{z}_{m}$ and $\bm{z}_{p}$ respectively.  
% \amit{is frozen assumption required here}%Also, let $\bm{z}$ be linearly separable with respect to both binary labels $y_{m}$ and $y_{p}$ taking values from $\{-1,1\}$.
\end{assumption}

%While it may be less surprising that a spurious feature is used if the invariant \chenhao{invariant is never defined} feature does not give 100\% accuracy,  we show that a spurious feature will be used even if the invariant feature is fully predictive of the label, following ~\cite{FailureModeOODGen}. We formalize it as, 
\begin{assumption}[\Prop-causal Feature Linear Separability]
\label{assm:fully_pred_inv}
The \prop-causal features ($\bm{z}_{p}$) of the latent representation ($\bm{z}$) are linearly separable/fully predictive for the \prop  labels $y_p$, i.e., $y_{p}^{i}\cdot (\hat{\bm{\epsilon}}_{p}\cdot \bm{z}_{p}^{i} + b_{p}) > 0 $, $\forall(\bm{x}^{i},y_{p}^{i})$ in training dataset $\mathcal{D}_{p}$ 
for some $\hat{\bm{\epsilon}}_{p}\in \mathbb{R}^{d_{p}}$ and $b_{p}\in \mathbb{R}$.  %For the case of zero-centered latent space we have $b_{p}=0$.
\end{assumption}

\para{The effect of spurious correlation between concept and label. }
Now we are ready to state the key lemma which will show that if there is a \textit{spurious correlation} between the main task and concept labels such that the main-task features $\bm{z}_m$ are predictive of the concept label for only a \emph{few special} points, % if non-\prop-features $\bm{z}_{m}$ are linearly separable w.r.t. to \prop label $y_{p}$ (\refassm{assm:spurious_linear_separably}) , 
then the probing  classifier $c_p(\bm{z})$ will use those  features. We operationalize spurious correlation as, 
\begin{assumption}[Spurious Correlation]
\label{assm:spurious_linear_separably}
For a subset of training points $\mathcal{S} \subset \mathcal{D}_p$ in the training dataset for a probing classifier, $\bm{z}_m$ is linearly-separable 
% \chenhao{I do not understand what it means to have a feature is linearly separable. It should be two classes are separable, right?} 
with respect to \prop label $y_p$, i.e., $y_p^{i}\cdot(\hat{\bm{\epsilon}}_{m}\cdot \bm{z}^i_{m} + b_{m})> 0 \ \forall i \in \mathcal{S}$  , where  $\hat{\bm{\epsilon}}_{m}\in \mathbb{R}^{d_{m}}$ and $b_{m}\in \mathbb{R}$. %For the case of zero-centered latent space we have $b_{m}=0$.
%A latent feature $\bm{z}_{sp}$ is spurious  for a task if the feature is not \emph{causally-derived} from the task label $y$ but for a subset of sample/examples, the latent feature is linearly-separable \chenhao{I do not understand what it means to have a feature is linearly separable. It should be two classes are separable, right?} with respect to task label $y$ i.e $y^{i}\cdot(\hat{\bm{\epsilon}}_{sp}\cdot \bm{z}_{sp} + b_{sp})\geq 1$ for some $\hat{\bm{\epsilon}}_{sp}\in \mathbb{R}^{d_{sp}}$ and $b_{sp}\in \mathbb{R}$. For the case of zero-centered latent space we have $b_{sp}=0$.
\end{assumption}
% \abhinav{connect this to the predictive correlation we have in experiment section, basically idea is that we wont have such feature and linear separability in real dataset hence new definition. POint in pred -corr paragraph this is empirical version of assm 3.3}
% \abhinav{repetitive - could remove the below three lines upto zero centered space. }
% Note that even for points in $\mathcal{S}$, there is no accuracy benefit of using $\bm{z}_m$ features since $\bm{z}_p$ already provides 100\% accuracy for all points in $\mathcal{D}_p$. 
% Without loss of generality, the lemma below assumes a zero-centered latent space ($b_m$ and $b_p$ are zero).
For simplicity, we assume that the  encoder $h(\cdot)$ which maps the input $\bm{X}$ to latent representation $\bm{Z}$ is frozen or non-trainable. Following \cite{FailureModeOODGen}, we assume 
max-margin as training loss; under some mild conditions on separable data, a classifier trained using logistic/exponential loss converges to max-margin classifier given infinite training time~\cite{GradDescentMaxMargin,MaxMarginJustificationJi}.
 %\amit{will be good to get more than one citation for using maxmargin and next claim}

%In addition, we assume that there is some correlation between the \prop and main-task label: a function using $\bm{z}_m$ may also be able to classify correctly some non-empty subset of points. 

%Let the main-task feature $\bm{z}_{m}$ in the latent space is spurious to the probing task and the \sensitive-probing feature $\bm{z}_{p}$ be spurious to the main task. We later state what we mean by a feature being spurious to a task in \refdef{def:spurious_linear_separably}. Next we will show that both main-task classifier $c_{m}(\bm{z})$ and the probing classifier $c_{p}(\bm{z})$ which uses the same latent space for their training will fail to learn a ``clean" classifier in presence of spuriously correlated feature in latent space. We say a classifier is clean when it doesnt \emph{use} any spuriously correlated feature for it's prediction. We will define what we mean by clean more concretely when stating \reflemma{lemma:sufficient_condition}.  Hence for ease of exposition, we describe our results for a general classifier $c(\bm{z})$ using the latent representation $\bm{z}\in\bm{Z}$ for predicting  label $y\in Y$. We generalize the latent representation to be of form $\bm{z}=[\bm{z}_{inv},\bm{z}_{sp}]$ where $\bm{z}_{inv}$ is the feature causally-derived from the task label $y$ and $\bm{z}_{sp}$ is the spurious to the task. We define a latent feature $z_{sp}$ being spurious as:

\begin{restatable}{lemma}{probthm}
% \begin{lemma}[Sufficient Condition for \prop-Probing Classifier]
\label{lemma:sufficient_condition}
Let the latent representation be frozen and disentangled such that $\bm{z}=[\bm{z}_{m},\bm{z}_{p}]$ (\refassm{assm:disentagled_latent}), and \prop-causal features $\bm{z}_{p}$ are fully predictive for the concept label $y_p$  (\refassm{assm:fully_pred_inv}). Let $c_{p}^{*}(\bm{z})=\bm{w}_{p}\cdot \bm{z}_{p}$ where $\bm{w}_{p}\in \mathbb{R}^{d_{p}}$ be the desired \emph{clean} linear classifier trained using the max-margin objective (\refappendix{subsec:app_max_margin_setup}) that only uses $\bm{z}_p$ for its prediction. Let $\bm{z}_{m}$ be the main task features, spuriously correlated s.t.  $\bm{z}_{m}$ are linearly-separable w.r.t. probing task label $y_{p}$ for the margin points of $c_{p}^{*}(\bm{z})$ (\refassm{assm:spurious_linear_separably}).  Then, assuming a zero-centered latent space \new{($b_{p}=0$)}, a concept-probing classifier $c_p$ trained using the max-margin objective will use spurious features, i.e., $c_p(\bm{z})=\bm{w}_{p}\cdot \bm{z}_{p} + \bm{w}_{m}\cdot \bm{z}_{m}$ where $\bm{w}_{m}\neq \bm{0}$ and $\bm{w}_{m}\in \mathbb{R}^{d_{m}}$.    
% \end{lemma}
\end{restatable}

\textit{Proof Sketch.}
Starting from $c_{p}^{*}(\bm{z})$, we show that there always exists a perturbed classifier which uses the main task features and has a bigger margin than $c_{p}^{*}(\bm{z})$. 
% \new{
Within some range of perturbation, for all margin points of $c_{p}^{*}$, using the main task features increases the margin by \refassm{assm:spurious_linear_separably}, and does not reduce the margin for non-margin points s.t. \new{it becomes the} same as the margin of $c_{p}^{*}$.
% }. 
Proof in \refappendix{subsec:app_proof_sufficient_condition}.
%\end{proof}

Our result shows that not just accuracy, even geometric skews in the dataset can yield an incorrect probing classifier. In \refappendix{subsec:app_proof_necessary_condition} we prove that the assumptions for  \reflemma{lemma:sufficient_condition} are both sufficient and necessary for a classifier to use non-concept-features $\bm{z}_{m}$ when  $\bm{z}_{p}$ is 1-dimensional. \reflemma{lemma:sufficient_condition}  generalizes a result from ~\cite{FailureModeOODGen} by using fewer assumptions (we do not restrict $\bm{z}_m$ to be binary, do not assume that $\bm{z}_m$ and $\bm{z}_p$ are conditionally independent given $y$, and do not assume monotonicity of classifier norm with dataset size). We present a similar result for the main task classifier: under spurious correlation of concept and main task labels, the main task classifier would use concept-causal features even when 100\% accuracy can be achieved using only main task features  (\reflemma{lemma:sufficient_condition_main_task}, \refappendix{subsec:app_generalized_probing_result}).

\subsection{Failure mode of post-hoc removal methods: Null-space removal (\INLP)}
\label{subsec:null_space_problem}

\begin{figure}[t]
\centering
    \begin{subfigure}[h]{.32\textwidth}
        \captionsetup{justification=centering}
        \centering
            \includegraphics[width=\linewidth,height=0.8\linewidth]{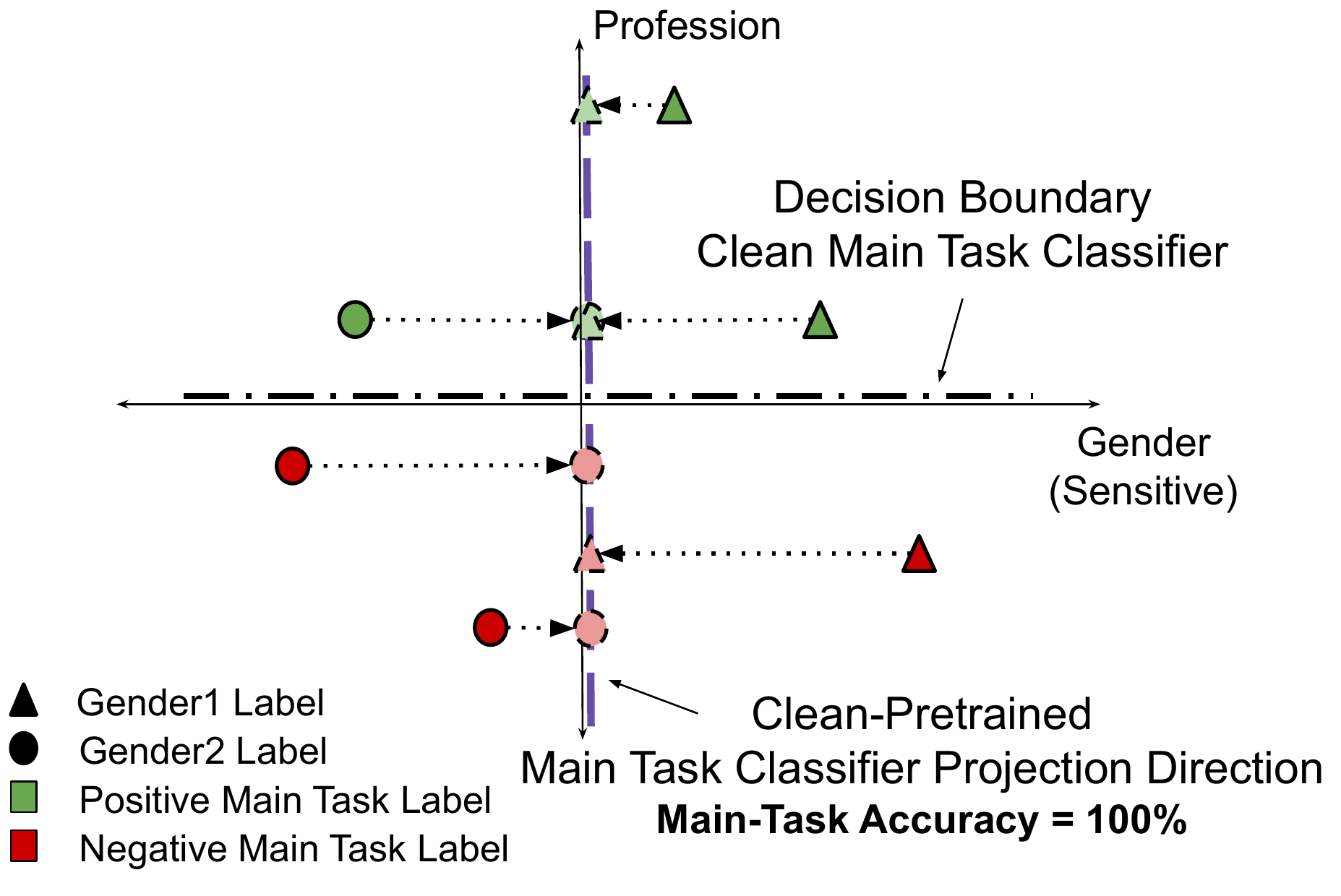}
    		\caption[]%
            {{\small Clean pretrained main task (\textit{Profession}) classifier.}}    
    		\label{fig:inlp_pretrained_main}
    \end{subfigure}
\hfill 
  \begin{subfigure}[h]{.32\textwidth}
  \captionsetup{justification=centering}
    \centering
    \includegraphics[width=\linewidth,height=0.8\linewidth]{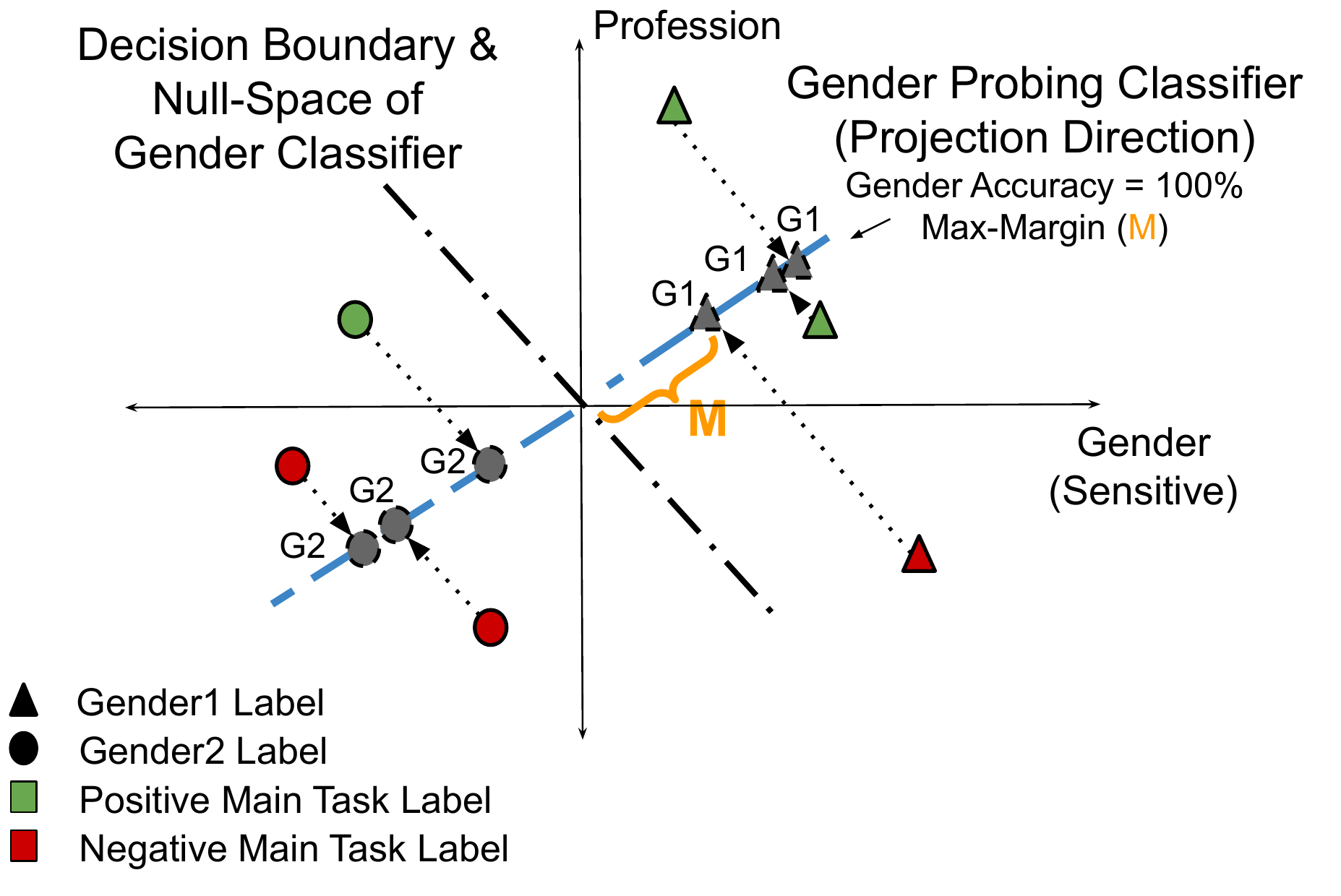}
		\caption[]%
        {{\small Probing (\textit{Gender}) classifier with a slanted projection direction.}}    
		\label{fig:inlp_prob_classifier} 
  \end{subfigure}
\hfill
  \begin{subfigure}[h]{.32\textwidth}
    \captionsetup{justification=centering}
    \centering
    \includegraphics[width=\linewidth,height=0.8\linewidth]{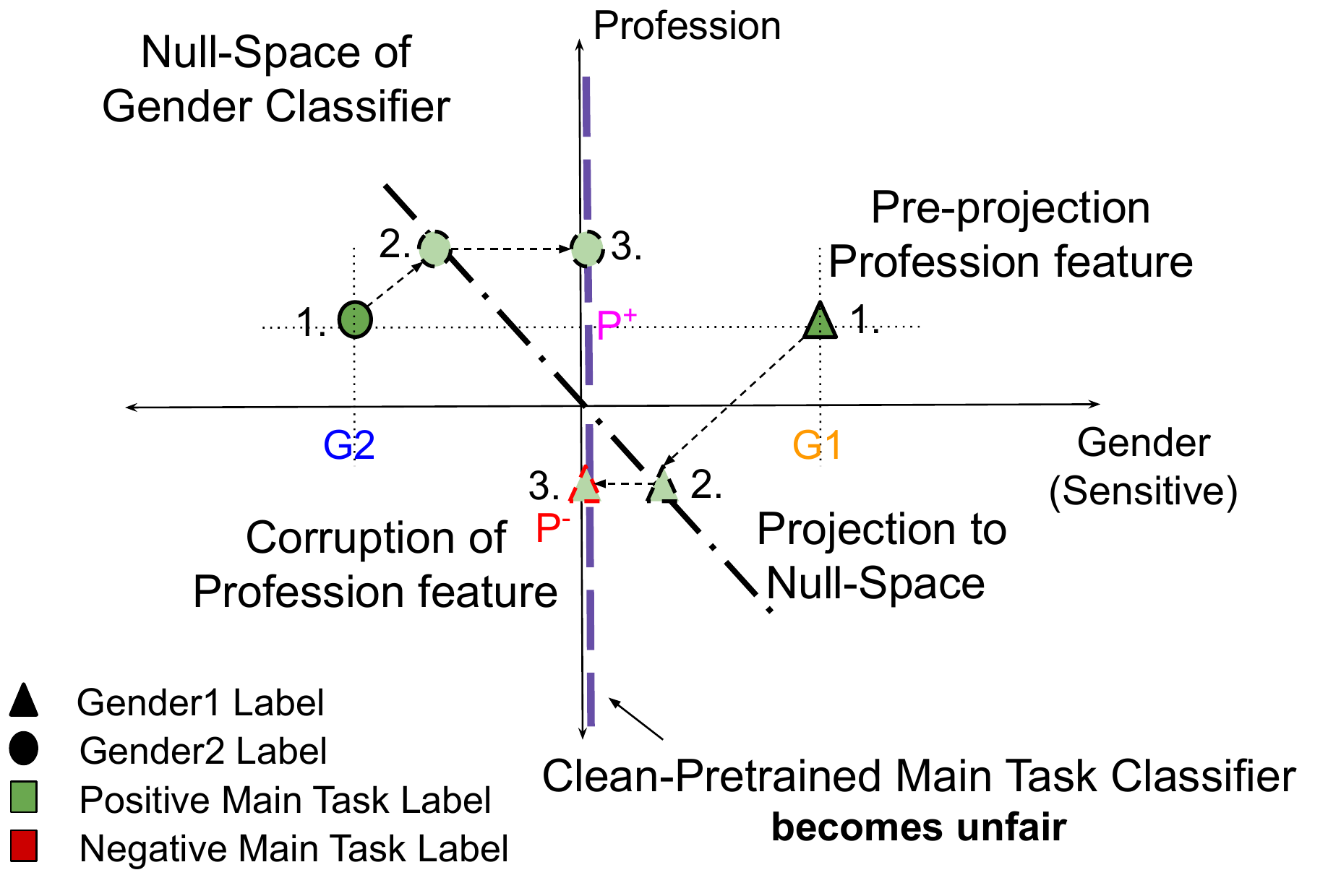}
		\caption[]%
        {{\small Main classifier becomes unfair after null space projection.}
        }    
		\label{fig:inlp_unfair_classifier} 
  \end{subfigure}

\caption{\textbf{Failure mode of null space removal. } Consider a main task (\textit{Profession}) classifier where \textit{Gender} is the spurious concept to be removed.  
Assume a 2-dimensional latent representation $\bm{z}$, where one dimension corresponds to profession and the other to the gender feature.  \textbf{(a)} A ``clean'' (fair) main task classifier that only uses the Profession feature, shown by its vertical projection direction, that is input to \INLP for concept removal. Its decision boundary is orthogonal to the projection direction. \textbf{(b)} From \reflemma{lemma:sufficient_condition}, \INLP trains a probing classifier for gender with a slanted  projection direction (ideal gender \new{projection direction} would \new{be} horizontal). \textbf{(c)} For two points having the same profession but different gender features (marked \textit{`1'}), projection to the null-space (\textit{`2'}) has their profession feature reversed (\textit{`3'}), thus  making the fair pretrained classifier become unfair 
% \amit{this shows unclean. But why unfair? need to explain more} 
(also see \refsec{subsec:null_space_problem}). 
%For the adversarial removal method, assume that the shared representation is a scalar value obtained by projecting the features in some direction. The adversarial goal is to find a projection direction such that the concept (gender) classifier obtains a random-guess accuracy of $50\%$ but have good performance on main-task (profession) prediction. \textbf{(d)} There can be two projection directions---shown by vertical and slanted line---that yield a \textit{random-guess} 50\% accuracy on gender prediction and, \textbf{(e)} have same 100\% accuracy for profession prediction. \textbf{(f)} However, the slanted projection direction has bigger margin for main-task and will be preferred, again leading to a final classifier that uses the Gender feature (also see \refsec{subsec:adv_rem_problem}).
% For null-space removal see failure mode paragraph  in \refsec{subsec:null_space_problem} and for adversarial removal see failure mode paragraph in \ref{subsec:adv_rem_problem}
} 
\label{fig:inlp_problem_illustration}
\end{figure}

The null space method~\cite{NullItOut:2020,AmnesicProbing}, henceforth referred as \textit{\INLP}, removes a concept from latent space  by projecting the latent space to a subspace that is not discriminative of that concept. First, it  estimates the subspace in the latent space discriminative of the \prop we want to remove  by training a probing classifier $c_{p}:\bm{Z}\rightarrow Y_{p}$, where $Y_{p}$ is the \prop label. % we want to remove from latent space $\bm{z}^{i}$. 
Then the projection is done onto the null-space of this probing classifier which is expected to be non-discriminative of the \prop. For instance, \cite{NullItOut:2020} use a linear probing classifier $c_{p}(\bm{z})$ to  ensure that the any linear classifier cannot recover the removed \prop from modified latent representation $\bm{z}'$ and hence the main task classifier ($c_{m}(\bm{z}')$)  becomes invariant to removed \prop. Also, they recommend running this removal step for multiple iterations to ensure the unwanted \prop is removed completely (details are in~\refappendix{subsec:app_inlp_setup}).  
Below we state the failure of the null-space method using $\bm{z}^{i(k)}$ to denote the representation \new{$\bm{z}^i$} after $k$ steps of \INLP.

\begin{restatable}{theorem}{inlpthm}
% \begin{theorem}[\INLP Failure Mode]
\label{theorem:null_space_failure}
Let $\bm{c}_{m}(\bm{z})$ be a pre-trained main-task classifier where the latent representation $\bm{z}=[\bm{z}_m, \bm{z}_p]$ satisfies \refassm{assm:disentagled_latent} and \ref{assm:fully_pred_inv}. Let $\bm{c}_{p}(\bm{z})$ be the probing classifier used by \INLP to remove the unwanted features $\bm{z}_{p}$ from the latent representation. %Let the  and \ref{assm:spurious_linear_separably} in \reflemma{lemma:sufficient_condition} be also satisfied for the probing classifier $c_{p}(\bm{z})$ such that $c_{p}(\bm{z})=\bm{w}_{p}\cdot\bm{z}_{p} + \bm{w}_{m}\cdot \bm{z}_{m}$ and $\bm{w}_{m}\neq \bm{0}$. Then,
Under \refassm{assm:spurious_linear_separably}, \reflemma{lemma:sufficient_condition} is satisfied for the probing classifier $c_{p}(\bm{z})$ such that $c_{p}(\bm{z})=\bm{w}_{p}\cdot\bm{z}_{p} + \bm{w}_{m}\cdot \bm{z}_{m}$ and $\bm{w}_{m}\neq \bm{0}$. Then,
\begin{enumerate}[leftmargin=*,topsep=0pt]{}
    \item \textbf{Damage in the first step of INLP.} The first step of \new{linear}-\INLP will corrupt the main-task features and this corruption is non-invertible with subsequent projection steps of \INLP.
    \begin{enumerate}
        \item \textbf{Mixing:} If $\bm{w}_{p}\neq \bm{0}$, the main task $\bm{z}_m$ and concept-causal features $\bm{z}_p$  will get mixed such that $\bm{z}^{i(1)}=[g(\bm{z}^{i}_{m},\bm{z}^{i}_{p}), f(\bm{z}^{i}_{p},\bm{z}^{i}_{m})] \neq [g(\bm{z}^i_{m}),f(\bm{z}^i_{p})]$ for some function ``$f$'' and ``$g$''.  Thus, the latent representation is no longer disentangled and removal of \prop-causal features will also lead to removal of main task features. 
        
        \item \textbf{Removal:} If $\bm{w}_{p}=\bm{0}$, then the first projection step of INLP will do opposite of what is intended, i.e., damage the main task features $\bm{z}_{m}$ (in case $\bm{z}_{m}\in \mathbb{R}$, it  will completely remove $\bm{z}_{m}$) but have no effect on the \prop-causal features $\bm{z}_{p}$. %This damage, however,  is non-invertible with subsequent projection steps.
        % This case is very unlikely since it requires \refassm{assm:spurious_linear_separably} to be satisfied for all the points in dataset. 
    \end{enumerate}
    
    \item \textbf{Removal in the long term:}  The L2-norm of the latent representation $\bm{z}$ decreases with every projection step as long as the parameters of probing classifier ($\bm{w}^{k}$) at a step $``k"$  does not lie completely in the space spanned by parameters of previous probing classifiers, i.e., span($\bm{w}^{1},\ldots,\bm{w}^{k-1}$), \new{$\bm{z}^{i(k-1)}$}, $\bm{z}^{i(0)}$ and $\bm{z}^{i(0)}$ in direction of $\bm{w}^{k}$ is not trivially zero. Thus, after sufficiently many steps, INLP can destroy all information in the representation \new{s.t.}  $\bm{z}^{i(\infty)}=[\bm{0},\bm{0}]$.
    %\chenhao{I do not quite understand the implication of 2} \abhinav{Maybe $x^{i(\infty)}$ is an overclaim. It will only go zero as long as probing classifier chooses different directions.}
\end{enumerate}
% \end{theorem}
\end{restatable}
\textit{Proof Sketch.}
From \reflemma{lemma:sufficient_condition}, in the first step, probing classifier for $\bm{z}_{p}$  will use $\bm{z}_{m}$ in addition to $\bm{z}_{p}$. Consequently, the projection matrix for \INLP  based on the probing classifier will be incorrect, hence corrupting the main task features $\bm{z}_{m}$ with $\bm{z}_{p}$ \textbf{(1a)} or damage $\bm{z}_{m}$ without any effect on $\bm{z}_{p}$ \textbf{(1b)}. %Next, we show that the projection matrix is symmetric and thus diagonalizable with one of the entries being zero. Hence it cannot be inverted, thus the damage in first step cannot be corrected. 
Next, we show that each step of the projection operation reduces the norm of latent representation $\bm{z}$; thus the latent representation can go to $\bm{0}$ as the number of steps increases \textbf{(2)}. Proof in \refappendix{sec:app_inlp_setup_proof}. 
% From \reflemma{lemma:sufficient_condition}, probing classifier will have non-zero weight on the spurious feature $\bm{z}_{m}$. Hence, after projection operation they get mixed up and we show projection operator is non-invertible. Then we show projection operation is lossy, i.e., removes some norm of $\bm{z}$ when null-space of $c_{p}(\bm{z})$ is not orthogonal to $\bm{z}$. For detailed proof see \refappendix{subsec:app_inlp_proof}.
%\end{proof}
%\vspace{-1em}
% \chenhao{maybe it is useful to have concrete examples for these failure modes?}

\textbf{Failure Mode:} \reffig{fig:inlp_pretrained_main}-\ref{fig:inlp_unfair_classifier} demonstrate the \textit{mixing} problem stated in \reftheorem{theorem:null_space_failure}, where a fair classifier becomes unfair after the first step of projection. Note that after first step the main task classifier's accuracy will drop because of this mixing of features, affecting \INLP-based probing methods like Amnesic Probing~\cite{AmnesicProbing} that interpret a drop in the main classifier's accuracy after \INLP projection as evidence that the main classifier was using the sensitive \prop.

%\chenhao{this paragraph is about empirical experiments and can be removed from here.}
%Since null-space removal is post-hoc method it already uses frozen latent representation. For adversarial removal, we introduce additional trainable hidden layer on top of the frozen latent representation produced by $h(\bm{x})$ to make our theoretical analysis close to real setting. Also, we assume both classifier trained on top of latent representation $\bm{Z}$, be linear. Since the encoder $\bm{h}(\bm{x})$ can be arbitrarily complex, we assume that linear classifier will be sufficient for good prediction performance.

\subsection{Failure mode of adversarial removal methods}
\label{subsec:adv_rem_problem}
To remove the unwanted features $\bm{z}_{p}$ from the latent representation,  adversarial removal methods jointly train the main classifier $c_{m}:\bm{Z}\rightarrow Y_{m}$ and the probing classifier $c_{p}:\bm{Z}\rightarrow Y_{p}$ by specifying $c_p$'s loss as an adversarial loss. % The goal is  such that the latent representation $\bm{z}$ no longer remains predictive of the unwanted attribute label $y_{p}$ . Unlike the post-hoc methods like adversarial removal, these methods jointly trains this adversarial classifier $c_{p}(\bm{z})$ with the main task classifier $c_{m}(\bm{z})$.  
For details refer to~\refappendix{subsec:app_adv_setup}.

% \begin{figure}[t]
% \centering
%     \begin{subfigure}[h]{.32\textwidth}
%     \centering
%     \includegraphics[width=\linewidth,height=0.7\linewidth]{figs/AdversarialRemoval-Accuracy-Gender-Illustration-Flip.png}
%     % 	\caption[]%
%     %     {{\small Spurious Feature Accuracy}}    
%     % 	\label{fig:illustrate_adv_gender} 
%     \end{subfigure}
% \hfill 
%     \begin{subfigure}[h]{.32\textwidth}
%         \centering
%         \includegraphics[width=\linewidth,height=0.7\linewidth]{figs/AdversarialRemoval-Accuracy-Profession-Illustration.png}
%     		\caption[]%
%             {{\small Main Task Accuracy}}    
%     		\label{fig:illustrate_adv_prof} 
%       \end{subfigure}
% \hfill 
%     \begin{subfigure}[h]{.32\textwidth}
%         \centering
%         \includegraphics[width=\linewidth,height=0.7\linewidth]{figs/adv-real-roberta-base-mnli-topicacc.png}
%     		\caption[]%
%             {{\small Adversarial Classifier Accuracy}}    
%     		\label{fig:adv_accuracy_prop_to_kappa} 
%       \end{subfigure}

% \caption{\textbf{Failure Mode of Adversarial Removal}: }
% \label{fig:adv_problem_illustration}
% \end{figure}

\begin{figure}[t]
\centering
\hfill
  \begin{subfigure}[h]{.32\textwidth}
    \captionsetup{justification=centering}
    \centering
    \includegraphics[width=\linewidth,height=0.8\linewidth]{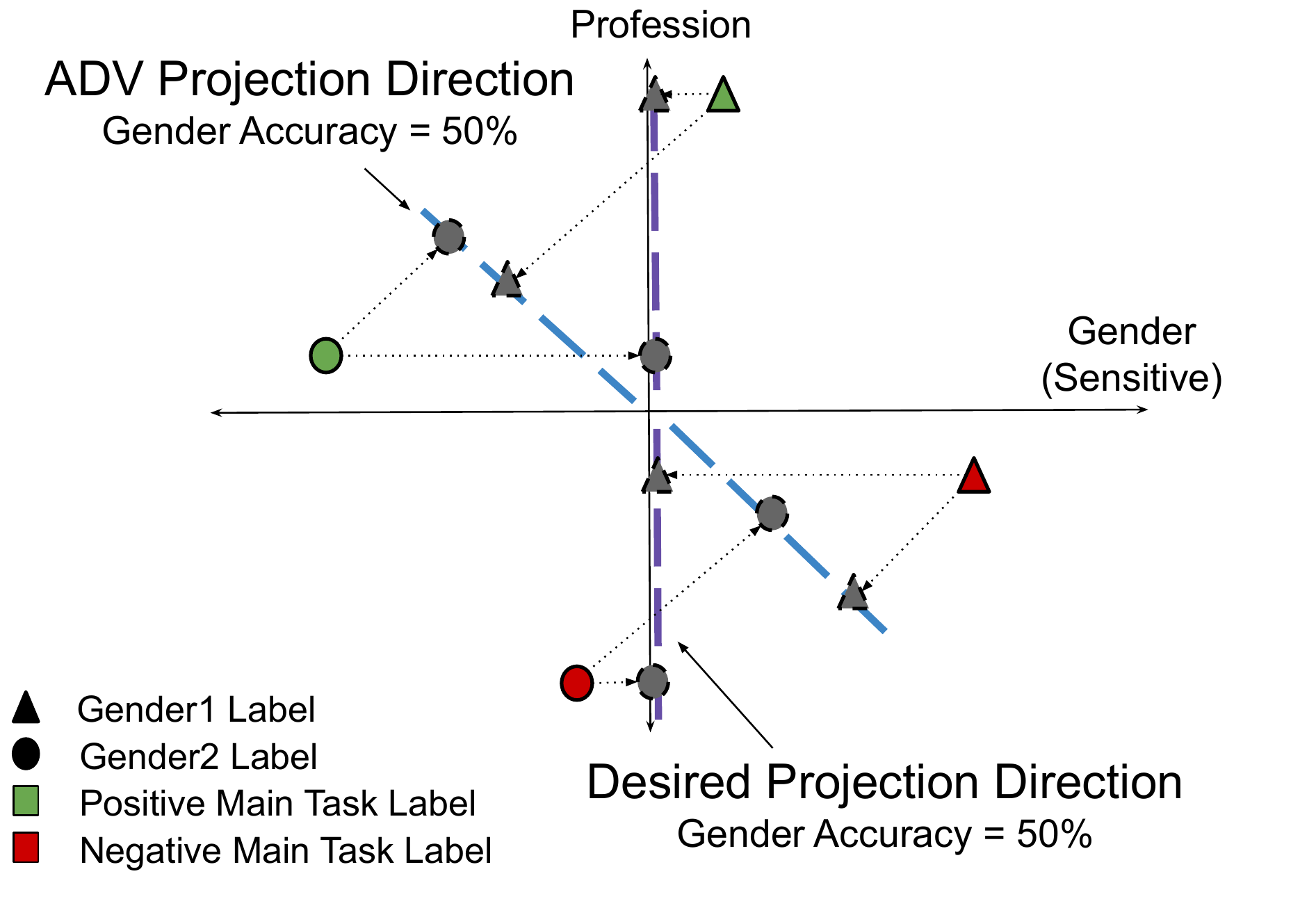}
		\caption[]%
        {{\small Probing (Gender) Accuracy}
        }    
		\label{fig:adv_gender_accuracy} 
  \end{subfigure}
\hfill
  \begin{subfigure}[h]{.32\textwidth}
    \captionsetup{justification=centering}
    \centering
    \includegraphics[width=\linewidth,height=0.8\linewidth]{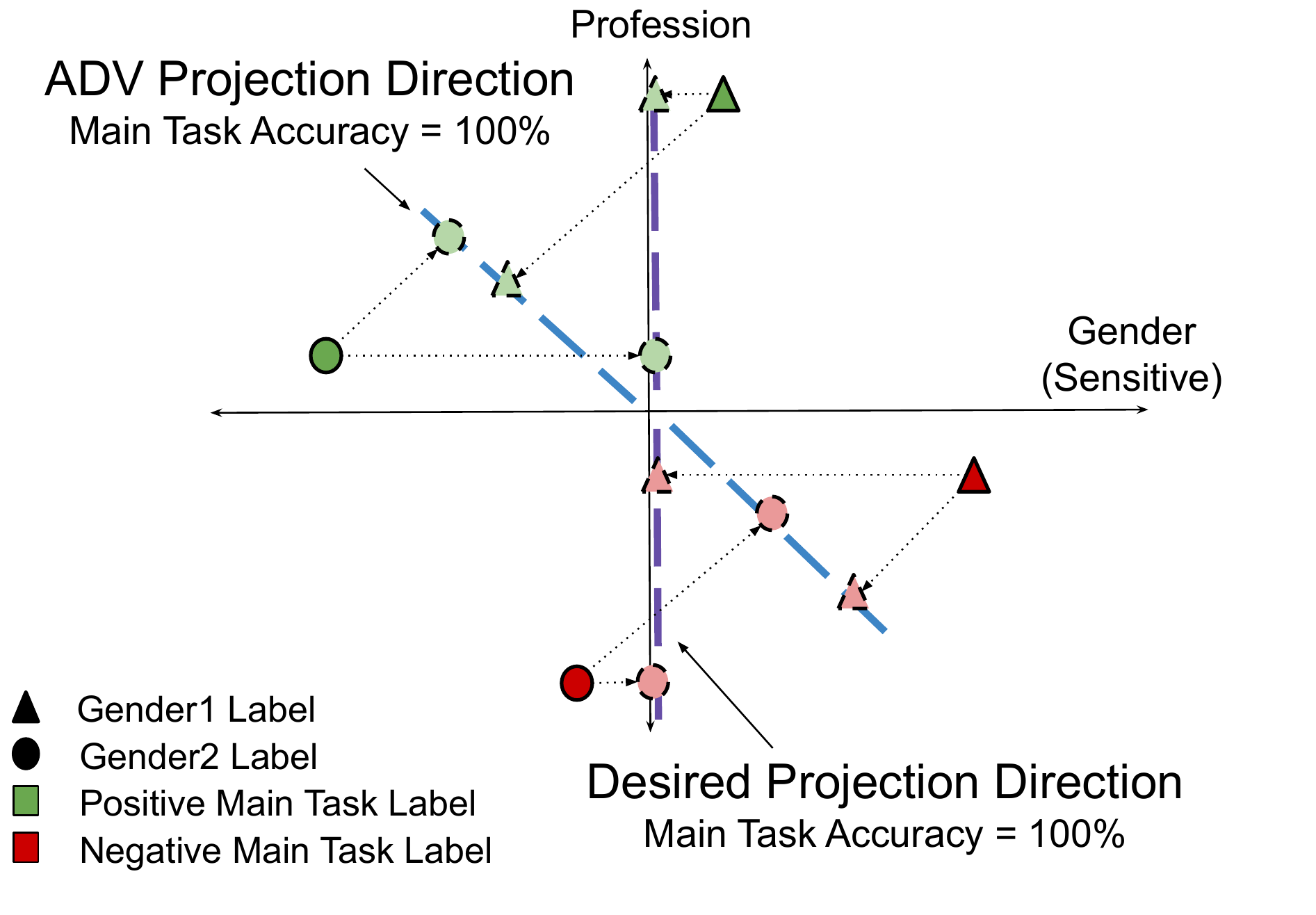}
		\caption[]%
        {{\small Main-Task (Profession) Accuracy}
        }    
		\label{fig:adv_profession_accuracy} 
  \end{subfigure}
\hfill
  \begin{subfigure}[h]{.32\textwidth}
    \captionsetup{justification=centering}
    \centering
    \includegraphics[width=\linewidth,height=0.8\linewidth]{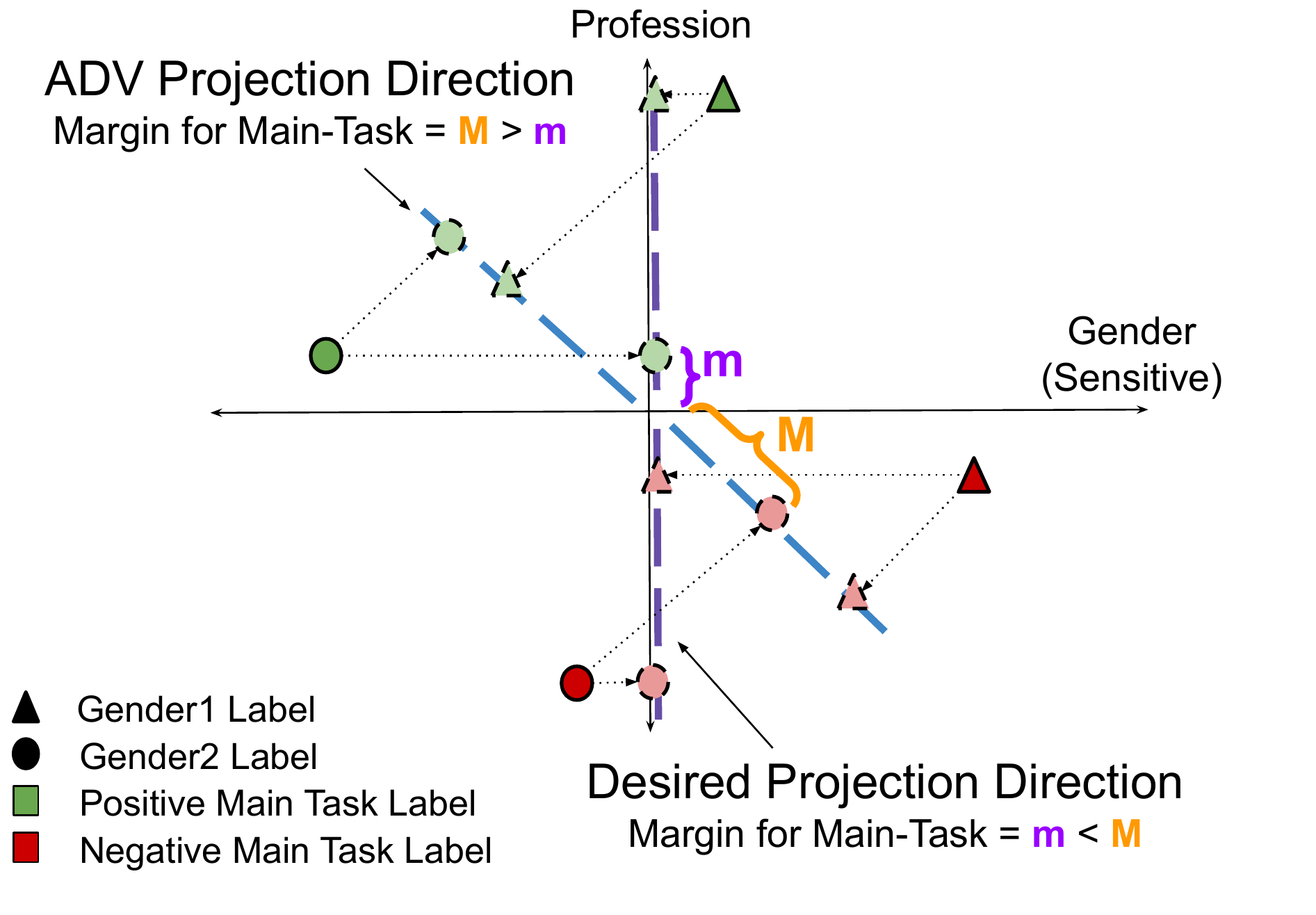}
		\caption[]%
        {{\small Main-Task (Profession) Margin}
        }    
		\label{fig:adv_profession_margin} 
  \end{subfigure}

\caption{\textbf{Failure mode of adversarial removal.}  As in \reffig{fig:inlp_problem_illustration}, the main task label is \textit{Profession} and \textit{Gender} is the spurious concept, each corresponding to one of the dimensions of the 2-dimensional  feature representation $\bm{z}$.  
%Assume a 2-dimensional representation, $\bm{z}$ where  Main task feature is Profession and sensitive feature is Gender to be removed. 
Assume that the shared representation is a scalar value obtained by projecting the two features in some direction. The adversarial goal is to find a projection direction such that the concept (gender) classifier obtains a random-guess accuracy of $50\%$ but has good accuracy on the main task label (profession).  \textbf{(a)} Two projection directions, shown by vertical and slanted lines, that yield  \textit{random-guess} 50\% accuracy on gender prediction,  and \textbf{(b)} have the same 100\% accuracy for profession prediction. \textbf{(c)} However, the slanted projection direction has a bigger margin for the main task and will be preferred, thus leading to a final classifier that uses the gender concept (see \refsec{subsec:adv_rem_problem}).
% For null-space removal see failure mode paragraph  in \refsec{subsec:null_space_problem} and for adversarial removal see failure mode paragraph in \ref{subsec:adv_rem_problem}
} 
\label{fig:adv_problem_illustration}
\end{figure}

%Let the training examples for probing task come from dataset $\mathcal{D}_{p}$ and the examples for the main-task comes from dataset $\mathcal{D}_{m}$. In general setting, these two dataset could be different. 
As in \reflemma{lemma:sufficient_condition},  we assume that the encoder $h:\bm{X}\rightarrow \bm{Z}$ mapping the input to the latent representation $\bm{Z}$ is frozen. To allow 
% training for the main task classifier
\new{for the removal of the unwanted features $\bm{z}_{p}$}
, we introduce additional representation layers after it. For simplicity in the proof, we assume a linear transformation to the latent representation $h_{2}:\bm{Z}\rightarrow \bm{\zeta}$. This layer is followed by the  linear main-task classifier $c_{m}:\bm{\zeta}\rightarrow Y_{m}$, as before. The  probing classifier $c_{p}:\bm{\zeta}\rightarrow Y_{p}$ is trained adversarially to remove $\bm{z}_{p}$ from the latent representation $\bm{\zeta}$. Thus, the goal of the adversarial method can be stated as removing the information of $\bm{z}_{p}$ from $\bm{\zeta}$. 
Let the main-task classifier satisfy assumptions of the generalized version of \reflemma{lemma:sufficient_condition} (\reflemma{lemma:sufficient_condition_main_task}, \refappendix{subsec:app_generalized_probing_result}). We also need an additional assumptions on the hard-to-classify margin points to ensure that main-task labels and \prop labels are correlated on the margin points of a \emph{clean} main-task classifier. Proof of the \reftheorem{theorem:adv_removal_failure} stated below is in \refappendix{sec:app_adv_setup_proof}.

\begin{assumption}[Label Correlation on Margin Points]
\label{assm:label_correlation}
For the margin points of a \emph{clean} classifier for the main task, the adversarial-probing labels $y_{p}$ and the main task labels $y_{m}$ are correlated, i.e., w.l.o.g., $y_{m}^{i}=y_{p}^{i}$ for all  margins points of the clean main task classifier.
%Let $\mathcal{X}^{\mu}$ be the set of margin point of purely-invariant classifier (\refdef{def:purely_invariant}) for the main classification task trained using max-margin objective. Let $y^{\mu}$ and $t^{\mu}$ be the main-task and adversarial task labels respectively for a margin point $\bm{x}^{\mu,i}\in \mathcal{X}^{\mu,i}$. Then for every $\bm{x}^{\mu,i}\in \mathcal{X}^{mu}$ we have $y^{\mu,i}=t^{\mu,i}$. Alternatively, for every $\bm{x}^{\mu,i}\in \mathcal{X}^{mu}$ we can have $y^{\mu,i}=(-1)\times t^{\mu,i}$.
\end{assumption}

%The following theorem shows that, it's impossible to make the main-classifier invariant of unwanted feature $\bm{z}_{p}$ using adversarial removal method:

\begin{restatable}{theorem}{advthm}
\label{theorem:adv_removal_failure}
%Let the pre-trained frozen encoder $h_{1}:\bm{x}\rightarrow \bm{z}$ which maps the input to latent representation $\bm{z}$ be frozen and 
Let the latent representation $\bm{z}$ satisfy \refassm{assm:disentagled_latent} and be frozen, $h_{2}(\bm{z})$ be a linear transformation over $\bm{Z}$ s.t. $h_{2}:\bm{Z}\rightarrow \bm{\zeta}$, the main-task classifier be $c_{m}(\bm{\zeta})=\bm{w}_{c_{m}}\cdot \bm{\zeta}$, and the adversarial \new{probing} classifier be $c_{p}(\bm{\zeta})=\bm{w}_{c_{p}}\cdot \bm{\zeta}$. Let all the assumptions of \reflemma{lemma:sufficient_condition_main_task} be satisfied for main-classifier $c_{m}(\cdot)$ when using $\bm{z}$  directly as input and \refassm{assm:fully_pred_inv} be satisfied on $\bm{z}$ w.r.t. the adversarial task. Let $h_{2}^{*}(\bm{z})$ be the desired encoder which is successful in removing $\bm{z}_{p}$ from $\bm{\zeta}$. Then there exists an undesired/incorrect encoder $h^{\alpha}_{2}(\bm{z})$ s.t. $h^{\alpha}_{2}(\bm{z})$ is dependent on $\bm{z}_{p}$ and the main-task classifier $c_{m}(h_{2}^{\alpha}(\bm{z}))$ has bigger margin than $c_{m}(h_{2}^{*}(\bm{z}))$  and has,

%Let the second trainable encoder be of form $h_{2}(\bm{z})=\bm{w}_{m}\cdot\bm{z}_{m}+\bm{w}_{p}\cdot \bm{z}_{p}$ mapping the latent $\bm{z}$ to a scalar value. Let, the linear main-task classifier be of form  $c_{m}(h_{2}(\bm{z}))=1\cdot h_{2}(\bm{z})$ and the adversarial probing classifier be of form $c_{p}(h_{2}(\bm{z}))=\beta\cdot h_{2}(\bm{z})$ where $\beta$ is also a trainable. Let \refassm{assm:disentagled_latent} and \ref{assm:fully_pred_inv} holds for the adversarial classifier. Let the desired second encoder be $h_{2}^{*}(\bm{z})=\bm{w}^{*}_{m}\cdot\bm{z}_{m}$ which is successful in removing the unwanted attribute $\bm{z}_{p}$ and thus not used by the main-task classifier. There there exist an incorrect $h_{2}(\bm{z})=\bm{w}_{m}\cdot\bm{z}_{m}+\bm{w}_{p}\cdot \bm{z}_{p}$ where $\bm{w}_{p}\neq \bm{0}$ which enables main classifier $c_{m}(h_{2}(\bm{z}))$ with incorrect encoder $h_{2}()\bm{z}$ to have bigger margin than $c_{m}(h^{*}_{2}(\bm{z}))$ with desired classifier and:
\begin{enumerate}[leftmargin=*,topsep=0pt]{}
    \item $Accuracy(c_{p}(h^{\alpha}_{2}(\bm{z})),y_{p}) = Accuracy(c_{p}(h^{*}_{2}(\bm{z})),y_{p})$;  when adversarial probing classifier $c_{p}(\cdot)$ is trained using any learning objective like max-margin or cross-entropy loss. Thus, the undesired encoder $h^{\alpha}_{2}(\bm{z})$ is indistinguishable from desired encoder $h^{*}_{2}(\bm{z})$ in terms of adversarial task prediction accuracy but better for main-task in terms of max-margin objective.
    
    \item $L_{h_{2}}\big(c_{m}(h_{2}^{\alpha}(\bm{z})),c_{p}(h_{2}^{\alpha}(\bm{z}))\big) < L_{h_2}\big(c_{m}(h_{2}^{*}(\bm{z})),c_{p}(h_{2}^{*}(\bm{z}))\big)$;  when \refassm{assm:label_correlation} is satisfied and \prop-causal features $\bm{z}_{p}^{M}$ of any margin point $\bm{z}^{M}$ of $c_{m}(h_{2}^{*}(\bm{z}))$ are more predictive of the main task label than $\bm{z}_{p}^{P}$ of any margin point $\bm{z}^{P}$ of $c_{p}(h_{2}^{*}(\bm{z}))$ is predictive for the probing label  (\refassm{assm:correlation_strength}). 
    Thus,  undesired encoder $h^{\alpha}_{2}(\bm{z})$ is preferable over desired encoder $h^{*}_{2}(\bm{z})$ for both main and combined adversarial objective. 
    Here $L_{h_2}=L(c_{m}(\cdot))-L(c_{p}(\cdot))$ is the combined adversarial loss w.r.t. to $h_{2}$ and $L(c(\cdot))$ is the max-margin loss for a classifier ``c'' (see \refsec{subsec:app_adv_setup}). 
\end{enumerate}
\end{restatable}
\textit{Proof Sketch.}
% \new{
%The combined loss for encoder is given by $L_{h_{2}}=L(c_{m}(\cdot))-L(c_{p}(\cdot))$ where $L(c_{m}(\cdot))$ is main-task loss and $L(c_{p}(\cdot))$ is probing loss. 
\textbf{(1)} The proof is by construction. Using \reflemma{lemma:sufficient_condition_main_task}, we show that there exists $h^{\alpha}_{2}$ s.t. $L(c_m(h_{2}^{\alpha}))<L(c_m(h_{2}^{*}))$, and that accuracy of the probing classifier remains the same when using either encoder. % since we create $h_{2}^{\alpha}$ from $h_{2}^{*}$ s.t. point doesn't change sides of decision boundary of $h_{2}^{*}$. 
%Thus, $h^{\alpha}_2$ may be preferred under the combned adversarial objective $L_{h2}$. %Assuming that $L(c_{p}(\cdot))$ is measured by accuracy of probing classifier we have same loss for both, thus $h_{2}^{\alpha}$ is preferred by overall objective. 
\textbf{(2)} Compared to $h^*_2$, we show that the improvement in main task loss when using $\bm{z}_p$ features is larger than the improvement in the probing loss for $h_{2}^{\alpha}$,  thus preferred by overall objective. 

\vspace{-1em}
\subsection{Implications for real-world data: A metric for quantifying degree of spuriousness}
\label{subsec:spurriousness_score}
Our theoretical analysis  shows that probing-based removal methods fail to make the main task classifier invariant to unwanted concepts.  However,  to verify whether the final classifier is using the \prop or not, the theorem statements require knowledge of the \prop's features $\bm{z}_p$. For practical usage,   we propose a metric that quantifies the degree of failure or  \emph{spuriousness} for both the main and probing classifier. For simplicity, we define it assuming that both main and \prop labels are binary.% Having such a metric can be useful in practice to gauge the quality of a probing classifier and hence the property removal result.

 Let $\mathcal{D}_{m,p}$ be the dataset where for every input $\bm{x}^{i}$ we have both the main task label $y_{m}$ and the \prop label $y_{p}$. %Let both $y_{m}$ and $y_{p}$ be binary taking value from set $\{-1,1\}$. 
 We define $2\times2$ groups, one for each combination of $(y_{m},y_{p})$. Without loss of generality, assume that the main-task label $y_{m}=1$ is spuriously correlated with \prop label $y_{p}=1$ and similarly $y_{m}=0$ is correlated with $y_{p}=0$. Thus, 
 $(y_{m}=1,y_{p}=1)$ and $(y_{m}=0,y_{p}=0)$ are the majority group $S_{maj}$ while 
 groups $(y_{m}=1,y_{p}=0)$ and $(y_{m}=0,y_{p}=1)$ make up the minority group $S_{min}$. We expect the main classifier to exploit this correlation and hence perform badly on $S_{min}$ where the correlation breaks. Following \cite{GroupDRO}, we posit that minority group accuracy i.e $Acc(S_{min})$ can be a good metric to evaluate the degree of \emph{spuriousness}. We bound the metric by comparing it with the accuracy on $S_{min}$ of a ``clean'' classifier that does not use the \prop features. % One problem with this metric is lack of calibration as we don't know the upper-bound of it's value. A very loose upper-bound is have a minority group accuracy as $100\%$. Hence, we define a new metric which attempts to resolve this problem.

\begin{definition}[Spuriousness Score]
\label{def:spurriousness_score}
Given a dataset, $\mathcal{D}_{m,p}=S_{min} \cup S_{maj}$ with  binary task label and binary \prop, let $Acc^{f}(S_{min})$ be the minority group accuracy of a given main task classifier ($f$) and $Acc^{*}(S_{min})$ be the minority group accuracy of a \emph{clean} main task classifier that does not use the spurious \prop. Then  spuriousness score of $f$ is:
%\begin{equation*}
%$    \psi(f) = \Big{|}1 - \frac{Acc^{f}(S_{min})}{Acc^{*}(S_{min})}\Big{|}$.
$    \psi(f) = {|}1 - Acc^{f}(S_{min})/Acc^{*}(S_{min}){|}$.
%\end{equation*}
\end{definition}
% \abhinav{State this defined with respect to a classifier, so maybe parameterize the score with the classifier}

To estimate $Acc^{*}(S_{min})$, we subsample the dataset such that $y_p$ takes a single value in the sample and train the main classifier on it, as in~\cite{ravichander-etal-2021-probing}. Here the probing label $y_{p}$ no longer is correlated with the main task label $y_{m}$. 
% There can be other ways to estimate $Acc^{*}(S_{min})$, e.g., by reweighting the data or using the accuracy on $S_{maj}$. However, we found that the former had high variance and the latter requires an equal-noise assumption, $Acc(S_{maj})=Acc(S_{min})$, on a \emph{clean}  main task classifier. 
% \abhinav{equivalent thing could be defined for prob classifier by reversing the zsp and zinv}
The spuriousness score of a \textit{probing} classifier can be defined analogously to \refdef{def:spurriousness_score}, by swapping the task and concept label (see \refdef{def:spurriousness_score_probe}). % could be equivalently made for probing classifiers using appropriate majority and minority groups. 
For creating a clean probing classifier, we subsample the dataset such that $y_{m}$ takes a single value and train the probing classifier. % on it \cite{ravichander-etal-2021-probing}.}
%\begin{enumerate*}
%    \item Experimentally, by training a clean classifier which dosent uses the spurious feature. One way to train a clean classifier is by sub-samping the dataset such that $|S_{min}|=|S_{maj}|$. Another way is to 
    
    %\item Theoretically, by using $Acc(S_{maj})$ on the given main-task classifier as a upper bound to $Acc(S_{maj})$. This is only true under the assumption that $Acc(S_{maj})=Acc(S_{min})$ on a \emph{clean}  main-task classifier which don't use the spurious feature. For all other classifier $Acc(S_{maj})\geq Acc(S_{min})$ since the majority group could use the spuriously correlated attribute to increase its accuracy.
%\end{enumerate*}

\section{Experimental Results}
\label{sec:feature_removal_expt_main}
Theorems~\ref{theorem:null_space_failure} and \ref{theorem:adv_removal_failure} show the failure of \prop removal methods under a simplified setup and max-margin loss. But current deep-learning models are not trained using max-margin objective and might not satisfy the required assumptions (\refassm{assm:disentagled_latent},\ref{assm:fully_pred_inv},\ref{assm:spurious_linear_separably},\ref{assm:label_correlation}). Thus,  we now verify the failure modes on three real-world datasets and one synthetic dataset, without making any restrictive assumptions. %All models are trained using cross-entropy loss without any restrictive assumption. 
%For all our removal experiment, we use the same dataset for both training the main classifier and training the probing classifier. 
We use RoBERTa \cite{roberta19} as default encoder and fine-tune it over each real-world dataset. For \syn dataset  we use the sum of pre-trained GloVe embeddings \cite{pennington2014glove} of words in a sentence as the default encoder. For details on the experimental setup,  refer \refappendix{sec:app_expt_setup}.
% \abhinav{Give a brief description of dataset --> the task and correlation info}

\vspace{-0.5em}
\subsection{Datasets: Main task and spurious/sensitive \prop}
\label{subsec:exp_dataset_desc}
\vspace{-0.5em}
\textbf{Real-world data.} We use three datasets: \mnli \cite{mnliDataset}, \pan \cite{pan16Dataset} and \aae \cite{aaeDataset}. In \mnli, given two sentences---premise and hypothesis---the main task is to predict whether hypothesis \emph{entails}, \emph{contradicts} or is \emph{neutral} with respect to premise. We simplify to a binary task of predicting whether a hypothesis \emph{contradicts} the premise or not. Since negation words like \emph{nobody,no,never} and \emph{nothing} have been reported to be spuriously correlated with the \textit{contradiction} label~\cite{GururanganMNLINegationArtifact}, we create a `negation' \prop denoting the presence of these words. The goal is to remove the negation \prop from an NLI model's representation space. In \pan, the main task is to detect whether a tweet mentions another user or not, as in \cite{AdvRemYoav}. The dataset contains \emph{gender} label for each tweet, which we consider as the sensitive \prop to be % which is which they observed to be predictive from the tweets. 
removed from the main model's representation. In \aae, again following \cite{AdvRemYoav}, the main-task is to predict binary sentiment labels from a tweet's text. The tweets are associated with \emph{race} of the author, the sensitive \prop to be removed from the main model's representation. %label and were found to be highly predictive from the given tweets \cite{AdvRemYoav}. So, we choose race feature to be removed from the latent space of sentiment classification task using the removal methods. 

\textbf{Synthetic-Text.} To understand the reasons for failure, we introduce a Synthetic-Text dataset where it is possible to change the input text based on a change in \prop (thus implementing Def.~\ref{def:causally_derived_feature}). Here we can directly evaluate whether the \prop is being used by the main-task classifier by  intervening on the \prop (adding or removing) and observing the change in model's prediction. The main-task is to predict whether a sentence contains a numbered word (e.g., \emph{one, fifteen, etc.}). We introduce a spurious concept (length) by increasing the length of sentences that contain \emph{numbered} words.

\textbf{Predictive correlation.}
To assess robustness of removal methods, we create multiple datasets with different \textit{predictive} correlation between the two labels $y_{m}$ and $y_{p}$. The predictive correlation $(\kappa)$ is a practical
measure for the \textit{spurious correlation} defined in \refassm{assm:spurious_linear_separably}, that does not require access to $\bm{z}_m$ features. It
 measures how informative one label is for predicting the other, 
    $\kappa = Pr(y_{m}\cdot y_{p}>0) = \frac{\sum_{i=1}^{N}\mathbf{1}[y_{m} \cdot y_{p}>0]}{N}$, 
where $N$ is the size of dataset and $\mathbf{1}[\cdot]$ is the indicator function that  is $1$ if the argument is true otherwise $0$. Predictive-correlation lies in $\kappa\in[0.5,1]$ where $\kappa=0.5$ indicates no correlation and $\kappa=1$ indicates that the attributes are fully correlated. For more details on the datasets and how we vary the predictive-correlation, refer to \refappendix{sec:app_expt_setup}; and for additional results see \refappendix{sec:app_additional_results}.

\textbf{Measuring spuriousness of a classifier.} We use the \textit{Spuriousness Score} (\refdef{def:spurriousness_score}) to measure the degree of reliance of the main task classifier on the spurious concept, and vice-versa for the probing classifier. % indirectly using accuracy of classifier on its minority group. 
In addition, for the Synthetic-Text dataset, we use a metric $\Delta$Prob that exactly implements \refdef{def:causally_derived_feature} for estimating a model's reliance on the spurious concept. Since we can modify the concept directly in input space for \syn, $\Delta$Prob changes the parts of a input sentence corresponding to spurious concept and measures the change in the main task classifier's prediction probability. As a sanity check, on the \syn dataset, $\Delta$Prob and Spuriousness Score are highly correlated (Pearson correlation=0.83, \refsec{sec:app_comp_sp_score}).  For all real-world datasets,  we use the Spuriousness Score. % direct metric  which  In our experiment, we will use $\Delta$Prob for \syn dataset where we can make a change in input sentences and use Spuriousness Score otherwise. In  we show the relationship between $\Delta$Prob and Spuriousness Score by measuring Pearson correlation  and observe correlation $>0.83$ on multiple experiments on \syn dataset. 

% \abhinav{add that use will use delta prob when defined or use proxy sp socre --> highlight as contribution}

% \abhinav{Add the correlation ref between delta prob and sp-score}

\vspace{-0.7em}
\subsection{Results: Null space removal}
\label{subsec:expt_null_space_failure}
\vspace{-0.7em}

In general, for any  model given as input to \INLP, it may be difficult to verify whether \INLP removed the correct features. Hence, we construct a benchmark where the input classifier is \textit{clean}, i.e., it does not use the \prop at all. %Henceforth we will refer this classifier as \emph{clean}-classifier. 
We do so by training on a subset of data with one particular value of spurious \prop label, as in ~\cite{ravichander-etal-2021-probing}.
Since the input classifier does not use the \prop-causal features, we expect that \INLP should not have any effect on the main task classifier.
%\new{We have kept the main task classifier frozen in all the experiments described in this section. We also experiment with the setting when the main task classifier is retrained after every projection step of \INLP (see \reffig{fig:app_inlp-head_retrain_mix})}.
\new{Note that we keep the main task classifier frozen in all the experiments described below. For the setting where the main task classifier is retrained after every projection step of \INLP, refer \refsec{subsec:app_extended_null_space_results} and \reffig{fig:app_inlp-head_retrain_mix}}.
% \abhinav{Changed the last line here. It said, concept-feature will not be in rep which might not be true.}
% Since the input classifier does not have any  \prop-causal features in its representation, we should expect that \INLP should not have any effect on the main task classifier.

\begin{figure*}[t]
\centering

\begin{subfigure}[h]{.28\textwidth}
    \centering
    \includegraphics[width=\linewidth,height=0.7\linewidth]{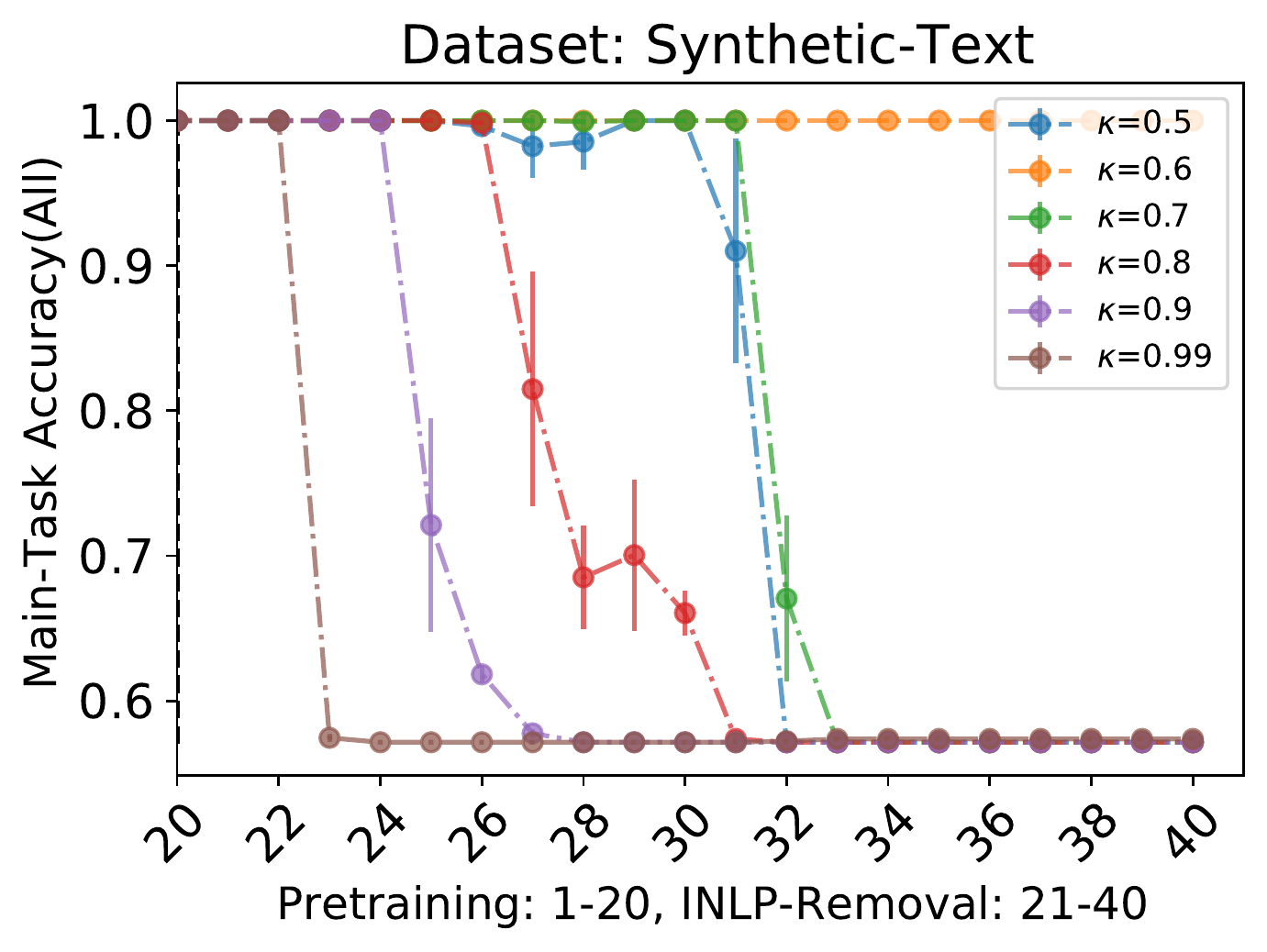}
    	\caption[]%
        {{\small  Main-Task Acc(All)}}    
    	\label{fig:inlp_syn_main_acc} 
    \end{subfigure}
\hfill 
\begin{subfigure}[h]{.28\textwidth}
    \centering
    \includegraphics[width=\linewidth,height=0.7\linewidth]{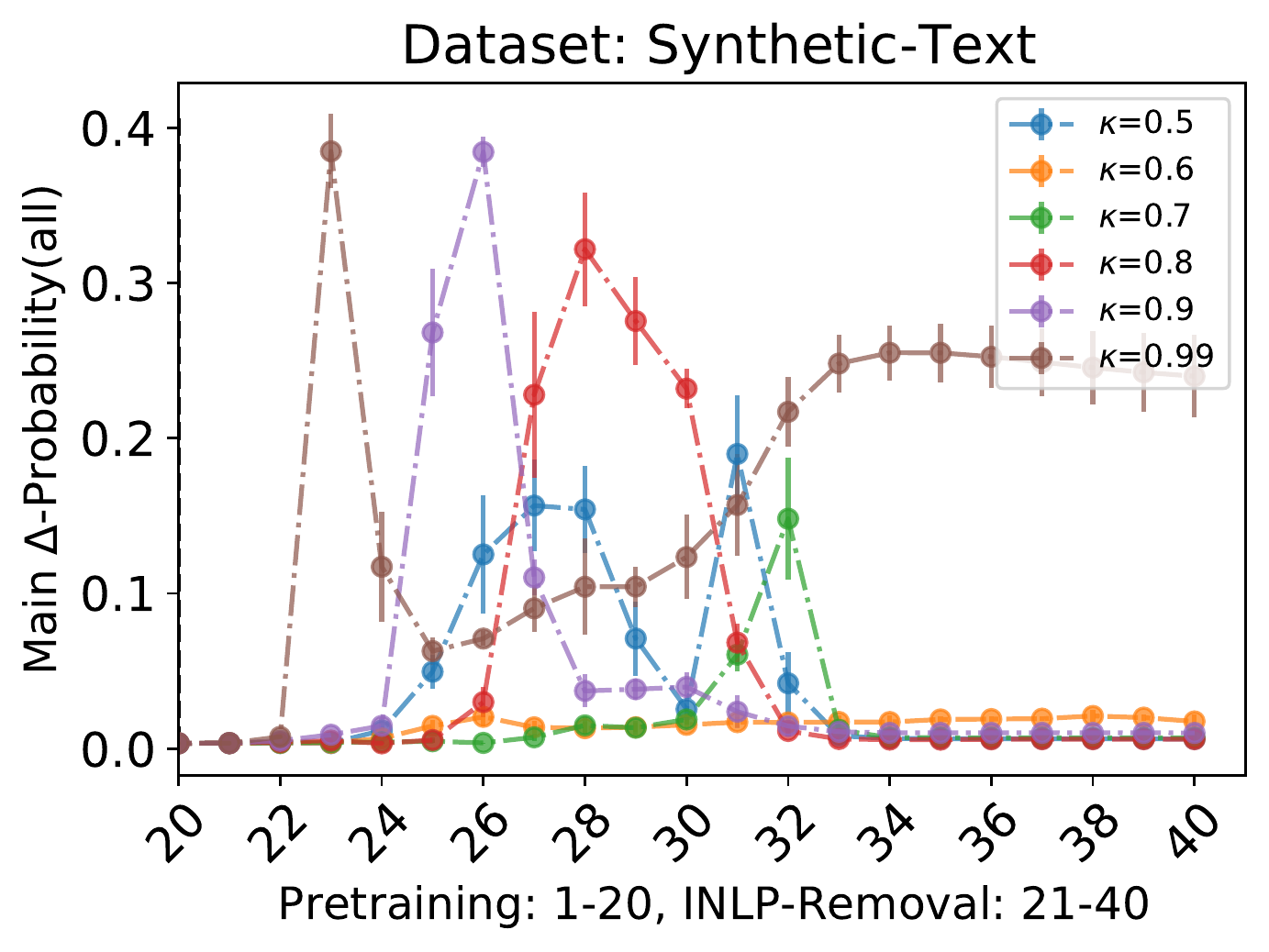}
    	\caption[]%
        {{\small Main-Task $\Delta$ Prob(All)}}    
    	\label{fig:inlp_syn_pdelta_main} 
    \end{subfigure}
\hfill
\begin{subfigure}[h]{.28\textwidth}
    \centering
    \includegraphics[width=\linewidth,height=0.7\linewidth]{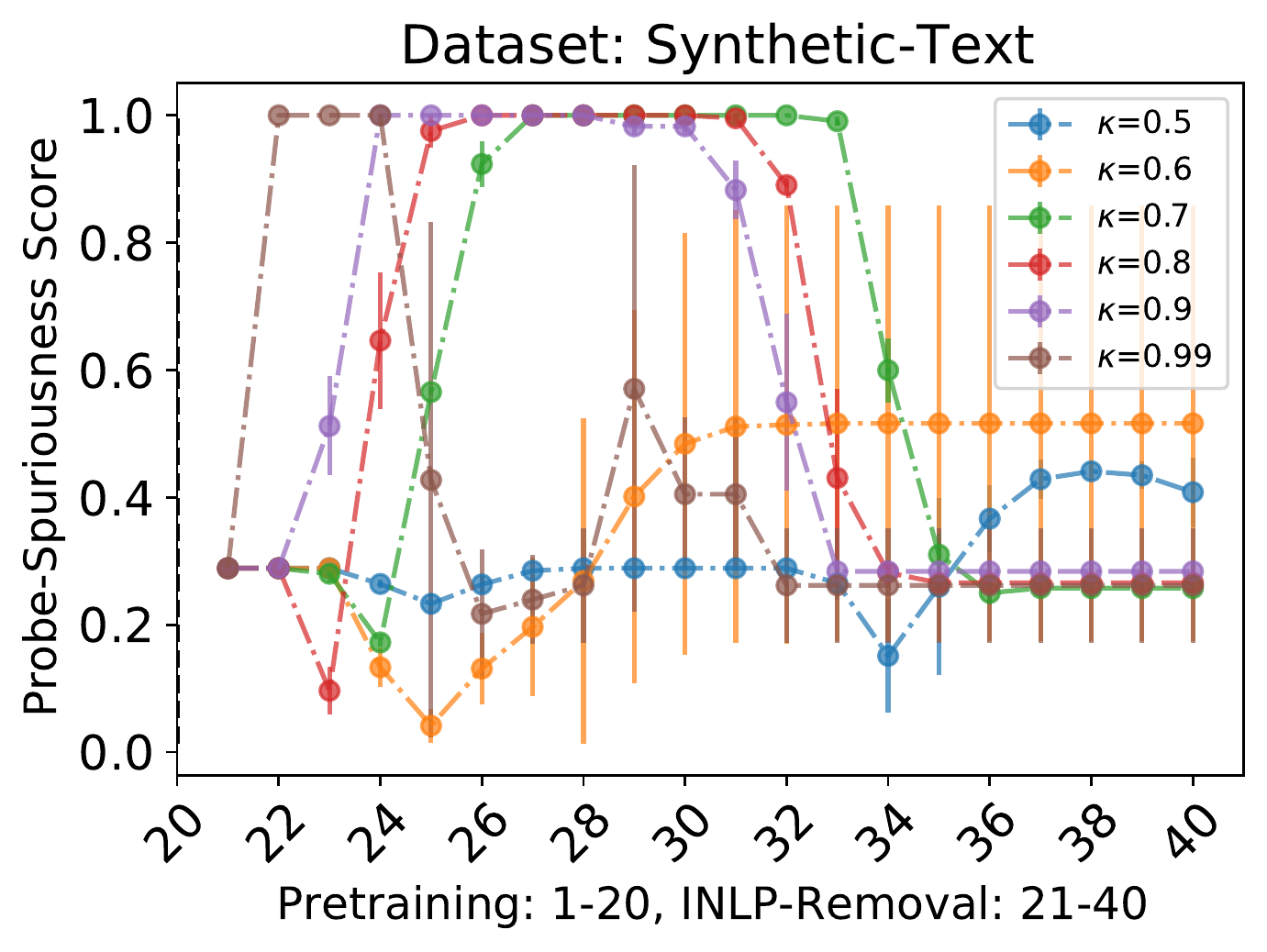}
    	\caption[]%
        {{\small Probe Spuriousness Score}}    
    	\label{fig:inlp_syn_probe_sp_score} 
    \end{subfigure}
\medskip
\begin{subfigure}[h]{.28\textwidth}
    \centering
    \includegraphics[width=\linewidth,height=0.7\linewidth]{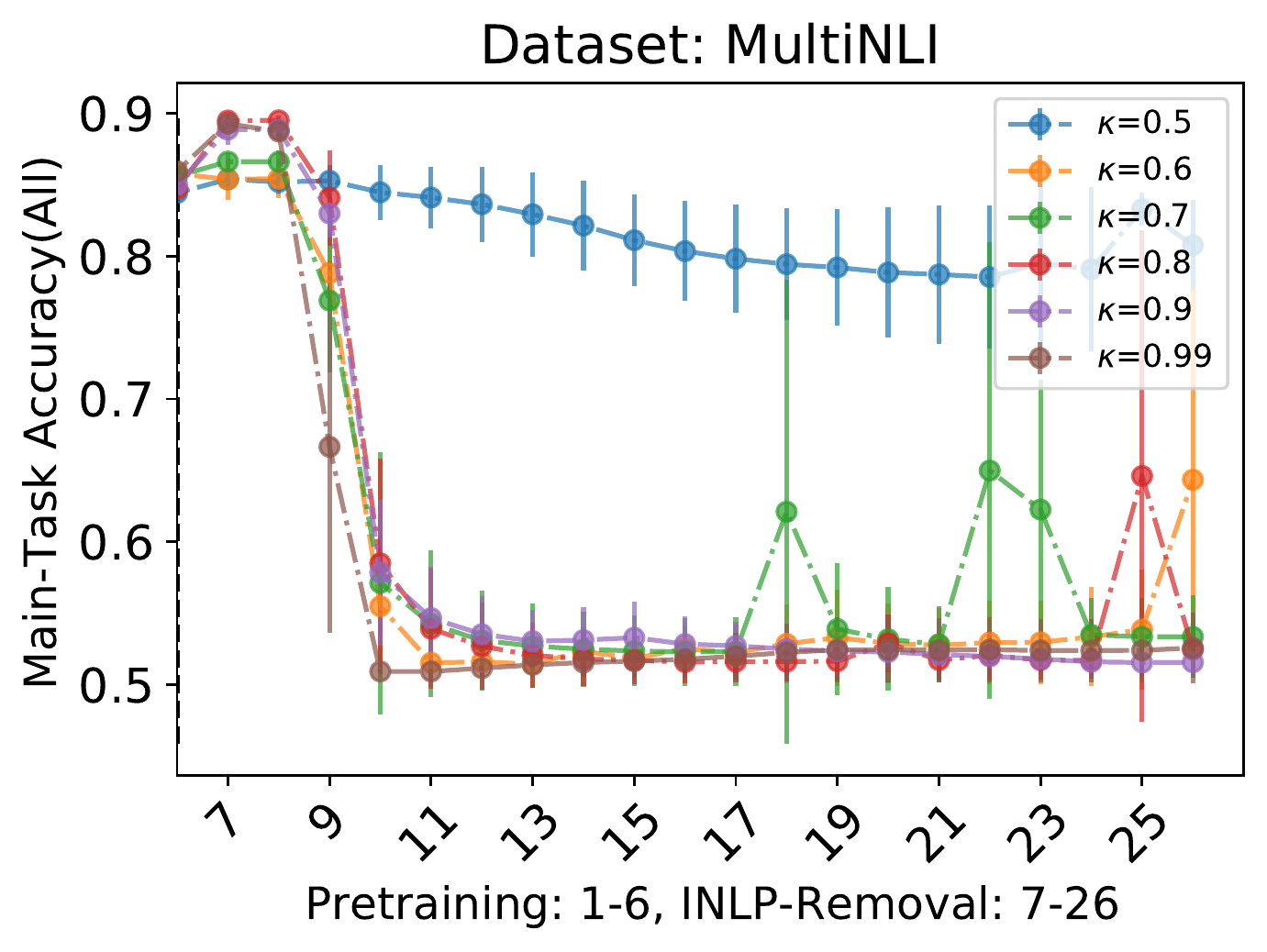}
    	\caption[]%
        {{\small Main-Task Acc(All)}}    
    	\label{fig:inlp_mnli_main_acc} 
    \end{subfigure}
\hfill
\begin{subfigure}[h]{.28\textwidth}
    \centering
    \includegraphics[width=\linewidth,height=0.7\linewidth]{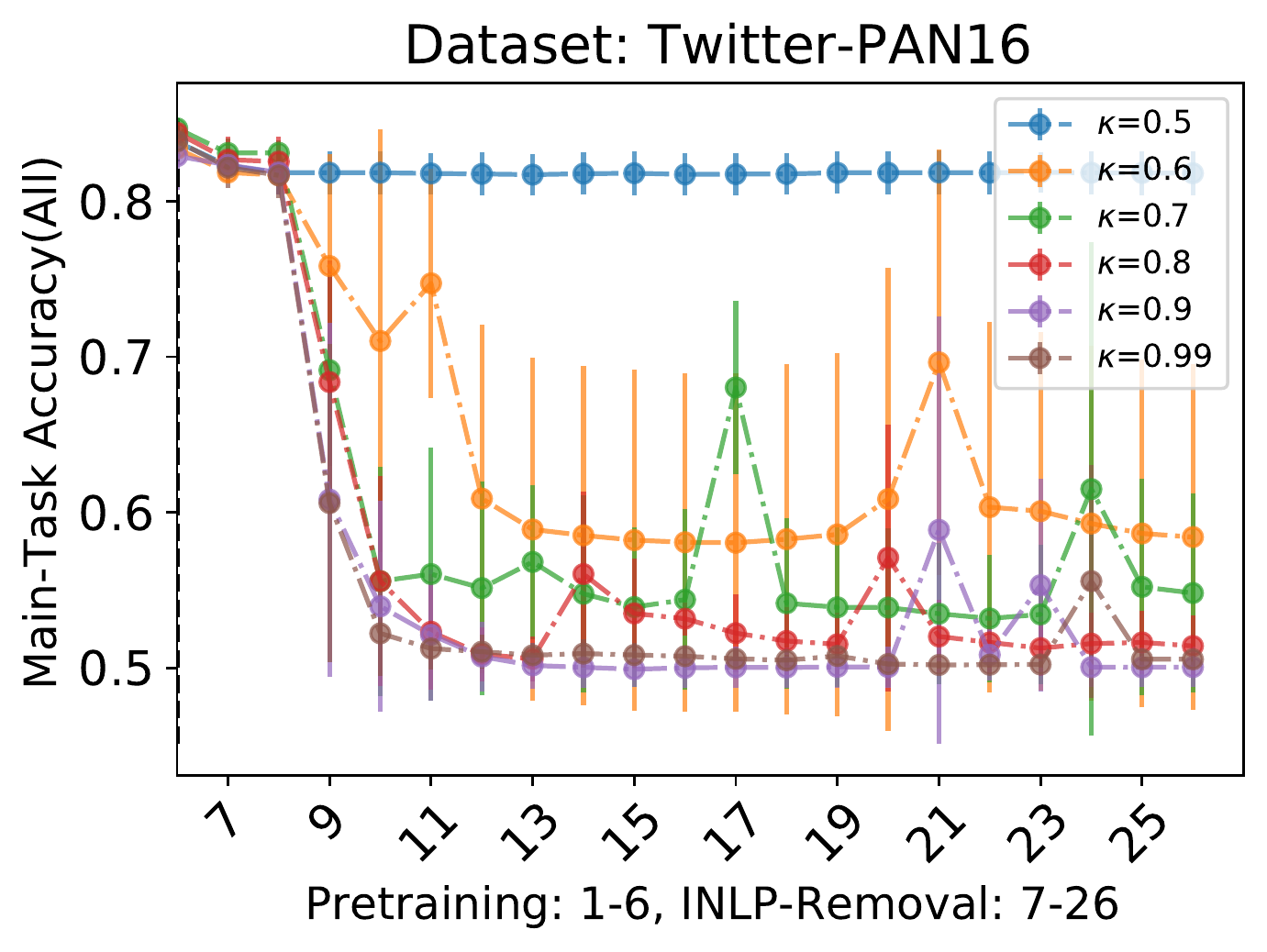}
    	\caption[]%
        {{\small Main-Task Acc(All)}}    
    	\label{fig:inlp_pan_main_acc} 
    \end{subfigure}
\hfill
\begin{subfigure}[h]{.28\textwidth}
    \centering
    \includegraphics[width=\linewidth,height=0.7\linewidth]{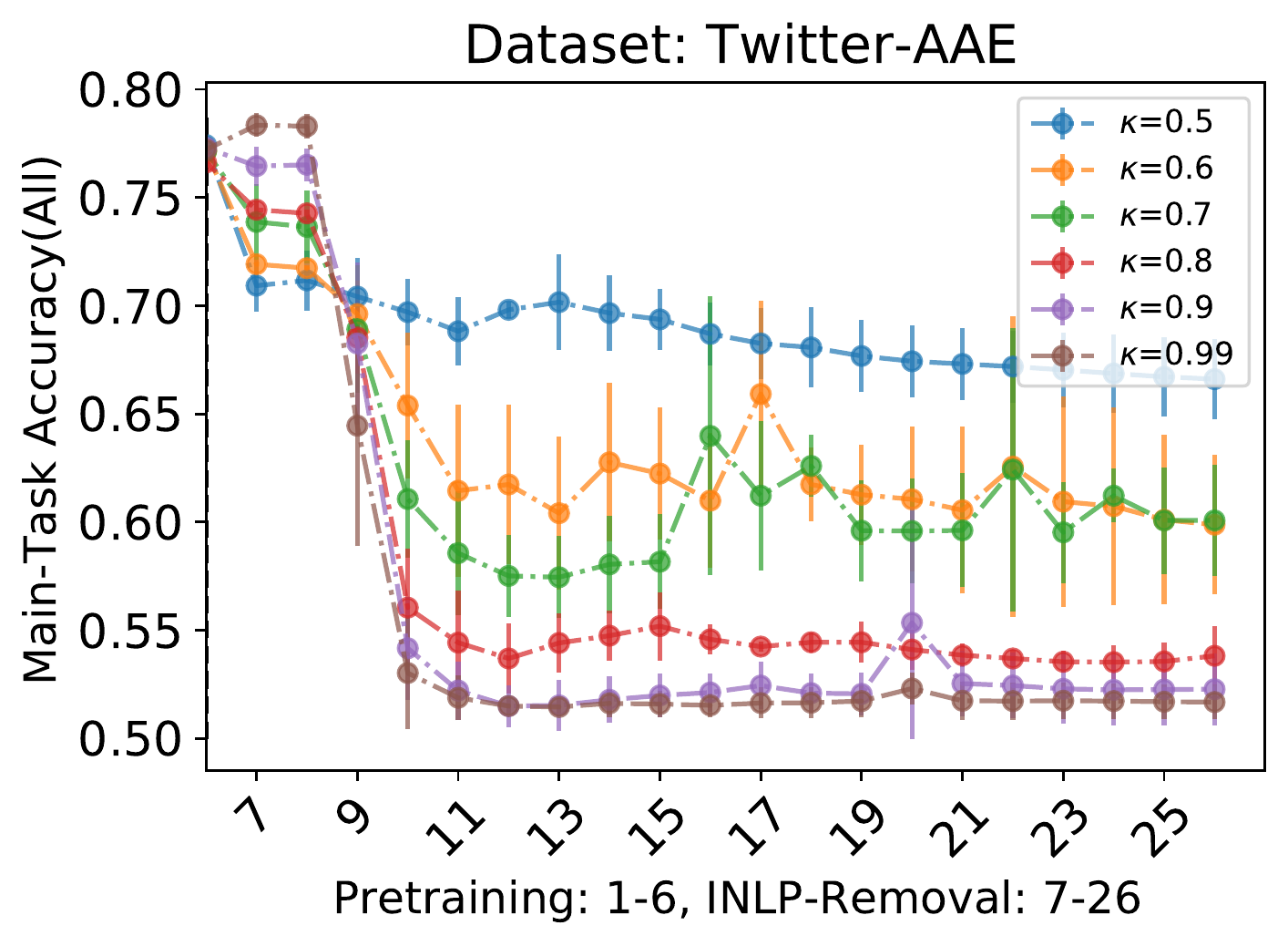}
    	\caption[]%
        {{\small 
        Main-Task Acc(All)}}    
    	\label{fig:inlp_aae_main_acc}
    \end{subfigure}
\hfill

\caption{\textbf{Null space removal failure. } Top row corresponds to the \syn dataset and bottom row shows the failure on three real-world datasets. In each figure, the x-axis shows the \INLP iteration and y-axis shows different evaluation metrics. Colored lines correspond to the different levels of predictive correlation ($\kappa$) in the datasets used by \INLP.  \textbf{(a), (d), (e), (f)} show that as \INLP removal progresses, main-task classifier is getting corrupted which leads to drop in its accuracy \new{(see \refsec{subsec:expt_null_space_failure})}.  
% \abhinav{Change the dataset name in the figure for syn text}
%Variation of main-task classifier's accuracy on 3 different datasets. For every dataset, from iteration 1-6 the main-task classifier is trained such that it doesn't  uses spurious feature. \chenhao{how can we guarantee that?} From iteration 7-26, every step represent one step of null-space removal. Different colored lines represent different predictive correlation ($\kappa$) between main and probing label in probing dataset used by null-space iteration. In the first row, main-task classifier's accuracy should not have dropped for any value of $\kappa$. In the second row, we can see that spuriousness score for probing classifier is increasing which we claim to be the main reason of failure \abhinav{should we say this is mixing happening}.\amit{yes you can talk about mixing--not here, in the main text}  For discussion refer \refsec{subsec:expt_null_space_failure} \abhinav{Remove the probing classifier's spuriousness score from 1-6, it is not defined there.}
% \chenhao{I do not understand why we need 1-6, it seems to mean different things from later steps.}
} 
\label{fig:expt_null_space_failure}
\end{figure*}

\textbf{Eventually all task-relevant features are destroyed. } We start with the \syn dataset by training a clean classifier on the main-task and inputting it to \INLP for removing the spurious concept. To keep the conditions favorable for \INLP, both the main task and concept-probing task can achieve 100\% accuracy using their causally derived features respectively. In \reffig{fig:inlp_syn_main_acc}, colored lines show  datasets with different levels of predictive correlation $\kappa$ that are provided to \INLP and iterations 21-40 show individual steps of null-space removal. Since, the given pre-trained classifier was \emph{clean}, i.e.,  not using the concept features, null-space removal shouldn't have any effect on it.  We observe that for all values of $\kappa$, the main-task classifier's accuracy eventually goes to 50\% random guess accuracy implying that the main-task related attribute has been removed by \INLP, as predicted by  \reftheorem{theorem:null_space_failure}. Higher the value of correlation $\kappa$, faster the removal of main-task attribute happens. We obtain a similar pattern over real-world datasets. %  \mnli, \pan and \aae. 
\reffig{fig:inlp_mnli_main_acc},\ref{fig:inlp_pan_main_acc} and \ref{fig:inlp_aae_main_acc} show a decrease in the main-task accuracy  even when the input  classifier for each dataset is ensured to be \emph{clean}: except for $\kappa=0.5$ (no correlation),  all values of $\kappa$ yield a  random-guess classifier after applying \INLP on \mnli while they yield classifiers with less than 60\% accuracy for \pan and \aae. %for all values of $\kappa$  whereas for \pan and \aae random-guess accuracy is obtained for higher values of $\kappa$. % for all values  thus demonstrating the failure of \INLP method.

\textbf{Early stopping increases the reliance on spurious features.}
To avoid full collapse of the main-task features, a stopping criterion in \INLP is to stop when the main-task classifier's performance drops~\cite{NullItOut:2020}. In \reffig{fig:inlp_syn_pdelta_main} we measure spuriousness, sensitivity of  the \syn main task classifier w.r.t. to the spurious concept, using $\Delta$Prob (see \refsubsec{subsec:exp_dataset_desc} and \ref{subsec:app_metric_desc}). At lower iterations of \INLP, %the change in main-task output due to change in spurious concept's value, 
$\Delta$Prob  is higher than that of the input classifier. For example, for $\kappa=0.8$, when the main-task classifier's performance drops at iteration 27, the classifier has a high $\Delta$Prob $\approx 25\%$, higher than the input classifier (\reffig{fig:inlp_syn_pdelta_main}). Hence it is possible that stopping prematurely will lead us to a classifier that is more reliant on the spurious concept than it was before, consistent with the statement 1(a) in \reftheorem{theorem:null_space_failure}. % stating that \INLP will lead to mixing of features in latent space.
%We replicate this analysis for the Multi-NLI in \refappendix{subsec:app_extended_null_space_results}, buta are unable to do so for \pan and \aae since it is difficult to modify input text based on changes in gender or race of the author. Here we utilize the Spuriousness score of the main classifier and find that it increases in the initial iterations, indicating the same conclusion.
The reason is the mixing of the main task and \prop-causal features in each iteration, as shown in \reffig{fig:inlp_syn_probe_sp_score} using the spuriousness score of the probing classifier  (\refdef{def:spurriousness_score}). At lower iterations, the spuriousness score of probing classifier increases to a very high value (close to max value 1), for all values of $\kappa$.
%as predicted by first statement in \reftheorem{theorem:null_space_failure}.
% \amit{is it true? more generally, we need to talk about the 3c plot, and say how spuriousness score is used for real datasets.}
% \chenhao{agreed, we should make that the center piece, rather than $\Delta$prob}

\textbf{Failure of causally-inspired probing. } Amnesic Probing~\cite{AmnesicProbing} declares that a \sensitive concept is being used by the model if, after removal of  the \prop from % causally derived from the concept 
from the latent representation using \INLP,  there is a drop in the main-task performance. But \reffig{fig:inlp_syn_main_acc}, \ref{fig:inlp_mnli_main_acc}, \ref{fig:inlp_pan_main_acc} and \ref{fig:inlp_aae_main_acc} show that even when the  input classifier for its corresponding main task is clean,  i.e., does not use the \sensitive concept, \INLP leads to drop in performance of the main-classifier. Hence,  removal-based methods like Amnesic probing will falsely conclude that the \sensitive concept is being used.

\vspace{-1em}
\subsection{Results: Adversarial removal}
\label{subsec:expt_adv_failure}
\vspace{-1em}

We now demonstrate failure of the adversarial-removal method (\ar) in removing the spurious \prop from the main classifier. We train a separate main-task classifier without any adversarial objective with standard cross-entropy loss objective (referred as \textit{ERM}). Then we compare  standard ERM training of the main classifier with the \ar method over the same number of epochs (20). %In all our experimental setup, both the main-task classifier and the removal-probing classifier is trained jointly for 20 epochs on the same dataset. 
We follow the training procedure of ~\cite{AdvRemYoav}
% \amit{add cite}  
and conduct a hyper-parameter sweep on the adversarial training strength to select the value which is most effective in removing the concept. For details, refer to \refappendix{sec:app_expt_setup}.

\begin{figure*}[t]
\centering
% \includegraphics[width=\textwidth,height=0.5
% \textwidth]{figs/adv-real-roberta-base-both-final.png}

% \begin{subfigure}[h]{.28\textwidth}
%     \centering
%     \includegraphics[width=\linewidth,height=0.7\linewidth]{figs/adv-MultiNLI-Main-Task Accuracy(All).pdf}
%     	\caption[]%
%         {{\small  Main-Task Acc(All)}}    
%     	\label{fig:adv_mnli_main_acc} 
%     \end{subfigure}
% \hfill 
\begin{subfigure}[h]{.28\textwidth}
    \centering
    \includegraphics[width=\linewidth,height=0.7\linewidth]{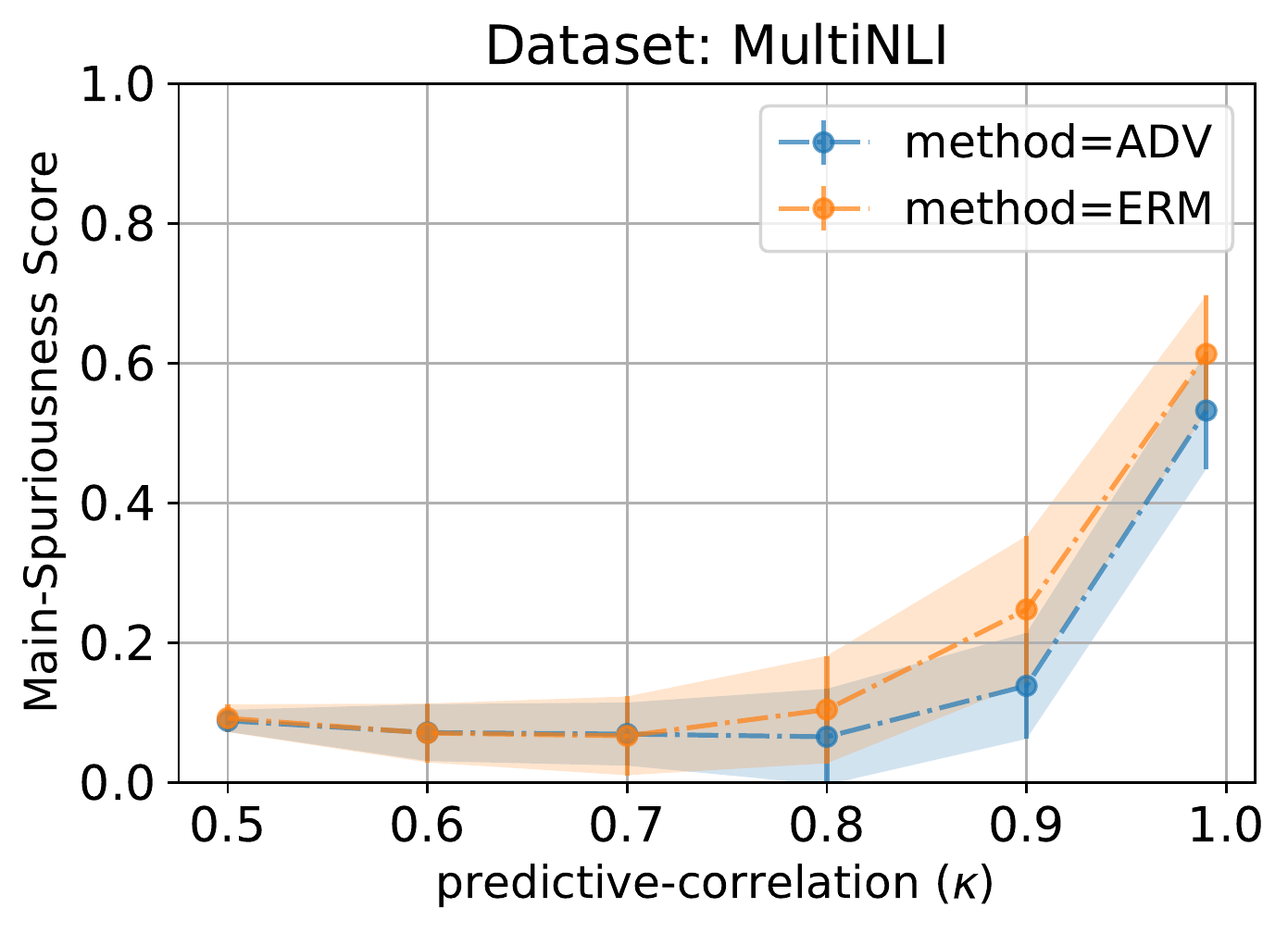}
    	\caption[]%
        {{\small Main Spuriousness Score}}    
    	\label{fig:adv_mnli_sp_score} 
    \end{subfigure}
\hfill
\begin{subfigure}[h]{.28\textwidth}
    \centering
    \includegraphics[width=\linewidth,height=0.7\linewidth]{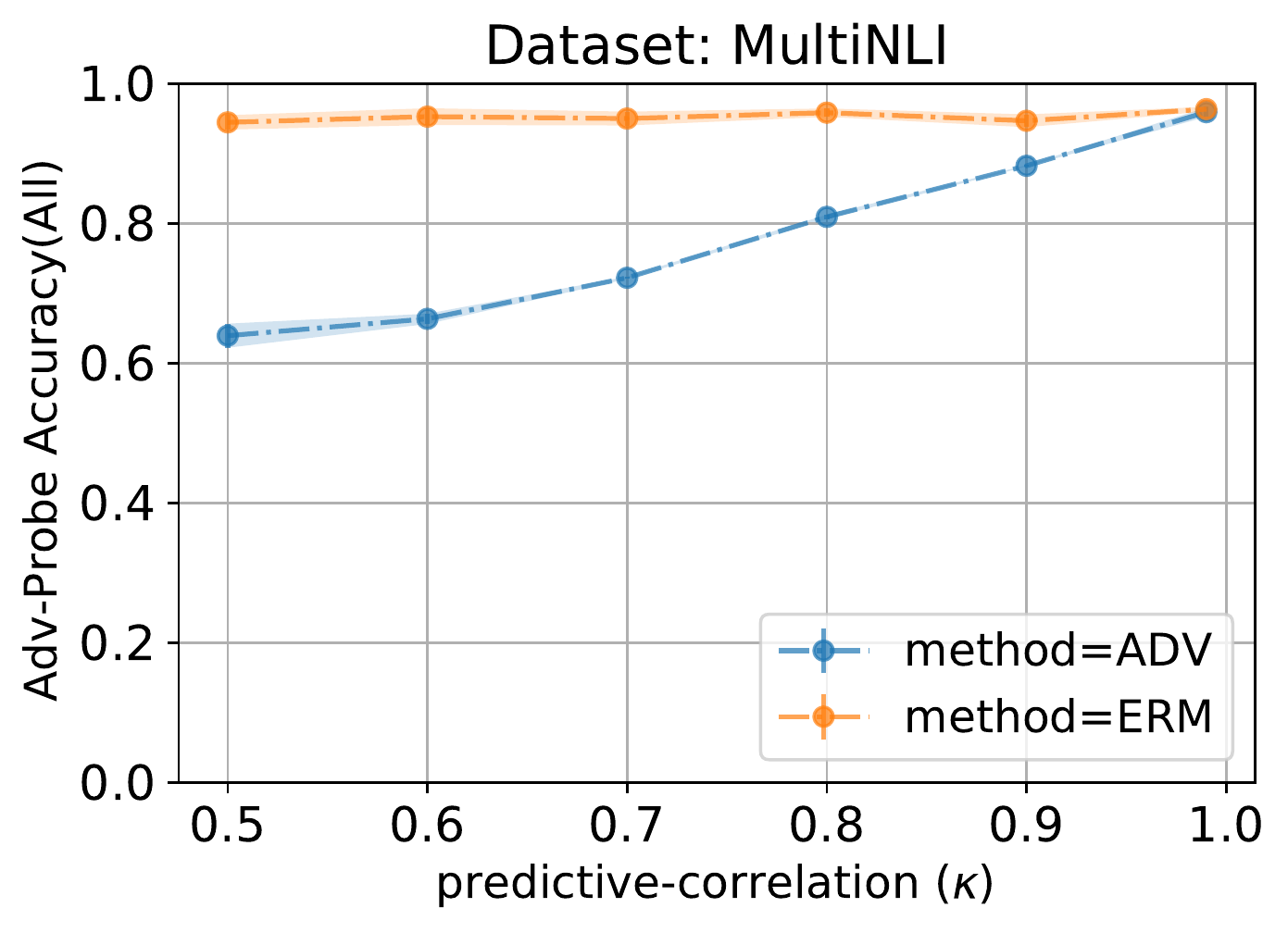}
    	\caption[]%
        {{\small Adv-Probe Acc(All)}}    
    	\label{fig:adv_mnli_probe_acc} 
    \end{subfigure}
\hfill
\begin{subfigure}[h]{.28\textwidth}
    \centering
    \includegraphics[width=\linewidth,height=0.7\linewidth]{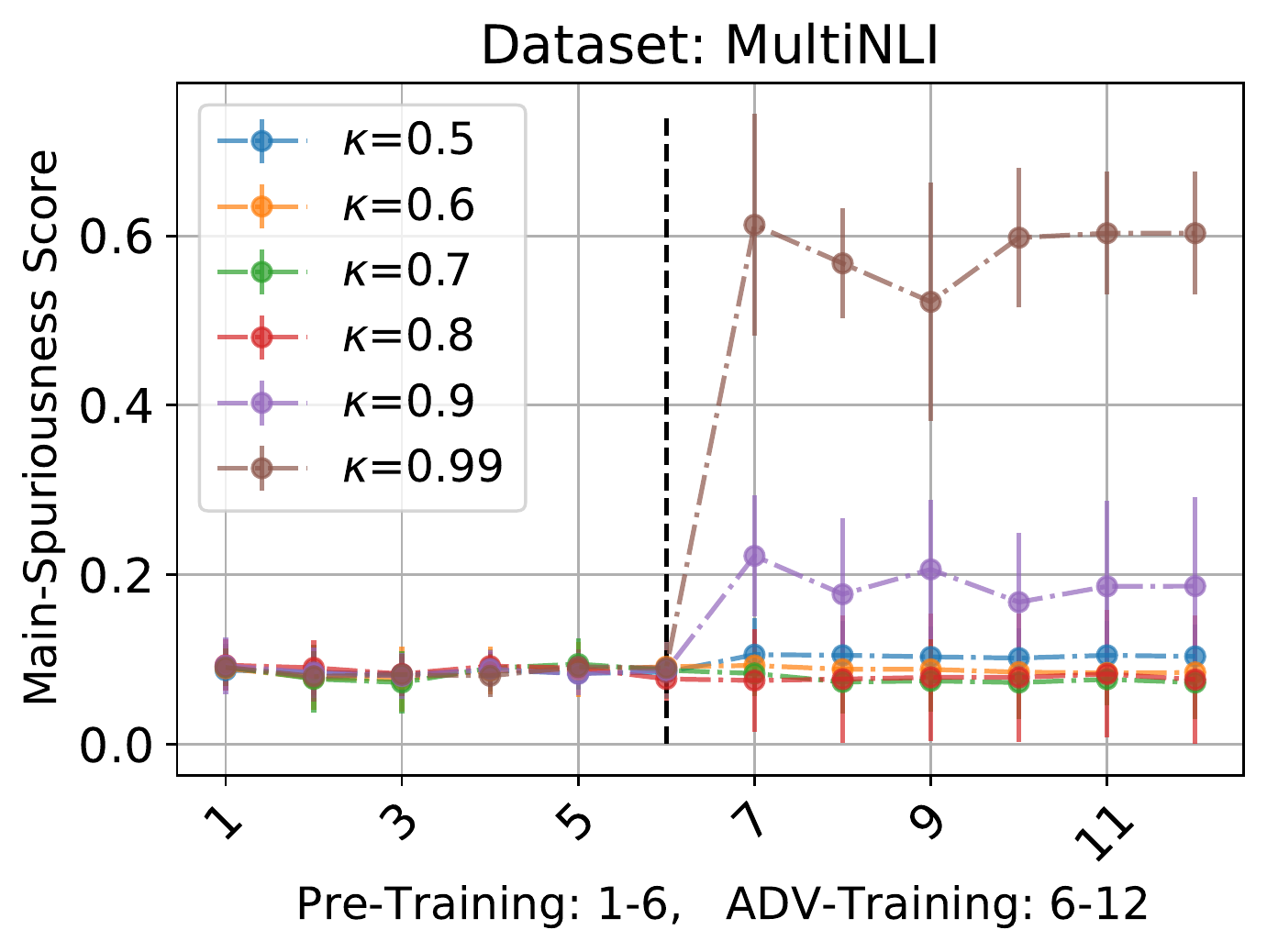}
    	\caption[]%
        {{\small \ar - Clean to Spurious}}    
    	\label{fig:adv_mnli_fairness_sp_score} 
    \end{subfigure}
\medskip
\begin{subfigure}[h]{.28\textwidth}
    \centering
    \includegraphics[width=\linewidth,height=0.7\linewidth]{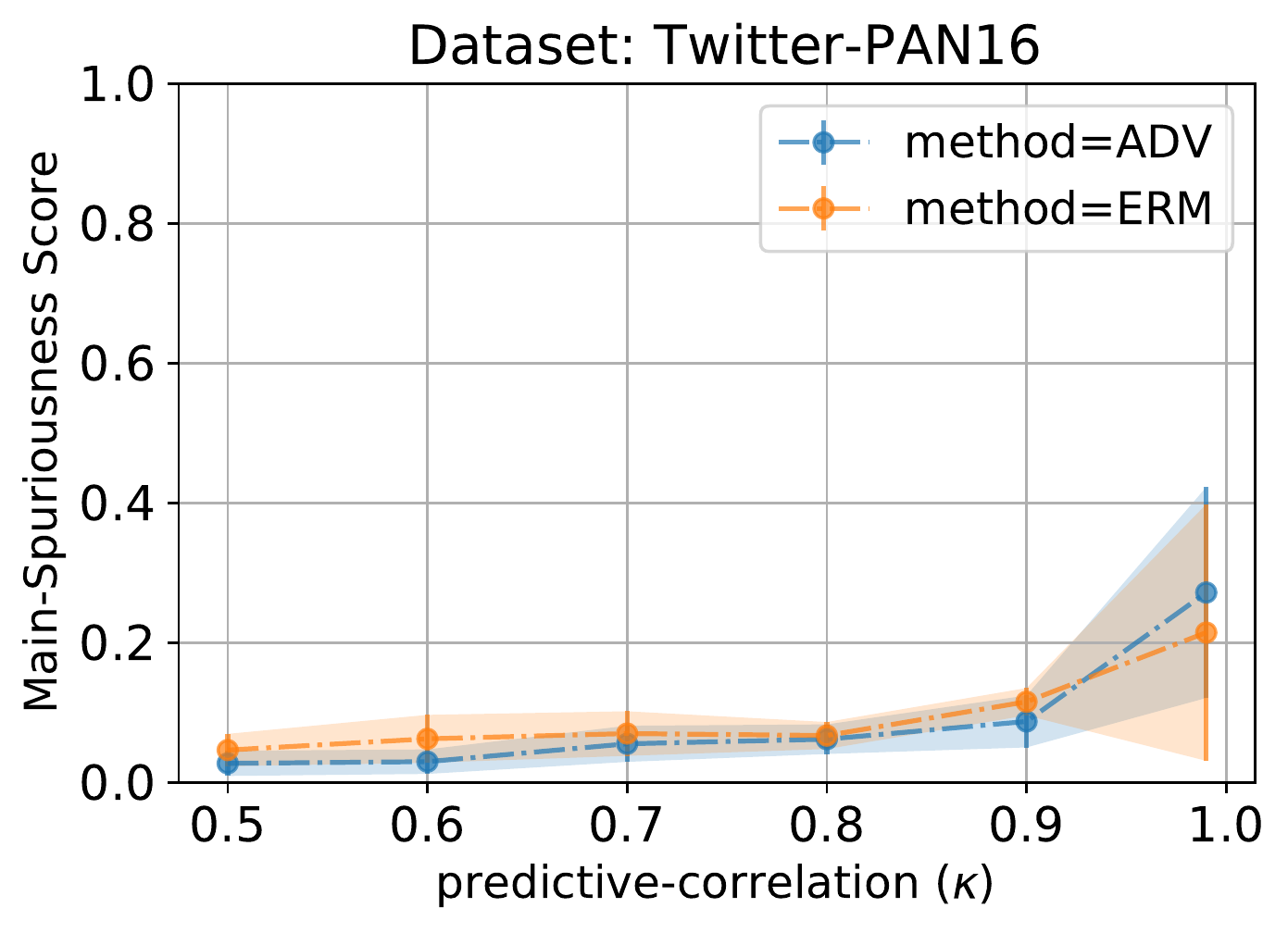}
    	\caption[]%
        {{\small Main Spuriousness Score}}    
    	\label{fig:adv_pan_sp_score} 
    \end{subfigure}
\hfill
\begin{subfigure}[h]{.28\textwidth}
    \centering
    \includegraphics[width=\linewidth,height=0.7\linewidth]{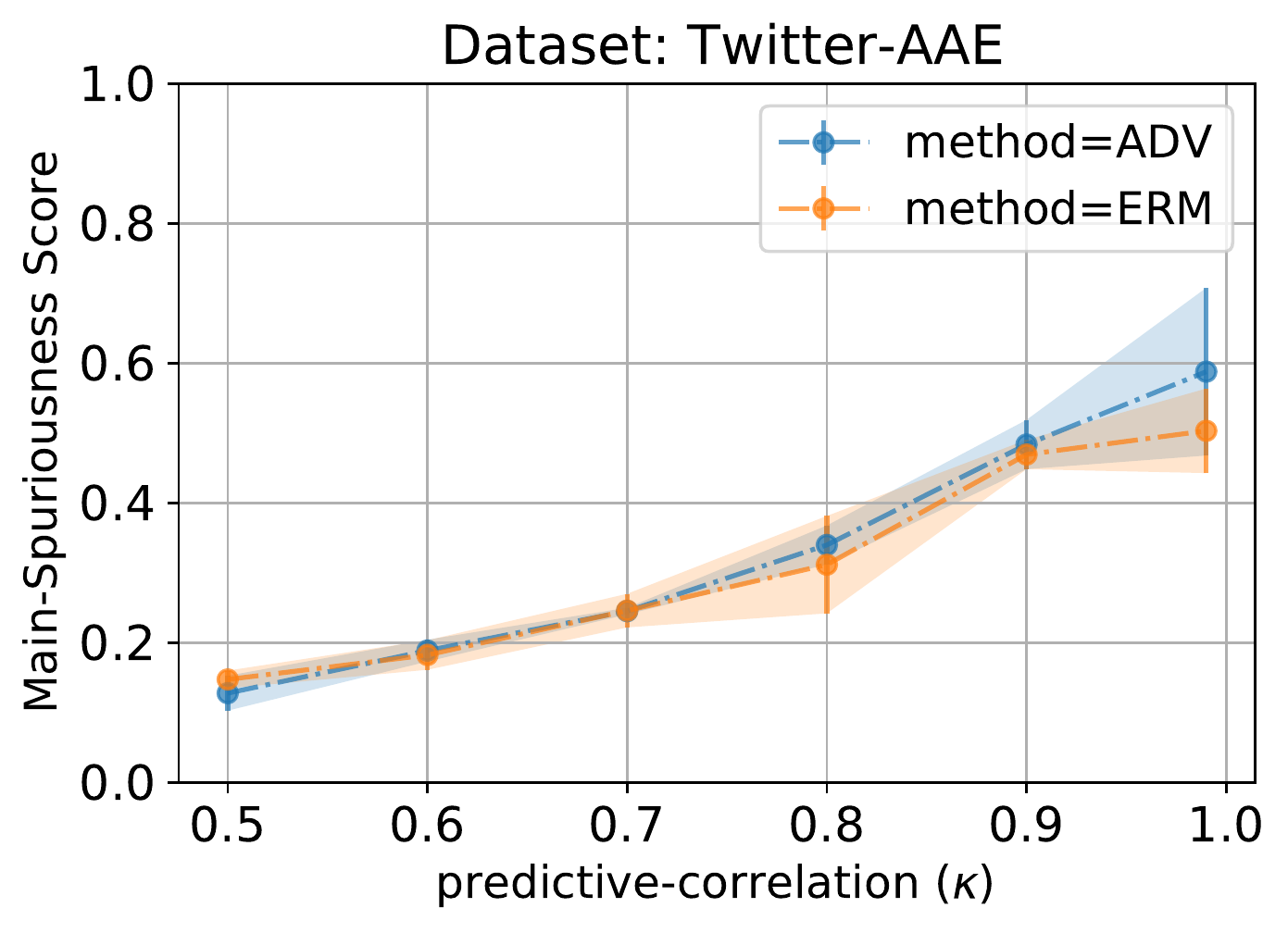}
    	\caption[]%
        {{\small Main Spuriousness Score}}    
    	\label{fig:adv_aae_sp_score} 
    \end{subfigure}
\hfill
\begin{subfigure}[h]{.28\textwidth}
    \centering
    \includegraphics[width=\linewidth,height=0.7\linewidth]{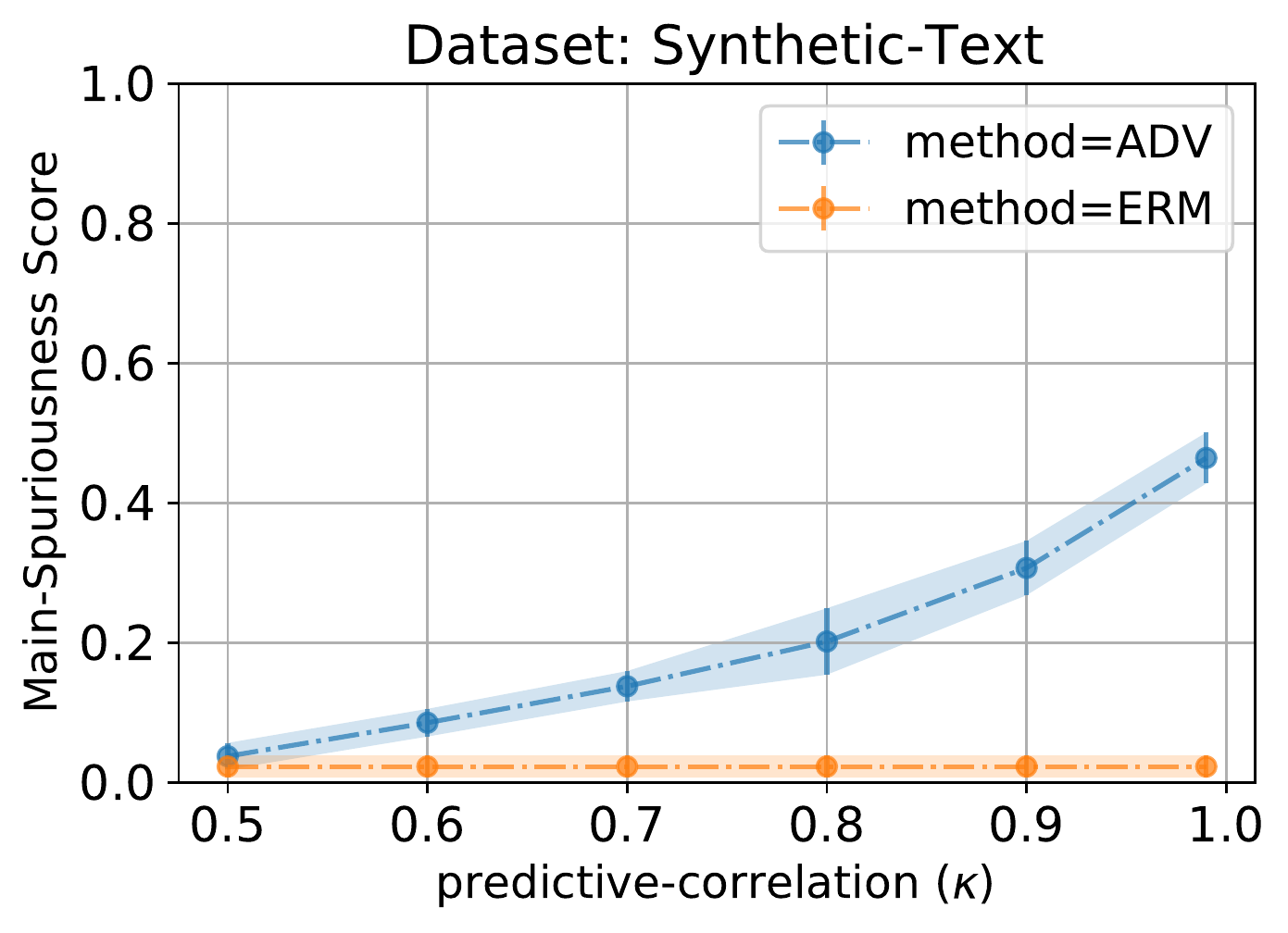}
    	\caption[]%
        {{\small 
        Main Spuriousness Score}}    
    	\label{fig:adv_syn_sp_score} 
    \end{subfigure}
\hfill

\caption{\textbf{Adversarial Removal Failure.} Top row explains  failure of the \ar method on \mnli. Bottom row shows the failure on \pan, \aae, and Synthetic-Text datasets. In each figure,  the x-axis shows different levels of predictive-correlation ($\kappa$) between the main task and \prop labels in the dataset used by \ar and the y-axis shows different evaluation metrics. Orange lines denote the ERM model and the blue lines denote the model trained using \ar. \textbf{(b), (d), (e), (f)} show that \ar is unable to completely remove the spurious concept from the  main task classifier. %For discussion see \refsec{subsec:expt_adv_failure}. 
% \abhinav{check the \syn figure once} %Variation of degree of spuriousness of main-classifier with predictive correlation ($\kappa$)  on 3 different dataset. In the first row, for every dataset, we vary the predictive correlation ($\kappa$) between the main-task label and the unwanted label and measure spuriousness score in the main-classifier. The red \chenhao{you mean orange?} curve show normal training of main classifier and the blue curve is for the classifier with adversarial training. We observe that adversarial training is not able to completely remove the spurious correlation and twitter-pan16 and twitter-aae they are almost same. The second row, shows the degree of spuriousness for the probing classifier. If the spurious feature would have been completely removed the spuriousness score would have been close to 1 for probing classifier.
% \amit{need to be a, b, c etc incorrect labels}
\textbf{(c)} shows a stronger failure where the \ar method introduces spuriousness into a clean input classifier.}
\label{fig:expt_adv_rem_failure}
\end{figure*}
\vspace{-0.5em}

\textbf{Cannot remove the spurious \prop fully.} 
For \mnli, \reffig{fig:adv_mnli_sp_score} shows the spuriousness score (\refdef{def:spurriousness_score}) of ERM and \ar classifiers  as we vary the predictive correlation ($\kappa$) between the main-task label and \sensitive concept label in the training dataset. %For all our experiment both probing and main-task dataset are kept same. 
While the spuriousness score for classifier trained using \ar (blue curve) is lower than that of ERM for all values of $\kappa$, it is substantially away from zero. Thus, the \ar method fails to  completely remove the spurious \prop completely from the latent representation. By inspecting the   \prop probing classifier accuracy for ERM and \ar in \reffig{fig:adv_mnli_probe_acc}, we obtain a possible explanation. The probe accuracy after adversarial training  doesn't decrease to 50\% but stops at accuracy proportional to the predictive correlation $\kappa$. This is expected since even if the \ar would have been successful in removing the concept-causal features, the main-task features would still be predictive of the \prop label by $\kappa$ due to the spurious correlation between them. 
However, the converse is not true: an accuracy of $\kappa$ does not imply that the concept is fully removed. The results substantiate the first statement of \reftheorem{theorem:adv_removal_failure}: given two representations where one (\emph{desired}) does not have \prop features while the other (\emph{undesired}) contains the \prop features, the undesired one may be better for the main task accuracy even as both may have the same probing accuracy. %r \emph{undesired/incorrect} one which doesn't could have probing accuracy but the the undesired one could be better in terms of main-task classifier objective.
% \amit{is the statement correct?}
\reffig{fig:adv_pan_sp_score}, \ref{fig:adv_aae_sp_score} and \ref{fig:adv_syn_sp_score} show the spuriousness score of \ar in comparison to classifier trained with ERM on \pan, \aae and \syn datasets respectively. The failure of \ar is worse here: there is no significant reduction in spuriousness score for \ar in comparison to ERM. For the \syn dataset, ERM has zero spuriousness score but \ar  has non-zero score. We expand more on this observation and include  additional  results on adversarial removal in %For more details on verifying this observation with other metrics and other ablations, see  
\refappendix{subsec:app_extended_adversarial_removal_results}. 

% \new{
\textbf{\ar makes a clean classifier use the spurious concept. }  In \reffig{fig:adv_mnli_fairness_sp_score} we provide a clean main task classifier (see \refsec{subsec:app_extended_adversarial_removal_results} for training details) as input to \ar method. For all values of $\kappa$, the input classifier's spuriousness score is low (iteration 1-6). From iteration 7 onwards, the \ar method corrupts the clean classifier as shown by increasing spuriousness scores. % Next in iteration 7-12, the \ar method is initialized with clean classifier from previous step and given multiple dataset with different  predictive-correlation ($\kappa$). We observe that for high correlation ($\kappa>0.8$) . 
 For more results,  see \reffig{fig:app_adv_mnli_unfair} in \refsec{subsec:app_extended_adversarial_removal_results}.
% }

% \amit{what about fairness expt? write a line that it is in suppl.}

\para{Comparison with previous work.}
%Our experimental results extend the observations from previous work on adversarial removal failure~\cite{AdvRemYoav}. 
% past work 
If post-removal the latent representation used by the main-task classifier is still predictive of the removed concept, \cite{AdvRemYoav} claimed it as a failure of the adversarial removal method. However, this claim may not be correct since a feature could be present in the latent space and yet not used by model \cite{ravichander-etal-2021-probing}. % a feature's presence in the latent representation does not necessarily mean that it is being used by the main classifier.  
Our proposed spuriousness score metric avoids this limitation. %That said, we do confirm the observation from \cite{AdvRemYoav} that adversarial removal can fail even if there is no correlation between concept and task label ($\kappa=0.5$), and extend their work by showing failure modes on multiple values of $\kappa$. %In \reffig{fig:adv_mnli_sp_score} we observe the spuriousness score at $\kappa=0.5$ is non-zero, validating the observation of \cite{AdvRemYoav} that even when the adversarial classifier has access to balance data ($\kappa=0.5$) where there is no spurious correlation between \sensitive concept label and main-task label, \ar methods will still fail. 

\textbf{Ablations. }In Appendix,  we report results on using BERT instead of RoBERTa as the input encoder (\refappendix{subsec:app_extended_null_space_results}, \ref{subsec:app_extended_adversarial_removal_results}), the effect of using different modeling choices like loss-function, regularization, e.t.c. %on \syn dataset
(\refappendix{subsec:app_extended_synthetic_results}), and the behavior of probing classifiers when \prop is not present in latent space (\refappendix{subsec:extended_probing_result}).

\vspace{-0.6em}
\section{Related work}
\vspace{-0.7em}
% \abhinav{Add the linear-concept erasure citation}

\textbf{\Prop removal methods.} When the removal of a \prop can be simulated in input space (e.g., in tabular data or simpler concepts), removing a concept directly using  data augmentation~\cite{kaushik2021learningthediff} or gradient regularization~\cite{rightfortherightreasonfinaledoshi,kancheti2022matching} can work. However, concept removal is non-trivial when change in a \prop cannot be propagated via change in  input tokens.   Combining the ideas of  null space and adversarial removal \cite{NullItOut:2020,ruslanINLP,AdvRemNeubig,AdvRemYoav}, methods like \cite{LinearConceptErasure,KernelConceptErasure} restrict the adversarial function to be a projection operation and derive a closed form solution. 
 %But all of them could be clubbed under same umbrella of adversarial removal.
%Given the, our work highlights the failure modes of popular removal methods. 
Other approaches use explanations of the classifier's prediction for concept removal~\cite{han-tsvetkov-2021-influence-tuning}. Our work highlights the difficulty of building an estimator for the features causally derived from a concept, as a general limitation for concept removal.

\textbf{Limitations of a probing classifier for model interpretability.}
%A possible reason is the use of a probing classifier to ascertain whether the representation is using the property. Typically, if the probing classifier accuracy is non-random guess then it is inferred that property is present in the latent space. 
 We also contribute to the growing literature on the limitations of a probing classifier's accuracy in capturing whether the main classifier is using a \prop~\cite{ProbingSurvey}.  %   but why ? a possible reason can be the predictive classifier prober.
It is known that probing classifiers capture not just the concept but any other features that may be correlated with it~\cite{GroupDRO,IRM:2020,kaushik2021learningthediff,victorStressTest}. As a result, many improvements have been proposed to better estimate whether a concept is being used, including the use of control labels or datasets~\cite{hewitt-liang-2019-control,ravichander-etal-2021-probing}. Parallelly, new causality-inspired probing methods~\cite{AmnesicProbing}  compare the main task accuracy on a representation without the concept constructed using the null space removal method. The hope is that such improvements can make probing more robust. 
Our results question this direction. To demonstrate the fundamental unreliability of probing classifier, we construct a setup that is most favorable for learning only the concept's features and still find that %: we assume that using the property's features yields 100\% accuracy and that property's features and other features are linearly separable in the representation space. In such a setup, one should expect a linear probing classifier to only use the property's features. Yet, under some simple conditions, we show that the
learned probing classifier includes non-zero weight for other features, limiting effectiveness of any interpretation method based on it. % . As we showed in \refsec{subsec:null_space_problem} and \ref{subsec:expt_null_space_failure},  causally-inspired probing methods~\cite{AmnesicProbing} have the same limitation because their removal method depends on the output of an auxiliary classifier that is equivalent to a probing classifier. %they need to depend on a probing classifier to estimate use of the property.

\section{Conclusion}
% \section{Implications for probing-based removal methods and future directions}

% \paragraph{Implications for probing-based removal methods and future directions:}
% \label{sec:future_direction}
% In our work, we 
% % evaluated the effectiveness 
% % discuss the unreliability of 
% show that probing classifiers are unreliable for concept removal.
\vspace{-0.5em}
Our theoretical and empirical evaluations show that it is difficult to create a probing-based \new{explainability and } removal method due to the fundamental limitation of learning a ``clean'' probing classifier. 
% using different metrics. 
% Before using a probing-based removal method, we recommend that practitioners validate it using the metrics proposed. Specifically, 
We  recommend two tests for validating removal methods. 
First, we provide a sanity-check: any reasonable removal method should not modify a ``clean'' classifier that does not use any spurious features to produce a final classifier that uses those features. 
% To validate this, we adapt a method from \cite{ravichander-etal-2021-probing} for generating a ``clean'' classifier by restricting the training set only for the initial training of the main task classifier; and provide two metrics: Spuriousness Score and $\Delta-$Probability to measure the spurious features captured by the model.  
Second, we propose a spuriousness score that can be used to evaluate the dependence of any classifier on spurious features. 
% This metric is inspired from the literature on spurious correlations and minority-group generalization (\cite{GroupDRO}), and it becomes useful to detect and debug any spuriousness in the probing classifier.
As a future step, we encourage the community to develop more such 
% benchmarks and 
sanity-check tests to evaluate proposed methods. 

% Overall, 
% While we acknowledge that full removal of a spurious concept may not be possible, we specify a simpler selection criteria on which probing-based removal methods failed: a spurious concept removal method should not remove the task-relevant features which are correlated with the spurious concept (and not caused by it). Given the sensitive nature of many feature-removal tasks (e.g., fairness on demographic groups) and the risk of counterproductive harms, we therefore question whether probing-based methods are the right way forward. 
Alternatively, we point attention to other approaches that may provide better guarantees for concept removal. An example is extending data augmentation techniques like counterfactual data augmentation (\cite{kaushik2021learningthediff,dash2022evaluating}) to non-trivial concepts. For a given training point that may include a spurious correlation, a new data point is generated that breaks the correlation but keeps the semantics of the rest of the text identical (hence the name, ``counterfactual''). This can be done by human labeling or handcrafted rules for modifying text (e.g., Checklists \cite{checklist}). Then the main classifier is regularized to have the same representation for such pairs of inputs (\cite{counterfactual_reg,mahajan2021domain}). By construction, with good quality pairs, such a method will not remove task-relevant features and will satisfy the sanity checks listed above. 
That said, a limitation is that the removal quality will depend on the diversity of the counterfactual examples generated and whether they capture all aspects of the spurious concept. 
Another direction is to take inspiration from the algorithmic fairness literature \cite{fairness_lit_future_hardt,fairness_survey_future_mehrabi} and focus on the predictions of the classifier rather than the representation. Compared to removal in latent space, enforcing certain fairness properties on model predictions is a more well-formed task,  more interpretable, and definitely more relevant if the final goal is fair decision making.
%Our theoretical and experimental evaluations show that it is difficult to create a probing-based removal method due to the fundamental limitation of learning a “clean” probing classifier. Thus, we would like the community to look for other methods which don’t use latent space concept removal using the probing classifier. An example, using data-augmentation techniques like counterfactual data augmentation (\cite{kaushik2021learningthediff}) for making a classifier less reliant on spurious features and at the same time not harming the main-task accuracy seems a promising direction. Also, taking inspiration from fairness literature, if the final goal is fair decision making then instead of trying to solve the difficult task of removing the sensitive feature one could make the classifier fair with respect to the sensitive attribute (\cite{fairness_removal}).

\textbf{Limitations. }
\label{sec:limitation_of_our_work}
A limitation of our theoretical work is assuming frozen or non-trainable latent representation which makes the analysis of task-classifier trained on top of them relatively easier. We address this limitation in our empirical work where we do not make such assumptions. Also, our work addresses failure modes of two popular methods, null space removal and adversarial removal.
We conjecture that any other method based on probing classifiers will lead to similar failure modes.
\bibliographystyle{plain} %{unsrtnat}
\bibliography{main_neurips_2022} 

% \input{checklist}

%Adding the appendix sections
\clearpage
\appendix
\section{Broader Impact and Ethical Consideration}
\label{sec:app_broader_impact}
% \textbf{Broader impact and ethical considerations.} 
Removal of spurious or sensitive concepts is an important problem to ensure that machine learning classifiers generalize better to new data and are fair towards all groups. %However, when concepts cannot be changed at the input level, the problem is a complex one since we cannot even access the features causally derived from the concept, which need to be removed. 
We found multiple limitations with current removal methods and recommend caution against the use of these methods in practice.
\section{Probing and Main Classifier Failure Proofs}
\label{sec:app_probing}

\subsection{Notation and Setup: Max-margin Classifier}
\label{subsec:app_max_margin_setup}
We assume that encoder $h:\bm{X}\rightarrow \bm{Z}$, mapping the input to latent representation is frozen/non-trainable. Thus for every input $\bm{x}^{i}$ in the dataset $\mathcal{D}$, we have a corresponding latent representation $\bm{z}^{i}$ which is fixed. Also, the latent representation $\bm{Z}$ is disentangled i.e  $\bm{z}=[\bm{z}_{m},\bm{z}_{p}]$ where $\bm{z}_{m}$ are the main task features,  i.e.,  causally derived from the main task label and $\bm{z}_{p}$ are the concept-causal features, causally derived from the \prop label. Let $c_{p}(\bm{z})=\bm{w}_{p}\cdot \bm{z}_{p} + \bm{w}_{m}\cdot \bm{z}_{m}$ be the linear probing classifier which we train using max-margin objective. The hyperplane $c_{p}(\bm{z})=0$ is the decision boundary of this linear classifier. The points which fall on one side of the decision boundary ($c_{p}(\bm{z})>0$) are assigned one label (say positive label 1) and the rest are assigned another label (say negative label -1). The \emph{margin} $\mathcal{M}_{c_{p}}$ \abhinav{remove this symbol if it is not needed later} of this probing classifier ($c_{p}(\bm{z}$)) is the distance of the closest latent representation ($\bm{z}$) from the decision boundary. The points which are closest to the decision boundary are called the \emph{margin points}. The distance of a given latent representation $\bm{z}^i$ having class label $y^i$, where $y^i\in\{-1,1\}$, from the decision boundary is given by 
\begin{equation}
\label{eq:margin_distance}
    \mathcal{M}_{c_{p}}(\bm{z}^i) := \frac{m_{c_{p}}(\bm{z}^i)}{\norm{\bm{w}}} = \frac{y^{i}_p\cdot c_{p}(\bm{z}^i)}{\norm{\bm{w}}} = \frac{y^{i}_p\cdot (\bm{w}\cdot \bm{z}^{i} + b)}{\norm{\bm{w}}} 
\end{equation}
where $\norm{\bm{w}}$ is the L2 norm of parameters $\bm{w}=[\bm{w}_{p},\bm{w}_{m}]$ of the probing classifier $c_{p}(\bm{z})$.

\paragraph{Max-Margin (MM):}
Then the max-margin classifier is trained by optimizing the  following objective:
\begin{equation}
\label{eq:max_margin}
    \argmax_{\bm{w},b} \Big{\{} \min_{i} \mathcal{M}_{c_{p}}(\bm{z}^i)  \Big{\}}
\end{equation}
For ease of exposition we convert this objective into multiple equivalent forms. To do this we observe that scaling the parameters of $c_{p}(\bm{z})$ by a positive scalar $\gamma$ i.e $\bm{w} \rightarrow \gamma \bm{w}$ and $b \rightarrow \gamma b$ does not change the distance of the point ($\mathcal{M}_{c_{p}}(\bm{z}^i)$) from the decision boundary.

\paragraph{MM-Denominator Version:}
We can use this freedom of scaling the parameters to set $m_{c_{p}}(\bm{z}^i)=1$ for the closest point of any given probing classifier $c_{p}(\bm{z})$ , thus all the data points will satisfy the constraint,
\begin{equation}
\label{eq:max_margin_denominator_constraint}
    m_{c_{p}}(\bm{z}^i) = y^i\cdot c_{p}(\bm{z}^i) \geq 1
\end{equation}
 giving us the final max-margin objective:
\begin{equation}
\label{eq:max_margin_denominator}
    \argmax_{\bm{w}} \Big{\{}  \frac{1}{\norm{\bm{w}}}  \Big{\}}
\end{equation}
under the constraint $m_{c_{p}}(\bm{z}^i)\geq 1$ corresponding to all the points in the dataset.

\paragraph{MM-Numerator Version:}
Alternatively, one can choose $\gamma$ such that $\norm{\bm{w}} = c$ where $c\in \mathbb{R}$ is some constant value.  The the modified objective becomes:
\begin{equation}
\label{eq:max_margin_numerator}
    \argmax_{\bm{w},b} \Big{\{}  m_{c_{p}}(\bm{z}^i)  \Big{\}}
\end{equation}
under constraint $\norm{\bm{w}} = c$ which is usually set to $1$.

We will use one of these  formulations in our proofs  based on the ease of exposition and give a clear indication when we do so. One can refer to Chapter $7$, Section $7.1$ of \cite{PatternRecognitionBook:2006} for further details about max-margin classifiers and different formulations of the  max-margin objective.

\subsection{Problem with learning a \emph{clean} main-task classifier}
\label{subsec:app_generalized_probing_result}

In this section, we will restate the assumptions and results of \reflemma{lemma:sufficient_condition} for the main-task classifier (instead of the probing classifier) and show  that the same results will hold.

\refassm{assm:disentagled_latent} remains the same since it is made on the latent-representation being disentangled and frozen/non-trainable. Next, parallel to \refassm{assm:fully_pred_inv}, we show that even when main-task feature is 100\% predictive of main-task and is linearly separable, the trained main-task classifier will also use the \prop-causal features. Formally,

\begin{assumption}[main-task feature Linear Separability]
\label{assm:fully_pred_inv_main_task}
The main-task features ($\bm{z}_{m}$) of the latent representation ($\bm{z}$) for every point are linearly separable/fully predictive for the main-task  labels $y_{m}$, i.e $y_{m}^{i}\cdot (\hat{\bm{\epsilon}}_{m}\cdot \bm{z}_{m}^{i} + b_{m}) > 0$ for all datapoints $(\bm{x}^{i},y_{m}^{i})$ for some \new{ $\hat{\bm{\epsilon}}_{m}\in \mathbb{R}^{d_{m}}$ and $b_m\in \mathbb{R}$}.  For the case of zero-centered latent space, we have $b_{m}=0$.
\end{assumption}

Next similar to \refassm{assm:spurious_linear_separably}, we define the spurious correlation between main-task and \prop label: a function using only $\bm{z}_{p}$ may also be able to classify correctly on some non-empty subset of points w.r.t. main-task label ($y_{m}$).

\begin{assumption}[Main-Task Spurious Correlation] 
\label{assm:spurious_linear_separably_main_task}
For a subset of training points $\mathcal{S} \subset \mathcal{D}_m$, main-task label $y_{m}$ is linearly separable using $\bm{z}_p$  i.e $y_m^{i}\cdot(\hat{\bm{\epsilon}}_{p}\cdot \bm{z}^i_{p} + b_{p}) > 0$ for some $\hat{\bm{\epsilon}}_{p}\in \mathbb{R}^{d_{p}}$ and $b_{p}\in \mathbb{R}$. For the case of zero-centered latent space we have $b_{p}=0$.
\end{assumption}
% \chenhao{I do not understand what it means to have a feature is linearly separable. It should be two classes are separable, right?}\amit{a better way to word it---main task label is linearly separable using zp}

Next we rephrase \reflemma{lemma:sufficient_condition} which shows that for only a \emph{few special} points if the \prop-causal features $\bm{z}_{p}$ are linearly-separable w.r.t. to main task classifier $y_{m}$ (\refassm{assm:spurious_linear_separably_main_task}), then the main-task classifier $c_{m}(\bm{z})$ will use those features.

\begin{lemma}[Sufficient Condition for Main-task Classifier]
\label{lemma:sufficient_condition_main_task}
Let the latent representation be frozen and disentangled such that $\bm{z}=[\bm{z}_{m},\bm{z}_{p}]$ (\refassm{assm:disentagled_latent}), where main-task-features $\bm{z}_{m}$ be fully predictive  (\refassm{assm:fully_pred_inv_main_task}). Let $c_{m}^{*}(\bm{z})=\bm{w}_{m}\cdot \bm{z}_{m}$ be the desired/clean linear main-task classifier trained using max-margin objective (\refappendix{subsec:app_max_margin_setup}) which only uses $\bm{z}_m$ for its prediction. Let $\bm{z}_{p}$ be the spurious feature s.t. for the margin points of $c_{m}^{*}(\bm{z})$, $\bm{z}_{p}$ be linearly-separable w.r.t. task label $y_{m}$ (\refassm{assm:spurious_linear_separably_main_task}).  Then, assuming the latent space is centered around $\bm{0}$ \new{(i.e. $b_m=0$ and $b_p=0$)}, the main-task classifier trained using max-margin objective will be of form $c_m(\bm{z})=\bm{w}_{m}\cdot \bm{z}_{m} + \bm{w}_{p}\cdot \bm{z}_{p}$ where $\bm{w}_{p}\neq \bm{0}$.    
\end{lemma}

The proof of \reflemma{lemma:sufficient_condition_main_task} is identical to \reflemma{lemma:sufficient_condition} and is  provided in \refappendix{subsec:app_proof_sufficient_condition}.  

% \amit{cleaner way to say it: The proof of Lemma C.5 is identical to that for Lemma 3.1  and is provided in Appendix C.2}

\subsection{Proof of Sufficient Condition:  \reflemma{lemma:sufficient_condition} and \reflemma{lemma:sufficient_condition_main_task}}
\label{subsec:app_proof_sufficient_condition}

\probthm*

In this section we prove that, given the assumption in \reflemma{lemma:sufficient_condition} is satisfied, they are sufficient for a probing classifier $c_{p}(\bm{z})$ to use the spuriously correlated main-task feature $\bm{z}_{m}$. See \refsec{subsec:app_max_margin_setup} for detailed setup and max-margin training objective. Also, we could use the same line of reasoning to prove a similar result for the main-task classifier i.e. when conditions in \reflemma{lemma:sufficient_condition_main_task} are satisfied, the main-task classifier will use the spuriously correlated \prop-causal feature $\bm{z}_{p}$. To keep the proof general for both the lemmas, we prove the result for a general classifier $c(\bm{z})$ trained to predict a task label $y$. Here the latent representation $\bm{z}$ be of form $\bm{z}=[\bm{z}_{inv},\bm{z}_{sp}]$ where $\bm{z}_{inv}$ are the features which are causally-derived from the \new{task concept} (``invariant'' features) and $\bm{z}_{sp}$ be the features spuriously correlated to the task label $y$. With respect to probing classifier $c_{p}(\bm{z})$ in  \reflemma{lemma:sufficient_condition} $\bm{z}_{inv}\coloneqq \bm{z}_{p}$ and $\bm{z}_{sp}\coloneqq \bm{z}_{m}$. Similarly, for the main-task classifier in \reflemma{lemma:sufficient_condition_main_task}, $\bm{z}_{inv}\coloneqq \bm{z}_{m}$ and $\bm{z}_{sp}\coloneqq \bm{z}_{p}$. For ease of exposition, we define two categories of classifiers based on which features they use:

\begin{definition}[Purely-Invariant Classifier] 
\label{def:purely_invariant}
A linear classifier of form $c(\bm{z}) = \bm{w}_{inv}\cdot \bm{z}_{inv} + \bm{w}_{sp} \cdot \bm{z}_{sp} + b$ is called \emph{"purely-invariant"} if it does not use the spurious features $\bm{z}_{sp}$ i.e.,  $\bm{w}_{sp}=\bm{0}$.
\end{definition}

\begin{definition}[Spurious-Using Classifier]
\label{def:spurious_using}
A linear classifier of form $c(\bm{z}) = \bm{w}_{inv}\cdot \bm{z}_{inv} + \bm{w}_{sp} \cdot \bm{z}_{sp} + b$ is called \emph{"spurious-using"} if it  uses the spurious features $\bm{z}_{sp}$ i.e. $\bm{w}_{sp}\neq \bm{0}$. 
% and $\bm{w}_{inv}\neq \bm{0}$.
\end{definition}

\begin{proof}[Proof of \reflemma{lemma:sufficient_condition} and \ref{lemma:sufficient_condition_main_task}]
Let $c_{inv}(\bm{z})=\bm{w}_{inv}\cdot \bm{z}_{inv}$  be the \emph{clean}/purely-invariant classifier trained using the max-margin objective using the \emph{MM-Denominator} formulation given in \refeqn{eq:max_margin_denominator} such that $\bm{w}_{inv}\neq \bm{0}$. The classifier $c_{inv}(\bm{z})$ is 100\% predictive of the task labels $y$ (from \refassm{assm:fully_pred_inv} for the probing task  or \refassm{assm:fully_pred_inv_main_task} for the main-task). Here the bias term $b=0$ since we assume the latent representation $\bm{z}$ is zero-centered. The norm of this classifier ($c_{inv}(\bm{z})$) is $\norm{\bm{w}_{inv}}$ and the distance of each input latent representation ($\bm{z}^i$) with class label $y^i$ ($y^i\in\{-1,1\}$)  from the decision boundary ($c_{inv}(\bm{z})=0$) is given by \refeqn{eq:margin_distance} i.e.:
\begin{equation}
\label{eq:margin_distance_inv}
    \mathcal{M}_{inv}(\bm{z}^i) = \frac{m_{inv}(\bm{z}^i)}{\norm{\bm{w}_{inv}}} = \frac{y^i\cdot c_{inv}(\bm{z}^i)}{\norm{\bm{w}_{inv}}} = \frac{y^{i}\cdot (\bm{w}_{inv}\cdot \bm{z}^{i}_{inv})}{\norm{\bm{w}_{inv}}}
\end{equation}
Since we have used the \emph{MM-Denominator} version of max-margin to train $c_{inv}(\bm{z})$, from \refeqn{eq:max_margin_denominator_constraint} we have $m_{inv}(\bm{z}^i)=1$ for the \emph{margin-points} and greater than $1$ for rest of the points.  Next we will construct a new  classifier parameterized by $\alpha\in[0,1]$ by perturbing the clean/purely-invariant classifier $c_{inv}(\bm{z})$ such that:
\begin{equation}
\label{eq:perturbed_spurious_classifier}
    c_{\alpha}(\bm{z}) = \alpha \big(\bm{w}_{inv}\cdot \bm{z}_{inv} \big) + \norm{\bm{w}_{inv}}\sqrt{1-\alpha^2} \big(\bm{\hat{\epsilon}}_{sp}\cdot \bm{z}_{sp}\big)
\end{equation}
where $\bm{\hat{\epsilon}}_{sp} \in \mathbb{R}^{d_{sp}}$ is a unit vector in spurious subspace of features, $d_{sp}$ is the dimension of the spurious feature subspace ($\bm{z}_{sp}$). We observe that the norm of this perturbed classifier $c_{\alpha}(\bm{z})$ is also $\norm{\bm{w}_{inv}}$, which is same as the clean/purely-invariant classifier $c_{inv}(\bm{z})$. Thus from \refeqn{eq:margin_distance}, the distance of any input $\bm{z}^i$ with class label $y^i$ from the decision boundary of this perturbed classifier $c_{\alpha}(\bm{z})$ is given by:
\begin{equation}
\label{eq:margin_distance_p}
    \mathcal{M}_{\alpha}(\bm{z}^i) = \frac{m_{\alpha}(\bm{z}_i)}{\norm{\bm{w}_{inv}}} = \frac{y^i\cdot c_{\alpha}(\bm{z}^i)}{\norm{\bm{w}_{inv}}}
\end{equation}

The perturbed classifier will be \emph{spurious-using} i.e use the spurious feature $\bm{z}_{sp}$ when $\alpha\in[0,1)$ since $\bm{w}_{sp}=(\norm{\bm{w}_{inv}}\sqrt{1-\alpha^2})\neq 0$ for these setting of $\alpha$. Thus to show that there exist a \emph{spurious-using} classifier which has a margin greater than the margin of the \emph{purely-invariant} classifier, we need to prove that there exist an $\alpha \in [0,1)$ such that $c_{\alpha}(\bm{z})$ has bigger margin than $c_{inv}(\bm{z})$ i.e.
$\min_{\bm{z}} \mathcal{M}_{\alpha}(\bm{z}) > \min_{\bm{z}} \mathcal{M}_{inv}(\bm{z})$. Since norm of parameters of both the classifier is same, substituting the expression of $\mathcal{M}_{\alpha}$ and $\mathcal{M}_{inv}$ from \refeqn{eq:margin_distance_inv} and \ref{eq:margin_distance_p}, we need to show $m_{\alpha}(\bm{z}^i)> 1$ for all $\bm{z}^i$. We have:
\begin{align}
    m_{\alpha}(\bm{z}^i) &= y^i \cdot \bigg(\alpha \big(\bm{w}_{inv}\cdot \bm{z}^{i}_{inv} \big) + \norm{\bm{w}_{inv}}\sqrt{1-\alpha^2} \big(\bm{\hat{\epsilon}}_{sp}\cdot \bm{z}^{i}_{sp}\big)\bigg)\\\label{eq:perturbed_margin}
    &= \alpha \cdot m_{inv}(\bm{z}^{i}) + y^i\norm{\bm{w}_{inv}}\sqrt{1-\alpha^2} \big(\bm{\hat{\epsilon}}_{sp}\cdot \bm{z}_{sp}^{i}\big)
\end{align}

Let $\mathcal{S}_{y}^{m}$ denote the set of \emph{margin-points} of purely-invariant classifier $c_{inv}(\bm{z})$ with class label $y$ having $m_{inv}(\bm{z}) = 1$ and $\mathcal{S}_{y}^{r}$ contain rest of points (non-margin points) having $m_{inv}(\bm{z})>1$ with the class label $y$.  Here ``m'' stands for margin-point in superscript of $\mathcal{S}$ and ``r'' stands for rest of point with label $y$. In rest of the proof, first we will show that for margin-points $\bm{z}^{m}\in (\mathcal{S}_{y=1}^{m} \cup \mathcal{S}_{y=-1}^{m})$, we need the assumption that spurious feature ($\bm{z}_{sp}^{m}$) be linearly separable with respect to class label $y$ (\refassm{assm:spurious_linear_separably} for probing task or \ref{assm:spurious_linear_separably_main_task} for main-task) for having $m_{\alpha}(\bm{z}^i)> 1$. But for all non-margin points $\bm{z}^{r}\in(\mathcal{S}_{y=1}^{r}\cup \mathcal{S}_
{y=-1}^{r})$, we can always choose $\alpha \in [0,1)$ such that  $m_{\alpha}(\bm{z}^i)> 1$. Below we handle margin and non-margin of points separately.

% In the proof we observe that for non-margin points of purely-invariant classifier  we can always get an $\alpha\in (0,1)$ such that $m_{\alpha}(\bm{z}^{r})>1$. But for margin points, we need linear separability of the spurious feature with respect to task label $y$ (\refassm{assm:spurious_linear_separably} or \ref{assm:spurious_linear_separably_main_task} which ensure that we have $m_{\alpha}(\bm{z}^{m})>1$. 

\paragraph{Case 1 : Margin Points $(\mathcal{S}_{y=1}^{m} \cup \mathcal{S}_{y=-1}^{m})$:}
For the margin-points in latent space, $\bm{z}^{m} \in \mathcal{S}_{y}^{m}$ we have $m_{inv}(\bm{z}^{m})=1$ and we need to show that there exists  $\alpha\in[0,1)$ such that $m_{\alpha}(\bm{z}^m)>1$ for all $\bm{z}^{m}\in \mathcal{S}^{m}_{y}$.  From \refeqn{eq:perturbed_margin} we have:
\begin{align}
    m_{\alpha}(\bm{z}^m) = \alpha\cdot 1 + y\norm{\bm{w}_{inv}}\sqrt{1-\alpha^2} \big(\bm{\hat{\epsilon}}_{sp}\cdot \bm{z}_{sp}^{m}\big) &> 1\\
    \label{eq:margin_main_ineq}
    \big(\norm{\bm{w}_{inv}}\sqrt{1-\alpha^2}\big) y\big(\bm{\hat{\epsilon}}_{sp}\cdot \bm{z}_{sp}^{m}\big)&>1-\alpha
\end{align}
From \refassm{assm:spurious_linear_separably} for probing task or \ref{assm:spurious_linear_separably_main_task} for the main-task, we know that spurious-feature $\bm{z}_{sp}^{m}$ of margin-points $\bm{z}^{m}$ are linearly-separable w.r.t to task label $y$. 
% Thus there exist an unit vector
\new{Since $\bm{\hat{\epsilon}}_{sp}\in\mathbb{R}^{d_{sp}}$ used in the perturbed classifier $c_{\alpha}(\bm{z})$ is arbitrary, let's set it to be an unit vector such that $y\big(\bm{\hat{\epsilon}}_{sp}\cdot \bm{z}_{sp}^{m}\big) > 0$ for all $\bm{z}^m \in \mathcal{S}_y^m $ (guaranteed by \refassm{assm:spurious_linear_separably} or \ref{assm:spurious_linear_separably_main_task})}. Also since $\alpha\in[0,1)$ and $\norm{\bm{w}_{inv}}> 0$, we have $(\norm{\bm{w}_{inv}}\sqrt{1-\alpha^2})> 0$. Hence the left hand side of \refeqn{eq:margin_main_ineq} is $> 0$ for such $\hat{\bm{\epsilon}}_{sp}$. If \refassm{assm:spurious_linear_separably} or \ref{assm:spurious_linear_separably_main_task} (corresponding to the task)  wouldn't have been satisfied then the above equation might have been inconsistent since right hand side is always $>0$; since we need to find a solution to \refeqn{eq:margin_main_ineq} when $\alpha\in[0,1)$ thus $(1-\alpha)> 0$,   but left hand side wouldn't have been  always greater than 0. This shows the motivation why we need \refassm{assm:spurious_linear_separably} or \ref{assm:spurious_linear_separably_main_task} for proving this lemma.  Continuing, let $\beta := (y\big(\bm{\hat{\epsilon}}_{sp}\cdot \bm{z}^{m}_{sp}\big))$, then squaring both sides and cancelling $(1-\alpha)$ since we need to find a solution to \refeqn{eq:margin_main_ineq} when $\alpha\in[0,1) \implies (1-\alpha)>0$,  we get:
\begin{align}
\label{eq:margin_pt_inequality}
    \norm{\bm{w}_{inv}}^2\cancel{(1-\alpha)}(1+\alpha)\bigg(y\big(\bm{\hat{\epsilon}_{sp}}\cdot \bm{z}_{sp}^{m}\big)\bigg)^2 &> \cancel{(1-\alpha)}(1-\alpha)\\
    \norm{\bm{w}_{inv}}^2(1+\alpha) \beta^2 &> (1-\alpha)\\
    \norm{\bm{w}_{inv}}^2\beta^2+\alpha\norm{\bm{w}_{inv}}^2\beta^2 &> 1-\alpha\\
    \alpha\bigg(1+\norm{\bm{w}_{inv}}^2\beta^2\bigg) &> \bigg(1-\norm{\bm{w}_{inv}}^2\beta^2\bigg)
\end{align}
%Since we need to find an $\alpha\in[0,1)$ which satisfy the above inequality (\refeqn{eq:margin_pt_inequality}) we can cancel $(1-\alpha)$ from both. Let $\beta := \bigg(y\big(\bm{\hat{\epsilon}_{sp}}\cdot \bm{x}_{sp}\big)\bigg)$, we get:
% \begin{align}
%     \norm{\bm{w}_{inv}}^2(1+\alpha) \beta^2 &\geq (1-\alpha)\\
%     \norm{\bm{w}_{inv}}^2\beta^2+\alpha\norm{\bm{w}_{inv}}^2\beta^2 &\geq 1-\alpha\\
%     \alpha\bigg(1+\norm{\bm{w}_{inv}}^2\beta^2\bigg) &\geq \bigg(1-\norm{\bm{w}_{inv}}^2\beta^2\bigg)\\
% \end{align}
After substituting back the value of $\beta$ and rearranging we get:
\begin{equation}
\label{eq:margin_pt_alpha}
    \alpha > \frac{1-\norm{\bm{w}_{inv}}^2 \cdot \bigg(y\big(\bm{\hat{\epsilon}_{sp}}\cdot \bm{z}_{sp}^{m}\big)\bigg)^2}{1+\norm{\bm{w}_{inv}}^2 \cdot \bigg(y\big(\bm{\hat{\epsilon}_{sp}}\cdot \bm{z}_{sp}^{m}\big)\bigg)^2} \coloneqq \alpha_{y}^{lb_{1}}(\bm{z}^{m})
\end{equation}

Lets define $\alpha_{y}^{lb_1}\coloneqq \max_{\bm{z}^{m}\in \mathcal{S}_{y}^{m}}(\alpha_{y}^{lb_1}(\bm{z}^{m}))$. Since $\norm{\bm{w}_{inv}}^2 \cdot \bigg(y\big(\bm{\hat{\epsilon}_{sp}}\cdot \bm{z}_{sp}^{m}\big)\bigg)^2 > 0$, the right hand side of above equation $\alpha_{y}^{lb_1}(\bm{z}^{m})<1$ for all $\bm{z}^{m}\in \mathcal{S}_{y}^{m} \implies \alpha_{y}^{lb_1}<1$, which sets a new lower bound on allowed value of $\alpha$ for which \refeqn{eq:margin_main_ineq} is satisfied. Thus when $\alpha\in(\alpha_{y}^{lb_1},1)$, $m_p(\bm{z}^m)>1$ for all $\bm{z}^{m}\in\mathcal{S}_y^m$. That is, the perturbed probing classifier $c_{\alpha}(\bm{z})$ has larger margin than purely-invariant/clean classifier $c_{inv}(\bm{z})$  for the margin points $\bm{z}^{m} \in \mathcal{S}_{y}^{m}.$

\paragraph{Case 2: Non-Margin Points $(\mathcal{S}_{y=1}^{r}\cup \mathcal{S}_{y=-1}^{r})$:} For the non-margin points $\bm{z}^r\in \mathcal{S}_y^{r}$ in the latent space we have $m_{inv}(\bm{z}^r)>1$. Let $\gamma :=  \min_{\bm{z}^r\in \mathcal{S}_y^r}\big(m_{inv}(\bm{z}^r)\big)$ thus we also have $\gamma>1$. Let $\alpha\neq0$ and we choose $\alpha$ such that:
\begin{align}
    \frac{1}{\alpha} &< \gamma\\
    \alpha &> \frac{1}{\gamma}
\end{align}
Substituting the value of $\gamma$ we get:
\begin{equation}
\label{eq:choice_alpha}
     \alpha >\frac{1}{ \min_{\bm{z}^r\in \mathcal{S}_y^r}\big(m_{inv}(\bm{z}^r)\big)} = \alpha_{y}^{lb_2}
\end{equation}
Since $\gamma>1$, thus right hand side in above equation $\alpha_{y}^{lb_2}<1$, which sets a new lower bound on allowed values of $\alpha$. Since $m_{inv}(\bm{z}^r)\geq \gamma >\frac{1}{\alpha}$ for all $\bm{z}^r\in\mathcal{S}_y^r$ for $\alpha \in (\alpha_{y}^{lb_2},1)$ (\refeqn{eq:choice_alpha}), we can write $m_{inv}(\bm{z}^r)=\frac{1}{\alpha}+\eta(\bm{z}^r)$ where $\eta(\bm{z}^r)\coloneqq (m_{inv}(\bm{z}^{r})-\frac{1}{\alpha}) >0$ for all $\bm{z}^r \in \mathcal{S}_y^r$. Now we need to show that there exist an $\alpha\in(\alpha_{y}^{lb_2},1)$ such that $m_{\alpha}(\bm{z}^r)>1$ for all $\bm{z}^r \in \mathcal{S}_y^r$. Thus from \refeqn{eq:perturbed_margin} we need:
\begin{align}
    m_{\alpha}(\bm{z}^r) = \alpha\cdot m_{inv}(\bm{z}^r) + \norm{\bm{w}_{inv}}\sqrt{1-\alpha^2} \bigg(y\big(\bm{\hat{\epsilon}}_{sp}\cdot \bm{z}^{r}_{sp}\big)\bigg) &> 1\\
    \alpha\cdot (\frac{1}{\alpha}+\eta(\bm{z}^r)) + \norm{\bm{w}_{inv}}\sqrt{1-\alpha^2} \bigg(y\big(\bm{\hat{\epsilon}}_{sp}\cdot \bm{z}^{r}_{sp}\big)\bigg) &> 1\\\label{eq:non_margin_pt_inequality}
    \norm{\bm{w}_{inv}}\sqrt{1-\alpha^2} \bigg(y\big(\bm{\hat{\epsilon}}_{sp}\cdot \bm{z}^{r}_{sp}\big)\bigg) &> - \big(\alpha\cdot\eta(\bm{z}^r)\big)
\end{align}
Since $\alpha\in(\alpha_{lb}^{2},1)$, we have $(\alpha\cdot \eta(\bm{z}^r)) > 0$ and $\norm{\bm{w}_{inv}}\sqrt{1-\alpha^2} > 0$. Let's define $\delta(\bm{z}^r) :=y\big(\bm{\hat{\epsilon}}_{sp}\cdot \bm{z}^{r}_{sp}\big)$. Thus for the latent-points $\bm{z}^r\in \mathcal{S}_y^r$ which have $\delta(\bm{z}^r) \geq 0$, \refeqn{eq:non_margin_pt_inequality} is always satisfied since left side of inequality is greater than or equal to zero and right side is always less than zero. For the points for which $\delta(\bm{z}^r)<0$ we have:
\begin{align}
    \norm{\bm{w}_{inv}}\sqrt{1-\alpha^2} \cdot{(-1)} \cdot |\delta(\bm{z}^r)| &> {-} \big(\alpha\cdot\eta(\bm{z}^r)\big)\\
    \norm{\bm{w}_{inv}}\sqrt{1-\alpha^2} |\delta(\bm{z}^r)| &< \big(\alpha\cdot\eta(\bm{z}^r)\big)\\
    \norm{\bm{w}_{inv}}^2 \big(1-\alpha^2\big) \delta(\bm{z}^r)^2 &< \big(\alpha\cdot\eta(\bm{z}^r)\big)^2\\
    \norm{\bm{w}_{inv}}^2 \delta(\bm{z}^r)^2 &< \alpha^2\cdot \bigg(\eta(\bm{z}^r)^2 + \norm{\bm{w}_{inv}}^2 \delta(\bm{z}^r)^2 \bigg)\\
    \alpha &> \sqrt{\frac{\norm{\bm{w}_{inv}}^2 \delta(\bm{z}^r)^2}{\eta(\bm{z}^r)^2 + \norm{\bm{w}_{inv}}^2 \delta(\bm{z}^r)^2}} = \alpha_{y}^{lb_3}(\bm{z}^{r})
    \label{eq:non_margin_pt_alpha}
\end{align}
Now different $\bm{z}^{r}$ will have different $\eta(\bm{z}^{r})$ which will give different lower bound of $\alpha$. Since the $m_{\alpha}(\bm{z}^{r})>1$ has to be satisfied for every point in $\bm{z}^{r}\in \mathcal{S}_{y}^{r}$ \new{we will choose} the maximum value of $\alpha_{y}^{lb_3}(\bm{z}^{r})$ which will give tightest lower bound on value of $\alpha$. Lets define $\alpha_{y}^{lb_3}\coloneqq \max_{\bm{z}^{r}\in \mathcal{S}_{y}^{r}}(\bm{z}^{r})$, then for $m_{\alpha}(\bm{z}^{r})>1$, we need $\alpha > \alpha_{y}^{lb_3}$. Also, since for all $\bm{z}^{r}\in \mathcal{S}_{y}^{r}$, $\eta(\bm{z}^{r})>0$ we have  $\alpha_{y}^{lb_3}(\bm{z}^{r})<1 \implies \alpha_{y}^{lb_3}<1 $. 

% $\eta_{min} = \min_{\bm{x}^r \in \mathcal{S}_y^r}(\eta(\bm{x}^r))$ which will give the tightest lower bound which $\alpha$ need to satisfy. Thus we have:
% \begin{equation}
% \label{eq:non_margin_pt_alpha}
%     \alpha > \sqrt{\frac{\norm{\bm{w}_{inv}}^2 \delta(\bm{z}^r)^2}{\eta_{min}^2 + \norm{\bm{w}_{inv}}^2 \delta(\bm{x}^r)^2}} = \alpha_{lb}^{3}
% \end{equation}
% Since $\eta_{min}>0$ and $\norm{\bm{w}_{inv}}^2 \delta(\bm{x}^r)^2 > 0$, the lower bound $\alpha_{lb}^{3}<1$ in \refeqn{eq:non_margin_pt_alpha}.

Finally, combining  all the lower bound of $\alpha $ from \refeqn{eq:margin_pt_alpha}, \refeqn{eq:choice_alpha} and \refeqn{eq:non_margin_pt_alpha} let the overall lower bound of $\alpha$ be $\alpha_{lb}$ given by:
% \begin{equation}
%     \alpha_{lb} = \max\Bigg\{\frac{1-\norm{\bm{w}_{inv}}^2 \cdot \bigg(y\big(\bm{\hat{\epsilon}_{sp}}\cdot \bm{x}_{sp}^{m}\big)\bigg)^2}{1+\norm{\bm{w}_{inv}}^2 \cdot \bigg(y\big(\bm{\hat{\epsilon}_{sp}}\cdot \bm{x}_{sp}^{m}\big)\bigg)^2} ,\>\>
%     \frac{1}{ \min_{\bm{x}^r\in \mathcal{S}_y^r}\big(m_{inv}(\bm{x}^r)\big)} ,\>\> \sqrt{\frac{\norm{\bm{w}_{inv}}^2 \delta(\bm{x}^r)^2}{\eta_{min}^2 + \norm{\bm{w}_{inv}}^2 \delta(\bm{x}^r)^2}}  \Bigg\}
% \end{equation}
\begin{equation}
\label{eq:alpha_sufficient_lb}
    \alpha_{lb} = \max\{ \alpha_{y=1}^{lb_1}, \alpha_{y=-1}^{lb_1}, \alpha_{y=1}^{lb_2},\alpha_{y=-1}^{lb_2}, \alpha_{y=1}^{lb_3},\alpha_{y=-1}^{lb_3}, \}
\end{equation}
This provides a way to construct a spurious-using classifier: given any \emph{purely-invariant} classifier, we can always choose $\alpha \in (\alpha_{lb},1) $ and construct a perturbed \emph{spurious-using} classifier from \refeqn{eq:perturbed_spurious_classifier} which has a bigger margin than \emph{purely-invariant}. Thus, given all the assumptions, there always exists a \emph{spurious-using} classifier which has greater margin than the \emph{purely-invariant} classifier.

% So, we can always choose $\alpha \in (\alpha_{lb},1)$ and construct a \emph{spurious-using} classifier, which will ensure that $m_{\alpha}(\bm{z_i})>1$ for all $\bm{z}_i$. Thus, given all the assumptions, there always exist a \emph{spurious-using} classifier which has greater margin than the \emph{purely-invariant} classifier completing our proof.

%\amit{it seems as if we require a condition on alpha to make the proof work. Rather, try to write, "This provides a way to construct a spurious-using classifier: we can always choose alpha more than lb to construct it."}

% In the proof we observe that for non-margin points we can always get an $\alpha\in (0,1)$ such that $m_{\alpha}(\bm{z}^{r})>1$. But for margin points, we need linear separability of the spurious feature with respect to task label $y$ (\refassm{assm:spurious_linear_separably} or \ref{assm:spurious_linear_separably_main_task} which ensure that we have $m_{\alpha}(\bm{z}^{m})>1$. 

%\amit{this last paragraph probably needs to be placed somewhere else. The proof ends, and then it seems we are pointing a limitation of the proof that it did not consider margin points? A better place will be just when we are starting the proof (line 668). Can say, "Below, we will provide the result for both margin and non-margin. for non-margin, we can always get...."}
\end{proof}

\subsection{Proof of necessary condition}
\label{subsec:app_proof_necessary_condition}

In this section we will show that \refassm{assm:spurious_linear_separably} is also a necessary condition for the probing classifier to use the \new{spuriously correlated main task features ($\bm{z}_{m}$) when the dimension of \prop-causal feature $d_{p}=1$}. That is, the probing classifier will use the spurious features  if and only if spurious features satisfy  \refassm{assm:spurious_linear_separably} for the margin points of the clean/purely-invariant (\refdef{def:purely_invariant})  probing classifier when the \prop-causal feature is 1-dimensional. Also, same line of reasoning will hold for the main-task classifier where we will show that main-task classifier will use the spurious feature ($\bm{z}_{p}$) iff spurious feature satisfies \refassm{assm:spurious_linear_separably_main_task} for the margin point of clean main-task classifier. Formally:

\begin{lemma}[Necessary Condition for \prop-Probing Classifier]
\label{lemma:necessary_condition_probing}
Let the latent representation be frozen and disentangled (\refassm{assm:disentagled_latent}) such that $\bm{z}=[\bm{z}_{m},z_{p}]$ where $z_{p}$ is the \prop-causal feature which is 1-dimensional scalar and fully predictive (\refassm{assm:fully_pred_inv}) and $\bm{z}_{m}\in\mathbb{R}^{d_{m}}$. Let $c_{p}^{*}(\bm{z})=w_{p}\cdot z_{p}$ be the desired clean/purely-invariant probing classifier trained using max-margin objective which only uses $z_{p}$ for prediction. Then the probing classifier trained using max-margin objective will be \emph{spurious-using} i.e. $c_{p}(\bm{z})=w_{p}\cdot z_{p} + \bm{w}_{m}\cdot \bm{z}_{m}$ where $\bm{w}_{m}\neq 0$ iff the spurious feature $\bm{z}_{m}$ is linearly separable w.r.t to probing task label $y_{p}$ for the margin point of $c_{p}^{*}(\bm{z})$ (\refassm{assm:spurious_linear_separably}).
\end{lemma}

\begin{lemma}[Necessary Condition for Main-task Classifier]
\label{lemma:necessary_condition_main_task}
Let the latent representation be frozen and disentangled (\refdef{assm:disentagled_latent}) such that $\bm{z}=[z_{m},\bm{z}_{p}]$ where $z_{m}$ is the main-task feature which is 1-dimensional scalar and fully predictive (\refassm{assm:fully_pred_inv_main_task}) and $\bm{z}_{p}\in\mathbb{R}^{d_{p}}$. Let $c_{m}^{*}(\bm{z})=w_{m}\cdot z_{m}$ be the desired clean/purely-invariant main-task classifier trained using max-margin objective which only uses $z_{m}$ for prediction. Then the main-task classifier trained using max-margin objective will be \emph{spurious-using} i.e. $c_{m}(\bm{z})=w_{m}\cdot z_{m} + \bm{w}_{p}\cdot \bm{z}_{p}$ where $\bm{w}_{p}\neq 0$ iff the spurious feature $\bm{z}_{p}$ is linearly separable w.r.t to main task label $y_{m}$ for the margin point of $c_{m}^{*}(\bm{z})$ (\refassm{assm:spurious_linear_separably_main_task}) .
\end{lemma}

Since proof of both \reflemma{lemma:necessary_condition_probing} and \ref{lemma:necessary_condition_main_task} follows same line of reasoning, hence for brevity, following \refsec{subsec:app_proof_sufficient_condition}, we will prove the lemma for a general classifier $c(\bm{z})$ trained using max-margin objective to predict the task-label $y$. Let the latent representation be of form $\bm{z}=[z_{inv},\bm{z}_{sp}]$ where $z_{inv}\in \mathbb{R}$ is the feature causally derived from the \new{task concept} and $\bm{z}_{sp}\in \mathbb{R}^{d}_{sp}$ is the feature spuriously correlated to task label $y$. With respect to probing classifier $c_{p}(\bm{z})$ in  \reflemma{lemma:necessary_condition_probing} $z_{inv}\coloneqq z_{p}$ and $\bm{z}_{sp}\coloneqq \bm{z}_{m}$. Similarly, for the main-task classifier in \reflemma{lemma:necessary_condition_main_task}, $z_{inv}\coloneqq z_{m}$ and $\bm{z}_{sp}\coloneqq \bm{z}_{p}$.

\begin{proof}[Proof of \reflemma{lemma:necessary_condition_probing} and \ref{lemma:necessary_condition_main_task}]
% \paragraph{Contradiction Assumption} 
Our goal is to show that \refassm{assm:spurious_linear_separably} for probing classifier or \refassm{assm:spurious_linear_separably_main_task} for the main-task classifier is necessary for obtaining a \emph{spurious-using} classifier for the case when $\bm{z}_{inv}$ is one-dimensional. We show this by assuming that optimal classifier is \emph{spurious-using} even when \refassm{assm:spurious_linear_separably} or \ref{assm:spurious_linear_separably_main_task} breaks and then show that this will lead to contradiction.

\textit{Contradiction Assumption:} Formally, let's assume that \refassm{assm:spurious_linear_separably} or \ref{assm:spurious_linear_separably_main_task} is not satisfied for probing or main task respectively, and the optimal classifier for the given classification task is \emph{spurious-using} $c_{*}(\bm{z})$, where:
\begin{align}
\label{eq:spurious_using_assumption}
    c_{*}(\bm{z}) = w_{inv}^{*}\cdot z_{inv} + \norm{\bm{w}^{*}_{sp}}(\bm{\hat{w}}_{sp}^{*}\cdot \bm{z}_{sp})
\end{align}
where $\norm{\bm{w}^{*}_{sp}}\neq0$ and  $\bm{\hat{w}}^{*}_{sp}\in \mathbb{R}^{d_{sp}}$ is a unit vector in spurious-feature subspace with dimension $d_{sp}$.

Let $c_{inv}(\bm{z})=w_{inv}\cdot z_{inv}$ %\amit{need to be consistent with notation. in earlier proof, hinv is the invariant classifier} 
be the optimal \emph{purely-invariant} classifier. Let both $c_{*}(\bm{z})$ and $c_{inv}(\bm{z})$ be trained using the max-margin objective using \emph{MM-Denominator} formulation in \refeqn{eq:max_margin_denominator}. Thus from the constraints of this formulation (\refeqn{eq:max_margin_denominator_constraint}), \new{for all $\bm{z}$} we have:
\begin{align}
    \label{eq:margin_numerator_optimal} m_{*}(\bm{z}) &= y\cdot c_{*}(\bm{z}) = y \cdot (w_{inv}^{*}\cdot z_{inv} + \norm{\bm{w}^{*}_{sp}}(\hat{\bm{w}}_{sp}^{*}\cdot \bm{z}_{sp})  ) \geq 1 \>\> ,\&\\
    \label{eq:margin_numerator_inv} m_{inv}(\bm{z}) &= y\cdot c_{inv}(\bm{z}) = y \cdot (w_{inv} \cdot z_{inv}) \geq 1
\end{align}
From \refassm{assm:fully_pred_inv} or \ref{assm:fully_pred_inv_main_task}, the invariant feature $\bm{z}_{inv}$ is 100\% predictive and linearly separable w.r.t task label $y$.  Then without loss of generality let's assume that: \begin{align}
    \label{eq:pos_x_inv} z_{inv}>0, \>\> \text{when} \>\> y=+1\\
    \label{eq:neg_x_inv} z_{inv}<0, \>\> \text{when} \>\> y=-1
\end{align}
From \refeqn{eq:pos_x_inv} and \ref{eq:neg_x_inv} we have $y\cdot z_{inv}>0$ thus from \refeqn{eq:margin_numerator_inv} we get:
\begin{equation}
\label{eq:pos_w_inv_pure}
    w_{inv}\geq0
\end{equation}

Also, from our \emph{contradiction-assumption} the max-margin trained classifier is \emph{spurious-using}, thus the norm of parameters of $c_{*}(\bm{z})$ is less or equal to $c_{inv}(\bm{z})$ (\refeqn{eq:max_margin_denominator}). Thus we have:
\begin{align}
    \sqrt{(w_{inv}^{*})^2 + (\norm{\bm{w}^{*}_{sp}})^2} &\leq |w_{inv}|\\
    \implies |w_{inv}^{*}| &< |w_{inv}| \quad\quad\quad (\norm{\bm{w}^{*}_{sp}}\neq0)\\
    \implies |w_{inv}^{*}| &< w_{inv} \quad\quad\quad (w_{inv}\geq0, \refeqn{eq:pos_w_inv_pure})   \\
     \implies \label{eq:inv_weight_comparison}w_{inv}^{*} &< w_{inv}
\end{align}
% since $\norm{\bm{w}^{*}_{sp}}\neq0$ and $w_{inv}^{p}>=0$ (\refeqn{eq:pos_w_inv_pure}).

Form our \emph{contradiction-assumption}, \refassm{assm:spurious_linear_separably} for \prop-probing task or \refassm{assm:spurious_linear_separably_main_task} for the main-task breaks by one of the following two ways:
\begin{enumerate}
    \item Opposite Side Failure: This occurs when the spurious part of margin points (of $c_{inv}(\bm{z})$) on opposite sides of decision-boundary of the optimal task classifier ($c_{*}(\bm{z})=0$) are not linearly-separable with respect to task label $y$. Formally, there exist two datapoints, $P^{m+} := [z_{inv}^{m+},\bm{z}_{sp}^{m+}]$ and $P^{m-} := [z_{inv}^{m-},\bm{z}_{sp}^{m-}]$ such that they are margin points of \emph{purely-invariant} classifier $c_{inv}(\bm{z})$ where $P^{m+}$ has class label $y=+1$ and $P^{m-}$ has class label $y=-1$ and $\forall \bm{\hat{\epsilon}}_{sp} \in \mathbb{R}^{d_{sp}}$, the spurious feature $\bm{z}_{sp}$ of both the points lies on same side of $\hat{\bm{\epsilon}}_{sp}$ i.e:
    \begin{equation}
    \label{eq:assm_opposite_side_failure}
        \big((\bm{\hat{\epsilon}}_{sp}\cdot\bm{z}_{sp}^{m+}) \cdot (\bm{\hat{\epsilon}}_{sp}\cdot\bm{z}_{sp}^{m-}) \big) \geq0
    \end{equation}
    
    \item Same Side Failure: This occurs when  the spurious part of margin points (of $c_{inv}(\bm{z})$) on same side of decision-boundary ($c_{*}(\bm{z})=0$) are linearly-separable. Formally, there exist two datapoints, $P_{y}^{m1}:=[z_{inv}^{m1},\bm{z}_{sp}^{m1}]$ and $P_{y}^{m2}:=[z_{inv}^{m2},\bm{z}_{sp}^{m2}]$ such that they are margin points of \emph{purely-invariant} classifier $c_{inv}(\bm{z})$ and both points have same class label $y$ and $\forall \bm{\hat{\epsilon}}_{sp} \in \mathbb{R}^{d_{sp}}$, w.l.o.g we have:
    \begin{equation}
    \label{eq:assm_same_side_failure}
        \big((\bm{\hat{\epsilon}}_{sp} \cdot \bm{z}_{sp}^{m1}) \cdot (\bm{\hat{\epsilon}}_{sp} \cdot \bm{z}_{sp}^{m2}) \big) \leq 0.
    \end{equation}
    
\end{enumerate}

%The following two lemma, one for each of the failure case above shows that when \refassm{assm:same_side_ls} breaks it implies $\norm{\bm{w}^{*}_{sp}}=0$ which is impossible since from our \emph{contradiction-assumption} the optimal classifier is \emph{spurious-using} which implied $\norm{\bm{w}^{*}_{sp}} \neq 0$ (\refeqn{eq:spurious_using_assumption}). 

We will use the following two lemma to proceed with our proof:
\begin{lemma}
\label{lemma:opposite_side_condraction}
If \refassm{assm:spurious_linear_separably} or \ref{assm:spurious_linear_separably_main_task} breaks by \emph{opposite-side failure} mode, it leads to contradiction.
\end{lemma}

\begin{lemma}
\label{lemma:same_side_condraction}
If \refassm{assm:spurious_linear_separably} or \ref{assm:spurious_linear_separably_main_task} breaks by \emph{same-side failure} mode,  it leads to contradiction.
\end{lemma}
This implies that our \emph{contradiction-assumption} which said that the max-margin trained optimal classifier is \emph{spurious-using} even when \refassm{assm:spurious_linear_separably} or \ref{assm:spurious_linear_separably_main_task} breaks, is wrong. Thus \refassm{assm:spurious_linear_separably} for \prop-probing task or \refassm{assm:spurious_linear_separably_main_task} for main-task is necessary for the optimal max-margin classifier to be \emph{spurious-using}. This  completes our proof.
\end{proof}

\begin{proof}[Proof of \reflemma{lemma:opposite_side_condraction}]
We have two points, $P^{m+} := [z_{inv}^{m+},\bm{z}_{sp}^{m+}]$ and $P^{m-} := [z_{inv}^{m-},\bm{z}_{sp}^{m-}]$,  which break the \refassm{assm:spurious_linear_separably} or \ref{assm:spurious_linear_separably_main_task}. From \refeqn{eq:pos_x_inv}, $z_{inv}>0$ for all the points with label $y=1$, thus we have $z_{inv}^{m+}>0$ and using \refeqn{eq:inv_weight_comparison} ($ w_{inv}^{*} < w_{inv}$) we get:
\begin{align}
    w_{inv}^{*} &< w_{inv}\\
    w_{inv}^{*}\cdot z_{inv}^{m+} &<  w_{inv}\cdot z_{inv}^{m+}\\
    \label{eq:lemma_opposite_margin_value_m+} w_{inv}^{*}\cdot z_{inv}^{m+} &<  1
\end{align}
where the right hand side $w_{inv}\cdot z_{inv}^{m+}=1$ since $P^{m+}$ is the margin-point of $c_{inv}(\bm{z})$ (\refeqn{eq:margin_numerator_inv}). Similarly from \refeqn{eq:neg_x_inv}, $z_{inv}<0$ for all the points with label $y=-1$, thus we have $z_{inv}^{m-}<0$ and using \refeqn{eq:inv_weight_comparison} ($ w_{inv}^{*} < w_{inv}$) we get:
\begin{align}
    w_{inv}^{*} &< w_{inv}\\
    (-1)\cdot w_{inv}^{*}\cdot z_{inv}^{m-} &<  (-1)\cdot w_{inv}\cdot z_{inv}^{m-}\\
    \label{eq:lemma_opposite_margin_value_m-} (-1)\cdot w_{inv}^{*}\cdot z_{inv}^{m-} &<  1
\end{align}
where the right hand side $(-1)\cdot (w_{inv}^{p}\cdot z_{inv}^{m-})=1$ since $P^{m-}$ is the margin-point of $c_{inv}(\bm{z})$ (\refeqn{eq:margin_numerator_inv}).

Next from \refeqn{eq:margin_numerator_optimal} we have $m_{*}(\bm{z})\geq 1$ for all $\bm{z}$ hence it is also true for $P^{m+}$ with $y=1$ and $P^{m-}$ with $y=-1$. Then:
\begin{align}
    m_{*}(P^{m+}) = y\cdot c_{*}(P^{m+}) = 1\cdot \bigg\{ w_{inv}^{*}z_{inv}^{m+} + \norm{\bm{w}_{sp}^{*}}\big( \bm{\hat{w}}^{*}_{sp}\cdot \bm{z}_{sp}^{m+} \big) \bigg\} &\geq 1\\ 
   \implies  w_{inv}^{*}z_{inv}^{m+} + \norm{\bm{w}_{sp}^{*}}\cdot \beta^{m+} &\geq 1\\
    \implies \label{eq:lemma_opposite_m+_inequality} w_{inv}^{*}z_{inv}^{m+} &\geq 1 -\norm{\bm{w}_{sp}^{*}}\cdot \beta^{m+}
\end{align}
where $\beta^{m+} = \big( \bm{\hat{w}}^{*}_{sp}\cdot \bm{z}_{sp}^{m+} \big)$. Also we have:
\begin{align}
    m_{*}(P^{m-}) = y\cdot c_{*}(P^{m-}) = -1\cdot \bigg\{ w_{inv}^{*}z_{inv}^{m-} + \norm{\bm{w}_{sp}^{*}}\big( \bm{\hat{w}}^{*}_{sp}\cdot \bm{z}_{sp}^{m-} \big) \bigg\} &\geq 1\\ 
    \implies -w_{inv}^{*}z_{inv}^{m-} - \norm{\bm{w}_{sp}^{*}}\cdot \beta^{m-} &\geq 1\\
    \implies \label{eq:lemma_opposite_m-_inequality} -w_{inv}^{*}z_{inv}^{m-} &\geq 1 +\norm{\bm{w}_{sp}^{*}}\cdot \beta^{m-}
\end{align}where $\beta^{m-} = \big( \bm{\hat{w}}^{*}_{sp}\cdot \bm{z}_{sp}^{m-} \big)$. From \refeqn{eq:assm_opposite_side_failure} we have $\big((\bm{\hat{\epsilon}}_{sp}\cdot\bm{z}_{sp}^{m+}) \cdot (\bm{\hat{\epsilon}}_{sp}\cdot\bm{z}_{sp}^{m-}) \big) \geq0$ for all $\bm{\hat{\epsilon}}_{sp} \in \mathbb{R}^{d_{sp}}$ which states the \emph{opposite-side failure} of \refassm{assm:spurious_linear_separably} or \ref{assm:spurious_linear_separably_main_task}. Thus:
\begin{equation}
\label{eq:lemma_opposite_beta_constraint}
    \beta^{m+}\cdot \beta^{m-} \geq 0
\end{equation}
Now we will show that
\refeqn{eq:lemma_opposite_margin_value_m+},
\ref{eq:lemma_opposite_margin_value_m-},
\ref{eq:lemma_opposite_m+_inequality} and \ref{eq:lemma_opposite_m-_inequality} cannot be satisfied simultaneously for any allowed value of $\beta^{m+}$ and $\beta^{m-}$ (given by \refeqn{eq:lemma_opposite_beta_constraint}) which are:
\begin{enumerate}
    \item $\beta^{m+}>0$ and $\beta^{m-}>0$: From \refeqn{eq:lemma_opposite_m-_inequality} we have $-w_{inv}^{*}z_{inv}^{m-}>1$ since $\norm{\bm{w}^{*}_{sp}}\neq0$ and $\beta^{m-}>0$. But from  \refeqn{eq:lemma_opposite_margin_value_m-} we have $-w_{inv}^{*}z_{inv}^{m-}<1$ which is a contradiction. 
    
    \item $\beta^{m+}<0$ and $\beta^{m-}<0$: From \refeqn{eq:lemma_opposite_m+_inequality} we have $w_{inv}^{*}z_{inv}^{m+}>1$ since $\norm{\bm{w}^{*}_{sp}}\neq0$ and $\beta^{m+}<0$. But from  \refeqn{eq:lemma_opposite_margin_value_m+} we have $w_{inv}^{*}z_{inv}^{m+}<1$ which is a contradiction. 
    
    \item $\beta^{m+}=0$ and $\beta^{m-}\in \mathbb{R}$: From \refeqn{eq:lemma_opposite_m+_inequality} we have $w_{inv}^{*}z_{inv}^{m+} \geq 1$ but from \refeqn{eq:lemma_opposite_margin_value_m+} we have $w_{inv}^{*}z_{inv}^{m+}<1$ which is a contradiction. 
    
    \item $\beta^{m+}\in \mathbb{R}$ and $\beta^{m-}=0$: From \refeqn{eq:lemma_opposite_m-_inequality} we have $-w_{inv}^{*}z_{inv}^{m-}\geq 1$ but from  \refeqn{eq:lemma_opposite_margin_value_m-} we have $-w_{inv}^{*}z_{inv}^{m-}<1$ which is a contradiction. 
\end{enumerate}
Thus we have a contradiction for all the possible values $\beta^{m+}$ and $\beta^{m-}$ could take, completing the proof of this lemma.

\end{proof}

\begin{proof}[Proof of \reflemma{lemma:same_side_condraction}]
We have two margin-points, $P_{y}^{m1}:=[z_{inv}^{m1},\bm{z}_{sp}^{m1}]$ and $P_{y}^{m2}:=[z_{inv}^{m2},\bm{z}_{sp}^{m2}]$, which break \refassm{assm:spurious_linear_separably} or \ref{assm:spurious_linear_separably_main_task}. From \refeqn{eq:pos_x_inv} and \refeqn{eq:neg_x_inv} we have $y\cdot z_{inv}^{m1}>0$ and $y\cdot z_{inv}^{m2}>0$. Using \refeqn{eq:inv_weight_comparison} ($ w_{inv}^{*} < w_{inv}$) we get:
\begin{align}
    w_{inv}^{*} &< w_{inv}\\
    w_{inv}^{*}\cdot (y\cdot z_{inv}^{mj}) &< w_{inv}\cdot (y\cdot z_{inv}^{mj})\\
    \label{eq:lemma_same_margin_value_my}
    y \cdot (w_{inv}^{*} \cdot z_{inv}^{mj}) &< 1
\end{align}
where $j\in\{1,2\}$ and right hand side $w_{inv}\cdot (y\cdot z_{inv}^{mj}) = 1$ since  $P_y^{mj}$ is the margin point of purely-invariant classifier $c_{inv}(\bm{z})$ (\refeqn{eq:margin_numerator_inv}).

From \refeqn{eq:margin_numerator_optimal} we have $m_{*}(\bm{z})\geq 1$ for all $\bm{z}$ thus also true for $P^{m1}_y$ and $P^{m2}_y$. Then:
\begin{align}
    m_{*}(P^{m1}_y) = y\cdot c_{*}(P^{m1}_{y}) = y\cdot \bigg\{ w_{inv}^{*}z_{inv}^{m1} + \norm{\bm{w}_{sp}^{*}}\big( \bm{\hat{w}}^{*}_{sp}\cdot \bm{z}_{sp}^{m1} \big) \bigg\} &\geq 1\\ 
    \implies y \cdot(w_{inv}^{*}z_{inv}^{m1}) + y \cdot (\norm{\bm{w}_{sp}^{*}}\cdot \beta^{m1}) &\geq 1\\
    \label{eq:lemma_same_m1_inequality} \implies y\cdot(w_{inv}^{*}z_{inv}^{m1}) &\geq 1 -y\cdot (\norm{\bm{w}_{sp}^{*}}\cdot \beta^{m1})
\end{align}
where $\beta^{m1} = \big( \bm{\hat{w}}^{*}_{sp}\cdot \bm{z}_{sp}^{m1} \big)$. Also we have:
\begin{align}
    m_{*}(P^{m2}_{y}) = y\cdot c_{*}(P^{m2}_{y}) = y\cdot \bigg\{ w_{inv}^{*}z_{inv}^{m2} + \norm{\bm{w}_{sp}^{*}}\big( \bm{\hat{w}}^{*}_{sp}\cdot \bm{z}_{sp}^{m2} \big) \bigg\} &\geq 1\\ 
    \implies y\cdot (w_{inv}^{*}z_{inv}^{m2}) + y\cdot (\norm{\bm{w}_{sp}^{*}}\cdot \beta^{m2}) &\geq 1\\
    \implies \label{eq:lemma_same_m2_inequality} y\cdot (w_{inv}^{*}z_{inv}^{m2}) &\geq 1 -y\cdot (\norm{\bm{w}_{sp}^{*}}\cdot \beta^{m2})
\end{align}
where $\beta^{m2} = \big( \bm{\hat{w}}^{*}_{sp}\cdot \bm{z}_{sp}^{m2} \big)$. Now from \refeqn{eq:assm_same_side_failure} we have $ \big((\bm{\hat{\epsilon}}_{sp} \cdot \bm{z}_{sp}^{m1}) \cdot (\bm{\hat{\epsilon}}_{sp} \cdot \bm{z}_{sp}^{m2}) \big) \leq 0$ for all unit vectors $\hat{\bm{\epsilon}}_{sp} \in \mathbb{R}^{d_{sp}}$ which states the \emph{same-side} failure mode of \refassm{assm:spurious_linear_separably} or \ref{assm:spurious_linear_separably_main_task}. Thus we have:
\begin{equation}
    \beta^{m1}\cdot \beta^{m2} \leq 0
\end{equation}

Now we will show that for all allowed values of $\beta^{m1}$ and $\beta^{m2}$, \refeqn{eq:lemma_same_margin_value_my}, \ref{eq:lemma_same_m1_inequality} and \ref{eq:lemma_same_m2_inequality} will lead to a contradiction. Following are the cases for different allowed values of $\beta^{m1}$ and $\beta^{m2}$:
\begin{enumerate}
    \item $\beta^{m1}=0$ and $\beta^{m2}\in \mathbb{R}$: Substituting $\beta^{m1}=0$ in  \refeqn{eq:lemma_same_m1_inequality} we get $y\cdot(w_{inv}^{*}z_{inv}^{m1})\geq1$, but from \refeqn{eq:lemma_same_margin_value_my} we have $y\cdot(w_{inv}^{*}z_{inv}^{m1})<1$. Thus we have a contradiction.
    
    \item $\beta^{m1}\in \mathbb{R}$ and $\beta^{m2}=0$: Substituting $\beta^{m2}=0$ in  \refeqn{eq:lemma_same_m2_inequality} we get $y\cdot(w_{inv}^{*}z_{inv}^{m2})\geq1$, but from \refeqn{eq:lemma_same_margin_value_my} we have $y\cdot(w_{inv}^{*}z_{inv}^{m2})<1$. Thus we have a contradiction.
    
    \item The only case which is left now is when both $\beta^{m1}$ and $\beta^{m2}$ is non-zero but of opposite sign. Without loss of generality, let $\beta^{m1}>0$, $\beta^{m2}<0$ and $y=(+1)$: Substituting $\beta^{m2}<0$ and $y=(+1)$ in  \refeqn{eq:lemma_same_m2_inequality} we get $y\cdot(w_{inv}^{*}z_{inv}^{m2})\geq1$, but from \refeqn{eq:lemma_same_margin_value_my} we have $y\cdot(w_{inv}^{*}z_{inv}^{m2})<1$. Thus we have a contradiction.
    
    \item $\beta^{m1}>0$, $\beta^{m2}<0$ and $y=(-1)$: Substituting $\beta^{m1}>0$ and $y=(-1)$ in  \refeqn{eq:lemma_same_m1_inequality} we get $y\cdot(w_{inv}^{*}z_{inv}^{m1})\geq1$, but from \refeqn{eq:lemma_same_margin_value_my} we have $y\cdot(w_{inv}^{*}z_{inv}^{m1})<1$. Thus we have a contradiction.
    
\end{enumerate}

Thus we have a contradiction for all the possible values $\beta^{m1}$, $\beta^{m2}$ and $y$ could take, completing the proof of this lemma.

\end{proof}

\section{Null-Space Removal Failure: Setup and Proof of \reftheorem{theorem:null_space_failure}}
\label{sec:app_inlp_setup_proof}

\subsection{Null-Space Setup}
\label{subsec:app_inlp_setup}
%\INLP method was proposed by \cite{NullItOut:2020} to remove the sensitive information (e.g. gender or race) or any undesired/spurious \prop in general, from the latent representations. Given any set of latent representation $\mathcal{X}=\{\bm{x}_1,\ldots,\bm{x}_m\}$ where $\bm{x}_i\in\mathbb{R}^{d}$ with corresponding discrete attributes $z_i \in \mathcal{Z}$ where $\mathcal{Z}=\{1,\ldots k\}$, the goal is to learn a \emph{guarding} transformation $g: \mathbb{R}^{d}\rightarrow \mathbb{R}^{d}$ such that any classifier $f_{z}(\bm{x})$ cannot predict $z_i$ from $g(\bm{x}_i)$. \cite{NullItOut:2020} focused on the case when $f_{z}(\bm{x})$ is a linear classifier thus guaranteeing that any linear classifier $f_{main}(\bm{x})$ which is using the representation ($x$) for some other downstream task could not use the \emph{guarded} feature ($z$).
% \amit{at the start of the section, refer to main paper---something like, ``As described in the main paper, let the given main-task...'' Right now it feels like you are adding new definitions here.}

% \amit{I see a lot of grammatical mistakes in the writing--it should be "main-task classifier have". I'd suggest passing the whole text through Word or a software like Grammarly and double-checking any errors.}

As described in \refsec{sec:feature_removal_problem}, the given main-task classifier have an encoder  $h:\bm{X}\rightarrow \bm{Z}$ mapping the input $\bm{X}$ to latent representation $\bm{Z}$. Post that, the main-task classifier $c_{m}:\bm{Z}\rightarrow Y_{m}$ is used to predict the main-task label $y^{i}_{m}$ from latent representation $\bm{z}^{i}$ for every input  $\bm{x}^{i}$.  Given this (pre) trained main-task classifier the goal of a post-hoc removal method is to remove any undesired/sensitive/spurious \prop from the latent representation $\bm{Z}$ without retraining the encoder $h$ or main-task classifier $c_{m}(\bm{z})$. 
% \abhinav{Shauli asked why we are not retraining the main classifier. We didn't because it was a post hoc method. Could do this later where we also retrain the main classifier. My hope is that this will show us increase spuriousness score in the main classifier which we don't see directly. So could be helpful for us.}

The null space method~\cite{NullItOut:2020,AmnesicProbing}, henceforth referred to as \emph{\INLP}, is one such post-hoc removal method that removes a concept from latent space  by projecting the latent space to a subspace that is not discriminative of that attribute. First, it  estimates the subspace in the latent space discriminative of the \prop we want to remove  by training a probing classifier $c_{p}(\bm{z})\rightarrow y_{p}$, where $y_{p}$ is the \prop label. \cite{NullItOut:2020} used a linear probing classifier $(c_{p}(\bm{z}))$ to  ensure that the any linear classifier cannot recover the removed \prop from modified latent representation $\bm{z}'$ and hence the main task classifier ($c_{m}(\bm{z}')$), which is also a linear layer, become invariant to removed attribute. Let linear probing classifier $c_{p}(\bm{z})$ be parametrized by a matrix $W$, and null-space of matrix $W$ is defined as space $N(W) = \{\bm{z}|W\bm{z}=\bm{0}\}$. Give the basis vectors for the $N(W)$ we can construct the projection matrix $P_{N(W)}$ such that $W(P_{N(W)}\bm{z})=\bm{0}$ for all $\bm{z}$. This projection matrix is defined as the guarding operator $g \coloneqq P_{N(\mathcal{W})}$ (estimated by $c_{p}(\bm{z})$), when applied on the $\bm{z}$ will remove the features which are discriminative of undesired \prop from $\bm{z}$. For the setting when $Y_{p}$ is binary we have:
\begin{equation}
    P_{N(W)} = I - \hat{w}\hat{w}^{T}
\end{equation}
where $I$ is the identity matrix and $\hat{w}$ is the unit vector in the direction of parameters of classifier $c_{p}(\bm{z})$ (\cite{NullItOut:2020}). Also, the authors recommend running this removal step for multiple iterations to ensure that the unwanted \prop is removed completely. Thus after  $S$ steps of removal, the final guarding function is:
\begin{equation}
    g \coloneqq \prod_{i=1}^{S}P_{N(W)}^{i}
\end{equation}
where $P_{N(W)}^{i}$ is the projection matrix at $i^{th}$ removal step. Amnesic Probing (\cite{AmnesicProbing}) builds upon this idea for testing whether \prop is being used by a given pre-trained classifier or not. The core idea is to remove the \prop we want to test from the latent representation. If the prediction of the given classifier is influenced by this removal then the \prop was being used by the given classifier otherwise not.

\subsection{Null-Space Removal Failure : Proof of \reftheorem{theorem:null_space_failure}}
\label{subsec:app_inlp_proof}
% \amit{have descriptive headers. Null-space proof is vague. better: Null-space removal failure: Proof of Theorem 3.2}

% \amit{since the proof has many steps, good to outline a summary here. Something like ..The proof proceeds in 2 steps \begin{enumerate}
%     \item We first show that mixing of features happens after the first projection step and that the mixing due to projection  is non-invertible in future steps.
%     \item Second, we show that projection is lossy.
% \end{enumerate}
% }

\inlpthm*

The proof of \reftheorem{theorem:null_space_failure} proceeds in following steps:
\begin{enumerate}
    \item First using \reflemma{lemma:sufficient_condition}, we show that even under very favourable conditions probing classifier will not be clean i.e will also use other features in addition to \prop-causal feature for it's prediction. Then, for the more likely case when probing classifier uses both main-task and \prop-causal feature, we show that after first step of null-space projection (\INLP), both the main-task features and \prop-causal features get \emph{mixed}.
    
    \item Next, for the extreme case when probing classifier uses only main-task feature, the first step of INLP will do opposite of what is intended. It will damage the main-task feature but will have no effect on the \prop-causal feature which we wanted to remove from latent space representation. 
    
    \item In addition, we also show that the \emph{damage} or \emph{mixing} of latent space after first step of \INLP projection cannot be corrected in subsequent step since the projection operation is non-invertible.
    
    \item Next, we show that the projection operation is lossy, i.e removes the norm of latent representation under some conditions. Hence after sufficient steps, \INLP could destroy all the information in latent representation.
\end{enumerate}

\begin{proof}[Proof of \reftheorem{theorem:null_space_failure}]

\textbf{First Claim (1a).  }
Let $c_{p}(\bm{z})=\bm{w}_p\bm{z}_{p}+\bm{w}_{m}\bm{z}_{m}$ be the linear probing classifier trained to predict the \prop label $y_{p}$ from the latent representation $\bm{z}$. Since all the assumptions of  \reflemma{lemma:sufficient_condition} are satisfied for the probing classifier $c_{p}(\bm{z})$, it is \emph{spurious using}, i.e.,  $\bm{w}_{m}\neq\bm{0}$ and for the claim 1(a) we have $\bm{w}_{p}\neq \bm{0}$. Since the \prop label $y_{p}$ is binary, the projection matrix for the first step of \INLP removal is defined as $P^{1}_{N(W)}=I-\hat{\bm{w}}\hat{\bm{w}}^{T}$ where $\hat{\bm{w}}^{T}=[\hat{\bm{w}}_{m},\hat{\bm{w}}_{p}]$, $\hat{\bm{w}}_{m}$ and $\hat{\bm{w}}_{p}$ are the unit norm parameters of $c_{p}(\bm{z})$ i.e $\bm{w}_{m}$ and $\bm{w}_{p}$ respectively. On applying this projection on the latent space representation $\bm{z}^{i}$ we get new projected representation $\bm{z}^{i(1)}$ s.t.:
\begin{align}
\label{eqn:inlp_proj1}
        \begin{bmatrix}
            \bm{z}_{m}^{i(1)}\\
            \bm{z}_{p}^{i(1)}\\
        \end{bmatrix}
        &= \Bigg{(}I - \begin{bmatrix}
            \hat{\bm{w}}_{m}\\
            \hat{\bm{w}}_{p}\\
        \end{bmatrix}
        \begin{bmatrix}
            \hat{\bm{w}}_{m} & \hat{\bm{w}}_{p} \\
        \end{bmatrix}\Bigg{)}
        \begin{bmatrix}
            \bm{z}_{m}^{i}\\
            \bm{z}_{p}^{i}\\
        \end{bmatrix}\\
        &=\begin{bmatrix}
           \bm{z}_{m}^{i}\\
            \bm{z}_{p}^{i}\\
        \end{bmatrix}
         - \hat{c}_{p}(\bm{z}^{i}) \cdot \begin{bmatrix}
            \hat{\bm{w}}_{m}\\
            \hat{\bm{w}}_{p}\\
        \end{bmatrix} \quad \quad \quad  \quad \quad \text{define  }  \hat{c}_{p}(\bm{z}^{i}) \coloneqq \bm{\hat{w}}_{m}\cdot \bm{z}_{m}^{i} + \bm{\hat{w}}_{p}\cdot \bm{z}_{p}^{i}\\
        &=\begin{bmatrix}
           \bm{z}_{m}^{i}-\hat{c}_{p}(\bm{z}^{i}) \hat{\bm{w}}_{m}\\
            \bm{z}_{p}^{i}-\hat{c}_{p}(\bm{z}^{i}) \hat{\bm{w}}_{p}\\
        \end{bmatrix}
        \label{eqn:mixing_features}\\
        &=\begin{bmatrix}
            g(\bm{z}_{m}^{i},\bm{z}_{p}^{i})  \\
            f(\bm{z}_{m}^{i},\bm{z}_{p}^{i}) \\
        \end{bmatrix}
\end{align}

% \amit{why do we want to show this? Make it clear for readers:the goal is to show that mixing will happen}
Next, we will show that the main task features and probing features get mixed after projection. To do so, we first show that $g(\bm{z}_{m}^{i},\bm{z}_{p}^{i})\neq g(\bm{z}_{m}^{i})$ for some function $g$ i.e projected main task features $\bm{z}_{m}^{i(1)}=g(\bm{z}_{m}^{i},\bm{z}_{p}^{i})$ are not independent of probing features $\bm{z}_{p}^{i}$. From \refeqn{eqn:mixing_features}, we have:
\begin{align}
    \bm{z}_{m}^{i(1)}&=g(\bm{z}_{m}^{i},\bm{z}_{p}^{i})\\
    &=\bm{z}^{i}_{m} - (\hat{\bm{w}}_{m}\cdot\bm{z}_{m}^{i} + \hat{\bm{w}}_{p}\cdot \bm{z}_{p}^{i})\hat{\bm{w}}_{m}\\
    &=(I-\hat{\bm{w}}_{m}\hat{\bm{w}}_{m}^{T})\bm{z}_{m}^{i}  - (\hat{\bm{w}}_{p}\cdot \bm{z}_{p}^{i}) \hat{\bm{w}}_{m} \label{eq:main-feat-mix}
\end{align}
\new{In this case we are given $\bm{w}_{p}\neq \bm{0}$ and $\bm{w}_{m}\neq \bm{0}$. Since $\bm{z}_{p}^{i}$ can take any value (subject to \refassm{assm:fully_pred_inv}), $\hat{\bm{w}}_{p}\cdot \bm{z}_{p}^{i}$ is not trivially zero for all $\bm{z}_{p}^{i}$  $\Longrightarrow (\hat{\bm{w}}_{p}\cdot \bm{z}_{p}^{i}) \hat{\bm{w}}_{m} \neq \bm{0}$}. Thus $g(\bm{z}_{m}^{i},\bm{z}_{p}^{i})$ is not independent of $\bm{z}_{p}^{i}$.

Next, we will show that $f(\bm{z}_{m}^{i},\bm{z}_{p}^{i})\neq f(\bm{z}_{p}^{i})$ for some function $f$ i.e projected probing feature $\bm{z}_{p}^{i(1)}=f(\bm{z}_{m}^{i},\bm{z}_{p}^{i})$ is not independent of the main task feature $\bm{z}_{m}^{i}$. From \refeqn{eqn:mixing_features}, we have:
% \amit{do not use h since h is used in the main paper as the frozen encoding. Use $f$ or $g'$}
\begin{align}
    \bm{z}_{p}^{i(1)}&=f(\bm{z}_{m}^{i},\bm{z}_{p}^{i})\\
    &=\bm{z}^{i}_{p} - (\hat{\bm{w}}_{m}\cdot\bm{z}_{m}^{i} + \hat{\bm{w}}_{p}\cdot \bm{z}_{p}^{i})\hat{\bm{w}}_{p}\\
    &=(I-\hat{\bm{w}}_{p}\hat{\bm{w}}_{p}^{T})\bm{z}_{p}^{i}  - (\hat{\bm{w}}_{m}\cdot \bm{z}_{m}^{i}) \hat{\bm{w}}_{p} \label{eq:prop-feat-mix}
\end{align}
% Since $\bm{w}_{p}\neq \bm{0}$ and $\bm{w}_{m}\neq \bm{0}$, we have $(\hat{\bm{w}}_{m}\cdot \bm{z}_{m}^{i}) \hat{\bm{w}}_{p} \neq \bm{0}$.
\new{Again, in this case we are given $\bm{w}_{p}\neq \bm{0}$ and $\bm{w}_{m}\neq \bm{0}$. Since $\bm{z}_{m}^{i}$ can take any value (subject to \refassm{assm:spurious_linear_separably}), $\hat{\bm{w}}_{m}\cdot \bm{z}_{m}^{i}$ is not trivially zero for all $\bm{z}_{m}^{i}$  $\Longrightarrow (\hat{\bm{w}}_{m}\cdot \bm{z}_{m}^{i}) \hat{\bm{w}}_{p} \neq \bm{0}$}.
Thus $f(\bm{z}_{m}^{i},\bm{z}_{p}^{i})$ is not independent of $\bm{z}_{m}^{i}$. Hence both \prop-feature $\bm{z}_{p}$ and the main-task feature $\bm{z}_{m}$ got mixed after the first step of projection.

Next, we will show that this mixing of the main task and \prop-causal feature cannot be corrected in subsequent steps of null-space projection. Formally, the following \reflemma{lemma:inlp_non_inv_p} proves that the above projection matrix ($P^{1}_{N(W)}$) which resulted in mixing of features is non-invertible. The subsequent steps of \INLP applies projection transformation which can be combined into one single matrix $P^{>1}_{N(W)} = \prod_{j>1} P^{j}_{N(W)}$. In order for mixing to be reversed, we need $P^{>1}_{N(W)} \times P^{1}_{N(W)} = I$, thus we need $P^{>1}_{N(W)} = (P^{1}_{N(W)})^{-1}$ which is not possible from \reflemma{lemma:inlp_non_inv_p}. Hence the mixing of the main-task feature and the \prop-causal feature which happened after the first step of projection cannot be corrected in the subsequent steps of \INLP thus completing the first claim of our proof.
% \amit{add more justification: why does invertibility of proj matrix imply that the mixing can be corrected?}
\begin{lemma}
\label{lemma:inlp_non_inv_p}
The projection matrix $P^{j}_{N(W)}$ at any projection step of \INLP is non invertible.
\end{lemma}
\begin{proof}[Proof of \reflemma{lemma:inlp_non_inv_p}]
The projection matrix for binary target case is defined as $P \coloneqq P^{j}_{N(W)}=I-A$ where $A=\hat{\bm{w}}\hat{\bm{w}}^{T}$ be a $n\times n$ matrix and $\bm{w}$ is the parameter vector of the probing classifier $c_{p}(\bm{z})$ trained at $j^{th}$-step of \INLP. We can see that $A$ is a symmetric matrix. Every symmetric matrix is diagonalizable (Equation W.9 in \cite{DigonalizationSymmMatrix}), hence we can write $A=Q\Lambda Q^{T}$, where $Q$ is a some orthonormal matrix such that $QQ^{T}=I$ and $\Lambda=diag(\lambda_1,\dots,\lambda_n)$ be a $n\times n$ diagonal matrix where the diagonal entries ($\lambda_1\ldots\lambda_n$) are the eigen-values of $A$. Since $QQ^{T}=I$ we can write $P = I-A = QQ^T-Q\Lambda Q^{T} = Q(I-\Lambda)Q^{T}$. Next, for the projection matrix $P$ to be invertible $P^{-1}$ should exist. We have:
\begin{align}
    P^{-1} &= \Big{(}Q(I-\Lambda)Q^{T} \Big{)}^{-1}\\
    &= (Q^T)^{-1} (I-\Lambda)^{-1} Q^{-1}\\
    &= Q (I-\Lambda)^{-1} Q^{T}
\end{align}
So projection matrix is only invertible when $(I-\Lambda)$ is invertible. We will show next that $(I-\Lambda)$ is not invertible thus completing our proof. We have $I-\Lambda=diag(1-\lambda_1,\ldots,1-\lambda_n)$, hence:

\begin{align}
    (I-\Lambda)^{-1}=diag(\frac{1}{1-\lambda_1},\ldots,\frac{1}{1-\lambda_2})
\end{align}
Now, if one of the eigenvalues of $A$ is $1$, then the diagonal matrix $(I-\Lambda)$ is not invertible. If one of the eigenvalues of $A$ is $1$, then there exist an eigenvector $\bm{x}$ such that $A\bm{x} = \hat{\bm{w}}\hat{\bm{w}}^{T} \times \bm{x} =1\times \bm{x}$.  The vector $\bm{x}=\hat{\bm{w}}$ is the eigenvector of $A$ with eigenvalue $1$: $A\bm{\hat{w}} = \hat{\bm{w}}\hat{\bm{w}}^{T} \times \bm{\hat{w}} = 1 \times \hat{\bm{w}}$ since $\hat{\bm{w}}^{T} \times \bm{\hat{w}} =1$ as it is a unit vector. Hence the projection matrix is not invertible.

% We can easily see that $\hat{\bm{w}}$ is the eigen-vector with eigen-value $1$, and $\gamma \hat{\bm{w}}$ is the eigen-space corresponding to eigen value $1$ for some $\gamma \in \mathbb{R}$.
% \amit{writing advice: avoid phrases like "easily see". Instead write the required math or english statements to show that w is eigenvector.} 

\end{proof}

\paragraph{First Claim (1b).} For a probing classifier of form $c_{p}^{(1)}(\bm{z})=\bm{w}_{p}\cdot\bm{z}_{p} + \bm{w}_{m}\cdot\bm{z}_{m}$ for the first step of \INLP projection ---denoted by superscript (1)--- trained to predict \prop label $y_{p}$ and \refassm{assm:disentagled_latent},\ref{assm:fully_pred_inv} and \ref{assm:spurious_linear_separably} of \reflemma{lemma:sufficient_condition} is satisfied then we have $\bm{w}_{m}\neq \bm{0}$ i.e main task feature $\bm{z}_{m}$ will be used by probing classifier along with the \prop feature $\bm{z}_{p}$. For this second case, we are given that $\bm{w}_{p} = \bm{0}$ i.e probing classifier will not use \prop feature at all. This is only possible when the main-task feature is fully predictive of the \prop label i.e \refassm{assm:spurious_linear_separably} is satisfied for all the points in the dataset, otherwise optimal probing classifier will use the \prop-causal feature to achieve better margin and accuracy. Moreover, even if we assume \refassm{assm:spurious_linear_separably} is satisfied \new{for all the points in the dataset}, to have $\bm{w}_{p}=\bm{0}$, the margin achieved by probing classifier using only main task feature \new{($\bm{z}_{m}$)} should be bigger than any other classifier i.e one using both the main-task feature and probing feature or using probing feature alone. Thus, it is very unlikely that the optimal probing classifier will have $\bm{w}_{p}=\bm{0}$. 

Having said this, even in the case when we have $\bm{w}_{p}=\bm{0}$,  we show that the first projection step of \INLP will do something unintended,  i.e.,  damage the main-task features while having no effect on \prop-causal features which we intended to remove. First, we will show that main-task features will get damaged. From \refeqn{eq:main-feat-mix} we have:
\begin{align}
    \bm{z}_{m}^{i(1)} &= (I-\hat{\bm{w}}_{m}\hat{\bm{w}}_{m}^{T})\bm{z}_{m}^{i}  - (\hat{\bm{w}}_{p}\cdot \bm{z}_{p}^{i}) \hat{\bm{w}}_{m}\\
    &= \bm{z}_{m}^{i} -  (\hat{\bm{w}}_{m} \cdot \bm{z}_{m}^{i}) \hat{\bm{w}}_{m} - \bm{0} \quad \quad  \quad \quad \quad (\text{since  }  \bm{w}_{p}=\bm{0})
\end{align}
Since $\bm{w}_{m}\neq \bm{0}$ \new{and $\bm{z}_{m}^{i}$ can take any value (subject to \refassm{assm:spurious_linear_separably}), $\hat{\bm{w}}_{m} \cdot \bm{z}_{m}^{i}$ is not trivially zero for all the $\bm{z}_{m}^{i}$ $\Longrightarrow (\hat{\bm{w}}_{m} \cdot \bm{z}_{m}^{i}) \hat{\bm{w}}_{m} \neq \bm{0}$}. Thus, projected main-task feature  $\bm{z}_{m}^{i(1)} \neq \bm{z}^{i}_{m}$. In case $\bm{z}_{m}\in \mathbb{R}$, we have $\hat{\bm{w}}_{m}=\hat{\bm{z}}_{m}^{i}$, thus $(\hat{\bm{w}}_{m} \cdot \bm{z}_{m}^{i}) \hat{\bm{w}}_{m} = \bm{z}_{m}^{i}$. Consequently, $\bm{z}_{m}^{i(1)} = \bm{z}_{m}^{i} - \bm{z}_{m}^{i}  = \bm{0}$. Thus, first projection step of \INLP leads to complete removal of main-task feature $\bm{z}_{m}$  when $\bm{z}_{m}\in \mathbb{R}$. Also, from \reflemma{lemma:inlp_non_inv_p}, this projection step is non-invertible and hence the main-task feature cannot be recovered back in the subsequent projection step. 

Next, we will show the first projection step has no effect on the \prop-causal features which we wanted to remove in the first place. From \refeqn{eq:prop-feat-mix}, we have:
\begin{align}
    \bm{z}_{p}^{i(1)} &=(I-\hat{\bm{w}}_{p}\hat{\bm{w}}_{p}^{T})\bm{z}_{p}^{i}  - (\hat{\bm{w}}_{m}\cdot \bm{z}_{m}^{i}) \hat{\bm{w}}_{p}\\
    &=\bm{z}_{p}^{i} - \bm{0} -\bm{0} \quad \quad \quad \quad \quad \quad \quad \quad  \quad \quad \quad (\text{since  }  \bm{w}_{p}=\bm{0})
\end{align}
Thus the first step of projection has no effect on the \prop-causal feature we wanted to remove. In the next step of projection, if we again have $\bm{w}_{p}=\bm{0}$, then this same case will repeat. Otherwise if \refassm{assm:spurious_linear_separably} still holds for main-task feature for the margin points of optimal probing classifier $c_{p}^{*(2)}(\bm{z})$ for this second step of projection, then we will have both $\bm{w}_{m}\neq \bm{0}$ and $\bm{w}_{p} \neq \bm{0}$ and the first case of this theorem will apply. 
% \amit{this is cp*(z) for the next iteration, which may be different from the first one, right? In that case, it is good to have the notation $c_{p, i}^*(z)$ to denote the i=2 iteration. }

% \amit{have a separation in the proof for the first and second part/statement of the theorem. You can use paragraph command with all capital and bold header line.}

\paragraph{Second Claim.}
Now for proving the second  statement, we will make use of the following lemma. The proof of the lemma is given below the proof of this theorem. 
\begin{lemma}
\label{lemma:inlp_norm_removal}
After every projection step of \INLP, the norm of every latent representation $\bm{z}^{i}$ decreases, i.e., $\norm{\bm{z}^{i(k)}}<\norm{\bm{z}^{i(k-1)}}$ for step $k$ and $k-1$,  if (1) $\bm{z}^{i(k-1)}\neq \bm{0}$, (2) $\bm{z}_{\hat{\bm{w}}^{k}}^{i(0)}\neq \bm{0}$ and (3) the parameters of probing classifier in step ``$k$''  i.e $\bm{\hat{w}}^{k}$ don't lie in the space spanned by parameters of previous probing classifier, span($\bm{\hat{w}}^{1},\ldots,\bm{\hat{w}}^{k-1}$).
\end{lemma}

% \abhinav{Directly complete the main result saying the main theorem satisfies the assumption of this lemma}

Next, we will show that starting from the first step, at every $k^{\text{th}}$-step of projection either we will have $\bm{z}^{i(k)}=\bm{0}$ or the norm will decrease after projection. Once we reach a step when $\bm{z}^{i(k)}=\bm{0}$, then after every subsequent projection we will have $\bm{z}^{i(k+1)}=\bm{0} \implies \norm{\bm{z}^{i(k+1)}}=0$ since:
\begin{equation}
\label{eq:norm_zero_next}
    \bm{z}^{i(k+1)}=P_{N(\bm{w}^{k})}\bm{z}^{i(k)} = P_{N(\bm{w}^{k})}\bm{0}=\bm{0}
\end{equation}
where $P_{N(\bm{w}^{k})}$ is the projection matrix at step "k". Since $\norm{\cdot}\geq0$ and $\norm{\bm{z}^{i(k)}}$ is decreasing with every stey, thus with large number of $\bm{z}^{i(\infty)}\rightarrow\bm{0}$.

We will start with the first step of projection. In the second statement of this \reftheorem{theorem:null_space_failure}, we are given that $\bm{z}^{i(0)}$ is not trivially zero in direction of $\bm{w}^{0}$ i.e $\bm{z}^{i(0)}_{\bm{w}^{0}}\neq \bm{0}$ (satisfying Assm(2) of above \reflemma{lemma:inlp_norm_removal}). We are also given that $\bm{z}^{i(0)}\neq \bm{0}$ (satisfying Assm(1) of above lemma) and since this is the first step of projection Assm(3) of above \reflemma{lemma:inlp_norm_removal}) is also satisfied. Thus from \reflemma{lemma:inlp_norm_removal}, we have $\norm{\bm{z}^{i(1)}}<\norm{\bm{z}^{i(0)}}$. Now, either $\bm{z}^{i(1)}=\bm{0}$, which will imply that $\norm{\bm{z}^{i(1)}}=0$ and will remain $0$ for all subsequent step (from \refeqn{eq:norm_zero_next}). Otherwise if $\bm{z}^{i(1)}\neq \bm{0}$, it satisfies the Assm(1)  of \reflemma{lemma:inlp_norm_removal}, for next step of projection. Since Assm (2) and (3) are already satisfied \new{(from the assumption in the second claim of \reftheorem{theorem:null_space_failure})}, again we will have $\norm{\bm{z}^{i(2)}}<\norm{\bm{z}^{i(1)}}$ and the same idea will repeat eventually making $\bm{z}^{i(k)}=\bm{0}$ at some step-k, thus completing our proof.

% Given that parameters of the \prop-probing classifier in the current step do not lie in the span of previous probing classifier's parameters, the norm of the latent representation is decreasing in every step/iteration. Thus given large number of \INLP projection iterations, $\norm{\bm{z}^{i(\infty})}\rightarrow \bm{0}$. This  completes the proof of the second statement. 
% \amit{this proof is not clear.why is norm decreasing?how is the intermediate lemma used?what assumptions are coming directly from the second claim statement?}
%The proof of \reflemma{lemma:inlp_non_inv_p} and \reflemma{lemma:inlp_norm_removal} is given below.
\end{proof}

\begin{proof}[Proof of \reflemma{lemma:inlp_norm_removal}]
After $(k-1)$-steps  of \INLP let the latent space representation $\bm{z}^{i}$ be denoted as  $\bm{z}^{i(k-1)}$. Let $\bm{\hat{w}}^{k}$ be the parameters of classifier $c_{p}(\bm{z}^{k-1})$ trained to predict the \prop label $y_{p}$ which we want to remove at step $k$. Then prior to the projection step in the $k^{th}$ iteration of the \INLP, we can write $\bm{z}^{i(k-1)}$ as:
\begin{equation}
    \bm{z}^{i(k-1)}_{B}=\bm{z}_{\bm{\hat{w}}^{k}}^{i(k-1)} \bm{\hat{w}}^{k} + \bm{z}_{N(\bm{\hat{w}}^{k})}^{i(k-1)} N(\bm{\hat{w}}^{k})
\end{equation}
where $B=\{\bm{\hat{w}}^{k},N(\bm{\hat{w}}^{k})\}$ is the basis set and $N(\bm{\hat{w}}^{k})$ is the null-space of $\bm{\hat{w}}^{k}$ . The parameter $\hat{\bm{w}}^{k}$ in this new basis is:
\begin{equation}
    \hat{\bm{w}}^{k}_{B} = I_{\hat{\bm{w}}^{k}} \hat{\bm{w}}^{k} + \bm{0} N(\bm{\hat{w}}^{k})
\end{equation}

where $I_{\hat{\bm{w}}^{k}}$ is identity matrix with dimension $d_{\hat{\bm{w}}^{k}} \times d_{\hat{\bm{w}}^{k}}$. 
%Since we have are transforming from one orthogonal basis ($A=\{\hat{\bm{z}}_{m}^{k-1},\hat{\bm{z}}_{p}^(k-1)\}$) to another orthogonal basis ($B=\{\bm{\hat{w}}^{k},N(\bm{\hat{w}}^{k})\}$), the basis-transformation matrix is given by orthonormal matrix $U$ such that $U^{T}U=I$ (Section 5.5 in \cite{arfken_math_methods}). Thus the parameter vector in new basis is given by $\hat{\bm{w}}^{k}_{B} = U\cdot\hat{\bm{w}}^{k}_{A}$ and representation vector in the new basis is given by $\hat{\bm{z}}^{k}_{B} = U\cdot\hat{\bm{z}}^{k}_{A}$ where $\hat{\bm{w}}^{k}_{A}$ and $\hat{\bm{z}}^{k}_{A}$ are parameter and latent vector in basis $A$. 
Now, in the new basis when we project the $\bm{z}^{k-1}$ to the null space of $\bm{\hat{w}}^{i(k)}$ we have:
\begin{align}
   \bm{z}^{i(k)} &=  P_{N(\bm{\hat{w}}^{k})} \bm{z}^{i(k-1)}\\
   \bm{z}^{i(k)}_{B} &= \big(I-\hat{\bm{w}}_{B}^{k}(\hat{\bm{w}}_{B}^{k})^{T})\big) \bm{z}_{B}^{i(k-1)}\\
   &= \Bigg{(}I - \begin{bmatrix}
             I_{\hat{\bm{w}}^{k}}\\
            \bm{0}\\
        \end{bmatrix}
        \begin{bmatrix}
            I_{\hat{\bm{w}}^{k}} & \bm{0} \\
        \end{bmatrix}\Bigg{)}
        \begin{bmatrix}
            \bm{z}_{\bm{\hat{w}}^{k}}^{i(k-1)}\\
           \bm{z}_{N(\bm{\hat{w}}^{k})}^{i(k-1)}\\
        \end{bmatrix}\\
    &= \begin{bmatrix}
            \bm{z}_{\bm{\hat{w}}^{k}}^{i(k-1)}\\
           \bm{z}_{N(\bm{\hat{w}}^{k})}^{i(k-1)}\\
        \end{bmatrix} - \begin{bmatrix}
            \bm{z}_{\bm{\hat{w}}^{k}}^{i(k-1)}\\
           \bm{0}\\
        \end{bmatrix}\\
    &=\label{eq:zero_post_null_space_projection} \begin{bmatrix}
            \bm{0}\\
           \bm{z}_{N(\bm{\hat{w}}^{k})}^{i(k-1)}\\
        \end{bmatrix}
  % &= \Big{(}I-\bm{\hat{w}}^{k}(\bm{\hat{w}}^{k})^{T}\Big{)} \Big{(}\bm{z}_{\bm{\hat{w}}^{k}} \bm{\hat{w}}^{k} + \bm{z}_{N(\bm{\hat{w}}^{k})} N(\bm{\hat{w}}^{k})\Big{)}\\
   %& = \bm{z}_{\bm{\hat{w}}^{k}} \bm{\hat{w}}^{k} + \bm{z}_{N(\bm{\hat{w}}^{k})} N(\bm{\hat{w}}^{k}) - \bm{z}_{\bm{\hat{w}}^{k}} \bm{\hat{w}}^{k} ()\\
   %&= \bm{x}_{N(\bm{\hat{w}}^{p})} N(\bm{\hat{w}}^{p})
\end{align}
Thus the norm of $\norm{\bm{z}^{i(k)}}=\sqrt{\norm{\bm{z}^{i(k-1)}_{N(\bm{\hat{w}}^{k})}} + 0}$ is less than $\norm{\bm{z}^{k-1}}=\sqrt{\norm{\bm{z}^{i(k-1)}_{\bm{\hat{w}}^{k}}}^{2}+\norm{\bm{z}^{i(k-1)}_{N(\bm{\hat{w}}^{k})}}^{2}}$ if $\bm{z}^{i(k-1)}_{\bm{\hat{w}}^{k}}\neq 0$. Next we will show that $\bm{z}^{i(k-1)}_{\bm{\hat{w}}^{k}}$ cannot be zero. From assumption (2) in \ref{lemma:inlp_norm_removal} $\bm{z}^{i(0)}_{\bm{w}^{k}}\neq \bm{0}$ i.e $\bm{z}^{i(0)}_{\bm{w}^{k}}$ is not trivially zero in the given latent representation $\bm{z}^{i(0)}$ before any projection from \INLP, thus $\bm{z}^{i(k-1)}_{\bm{\hat{w}}^{k}}$ is not trivially zero from beginning. Also, from \refeqn{eq:zero_post_null_space_projection}, we observe that at any step ``$k$'' \INLP removes the part of the representation from $\bm{z}^{i(k-1)}$ which is in the direction of $\bm{\hat{w}}^{k}$ i.e. $\bm{z}^{i(k-1)}_{\bm{\hat{w}}^{k}}$. Consequently, a sequence of removal steps with parameters $\bm{\hat{w}}^{1},\ldots,\bm{\hat{w}}^{k-1}$ will remove the part of $\bm{z}$ which lies in the span($\bm{\hat{w}}^{1},\ldots,\bm{\hat{w}}^{k-1}$). Thus $\bm{z}^{i(k-1)}_{\bm{\hat{w}}^{k}}= \bm{0}$ if $\bm{\hat{w}}^{k}$ lies in the span of parameters of previous classifier i.e span($\bm{\hat{w}}^{1},\ldots,\bm{\hat{w}}^{k-1}$) which violates the assumption $(3)$ in \reflemma{lemma:inlp_norm_removal}. Thus $\bm{z}^{i(k-1)}_{\bm{\hat{w}}^{k}}$ is neither trivially zero from the beginning nor it could have been removed in the previous steps of projection as long as the assumption in \reflemma{lemma:inlp_norm_removal} is satisfied, which completes our proof of the lemma.

% \abhinav{This doesn't show iff for last line. Make it more precise later.}

% \amit{the motivation of this para was not clear. I'd suggest first stating that, "Next we show that zwk cannot be zero" before going to Eqn 84. Also the last line can be reworded: "which is not the case" to "which violates assumption 2) in the Lemma's statement."}

% \amit{to be consistent, use k as the step id since it is used above}
%Thus if none of the \prop-probing classifier from previous step had parameters $=\gamma \bm{\hat{w}}^{k}$ for some $ \gamma \in \mathbb{R}$, then $\bm{z}_{\bm{\hat{w}}}^{k}$ wouldn't have been removed in previous step and hence $\norm{\bm{z}_{\bm{\hat{w}}}^{k}}\neq 0$.

% \abhinav{should we keep the next part? They conjecture why the parameter of current iteration will not lie in the span of parameters of previous steps classifier} \amit{yes, keep it as a remark just after proof.}

\end{proof}

\begin{remark}

The following lemma from \cite{NullItOut:2020} tells us some of the sufficient conditions when the parameters of the probing classifier at the current iteration will not be same as the previous step.

\begin{lemma}[Lemma A.1 from \cite{NullItOut:2020}]
\label{lemma:a1_null_space_paper}
If the \prop-probing classifier is being trained using SGD (stochastic gradient descent) and the loss function is convex, then parameters of the probing classifier at step $k$, $\bm{w}^{k}$,  are orthogonal to parameters at step $k-1$, $\bm{w}^{k-1}$.
\end{lemma}
We conjecture that \reflemma{lemma:a1_null_space_paper} will be true for any loss since after $k-1$ steps of projection, the component of $\bm{z}$ in the direction of span($\bm{w}^{1},\ldots,\bm{w}^{k-1}$) will be removed. Hence the \prop-probing classifier at step $k$ should be orthogonal to span($\bm{w}^{1},\ldots,\bm{w}^{k-1}$) in order to have non-random guess accuracy on probing task.

\end{remark}

\section{Adversarial Removal: Setup and Proof}
\label{sec:app_adv_setup_proof}

\subsection{Adversarial Setup}
\label{subsec:app_adv_setup}
As described in \refsec{subsec:adv_rem_problem}, let $h:\bm{X}\rightarrow \bm{Z}$ be an encoder mapping the input $\bm{x}$ to latent representation $\bm{z}$. The main task classifier $c_{m}:\bm{Z}\rightarrow Y_{m}$ is applied on top of $\bm{z}$ to predict the main task label $y_{m}$ for every input $\bm{x}$. 
% \amit{Again, it looks like a new addition. rather, connect to main paper, "As described in section X of main paper, the goal of "}
As described in \refsec{subsec:adv_rem_problem}, the goal of an adversarial removal method, henceforth referred as \ar, is to remove any undesired/sensitive/spurious \prop from the latent representation $\bm{z}$. Once the \prop is removed from the latent representation, any (main-task) classifier using the latent representation $\bm{Z}$ will not be able to use it \cite{AdvDomAdapGanin,AdvRemNeubig,AdvRemYoav}. These methods jointly train the main-task classifier $c_{m}(\bm{z})$ and the probing classifier $c_{p}:\bm{Z}\rightarrow Y_{p}$. The probing classifier is adversarially trained to predict the \prop label $y_{p}^{i}$ from latent representation $\bm{z}^{i}$. Hence, \ar methods optimize the following two objectives simultaneously:
% \amit{need to change notation in the equation. It says that h is learnable, but even in adv section of main paper, we say that h is frozen and we are only allowed to train a few layers on top of h. Need to introduce h2, either here or in a followup equation to show our custom formulation.}
\begin{align}
    & arg\min_{c_{p}} L(c_{p}(h(\bm{x})),y_{p})\label{eq:adv_prob_loss}\\
    & arg\min_{h,c_{m}} \Big{\{}L(c_{m}(h(\bm{x}),y_{m}) - L(c_{p}(h(\bm{x}),y_{p}) \Big{\}} \label{eq:adv_enc_loss}
\end{align}

Here $L(\cdot)$ is a loss function which estimates the error given the ground truth $y_{m}/y_{p}$ and corresponding prediction $c_{m}(\bm{z})/c_{p}(\bm{z})$. The above adversarial objective between the encoder and probing classifier is a min-max game. The encoder wants to learn a latent representation $\bm{z}$ s.t. it maximize the loss of probing classifier but at the same time probing classifier tries to minimize it's loss. The desired solution and simultaneously a valid equilibrium point of the above min-max objective is an  encoder $h$ such that it removes all the features from latent space that are useful for prediction of $y_p$ while keeping intact other features causally derived from the main task prediction. 
% \amit{what about cm and cp, they are also in arg min?}
In practice, the optimization to above objective is performed using a gradient-reversal (GRL) layer (\cite{AdvDomAdapGanin}). It introduces an additional layer $g_{\lambda}$ between the latent representation $h(\bm{z})$ and the adversarial classifier $c_{p}(\bm{z})$. The $g_{\lambda}$ layer acts as an identity layer (i.e., has no effect) during the forward pass but scales the gradient by $(-\lambda)$ when back-propagating it during the backward pass. Thus resulting combined objective is:
\begin{equation}
\label{eq:adv_removal_actual_obj}
    arg\min_{h,c_{m},c_{p}} \Big{\{}L(c_{m}(h(\bm{z})),y_{m})+L(c_{p}(g_{\lambda}(h(\bm{z}))),y_{p})\Big{\}}
\end{equation}

\paragraph{Setup for theoretical result:}
As stated in \reftheorem{theorem:adv_removal_failure}, for our theoretical result showing the failure mode of adversarial removal, we assume that the encoder is divided into two sub-parts. The first encoder i.e $h_{1}:\bm{X}\rightarrow \bm{Z}$ is frozen (non-trainable) and maps the input $\bm{x}^{i}$ to intermediate latent representation $\bm{z}^{i}$ which is frozen and disentangled (\refassm{assm:disentagled_latent}). The second encoder $h_{2}:\bm{Z}\rightarrow \bm{\zeta}$ is a linear transformation mapping the intermediate latent representation $\bm{z}^{i}$ to final latent representation $\bm{\zeta}^{i}$  and is trainable. On top of this final latent representation $\bm{\zeta}^{i}$, we train the main task classifier $c_{m}(\bm{\zeta}^{i})$ and probing classifier $c_{p}(\bm{\zeta}^{i})$. Thus the training objective from \refeqn{eq:adv_prob_loss} and \ref{eq:adv_enc_loss} becomes:
\begin{align}
    % arg\min_{h_{2},c_{m},c_{p}} \Big{\{}L(c_{m}(h_{2}(\bm{\zeta})),y_{m})+L(c_{p}(g_{\lambda}(h_{2}(\bm{\zeta}))),y_{p})\Big{\}}
    & arg\min_{c_{p}} L(c_{p}(h_2(\bm{z})),y_{p})\label{eq:adv_prob_loss_theory}\\
    & arg\min_{h_2,c_{m}} \Big{\{}L(c_{m}(h_{2}(\bm{z}),y_{m}) - L(c_{p}(h_{2}(\bm{z}),y_{p}) \Big{\}} \label{eq:adv_enc_loss_theory}
\end{align}

\subsection{Adversarial Proof}
\label{subsec:app_adv_proof}
For a detailed discussion of adversarial training objective and specific setup for our theoretical result refer \refsec{subsec:app_adv_setup}.
% \amit{good to restate the theorem here. For people reading on computer, it is hard to go back and forth between pages. I've added the restatable package but you need to copy and edit for the previous theorem.}

We formally state the new assumption made in the second statement of \reftheorem{theorem:adv_removal_failure}. This assumption imposes constraints on strength of correlation between main task label and \prop-causal feature  i.e it requires the \prop-causal feature to be more predictive of main task label than for probing task. 
\begin{assumption}[Strength of Correlation]
\label{assm:correlation_strength}
% Let $h_{2}^{*}$ be the desired encoder successful in removing $\bm{z}_{p}$ from $\bm{\zeta}$. Let $\bm{z}^{M}$ be the margin point of main-task classifier ($c_m$) and $\bm{z}^{P}$ be the margin point of probing classifier ($c_{p}$) in the latent-representation space $\bm{Z}$ when using desired encoder $h_{2}^{*}$.
Let $\hat{\bm{w}}_{p}\in \mathbb{R}^{d_{p}}$ be the unit vector s.t. $\bm{z}_{p}$ is linearly separable for \prop-label $y_{p}$ (see \refassm{assm:fully_pred_inv}) \new{and let $h_{2}^{*}(\bm{z})$ be the desired encoder which is successful in removing the \prop-causal features $\bm{z}_{p}$ from $\bm{\zeta}$}. Then, \prop-causal features $\bm{z}_{p}^{M}$ of any margin point $\bm{z}^{M}$ of $c_{m}(h_{2}^{*}(\cdot))$ is more predictive of the main task than \prop-causal features $\bm{z}_{p}^{P}$ of any margin-point $\bm{z}^{P}$ of $c_{p}(h_{2}^{*}(\cdot))$ for probing task by a factor of $|\beta|\in\mathbb{R}$ where $|\beta|$ is the  norm of parameter of probing classifier $c_{p}\new{(h_{2}^{*}(\cdot))}$ i.e $y_{m}(\hat{\bm{w}}_{p}\cdot\bm{z}_{p}^{M}) > |\beta|y_{p}(\hat{\bm{w}}_{p}\cdot\bm{z}_{p}^{P})$.
\end{assumption}

\advthm*

\begin{proof}[Proof of \reftheorem{theorem:adv_removal_failure}]
%\amit{a proof cannot include new assumptions. the first line can be removed. Then when you restrict w to be 1 dimensional, need a justification on why it is wlog}
Let the main classifier be of the form $c_{m}(\bm{\zeta})=\bm{w}_{c_{m}} \cdot \bm{\zeta}$ where $\bm{w}_{c_{m}}$ and $\bm{\zeta}$ are $d_{\bm{\zeta}}$ dimensional vectors. %where $d_{\bm{\zeta}}$ is the dimension of $\bm{\zeta}$. 
Since both parameters $\bm{w}_{c_{m}}$ and $\bm{\zeta}$ are learnable, for ease of exposition we constrain  $\bm{w}_{c_{m}}$ to be $[1,0,\ldots,0]$. This constraint on $\bm{w}_{c_{m}}$ is w.l.o.g. since $\bm{w}_{c_{m}}$ makes the prediction for main-task by projecting $\bm{\zeta}$ into one specific direction to get a scalar $(\bm{w}_{c_{m}}\cdot\bm{\zeta})$. We constrain that direction to be the first dimension of $\bm{\zeta}$. Since the encoder $h_{2}(\bm{z}):\bm{Z}\rightarrow \bm{\zeta}$ is trainable it could learn to encode the scalar $(\bm{w}_{c_{m}}\cdot\bm{\zeta})$ in the first dimension of $\bm{\zeta}$.  
%\amit{"This constraint on $\bm{w}$ is wlog for a linear class for $h_2$ because...."}
% \amit{can we claim it is wlog? is it wlog only within the space of linear h2 and cm?}
Thus, effectively a single dimension of the representation $\bm{\zeta}$ encodes the main-task information.
As a result, the main classifier is effectively of the form $c_{m}(\bm{\zeta})=\bm{w}_{c_{m}}^{(0)}\times \bm{\zeta}^{(0)} = 1 \times \bm{\zeta}^{(0)}$ where  $\bm{w}_{c_{m}}^{(0)}$ and $\bm{\zeta}^{(0)}$ are the first elements of $\bm{w}_{c_{m}}$ and $\bm{\zeta}$ respectively and $\bm{w}_{c_{m}}^{(0)}=1$. 
We can now write the goal of the adversarial method as removing the information of $\bm{z}_{p}$ from $\bm{\zeta}^{(0)}$ because the other dimensions are not used by the main classifier.
% \amit{", because the other dimensions are not used by the main classifier"}.
Also, the adversarial \new{probing} classifier can be written effectively as $c_{p}(\bm{\zeta})=\beta \times \bm{\zeta}^{(0)}$ where $\beta \in \mathbb{R}$ is a trainable parameter. Since both the main and adversarial classifier are using only $\bm{\zeta}^{(0)}$, the second encoder, with a slight abuse of notation, can be simplified as $\zeta \coloneqq \bm{\zeta}^{(0)} \coloneqq h_{2}(\bm{z})= \bm{w}_{m}\cdot \bm{z}_{m} + \bm{w}_{p} \cdot \bm{z}_{p}$ ,
% \abhinav{explicitly say with abuse of notation}
% \amit{in thm notation, $\bm{\zeta}=h2$ but here $\bm{\zeta}^0=h2$. Use $h2^0$ everywhere to make notation consistent.I suggest the same for $wm^0$ and $wp^0$}
where $\zeta \in \mathbb{R}$ and $\bm{w}_{m}$ and $\bm{w}_{p}$ are the weights that determine the first dimension of $\bm{\zeta}$. 
% \amit{"and $\bm{w}_{m}$ and $\bm{w}_{p}$ are the weights that determine the first dimension of $\bm{\zeta}$."}. 
Also, let the desired (correct) second encoder which is successful in removing the \prop-causal feature $\bm{z}_{p}$ from $\zeta$ be   $ \zeta^{*} \coloneqq \bm{\zeta}^{*(0)} \coloneqq h_2^{*}(\bm{z})=\bm{w}^{*}_{m}\cdot \bm{z}_{m}$. Thus using \refeqn{eq:adv_prob_loss_theory} and \ref{eq:adv_enc_loss_theory}, our overall objective for adversarial removal method becomes:
\begin{align}
    & arg\min_{\beta} L(c_{p}(h_{2}(\bm{z})),y_{p})\label{eq:adv_prob_loss_theoryf}\\
    & arg\min_{h_2} \Big{\{}L(c_{m}(h_{2}(\bm{z}),y_{m}) - L(c_{p}(h_{2}(\bm{z}),y_{p}) \Big{\}} \label{eq:adv_enc_loss_theoryf}
\end{align}

\paragraph{1. First claim.} 

The ideal main classifier with desired encoder can be written as,  $c_{m}(\zeta^{*}) = 1\times  h^{*}_{2}(\bm{z}) = \bm{w}_{m}^{*}\cdot \bm{z}_{m}$. Therefore, it can be trained using  the \emph{MM-Denominator} formulation of the  max-margin objective and would satisfy the constraint in \refeqn{eq:max_margin_denominator_constraint}:
\begin{equation}
\label{eq:adv_margin_h*}
    m(c_{m}(\zeta^{*i}))=m(h^{*}_{2}(\bm{z}^{i}))= y_{m}^{i}\cdot h_{2}^{*}(\bm{z}^{i})\geq 1
\end{equation}
for all the points $\bm{x}^{i}$ with latent representation $\bm{z}^{i}$ and $m(\cdot)$ is the numerator of the distance of point from the decision boundary of classifier (\refeqn{eq:margin_distance}).
% \amit{notation is getting inconsistent again. I'd suggest use parenthesis when referring to first dimension of zeta, $\zeta^{(0)}$ or $h2^{(0)}$, and use this standard superscript for the row of data}

However, the main task classifier which does not  use the desired encoder is of the form, $c_{m}(\zeta) = 1 \times h_{2}^{\new{\alpha}}(\bm{z}) = \bm{w}_{m}\cdot \bm{z}_{m} + \bm{w}_{p}\cdot \bm{z}_{p}$. 
%\amit{why cm(zeta) here but not for ideal classifier above. be consistent.} 
Since this main task classifier is trained using max-margin objective by \emph{MM-Denominator} formulation, it would satisfy the constraint in \refeqn{eq:max_margin_denominator_constraint}:
\begin{equation}
\label{eq:adv_margin_h}
    m(c_{m}(\zeta^{i})) = m(h_{2}^{\new{\alpha}}(\bm{z}^{i})) = y_{m}^{i}\cdot h_{2}^{\new{\alpha}}(\bm{z}^{i}) \geq 1
\end{equation}

% Since all assumptions of \reflemma{lemma:sufficient_condition_main_task} are satisfied, we use the lemma to conclude that the main-task classifier trained using max-margin objective will be  \emph{spurious-using} i.e will use the \prop-causal feature $\bm{z}_{p}$. 
Since in our case, main task classifier is the same as the encoder i.e $c_{m}(\zeta) = 1 \times h_{2}^{\new{\alpha}}(\bm{z}) = \bm{w}_{m}\cdot\bm{z}_{m} + \bm{w}_{p}\cdot\bm{z}_{p}$, and the latent representation $\bm{z}$ satisfies the \refassm{assm:disentagled_latent}, \ref{assm:fully_pred_inv_main_task} and \ref{assm:spurious_linear_separably_main_task}, from \reflemma{lemma:sufficient_condition_main_task} the main-task classifier is \emph{spurious-using} i.e $\bm{z}_{p}\neq \bm{0}$. Hence there exists an undesired/incorrect encoder $h_{2}^{\new{\alpha}}(\bm{z})$ such that the main classifier $c_{m}(\zeta)=h_{2}^{\new{\alpha}}(\bm{z})$ has bigger margin than $c_{m}(\zeta^{*})=h_{2}^{*}(\bm{z})$. 
Next, we show that the accuracy of the adversarial classifier remains the same irrespective of whether the desired ($h^{*}_{2}(\bm{z})$) or undesired encoder $h_2^{\new{\alpha}}(\bm{z})$ is used. The accuracy of the adversarial classifier $c_{p}(\zeta)=\beta\times \zeta$, using the desired/correct encoder $\zeta = h_{2}^{*}(\bm{z})$ is given by:
\begin{equation}
    Accuracy(c_{p}(\zeta^{*}),y_{p}) = \frac{\sum_{i=1}^{N} \mathbf{1}\Big{(} sign\big{(}\beta \cdot h_{2}^{*}(\bm{z}^{i})\big{)} == y_{p}^{i} \Big{)}}{N}
\end{equation}
where $\mathbf{1}(\cdot)$ is an indicator function which takes the value $1$ if the argument is true otherwise $0$, and $sign(\gamma)=+1$ if $\gamma\geq0$ and $-1$ otherwise. Combining \refeqn{eq:adv_margin_h*} and \refeqn{eq:adv_margin_h},  since $y_{m}^{i}\in \{-1,1\}$, we see that whenever $h_{2}^{\new{\alpha}}(\bm{z}^{i})>1$ we also have $h_{2}^{*}(\bm{z}^{i})>1$ and similarly whenever $h_{2}^{\new{\alpha}}(\bm{z}^{i})<-1$, we have $h_{2}^{*}(\bm{z}^{i})<-1$. Thus, 
\begin{equation}
\label{eq:adv_same_sign_h_h*}
    h_{2}^{\new{\alpha}}(\bm{z})\cdot h_{2}^{*}(\bm{z})>0
\end{equation}

From \refeqn{eq:adv_same_sign_h_h*}, $h_{2}^{*}(\bm{z}^{i})$ and $h_{2}^{\new{\alpha}}(\bm{z}^{i})$ has the same sign for every input $\bm{z}^{i} \implies sign(\beta\cdot h_{2}^{*}(\bm{z}^{i})) =  sign(\beta\cdot h_{2}^{\new{\alpha}}(\bm{z}^{i}))$. Thus we can replace $h_{2}^{*}(\bm{z}^{i})$ with $h_{2}^{\new{\alpha}}(\bm{z}^{i})$ in the above equation and we have:
\begin{align*}
    Accuracy(c_{p}(\zeta^{*}),y_{p}) &= \frac{\sum_{i=1}^{N} \mathbf{1}\Big{(} sign\big{(}\beta \cdot h_{2}^{\new{\alpha}}(\bm{z}^{i})\big{)} == y_{p}^{i} \Big{)}}{N}\\
    \new{Accuracy(c_{p}(h_2^{*}(\bm{z^{i}})),y_{p})} &= \new{Accuracy(c_{p}(h_2^{\alpha}(\bm{z^{i}})),y_{p})}
\end{align*}
thus completing the first part our proof.
%\amit{the first equation (rhs on first line) above is repetitive and can be removed}

\paragraph{2. Second claim.}
Since we are training  both the main task and the  probing classifier with a max-margin objective (see MM-Numerator version at \refeqn{eq:max_margin_numerator}), we can effectively write the adversarial objective (from \ref{eq:adv_prob_loss_theoryf} and \ref{eq:adv_enc_loss_theoryf}) as:
\begin{align}
    arg\max_{\beta}(P(\beta)) \coloneqq \>\> & arg\max_{\beta} \Big{\{} \min_{\bm{z}^{i}} {m}_{c_{p}}(h_{2}(\bm{z}^{i})) \Big{\}} \label{eq:adv_prob_loss_theoryfmm}\\
   arg\max_{h_2}(E(h_{2})) \coloneqq \>\> & arg\max_{h_2} \Big{\{} \min_{\bm{z}^{i}} {m}_{c_{m}}(h_{2}(\bm{z}^{i})) - \min_{\bm{z}^{i}}{m}_{c_{p}}(h_{2}(\bm{z}^{i})) \Big{\}} \label{eq:adv_enc_loss_theoryfmm}
\end{align}
where $m_{c_{m}}(h_{2}(\bm{z}))$ and $m_{c_{p}}(h_{2}(\bm{z}))$ are the numerator of margin of a point (\refeqn{eq:margin_distance}). Next, our goal is to show that the desired encoder $h^{*}_{2}$ is not an equilibrium point of the above adversarial objective. To do so, we will create an undesired/incorrect encoder $h_{2}^{\new{\alpha}}(\bm{z})$ by perturbing $h_{2}^{*}$ by small amount and showing that the combined encoder objective $E(h_{2}^{\new{\alpha}})>E(h_{2}^{*})$ (\refeqn{eq:adv_enc_loss_theoryfmm}) irrespective of choice of $\beta$ chosen by the probing objective $P(\beta)$ (\refeqn{eq:adv_prob_loss_theoryfmm}).

\paragraph{Construction of the undesired/incorrect encoder.}
We have $h_{2}^{*}(\bm{z})=\norm{\bm{w}_{m}}(\hat{\bm{w}}_{m}^{*}\cdot \bm{z}_{m})$ where $\hat{\bm{w}}_{m}^{*}\in \mathbb{R}^{d_m}$ is a unit vector. We will perturb this desired encoder by parameterizing with $\alpha\in[0,1)$ s.t.:
\begin{equation}
\label{eq:undesired_encoder}
    h_{2}^{\alpha}(\bm{z}) = \alpha \norm{\bm{w}_{m}^{*}}(\hat{\bm{w}}_{m}^{*}\cdot \bm{z}_{m}) + \sqrt{1-\alpha^{2}} \norm{\bm{w}_{m}^{*}} (\hat{\bm{\epsilon}}_{p}\cdot \bm{z}_{p})
\end{equation}
where $\hat{\bm{\epsilon}}_{p}\in\mathbb{R}^{d_{p}}$ is a unit vector. The clean main-task classifier is defined as $c_{m}^{*}(h_{2}^{*}(\bm{z}))=h_{2}^{*}(\bm{z})$. The main-task classifier $c_{m}$ when using the incorrect encoder takes form $c_{m}(h_{2}^{\new{\alpha}}(\bm{z}))=h_{2}^{\new{\alpha}}(\bm{z})$. As stated in the theorem statement, all the assumptions of \reflemma{lemma:sufficient_condition_main_task} are satisfied. Since \refassm{assm:spurious_linear_separably_main_task} (one of the assumptions of \reflemma{lemma:sufficient_condition_main_task}) are satisfied, there exists a unit-vector in $\mathbb{R}^{d}_{p}$ such that \prop-causal features of margin points of the  main task classifier using encoder $h_{2}^{*}$ are linearly separable w.r.t main-task label. Let $\hat{\bm{\epsilon}}_{p}$ in our constructed undesired encoder $h_{2}^{\alpha}$ (\refeqn{eq:undesired_encoder}) be set to that unit vector such that:
\begin{equation}
\label{eq:adv_margin_sep_xt}
    y_{m}^{M} \cdot (\hat{\bm{\epsilon}}_{p}\cdot \bm{z}_{p}^{M}) > 0
\end{equation}
where $\bm{z}_{p}^{M}$ is the \prop-causal feature of margin point $\bm{z}^{M}$ of the main-task classifier when using encoder $h_{2}^{*}$. Now since all the assumption of \reflemma{lemma:sufficient_condition_main_task} is satisfied, the margin of main-task classifier when using undesired encoder $h_{2}^{\alpha}(\bm{z})$ is bigger than when desired encoder $h_{2}^{*}$ is used for some $\alpha\in(\alpha_{lb}^{1},1)$. Consequently, we have:
\begin{equation}
\label{eq:main_margin_delta}
    m_{c_{m}}(h_{2}^{\alpha}(\bm{z}^{M})) > m_{c_{m}}(h_{2}^{*}(\bm{z}^{M}))
\end{equation}
where $\bm{z}^{M}$ is the margin point of $c_{m}(h^{*}_{2})$. %and by definition $m_{c_{m}}(h_{2}^{*}(\bm{z}^{M}))=1$.
Since \refassm{assm:fully_pred_inv} is satisfied, we have a fully predictive \prop-causal feature $\bm{z}_{p}$ for prediction of adversarial label $y_{p}$ such that for some unit vector $\hat{\bm{w}}_{p}\in \mathbb{R}^{d_{p}}$ we have:
\begin{equation}
\label{eq:adv_fully_pred_xt}
    y_{p}^{i}\big{(} \hat{\bm{w}}_{p}\cdot \bm{z}^{i}_{p} \big{)} > 0 \>\>\> \forall  (\bm{z}^{i},y_{p}^{i})
\end{equation}
Next, since \refassm{assm:label_correlation} is also satisfied for this second part of theorem, we have $y_{p}^{i}=y_{m}^{i}$ for every margin point of the desired/correct main-task classifier using the desired/correct encoder $h_{2}^{*}(\bm{z})$. Thus we can assign $\hat{\bm{\epsilon}_{p}}\coloneqq \hat{\bm{w}}_{p}$ which satisfies the inequality in \refeqn{eq:adv_margin_sep_xt}. Hence, our \emph{incorrect} encoder $h_{2}^{\alpha}(\bm{z})$ take the following form:
\begin{equation}
     h_{2}^{\alpha}(\bm{z}) =  \alpha\norm{\bm{w}^{*}_{m}}\big{(}\bm{\hat{w}}^{*}_{m}\cdot \bm{z}_{m}\big{)} + \sqrt{1-\alpha^{2}}\norm{\bm{w}^{*}_{m}} \big{(} \bm{\hat{w}}_{p} \cdot \bm{z}_{p} \big{)}
\end{equation}
Note that when $\alpha=1$, we recover back the correct encoder $h_{2}^{*}$. Thus to perturb the $h_{2}^{*}$, we set $\alpha$ close to but less than 1.

\paragraph{Showing $h_{2}^{*}$ is not the equilibrium point.}
From \refeqn{eq:adv_enc_loss_theoryfmm}, we want to show that for some $\alpha\in[0,1)$ \new{s.t. $\alpha \rightarrow 1$ ($\alpha$ close to but less than $1$)}, the undesired encoder $h^{\alpha}_{2}$ has bigger combined objective than desired encoder $h^{*}_{2}$. Since the combined adversarial objective for encoder $h_{2}(\bm{z})$ ($E(h_{2}(\bm{z}))$ in \refeqn{eq:adv_enc_loss_theoryfmm}) is evaluated on the margin points of main-task and probing task classifier. We use the following lemma to show that for small perturbation of the optimal encoder ($\alpha\rightarrow1$), the margin point of main-task classifier and probing classifier when using perturbed encoder $h_{2}^{\alpha}$ remains same or is a subset of margin points when using desired encoder $h_{2}^{*}$. The proof of the lemma below is given after the proof of the current theorem.

\begin{lemma}
\label{lemma:same_margin_point}
There exist an $\alpha_{lb}^{2}\in[0,1)$ s.t. when $\alpha>\alpha_{lb}^{2}$ we have:
\begin{enumerate*}[label=(\roman*)]
    \item margin points of probing classifier when using perturbed encoder $h_{2}^{\alpha}$ is same or is a subset of margin points when using desired encoder $h_{2}^{*}$.
    \item margin points of main-task classifier when using perturbed encoder $h_{2}^{\alpha}$ is same or is a subset of margin points when using desired encoder $h_{2}^{*}$.
\end{enumerate*}
\end{lemma}

Let $\bm{z}^{M_{*}}$ be one of the margin point of main-task classifier $c_{m}$ and $\bm{z}^{P_{*}}$ be one of the margin point of probing classifier $c_{p}$ when using the correct encoder $h_{2}^{*}$. Let $\bm{z}^{M_{\alpha}}$ be one of the margin point of main-task classifier $c_{m}$ and $\bm{z}^{P_{\alpha}}$ be one of the margin point of probing classifier $c_{p}$ when using the perturbed encoder $h_{2}^{\alpha}$.
% From \reflemma{lemma:same_margin_point} the margin points of main-task and probing classifier when using $h_{2}^{\alpha}$ is same or subset of margin points when using $h_{2}^{*}$. Thus the combined encoder $E(h_{2}^{\alpha})$ will be evaluated on the same point  . 
Thus we want to show that for all $(\bm{z}^{M_{*}},\bm{z}^{P_{*}},\bm{z}^{M_{\alpha}},\bm{z}^{P_{\alpha}})$ tuple, there exists some $\alpha$ close to but less than 1  s.t. we have:
\begin{align}
    E(h^{\alpha}_{2}) &> E(h_{2}^{*})\\
    m_{c_{m}}(h_{2}^{\alpha}(\bm{z}^{M_{\alpha}}))-m_{c_p}(h_{2}^{\alpha}(\bm{z}^{P_{\alpha}})) &>  m_{c_{m}}(h_{2}^{*}(\bm{z}^{M_{*}}))-m_{c_p}(h_{2}^{*}(\bm{z}^{P_{*}}))\\
    m_{c_{m}}(h_{2}^{\alpha}(\bm{z}^{M_{\alpha}}))- m_{c_{m}}(h_{2}^{*}(\bm{z}^{M_{*}}))  &> m_{c_p}(h_{2}^{\alpha}(\bm{z}^{P_{\alpha}})) -m_{c_p}(h_{2}^{*}(\bm{z}^{P_{*}})) \label{eq:equilibrium_master_ineq}
\end{align}

\paragraph{For $\beta<0$ .} From \reflemma{lemma:same_margin_point}, both $\bm{z}^{P_{\alpha}}$ and $\bm{z}^{P_{*}}$ are the margin point of probing classifier when using the desired encoder $h_{2}^{*}$. Thus we have:
\begin{align}
    m_{c_{p}}(h_{2}^{*}(\bm{z})^{P_{\alpha}}) &= m_{c_{p}}(h_{2}^{*}(\bm{z}^{P_{*}}))\\
    y_{p}^{P_{\alpha}}\beta (\bm{w}^{*}_{m}\cdot \bm{z}_{m}^{P_{\alpha}}) &= y_{p}^{P_{*}}\beta (\bm{w}^{*}_{m}\cdot \bm{z}_{m}^{P_{*}})
\end{align}
\new{
Also, since $\alpha\in[0,1)$, from above equation we have:
\begin{align}
     y_{p}^{P_{\alpha}}\beta \alpha(\bm{w}^{*}_{m}\cdot \bm{z}_{m}^{P_{\alpha}}) < y_{p}^{P_{*}}\beta (\bm{w}^{*}_{m}\cdot \bm{z}_{m}^{P_{*}})
\end{align}
}
From \refeqn{eq:adv_fully_pred_xt} we have $y_{p}^{P_{\alpha}}(\hat{\bm{w}}_{p}\cdot \bm{z}_{p}^{P_{\alpha}})>0$. Since $\beta<0$ \new{and $\alpha\in[0,1)$}, we have $\new{\sqrt{1-\alpha^{2}}} \beta \norm{\bm{w}_{m}^{*}} y_{p}^{P_{\alpha}}(\hat{\bm{w}}_{p}\cdot \bm{z}_{p}^{P_{\alpha}})<0$. Adding this to LHS of the above equation we get:
\begin{align}
    y_{p}^{P_{\alpha}}\beta \new{\alpha}(\bm{w}^{*}_{m}\cdot \bm{z}_{m}^{P_{\alpha}}) + \new{\sqrt{1-\alpha^{2}}} \beta \norm{\bm{w}_{m}^{*}} y_{p}^{P_{\alpha}}(\hat{\bm{w}}_{p}\cdot \bm{z}_{p}^{P_{\alpha}})  &< y_{p}^{P_{*}}\beta (\bm{w}^{*}_{m}\cdot \bm{z}_{m}^{P_{*}})\\
    y_{p}^{P_{\alpha}} \beta \Big{\{}  \new{\alpha}(\bm{w}^{*}_{m}\cdot \bm{z}_{m}^{P_{\alpha}}) + \new{\sqrt{1-\alpha^{2}}} \norm{\bm{w}_{m}^{*}} (\hat{\bm{w}}_{p}\cdot \bm{z}_{p}^{P_{\alpha}}) \Big{\}} &< y_{p}^{P_{*}}\beta (\bm{w}^{*}_{m}\cdot \bm{z}_{m}^{P_{*}})\\
    y_{p}^{P_{\alpha}}\beta h_{2}^{\alpha}(\bm{z}^{P_{\alpha}}) &< y_{p}^{P_{*}}\beta (\bm{w}^{*}_{m}\cdot \bm{z}_{m}^{P_{*}}) \\
    m_{c_{p}}(h_{2}^{\alpha}(\bm{z}^{P_{\alpha}})) &< m_{c_{p}}(h_{2}^{*}(\bm{z}^{P_{*}})) \label{eq:LHS_neg}
\end{align}

% If the probing objective in \refeqn{eq:adv_prob_loss_theoryfmm} had chosen $\beta<0$ then from \reflemma{lemma:yp_h_compare} for all $\bm{z}$ we have:
% \begin{align}
%     y_{p}h_{2}^{\alpha}(\bm{z}) > y_{p}h_{2}^{*}(\bm{z})\\
%     y_p(\beta\cdot h_{2}^{\alpha}(\bm{z})) < y_{p}(\beta\cdot h_{2}^{*}(\bm{z}))\\
%     m_{c_{p}}(h_{2}^{\alpha}(\bm{z})) < m_{c_{p}}(h_{2}^{*}(\bm{z}))
% \end{align}
From \refeqn{eq:LHS_neg} the RHS of \refeqn{eq:equilibrium_master_ineq} is less than zero. Also, from \refeqn{eq:main_margin_delta}, for $\alpha\in(\alpha^{1}_{lb},1)$ we have $m_{c_{m}}(h_{2}^{\alpha}(\bm{z}^{M_{\alpha}}))- m_{c_{m}}(h_{2}^{*}(\bm{z}^{M_{*}})) > 0$ where value of $\alpha_{lb}^{1}$ is given by \reflemma{lemma:sufficient_condition_main_task}. Thus the LHS of \refeqn{eq:equilibrium_master_ineq} is greater than $0$. Thus the inequality in \ref{eq:equilibrium_master_ineq} is always satisfied when $\beta<0$ and $\alpha\in(max\{\alpha_{lb}^{1},\alpha_{lb}^{2}\},1)$. The constraint $\alpha>\alpha_{lb}^{1}$ is enforced by \reflemma{lemma:sufficient_condition_main_task} when constructing the perturbed encoder and $\alpha>\alpha_{lb}^{2}$ is enforced by \reflemma{lemma:same_margin_point} which ensures $\bm{z}^{P_{\alpha}}$ is also a margin point of probing classifier when using desired encoder $h_{2}^{*}$. Hence, we have shown that \new{when $\beta<0$},  $h_{2}^{*}$ is not the equilibrium point since there  exist a perturbed undesired encoder $h_{2}^{\alpha}$ such that the combined encoder objective is greater in \refeqn{eq:adv_enc_loss_theoryfmm} and consequently the optimizer will try to move away from/change $h_{2}^{*}$.

\paragraph{For $\beta>0$ .} Next we have to show that there exist $\alpha\in[0,1)$ s.t. \new{when $\alpha \rightarrow 1$} we have \refeqn{eq:equilibrium_master_ineq} satisfied. Thus we solve for allowed values of $\alpha$:
\begin{align}
    \Big{\{} m_{c_{m}}(h_{2}^{\alpha}(\bm{z}^{M_{\alpha}}))- m_{c_{m}}(h_{2}^{*}(\bm{z}^{M_{*}})) \Big{\}}  &> \Big{\{} m_{c_p}(h_{2}^{\alpha}(\bm{z}^{P_{\alpha}})) -m_{c_p}(h_{2}^{*}(\bm{z}^{P_{*}})) \Big{\}}\label{eq:pos_beta_substitute}\\ 
    \Big{\{}y_{m}h^{\alpha}_{2}(\bm{z}^{M_{\alpha}}) - y_{m}h^{*}_{2}(\bm{z}^{M_{*}}) \Big{\}} &> \Big{\{} y_{p}(\beta\cdot h^{\alpha}_{2}(\bm{z}^{P_{\alpha}})) -  y_{p}(\beta \cdot h^{*}_{2}(\bm{z}^{P_{*}})) \Big{\}}
    % \Big{\{} (\alpha-1)\norm{\bm{w}_{m}^{*}}y_{m}(\hat{\bm{w}}_{m}^{*}\cdot \bm{z}^{M}_{m}) + \sqrt{1-\alpha^{2}}\norm{\bm{w}_{m}^{*}}y_m(\hat{\bm{w}_{p}}\cdot \bm{z}_{p}^{M}) \Big{\}} &> \Big{\{} \Big{\}}
\end{align}
From the second statement from \reflemma{lemma:same_margin_point}, $\bm{z}^{M_{\alpha}}$ and $\bm{z}^{M_{*}}$ both are margin point of main-task classifier using the desired encoder $h_{2}^{*}$. Thus we have $m_{c_{m}}(h_{2}^{*}(\bm{z}^{M_{\alpha}}))= m_{c_{m}}(h_{2}^{*}(\bm{z}^{M_{*}})) \implies y_{m}^{M_{\alpha}}h_{2}^{*}(\bm{z}^{M_{\alpha}}) = y_{m}^{M_{*}}h_{2}^{*}(\bm{z}^{M_{*}}) \implies y_{m}^{M_{\alpha}}(\bm{w}_{m}^{*}\cdot \bm{z}_{m}^{M_{\alpha}}) = y_{m}^{M_{*}}(\bm{w}_{m}^{*}\cdot \bm{z}_{m}^{M_{*}})$. Substituting this observation in LHS of \refeqn{eq:pos_beta_substitute} we get $m_{c_{m}}(h_{2}^{\alpha}(\bm{z}^{M_{\alpha}}))- m_{c_{m}}(h_{2}^{*}(\bm{z}^{M_{*}}))=$ 
\begin{align}
    &=y_{m}^{M_{\alpha}} \Big{\{} \alpha ({\bm{w}}_{m}^{*}\cdot \bm{z}_{m}^{M_{\alpha}}) + \sqrt{1-\alpha^{2}}\norm{\bm{w}^{\new{*}}_{m}}(\hat{\bm{w}}_{p}\cdot \bm{z}_{p}^{M_{\alpha}}) \Big{\}} - y_{m}^{M_{*}} \Big{\{} ({\bm{w}}^{*}_{m}\cdot \bm{z}_{m}^{M_{*}}) \Big{\}}\\
    &=y_{m}^{M_{\alpha}} \Big{\{} \alpha ({\bm{w}}_{m}^{*}\cdot \bm{z}_{m}^{M_{\alpha}}) + \sqrt{1-\alpha^{2}}\norm{\bm{w}^{\new{*}}_{m}}(\hat{\bm{w}}_{p}\cdot \bm{z}_{p}^{M_{\alpha}}) \Big{\}} - y_{m}^{M_{\alpha}} \Big{\{} ({\bm{w}}^{*}_{m}\cdot \bm{z}_{m}^{M_{\alpha}}) \Big{\}}\\
    &= (\alpha-1) y_{m}^{M_{\alpha}} \Big{\{} \norm{\bm{w}_{m}^{*}}(\hat{\bm{w}}_{m}^{*}\cdot \bm{z}_{m}^{M_{\alpha}}) \Big{\}} + \sqrt{1-\alpha^{2}} y_{m}^{M_{\alpha}} \Big{\{} \norm{\bm{w}^{\new{*}}_{m}}(\hat{\bm{w}}_{p}\cdot \bm{z}_{p}^{M_{\alpha}}) \Big{\}}\\
    &= (\alpha-1) y_{m}^{M_{}} \Big{\{} \norm{\bm{w}_{m}^{*}}(\hat{\bm{w}}_{m}^{*}\cdot \bm{z}_{m}^{M_{}}) \Big{\}} + \sqrt{1-\alpha^{2}} y_{m}^{M_{}} \Big{\{} \norm{\bm{w}^{\new{*}}_{m}}(\hat{\bm{w}}_{p}\cdot \bm{z}_{p}^{M_{}}) \Big{\}}\label{eq:LHS_pos_beta_ineq}
\end{align}

where for ease of exposition we have defined $M\coloneqq M_{\alpha}$. Now again for RHS of \refeqn{eq:pos_beta_substitute}, from \reflemma{lemma:same_margin_point}, $\bm{z}^{P_{\alpha}}$ and $\bm{z}^{P_{*}}$ both are margin point of probing classifier using the desired encoder $h_{2}^{*}$. Thus we have $m_{c_{p}}(h_{2}^{*}(\bm{z}^{P_{\alpha}}))= m_{c_{p}}(h_{2}^{*}(\bm{z}^{P_{*}})) \implies y_{p}^{P_{\alpha}}\new{\beta}h_{2}^{*}(\bm{z}^{P_{\alpha}}) = y_{p}^{P_{*}}\new{\beta}h_{2}^{*}(\bm{z}^{P_{*}}) \implies y_{p}^{P_{\alpha}}(\bm{w}_{m}^{*}\cdot \bm{z}_{m}^{P_{\alpha}}) = y_{p}^{P_{*}}(\bm{w}_{m}^{*}\cdot \bm{z}_{m}^{P_{*}})$. Substituting this observation in RHS of \refeqn{eq:pos_beta_substitute} we get $m_{c_{p}}(h_{2}^{\alpha}(\bm{z}^{P_{\alpha}}))- m_{c_{p}}(h_{2}^{*}(\bm{z}^{P_{*}}))=$ 
\begin{align}
    &=y_{p}^{P_{\alpha}} \Big{\{} \alpha \beta ({\bm{w}}_{m}^{*}\cdot \bm{z}_{m}^{P_{\alpha}}) + \sqrt{1-\alpha^{2}} \norm{\bm{w}^{\new{*}}_{m}}\beta(\hat{\bm{w}}_{p}\cdot \bm{z}_{p}^{P_{\alpha}}) \Big{\}} - y_{p}^{P_{*}} \Big{\{} \beta ({\bm{w}}^{*}_{m}\cdot \bm{z}_{m}^{P_{*}}) \Big{\}}\\
    &=y_{p}^{P_{\alpha}} \Big{\{} \alpha \beta ({\bm{w}}_{m}^{*}\cdot \bm{z}_{m}^{P_{\alpha}}) + \sqrt{1-\alpha^{2}}\norm{\bm{w}^{\new{*}}_{m}}\beta(\hat{\bm{w}}_{p}\cdot \bm{z}_{p}^{P_{\alpha}}) \Big{\}} - y_{p}^{P_{\alpha}} \Big{\{} \beta ({\bm{w}}^{*}_{m}\cdot \bm{z}_{m}^{P_{\alpha}}) \Big{\}}\\
    &= (\alpha-1) y_{p}^{P_{\alpha}} \Big{\{} \norm{\bm{w}_{m}^{*}}\beta(\hat{\bm{w}}_{m}^{*}\cdot \bm{z}_{m}^{P_{\alpha}}) \Big{\}} + \sqrt{1-\alpha^{2}} y_{p}^{P_{\alpha}} \Big{\{} \norm{\bm{w}^{\new{*}}_{m}}\beta(\hat{\bm{w}}_{p}\cdot \bm{z}_{p}^{P_{\alpha}}) \Big{\}}\\
    &= (\alpha-1) y_{p}^{P_{}} \Big{\{} \norm{\bm{w}_{m}^{*}}\beta(\hat{\bm{w}}_{m}^{*}\cdot \bm{z}_{m}^{P_{}}) \Big{\}} + \sqrt{1-\alpha^{2}} y_{p}^{P_{}} \Big{\{} \norm{\bm{w}^{\new{*}}_{m}}\beta(\hat{\bm{w}}_{p}\cdot \bm{z}_{p}^{P_{}}) \Big{\}}\label{eq:RHS_pos_beta_ineq}
\end{align}

where for ease of exposition we have defined $P\coloneqq P_{\alpha}$. Substituting RHS (\refeqn{eq:RHS_pos_beta_ineq}) and LHS (\refeqn{eq:LHS_pos_beta_ineq})  back in \refeqn{eq:pos_beta_substitute} and rearranging we get:
\begin{align}
    \sqrt{1-\alpha^{2}} \Big{\{} y_m^{M}(\hat{\bm{w}_{p}}\cdot \bm{z}_{p}^{M}) - y_p^{P}\beta(\hat{\bm{w}_{p}}\cdot \bm{z}_{p}^{P}) \Big{\}} &> (1-\alpha) \Big{\{} y_{m}^{M}(\hat{\bm{w}}_{m}^{*}\cdot \bm{z}^{M}_{m}) - y_{p}^{\new{P}}\beta(\hat{\bm{w}}_{m}^{*}\cdot \bm{z}^{P}_{m}) \Big{\}} \label{eq:equilibrium_alpha_ineq}
\end{align}
Now, since \refassm{assm:fully_pred_inv_main_task} is satisfied, the main task feature $\bm{z}_{m}^{M}$ is linearly separable  w.r.t main-task label $y_{m}^{M}$. Thus we have $ y_{m}^{M}(\hat{\bm{w}}_{m}^{*}\cdot \bm{z}^{M}_{m}) > 0$.
% \abhinav{Write about the case when main-task is is fully predictive of probing label}. Now we have two cases:

\paragraph{Case 1: Main-task feature is not fully predictive of probing label}($\exists\bm{z}$ s.t. $y_{p}(\hat{\bm{w}}_{m}^{*}\cdot \bm{z}_{m})<0$). 
Since main-task feature is not fully predictive of the probing label $y_{p}$, there will be some points which will be misclassified (will be on the opposite side of decision boundary) when probing classifier uses desired encoder $c_{p}(h_{2}^{*}(\bm{z})) = \beta h_{2}^{*}(\bm{z})=\bm{w}_{m}^{*}\cdot \bm{z}_{m}$. Thus margin for those points will be negative and one of them will be the margin point $\bm{z}^{P}$ of the probing classifier. That is, $m_{c_{p}}(h_{2}^{*}(\bm{z}^{P})) =  y_{p}^{P}\beta(\hat{\bm{w}}_{m}^{*}\cdot \bm{z}_{m}^{P}) < 0$. Then the term  $\big(y_{m}^{M}(\hat{\bm{w}}_{m}^{*}\cdot \bm{z}^{M}_{m}) - y_{p}^{P}\beta(\hat{\bm{w}}_{m}^{*}\cdot \bm{z}^{P}_{m})\big) >0$ in the above \refeqn{eq:equilibrium_alpha_ineq}. Hence, rewriting the above equation we have:
\begin{align}
    \frac{\Big{\{} y_{m}^{M}(\hat{\bm{w}_{p}}\cdot \bm{z}_{p}^{M}) - y_{p}^{P}\beta(\hat{\bm{w}_{p}}\cdot \bm{z}_{p}^{P}) \Big{\}}}{\Big{\{} y_{m}^{M}(\hat{\bm{w}}_{m}^{*}\cdot \bm{z}^{M}_{m}) - y_{p}^{P}\beta(\hat{\bm{w}}_{m}^{*}\cdot \bm{z}^{P}_{m}) \Big{\}}} > \frac{1-\alpha}{\sqrt{1-\alpha^{2}}}
\end{align}
Next, from \reflemma{lemma:same_margin_point} both $\bm{z}^{M} \new{\coloneqq} \bm{z}^{M_{\alpha}}$ and $\bm{z}^{P} \new{\coloneqq} \bm{z}^{P_{\alpha}}$ are also the margin point of main-task and probing classifier respectively when the classifiers use the desired encoder $h_{2}^{\alpha}$. Then, since \refassm{assm:correlation_strength} is satisfied the numerator in LHS of above equation $\big( y_{m}^{M}(\hat{\bm{w}_{p}}\cdot \bm{z}_{p}^{M}) - y_{p}^{P}\beta(\hat{\bm{w}_{p}}\cdot \bm{z}_{p}^{P}) \big)>0$. Thus, the whole LHS in the  above equation is greater than zero. Denoting the LHS by $\gamma(\bm{z}^{M},\bm{z}^{P})$ gives us:
\begin{align}
    \gamma(\bm{z}^{M},\bm{z}^{P}) &> \frac{1-\alpha}{\sqrt{1-\alpha^{2}}}\label{eq:alpha_lb3_step1}\\
    \gamma^{2}(\bm{z}^{M},\bm{z}^{P}) &> \frac{\cancel{(1-\alpha)}(1-\alpha)}{\cancel{(1-\alpha)}(1+\alpha)}\\
    \gamma^{2}(\bm{z}^{M},\bm{z}^{P})+ \alpha \gamma^{2}(\bm{z}^{M},\bm{z}^{P}) &> 1- \alpha \\
    \big(1+\gamma^{2}(\bm{z}^{M},\bm{z}^{P})\big)\alpha &> 1-\gamma^{2}(\bm{z}^{M},\bm{z}^{P})\\
    \alpha &> \frac{1-\gamma^{2}(\bm{z}^{M},\bm{z}^{P})}{1+\gamma^{2}(\bm{z}^{M},\bm{z}^{P})} = \alpha_{lb}^{3}(\bm{z}^{M},\bm{z}^{P}))\label{eq:alpha_lb3_steplast}
\end{align}
Since $\gamma^{2}(\bm{z}^{M},\bm{z}^{P})>0$, $\alpha_{lb}^{3}(\bm{z}^{M},\bm{z}^{P}))<1$. Let $\alpha_{lb}^{3} = \max_{(\bm{z}^{M},\bm{z}^{P})}(\alpha_{lb}^{3}(\bm{z}^{M},\bm{z}^{P})))$ \new{which is $<1$} gives us the tight lower-bound on $\alpha$ such that \refeqn{eq:equilibrium_master_ineq} is satisfied for any pair of margin point $\bm{z}^{M}$ and $\bm{z}^{P}$. 

\paragraph{Case 2: Main-task is fully predictive of probing label. }
($\forall\bm{z}$, $y_{p}(\hat{\bm{w}}_{m}^{*}\cdot \bm{z}_{m})>0$). Since \refassm{assm:fully_pred_inv_main_task} (from \reflemma{lemma:sufficient_condition_main_task}) is satisfied, we have that main-task features are fully predictive of main-task label i.e $y_{m}(\hat{\bm{w}}_{m}^{*}\cdot \bm{z}_{m})>0$ for all $\bm{z}$. Thus for this case $y_{m}(\hat{\bm{w}}_{m}^{*}\cdot \bm{z}_{m})>0$ and $y_{p}(\hat{\bm{w}}_{m}^{*}\cdot \bm{z}_{m})>0 \implies y_{m}=y_{p}$ for all $\bm{z}$. Also, for this case, there will be no misclassified points for the probing classifier when using the desired encoder $h_{2}^{*}$. Thus the margin point for both the main and the probing classifier is same i.e $\bm{z}^{M}=\bm{z}^{P}$. Since \refassm{assm:correlation_strength} is satisfied, $y_{m}=y_{p}$ for all $\bm{z}$, $y_{p}(\hat{\bm{w}}_{p}\cdot \bm{z})>0$ for all $\bm{z}$ from \refassm{assm:fully_pred_inv} and $\bm{z}^{P}=\bm{z}^{M}$ we have:
\begin{align}
    y_m^{M}(\hat{\bm{w}_{p}}\cdot \bm{z}_{p}^{M}) &> y_p^{P}\beta(\hat{\bm{w}_{p}}\cdot \bm{z}_{p}^{P}) \quad \quad \quad (\refassm{assm:correlation_strength})\\
     1\cdot \cancel{(y_p^{P}(\hat{\bm{w}_{p}}\cdot \bm{z}_{p}^{P}))} &> \beta \cancel{(y_p^{P}(\hat{\bm{w}_{p}}\cdot \bm{z}_{p}^{P}))}\\
     \beta &<1 \label{eq:beta_less_1}
\end{align}
Thus, in this case the RHS in \refeqn{eq:equilibrium_alpha_ineq}, could be simplified to : $y_{m}^{M}(\hat{\bm{w}}_{m}^{*}\cdot \bm{z}^{M}_{m}) - y_{p}^{P}\beta(\hat{\bm{w}}_{m}^{*}\cdot \bm{z}^{P}_{m}) = y_{m}^{M}(\hat{\bm{w}}_{m}^{*}\cdot \bm{z}^{M}_{m}) - \beta y_{m}^{M}(\hat{\bm{w}}_{m}^{*}\cdot \bm{z}^{M}_{m}) = (1-\beta)y_{m}^{M}(\hat{\bm{w}}_{m}^{*}\cdot \bm{z}^{M}_{m})>0$ since $0<\beta<1$ from above \refeqn{eq:beta_less_1} and $y_{\new{m}}(\hat{\bm{w}}_{m}^{*}\cdot \bm{z}^{M}_{m})>0$ from \refassm{assm:fully_pred_inv_main_task}. Thus we can rewrite \refeqn{eq:equilibrium_alpha_ineq} as:
\begin{align}
    \frac{\Big{\{} y_m^{M}(\hat{\bm{w}_{p}}\cdot \bm{z}_{p}^{M}) - y_p^{P}\beta(\hat{\bm{w}_{p}}\cdot \bm{z}_{p}^{P}) \Big{\}}}{\Big{\{} y_{m}^{M}(\hat{\bm{w}}_{m}^{*}\cdot \bm{z}^{M}_{m}) - y_{p}^{P}\beta(\hat{\bm{w}}_{m}^{*}\cdot \bm{z}^{P}_{m}) \Big{\}}} > \frac{1-\alpha}{\sqrt{1-\alpha^{2}}}
\end{align}
Again, from \reflemma{lemma:same_margin_point} both $\bm{z}^{M} \new{\coloneqq} \bm{z}^{M_{\alpha}}$ and $\bm{z}^{P} \new{\coloneqq} \bm{z}^{P_{\alpha}}$ are also the margin point of main-task and probing classifier respectively when the classifiers use the desired encoder $h_{2}^{\alpha}$.
Thus from \refassm{assm:correlation_strength}, we have numerator of LHS in above equation greater than 0, thus we can follow the same steps from \refeqn{eq:alpha_lb3_step1} to \ref{eq:alpha_lb3_steplast} to get the $\alpha_{lb}^{3}$ for this case.

So far, we have three lower bounds on $\alpha$ needed for this proof, so lets define $\alpha_{lb}=max\{\alpha_{lb}^{1},\alpha_{lb}^{2},\alpha_{lb}^{3}\}$, where $\alpha_{lb}^{1}$ is enforced by \reflemma{lemma:sufficient_condition_main_task} on undesired encoder $h_{2}^{\alpha}$ construction, $\alpha_{lb}^{2}$ is enforced by \reflemma{lemma:same_margin_point} and $\alpha_{lb}^{3}$ is enforced by \refeqn{eq:equilibrium_master_ineq}. Thus, when $\alpha \in (\alpha_{lb},1]$ we have a bigger combined objective (\refeqn{eq:adv_enc_loss_theoryfmm}) for $h_{2}^{\alpha}$ than $h_{2}^{*}$. Thus, we can always perturb the desired encoder $h_{2}^{\alpha}$ by choosing $\alpha\in(\alpha_{lb},1]$ close to but less than 1 to create $h_{2}^{\alpha}$ which will have better combined encoder objective. Hence any optimizer will prefer to change the desired encoder $h_{2}^{*}$ and it is not an equilibrium solution to the overall adversarial objective.

\end{proof}

\begin{proof}[Proof of \reflemma{lemma:same_margin_point}]
First, we will prove the statement for the probing classifier. Let $\bm{z}^{M}$ be one of the margin points of the probing classifier when using the desired encoder $h_{2}^{*}$ and let $\bm{z}^{R}$ be any other (non-margin) points. Then we have to show that the margin-point of the probing classifier when using perturbed encoder $h_{2}^{\alpha}$ cannot be $\bm{z}^{R}$. This will imply that the margin points for probing classifier when using $h_{2}^{\alpha}$ has to be the same or a subset of margin points when using $h_{2}^{*}$. Since norm of parameters of both $c_{p}(h_{2}^{\alpha}(\bm{z}))=\beta h_{2}^{\alpha}(\bm{z})$ and  $c_{p}(h_{2}^{*}(\bm{z}))=\beta h_{2}^{*}(\bm{z})$ is the same and margin-point of a classifier is the point which have minimum margin, we have to show that $m_{c_{p}(h_{2}^{\alpha})}(\bm{z}^{R})>m_{c_{p}(h_{2}^{\alpha})}(\bm{z}^{M})$ \new{for some $\alpha \in [0,1)$}. We have:
\begin{align}
    m_{c_{p}(h_{2}^{\alpha})}(\bm{z}) &= \alpha y_{p} \Big{\{} \beta(\bm{w}^{*}_{m}\cdot \bm{z}_{m}) \Big{\}} + \sqrt{1-\alpha^{2}}\norm{\bm{w}_{m}^{*}} y_{p}\Big{\{} \beta(\hat{\bm{w}}_{p}\cdot \bm{z}_{p}) \Big{\}}\\
    &=\alpha m_{c_{p}(h_{2}^{*})}(\bm{z}) + \sqrt{1-\alpha^{2}}\norm{\bm{w}_{m}^{*}} y_{p}\Big{\{} \beta(\hat{\bm{w}}_{p}\cdot \bm{z}_{p}) \Big{\}}
\end{align}

Thus we have to find an \new{$\alpha \in[0,1)$} s.t.:
\begin{align*}
    \alpha m_{c_{p}(h_{2}^{*})}(\bm{z}^{R}) + \sqrt{1-\alpha^{2}}\norm{\bm{w}_{m}^{*}} y_{p}^{R}\Big{\{} \beta(\hat{\bm{w}}_{p}\cdot \bm{z}^{R}_{p}) \Big{\}} >& \\
    \alpha m_{c_{p}(h_{2}^{*})}(\bm{z}^{M})& + \sqrt{1-\alpha^{2}}\norm{\bm{w}_{m}^{*}} y_{p}^{M}\Big{\{} \beta(\hat{\bm{w}}_{p}\cdot \bm{z}^{M}_{p}) \Big{\}}
\end{align*}

Rearranging we get:
\begin{align}
\label{eq:prob_same_mp_master}
    \alpha \Big{\{} m_{c_{p}(h_{2}^{*})}(\bm{z}^{R}) -  m_{c_{p}(h_{2}^{*})}(\bm{z}^{M}) \Big{\}} > \sqrt{1-\alpha^{2}}\norm{\bm{w}_{m}^{*}} \beta \Big{\{}  y_{p}^{M}(\hat{\bm{w}}_{p}\cdot \bm{z}^{M}_{p}) - y_{p}^{R} (\hat{\bm{w}}_{p}\cdot \bm{z}^{R}_{p}) \Big{\}}
\end{align}

Since $\bm{z}^{M}$ is the margin point of the probing classifier when using $h_{2}^{*}$, we have $ m_{c_{p}(h_{2}^{*})}(\bm{z}^{R}) > m_{c_{p}(h_{2}^{*})}(\bm{z}^{M})$. Now, if $\beta \Big{\{} y_{p}^{M}(\hat{\bm{w}}_{p}\cdot \bm{z}^{M}_{p}) - y_{p}^{R}(\hat{\bm{w}}_{p}\cdot \bm{z}^{R}_{p}) \Big{\}} \leq 0$, then above equation is trivially satisfied for all values of $\alpha\in(0,1)$, since RHS of above equation is greater than 0 and LHS is less than 0. For the case when $\beta \Big{\{} y_{p}^{M}(\hat{\bm{w}}_{p}\cdot \bm{z}^{M}_{p}) - y_{p}^{R}(\hat{\bm{w}}_{p}\cdot \bm{z}^{R}_{p}) \Big{\}} > 0$ we need:
\begin{align}
    \frac{\alpha}{\sqrt{1-\alpha^{2}}} &> \frac{\norm{\bm{w}_{m}^{*}} \beta \Big{\{} y_{p}^{M}(\hat{\bm{w}}_{p}\cdot \bm{z}^{M}_{p}) - y_{p}^{R}(\hat{\bm{w}}_{p}\cdot \bm{z}^{R}_{p}) \Big{\}}}{\Big{\{} m_{c_{p}(h_{2}^{*})}(\bm{z}^{R}) -  m_{c_{p}(h_{2}^{*})}(\bm{z}^{M}) \Big{\}}} \coloneqq \gamma(\bm{z}^{M},\bm{z}^{P})>0\\
    \frac{\alpha^{2}}{1-\alpha^{2}} &> \gamma^{2}(\bm{z}^{M},\bm{z}^{P})\\
    \alpha^{2}(1+\gamma^{2}(\bm{z}^{M},\bm{z}^{P}))&>\gamma^{2}(\bm{z}^{M},\bm{z}^{P})\\
    \alpha &> \sqrt{\frac{\gamma^{2}(\bm{z}^{M},\bm{z}^{P})}{1+\gamma^{2}(\bm{z}^{M},\bm{z}^{P})}} \coloneqq \alpha_{lb}^{p}(\bm{z}^{M},\bm{z}^{P})\label{eq:prob_same_mp_alpha}
\end{align}
Since we have $\gamma>0 \implies \alpha_{lb}^{p}(\bm{z}^{M},\bm{z}^{P})<1$. Lets define $\alpha_{lb}^{p} \coloneqq \max_{(\bm{z}^{M},\bm{z}^{\new{P}})}(\alpha_{lb}^{p}(\bm{z}^{M},\bm{z}^{\new{P}}))<1$, which gives the tightest lower bound on $\alpha$ s.t. when $\alpha\in(\alpha_{lb}^{p},1)$, the margin point of the probing classifier when using the perturbed encoder is same or is a subset of margin point when using desired encoder $h_{2}^{*}$. This completes the first part of the proof.

Next, we prove the second part of this lemma for the main-task classifier. Let $\bm{z}^{M}$ be one of the margin points of the main-task classifier when using the desired encoder $h_{2}^{*}$ and let $\bm{z}^{R}$ be any other (non-margin) point. Then we have to show that the margin-point of the main-task classifier when using perturbed encoder $h_{2}^{\alpha}$ cannot be $\bm{z}^{R}$. Since norm of parameter of both $c_{m}(h_{2}^{\alpha}(\bm{z}))=h_{2}^{\alpha}(\bm{z})$ and $c_{m}(h_{2}^{*}(\bm{z}))=h_{2}^{*}(\bm{z})$ is same and margin-point of a classifier is the point which have minimum margin, we have to show that $m_{c_{m}(h_{2}^{\alpha})}(\bm{z}^{R})>m_{c_{m}(h_{2}^{\alpha})}(\bm{z}^{M})$ \new{for some $\alpha\in[0,1)$}. We have:
\begin{align}
    m_{c_{m}(h_{2}^{\alpha})}(\bm{z}) &= \alpha y_{m} \Big{\{} (\bm{w}^{*}_{m}\cdot \bm{z}_{m}) \Big{\}} + \sqrt{1-\alpha^{2}}\norm{\bm{w}_{m}^{*}} y_{m}\Big{\{} (\hat{\bm{w}}_{p}\cdot \bm{z}_{p}) \Big{\}}\\
    &=\alpha m_{c_{m}(h_{2}^{*})}(\bm{z}) + \sqrt{1-\alpha^{2}}\norm{\bm{w}_{m}^{*}} y_{m}\Big{\{} (\hat{\bm{w}}_{p}\cdot \bm{z}_{p}) \Big{\}}
\end{align}
Thus we have find an $\alpha$ s.t.:
\begin{align*}
    \alpha m_{c_{m}(h_{2}^{*})}(\bm{z}^{R}) + \sqrt{1-\alpha^{2}}\norm{\bm{w}_{m}^{*}} y_{m}^{R}\Big{\{} (\hat{\bm{w}}_{p}\cdot \bm{z}^{R}_{p}) \Big{\}} >& \\
    \alpha m_{c_{m}(h_{2}^{*})}(\bm{z}^{M})& + \sqrt{1-\alpha^{2}}\norm{\bm{w}_{m}^{*}} y_{m}^{M}\Big{\{} (\hat{\bm{w}}_{p}\cdot \bm{z}^{M}_{p}) \Big{\}}
\end{align*}
Rearranging we get:
\begin{align}
    \alpha \Big{\{} m_{c_{m}(h_{2}^{*})}(\bm{z}^{R}) -  m_{c_{m}(h_{2}^{*})}(\bm{z}^{M}) \Big{\}} > \sqrt{1-\alpha^{2}}\norm{\bm{w}_{m}^{*}} \Big{\{} y_{m}^{M}(\hat{\bm{w}}_{p}\cdot \bm{z}^{M}_{p}) - y_{m}^{R}(\hat{\bm{w}}_{p}\cdot \bm{z}^{R}_{p}) \Big{\}}
\end{align}
Since $\bm{z}^{M}$ is the margin point of the probing classifier when using $h_{2}^{*}$, we have $ m_{c_{m}(h_{2}^{*})}(\bm{z}^{R}) > m_{c_{m}(h_{2}^{*})}(\bm{z}^{M})$. We notice that, apart from $y_{p}$ being set to $y_{m}$ and $\beta$ being set to $1$, the above equation is identical to \refeqn{eq:prob_same_mp_master}. Since our argument (from \refeqn{eq:prob_same_mp_master} to \ref{eq:prob_same_mp_alpha}) to derive the allowed value of $\alpha$ doesn't depend on $y_{p}$ and $\beta$, we could follow the same argument to get a lower bound $\alpha_{lb}^{m}$ s.t. the main-task classifier has the same \new{or subset of} the margin points when using \new{the perturbed encoder as it has when using the desired encoder.} 

Let us define $\alpha_{lb}^{2} = \max\{\alpha_{lb}^{p},\alpha_{lb}^{m}\}$. Thus when $\alpha\in(\alpha_{lb}^{2},1)$, both the statements of this lemma are satisfied thus completing our proof.

\end{proof}
\section{Experimental Setup}
\label{sec:app_expt_setup}

\subsection{Dataset Description}
\label{subsec:app_dataset_desc}
As described in \refsec{sec:feature_removal_expt_main}, we demonstrate the failure of Null-Space Removal (\refsec{subsec:expt_null_space_failure}) and Adversarial Removal (\refsec{subsec:expt_adv_failure}) in removing the undesired \prop from the latent representation on three real-world datasets: \mnli\cite{mnliDataset}, \pan\cite{pan16Dataset} and \aae\cite{aaeDataset}; and a synthetic dataset, \syn. The detailed generation and evaluation strategies for each dataset are given below. 

\paragraph{\mnli Dataset. }
In the MultiNLI dataset, given two sentences---premise and hypothesis---the main task is to predict whether the hypothesis \emph{entails}, \emph{contradicts} or is \emph{neutral} to the premise. As described in \refsec{sec:feature_removal_expt_main}, we simplify it to a binary task of predicting whether a hypothesis \emph{contradicts} the premise. The binary main-task label, $y_{m}=1$ when a given hypothesis \emph{contradicts} the premise otherwise it is -1. That is, we relabel the MNLI dataset by assigning label $y_{m}=1$ to examples with contradiction labels and $y_{m}=-1$ to the example with neutral or entailment label. It has been reported that the \emph{contradiction} label is spuriously correlated with the negation words like \emph{nobody, no, never} and \emph{nothing}\cite{GururanganMNLINegationArtifact}. Thus, we created a `negation' \prop denoting the presence of these words in the hypothesis of a given (hypothesis, premise) pair. The \prop-label $y_{p}=1$ when the \emph{negation} \prop is present in the hypothesis otherwise it is $-1$.

The standard MultiNLI dataset \footnote{MultiNLI dataset and its license could be found online at: \url{https://cims.nyu.edu/~sbowman/multinli/}} has approximately 90\% of data points in the training set, 5\% as publicly available development set and the rest of $5\%$ in a separate held-out validation set accessible through online competition leader-board not accessible to the public. Thus, we create our own train and test split by subsampling $10k$ examples from the initial training set, converting it into binary contradiction vs. non-contradiction labels, labeling the negation-\prop label, and splitting them into 80-20 train and test split. For pre-training a clean classifier that does not use the spurious-\prop, we create a special training set following the method described in \refsec{subsec:app_clean_dataset}. For evaluating the robustness of both null-space and adversarial removal methods, we create multiple datasets with different \emph{predictive-correlation} as described in \refsec{subsec:app_controlling_sp_corr} .

\paragraph{\pan Dataset. }
In \pan dataset \cite{pan16Dataset}, following  \cite{AdvRemYoav}, given a tweet,  the main task is to predict whether it contains  a mention of another user or not. The dataset contains manually annotated binary Gender information (i.e Male or Female) of 436 Twitter users with at least 1k tweets each. The Gender annotation was done by assessing the name and photograph of the LinkedIn profile of each user \cite{AdvRemYoav}. The unclear cases were discarded in this process. We consider ``Gender'' as a sensitive concept that should not be used for main-task prediction. The dataset contains 160k tweets for training and 10k tweets for the test. We merged the full dataset, subsampled 10k examples, and created an 80-20 train and test split. For pre-training a clean classifier, we create a special training set  following the method described in \refsec{subsec:app_clean_dataset}. To generate datasets with different predictive correlation, we follow the method from \ref{subsec:app_controlling_sp_corr}. The dataset is acquired and processed using the code\footnote{The code for \pan and \aae dataset acquisition is available at: \url{https://github.com/yanaiela/demog-text-removal} } made available by the \cite{AdvRemYoav}. According to Twitter's policy, one has to download tweets from a personal account using Twitter Academic Research access and cannot be released to the public or used for commercial purposes. We also adhere to this policy and don't release any data to the public or use it elsewhere.

\paragraph{\aae Dataset. }
In \aae dataset \cite{aaeDataset}, again following \cite{AdvRemYoav},  the main task is to predict a binary sentiment (Positive or Negative) from a given tweet. The dataset contains 59.2 million tweets by 2.8 million users. Each tweet is associated with ``race'' information of the user which is labeled based on  both words in the tweet and the geo-location of the user.  We consider ``race'' as the sensitive \prop which should not be used for the main task of sentiment prediction. We use the AAE (African America English) and SAE (Standard American English) as a proxy for non-Hispanic blacks and non-Hispanic whites automatically labeled using code made available by \cite{AdvRemYoav}.
% \amit{did you label it or just used the labels from 13?} 
Again, we subsampled 10k examples with 80-20 split from the dataset and followed the method described in \refsec{subsec:app_clean_dataset} and \ref{subsec:app_controlling_sp_corr} to generate a clean dataset for pre-training a clean classifier and datasets with different predictive correlation respectively. The dataset is made publicly available online\footnote{TwitterAAE dataset could be found online at: \url{http://slanglab.cs.umass.edu/TwitterAAE/}} only for research-purpose.

\paragraph{Synthetic Dataset. }
To accurately evaluate the whether a classifier is using the spurious \prop or not, we introduce a \syn dataset where it is possible to change the text input based on the change in concept (thus implementing \refdef{def:causally_derived_feature}). The main-task is to predict whether a sentence contains a numbered word (e.g. \emph{one, fifteen} etc) or not, and the spurious \prop is the length of the sentence which is correlated with the main task label. To create a sentence with numbered words, we randomly sample 10 words from the following set and combine  them to form  the sentence.

\begin{align*}
    \text{Numbered Words} =  
     &\text{  one, two, three, four, five, six, seven, eight, }\\
    & \text{  nine, ten, eleven, twelve, thirteen, fourteen, }\\
    & \text{  fifteen, sixteen, seventeen, eighteen, twenty, }\\
    & \text{  thirty, forty, fifty, sixty, seventy, eighty, }\\
    & \text{  ninety, hundred, thousand} \\
\end{align*}
Otherwise, a sentence is  created by adding 10 non-numbered words randomly sampled from the following set. 
\begin{align*}
    \text{Non-Numbered Words} = & \text{   nice, device, try, picture, signature, trailer, }\\
    & \text{   harry, potter, malfoy, john, switch, taste, }\\
    & \text{   glove, balloon,  dog,  horse, switch,  watch, }\\
    & \text{   sun,  cloud,  river,  town,  cow,  shadow, }\\
    & \text{   pencil,  eraser}\\
\end{align*}
Next, we introduce the spurious \prop (length) by increasing the length of the sentences which contain numbered words. We do so by adding a special word ``pad'' 10 times. In our experiments,  we use 1k examples created using the above method and create an 80-20 split for the train and test set. Again, we follow the method described in \refsec{subsec:app_clean_dataset} and \ref{subsec:app_controlling_sp_corr} to generate a clean dataset for training a clean classifier and to generate datasets with different predictive correlations respectively. \new{To simulate a real-world setting,  we also introduce noise in the main-task and the probing label. To introduce noise (denoted by $n=x$) in the labels, we randomly flip $100x$\% of the main-task and probing label in the dataset. Wherever applicable, we will explicitly mention the amount of noise we add in the labels.}

\subsection{Creating a ``clean'' dataset with no spurious correlation with main-label}
\label{subsec:app_clean_dataset}
\new{Unless otherwise specified}, to construct a new dataset with no spurious correlation between the main-task and the \prop label, we subsample only those examples from the the given dataset which have a fixed value of the spurious-\prop label $(y_{p})$. Thus, if we train main-task classifier using this dataset, it cannot use the spurious-\prop since they are not discriminative of the main task label \cite{ravichander-etal-2021-probing}.

In \mnli dataset,  we select only those examples which have no \emph{negation} words in the sentence for creating a clean dataset. Similarly, for \pan dataset, we only select those examples which have gender label $y_{p}=-1$ (Female)  in the processed dataset. 
% \abhinav{find the gender which is negative in the code} 
And for \aae dataset, we only select those examples which have \emph{non-Hispanic whites} race label.

% \subsection{Training a Creating a ``clean'' classifier with no spurious correlation by re}
% \label{subsec:app_clean_dataset_rebalancing}
% An alternate way to create clean dataset is by 

\subsection{Creating datasets with spurious correlated main and \prop label}
\label{subsec:app_controlling_sp_corr}
In our experimental setup, both the main-task label ($y_{m}$) and \prop label ($y_{p}$) are binary ($-1$ or $1$). This creates $2\times2$ subgroups for each combination of ($y_{m},y_{p}$). In \mnli dataset, the contradiction label $(y_{m}=1)$ is correlated with the presence of negation words $y_{p}=1$, this implies that the not-contradiction label $y_{m}=-1$ is also correlated with \emph{absence} of negation words in the sentence $y_{p}=-1$. Thus, the input example with $(y_{m}=1,y_{p}=1)$ and $(y_{m}=-1,y_{p}=-1)$ form the majority group, henceforth referred as $S_{maj}$ while groups $(y_{m}=1,y_{p}=-1)$ and $(y_{m}=-1,y_{p}=1)$ forms the minority group $S_{min}$. To evaluate the robustness of the removal methods, we create multiple datasets with different \emph{predictive correlation} ($\kappa$) between the two labels $y_{m}$ and $y_{p}$ where $\kappa=P(y_{m}\cdot y_{p})>0$ as defined in \refsec{sec:feature_removal_expt_main}. In other words, to create a dataset with a particular predictive correlation $\kappa$, we vary the size of $S_{maj}$ and $S_{min}$. More precisely, the predictive correlation can be equivalently defined in terms of the size of the these groups as:
\begin{equation}
    \kappa = \frac{|S_{maj}|}{|S_{maj}|+|S_{min}|}
\end{equation}

Similarly for \pan, \aae, and \syn datasets, we create datasets with different levels of spurious correlation between $y_{m}$ and $y_{p}$ by creating the $S_{maj}$ and $S_{min}$ to have the desired predictive correlation ($\kappa$).

\subsection{Encoder for real datasets}
\label{subsec:app_real_encoder}
For all the experiments on real datasets in \refsec{sec:feature_removal_expt_main} we used RoBERTa as default encoder $h$. In \refsec{sec:app_additional_results}, we report the results when using BERT instead of  RoBERTa as input encoder. 
\paragraph{RoBERTa}
We use the Hugging Face\cite{HuggingFace} \emph{transformers} implementation of RoBERTa\cite{roberta19} \emph{roberta-base} model, starting with pretrained weights for encoding the text-input to latent representation.  We use a default tokenizer and model configuration in our experiment. 

\paragraph{BERT}
We use the Hugging Face\cite{HuggingFace} \emph{transformers} implementation of BERT\cite{devlin-etal-2019-bert} \emph{bert-base-uncased} model, starting with pretrained weights for encoding the text-input to latent representation. We use a default tokenizer and model configuration in our experiment.% As with RoBERTa, the parameters of the encoder (BERT) fine-tuned along when training the main-task classifier for null-space removal. For adversarial removal the encoder, main-task classifier and the adversarial probing classifier is trained jointly.

For both BERT and RoBERTa, the parameters of the encoder were fine-tuned as a part of training the main-task classifier for null-space removal and then frozen. For adversarial removal, the encoder, main-task classifier and the adversarial probing classifier are trained jointly.
For both BERT and RoBERTa, we use the pooled output ($[CLS]$ token for BERT) from the the model, as the latent representation and is given to main-task and probing classifier. Main-task and probing classifier are a linear transformation layer followed by a softmax layer for prediction. \new{We use a batch size of $32$ samples for all training procedures that use BERT or RoBERTa for encoding the input.}

\subsection{Encoder for synthetic Dataset}
\label{subsec:app_syn_encoder}
\paragraph{nBOW: neural Bag of Word.}
For \syn dataset, we use sum of pretrained-GloVe embedding\cite{pennington2014glove} of the words in the sentence to encode the sentence into latent representation. We used Gensim \cite{rehurek2011gensim} library for acquiring the 100-dimensional GloVe embedding (\emph{glove-wiki-gigaword-100}). Throughout all our experiments, the word embedding was not trained. Post encoding, the latent representation were further passed through hidden layers consisting of a linear transformation layer followed by ReLU non-linearity. We will specify how many such hidden layers were used when discussing specific experiments in \refsec{sec:app_additional_results}. The hidden layer dimensions were fixed to $50$ dimensional space. \new{We use a batch size of $32$ samples for all training procedures that use nBOW for encoding the input.}

\subsection{Null-Space Removal  Experiment Setup}
\label{subsec:app_inlp_expt_setup}
For null-space removal (\INLP) experiment on both real and synthetic dataset the following procedure is followed:
\begin{enumerate}
    \item Pretraining Phase: A \emph{clean} pretrained main-task classifier is trained using the \emph{clean} dataset obtained by method described in \refappendix{subsec:app_clean_dataset}. This is to ensure that the main-task classifier does not use the spurious feature, so that the \INLP method doesn't have any effect on the main-task classifier. The main-task classifier is a linear-transformation on the latent-representation provided by encoder followed by softmax layer for prediction. Both the encoder and main-task classifier is fine-tuned during this process.
    
    \item Removal Phase: Both the encoder and main-task classifier is frozen (made non-trainable). Next, a probing classifier is trained from the latent representation of the encoder (refer \refappendix{subsec:app_real_encoder} and \ref{subsec:app_syn_encoder} for more details about encoder). The probing classifier is also a linear transformation layer followed by softmax layer for prediction. \new{For experiments on real-world datasets using BERT or RoBERTa as encoder, we train the the probing classifier for 1 epoch (one full pass though the probing dataset) before each projection step. For experiment on the \syn dataset,  we train the probing classifier for 10 epochs before each projection step.} \new{Note that, we also experiment with the setting when the main task classifier is also trained after every step of \INLP projection (see \refsec{subsec:app_extended_null_space_results} and \reffig{fig:app_inlp-head_retrain_mix} for results)}. \new{The main task classifier is a linear transformation layer followed by a softmax layer for prediction trained using cross-entropy objective to predict the main task label. The main task classifier is trained for 1 epoch for the real-world datasets and 10 epochs for the \syn dataset for the setting when we train the main task classifier in INLP removal phase. The encoder is frozen for both the setting (with or without main task classifier training)  though-out the INLP removal phase.}
\end{enumerate}

The main-task classifier and encoder in the pretraining phase and the probing classifier in the removal phase is trained using cross-entropy loss for both real and synthetic datasets. For the real dataset, a fixed learning rate of $1\times 10^{-5}$ is used when RoBERTa is used as encoder and $5\times 10^{-5}$ when using BERT as encoder. For synthetic experiments, a fixed learning rate of $5\times 10^{-3}$ is used when training both the nBOW encoder and main-task classifier in the pretraining stage and probing classifier in removal stage.

\subsection{Adversarial Removal Experiment Setup}
\label{subsec:app_adv_expt_setup_desc}
For adversarial removal (\ar) experiment,  for both real and synthetic datasets, first the input text is encoded to latent representation using the encoder (\refappendix{subsec:app_real_encoder} and \ref{subsec:app_syn_encoder}). Then for the main-task classifier, a linear transformation layer followed by a softmax layer is applied for the main-task prediction. The same latent representation output from the encoder is given to the probing classifier which is a separate linear transformation layer followed by a softmax layer. All components of the model, encoder, main-task classifier, and probing classifier are trained using the following modified objective from \refeqn{eq:adv_removal_actual_obj}:
\begin{equation*}
    arg\min_{h,c_{m},c_{p}} \Big{\{}L(c_{m}(h(\bm{z})),y_{m})+ \lambda L(c_{p}(g_{-1}(h(\bm{z}))),y_{p})\Big{\}}
\end{equation*}
where $h$ is the encoder, $c_{m}$ is the main task classifier, $c_{p}$ is the probing classifier, $g_{-1}$ is the gradient reversal layer with fixed reversal strength of $-1$. The first term in the objective is for training the main task classifier and the second term is the adversarial objective for training the probing classifier using gradient reversal method \cite{AdvDomAdapGanin,AdvRemYoav} . The hyperparameter $\lambda$ controls the strength of the adversarial objective. In our experiment we very $\lambda \in \{ 0.00001,0.0001,0.001,0.01,0.1,0.5,1.0,2.0 \}$. When describing the experimental results in \refsec{subsec:app_extended_adversarial_removal_results} we choose the $\lambda$ which performs the best for all datasets with different predictive correlation $\kappa$ in removing the undesired \prop from the latent representation.

\subsection{Metrics Description}
\label{subsec:app_metric_desc}
Analogous to spuriousness score (\refdef{def:spurriousness_score}) for main-task classifier we define the score for probing classifier below. 
\begin{definition}[Probe Spuriousness Score]
\label{def:spurriousness_score_probe}
Given a dataset, $\mathcal{D}_{m,p}=S_{min} \cup S_{maj}$ with  binary task label and binary \prop, let $Acc^{f}(S_{min})$ be the minority group accuracy of a given probing classifier ($f$) and $Acc^{*}(S_{min})$ be the minority group accuracy of a \emph{clean} probing classifier that does not use the main-task feature. Then  spuriousness score of $f$ is:
%\begin{equation*}
%$    \psi(f) = \Big{|}1 - \frac{Acc^{f}(S_{min})}{Acc^{*}(S_{min})}\Big{|}$.
$    \psi(f) = {|}1 - Acc^{f}(S_{min})/Acc^{*}(S_{min}){|}$.
%\end{equation*}
\end{definition}

For simplicity, in all our experiments we assume that both the main and the correlated attribute labels are binary. We measure the degree of spuriousness using the following two metrics:
\begin{enumerate}
    \item Spuriousness Score: As defined in \refsec{subsec:spurriousness_score}, this metric help us quantify, how much a classifier is using the spurious feature (see \refdef{def:spurriousness_score} and \ref{def:spurriousness_score_probe}).
    
    \item $\Delta$ Probability: In \syn dataset as described in \ref{subsec:app_dataset_desc}, we have the ability to change the input corresponding to the change in \prop label (thus implementing \refdef{def:causally_derived_feature}) . Thus we could measure if the main-classifier is using the spurious-\prop by changing the \prop in the input and measuring the corresponding change in the main-task classifier's prediction probability. The Higher the change in prediction probability higher the main-task classifier is dependent on spurious-\prop.

\end{enumerate}

\subsection{Compute and Resources}
\label{subsec:compute_resource}
We used an internal cluster of Nvidia P40 and P100 GPUs for all our experiment. Each experiment setting was run on three random seed and mean results with variance are reported in all the experiment.

\section{Additional Results}
\label{sec:app_additional_results}

\subsection{Probing classifier Quality}
\label{subsec:extended_probing_result}
\reffig{fig:app_prob_all} shows different failure modes of the probing classifier. In \reffig{fig:app_prob_syn_causal_same} and \ref{fig:app_prob_mnli_causal_same}, a \emph{clean} main-task classifier which doesn't use the \prop feature is trained on \syn and \mnli dataset respectively using the method described in \refsec{subsec:app_clean_dataset}. Thus the latent representation doesn't have the \prop feature. Then, to test the presence of \prop-causal feature in the latent representation we train a probing classifier to predict \prop-label.  The first row show the accuracy of the probing classifier for testing the presence of \prop in latent space. When $\kappa=0.5$ i.e no correlation between the main-task and the \prop label, the probing accuracy is approximately 50\% which correctly shows the absence of the \prop-causal feature in the latent representation. The accuracy increases as the  correlation $\kappa$ between the main and \prop-causal feature increases in dataset. This shows that even when \prop-causal feature is not present in the latent representation, probing classifier will still claim presence of \prop-causal feature if any correlated feature (main-task feature in this case) is present in the latent space. In \reffig{fig:app_prob_mnli_causal_rebalance}, the latent space contains the \prop-causal feature as shown by accuracy of approximately 94.5\% when $\kappa=0.5$. But as $\kappa$ increases the probing classifier's accuracy increases in the presence of correlated main-task feature which falsely increases the confidence of presence of the \prop-causal feature. The second row shows the spuriousness-score of \prop-probing classifier is increasing as the correlation between the main-task and \prop-causal feature increases which implies that the probing classifier is using relatively large \emph{amount} of correlated main-task feature for \prop-label prediction in all settings. 

For all the experiments \new{in this section (\refsec{subsec:extended_probing_result})} with \syn dataset,  \new{we didn't introduce any noise in the probing label (i.e. n=0.0)} and have 1 hidden layer when training the encoder (see \refsec{subsec:app_syn_encoder} for details). For the experiment on \mnli dataset, we use RoBERTa as the default encoder and rest of setup is same as described in \refsec{subsec:app_real_encoder}. 

\begin{figure*}[]
\begin{subfigure}[h]{0.32\textwidth}
    \centering
    \captionsetup{justification=centering}
    \includegraphics[width=\linewidth,height=1.0\linewidth]{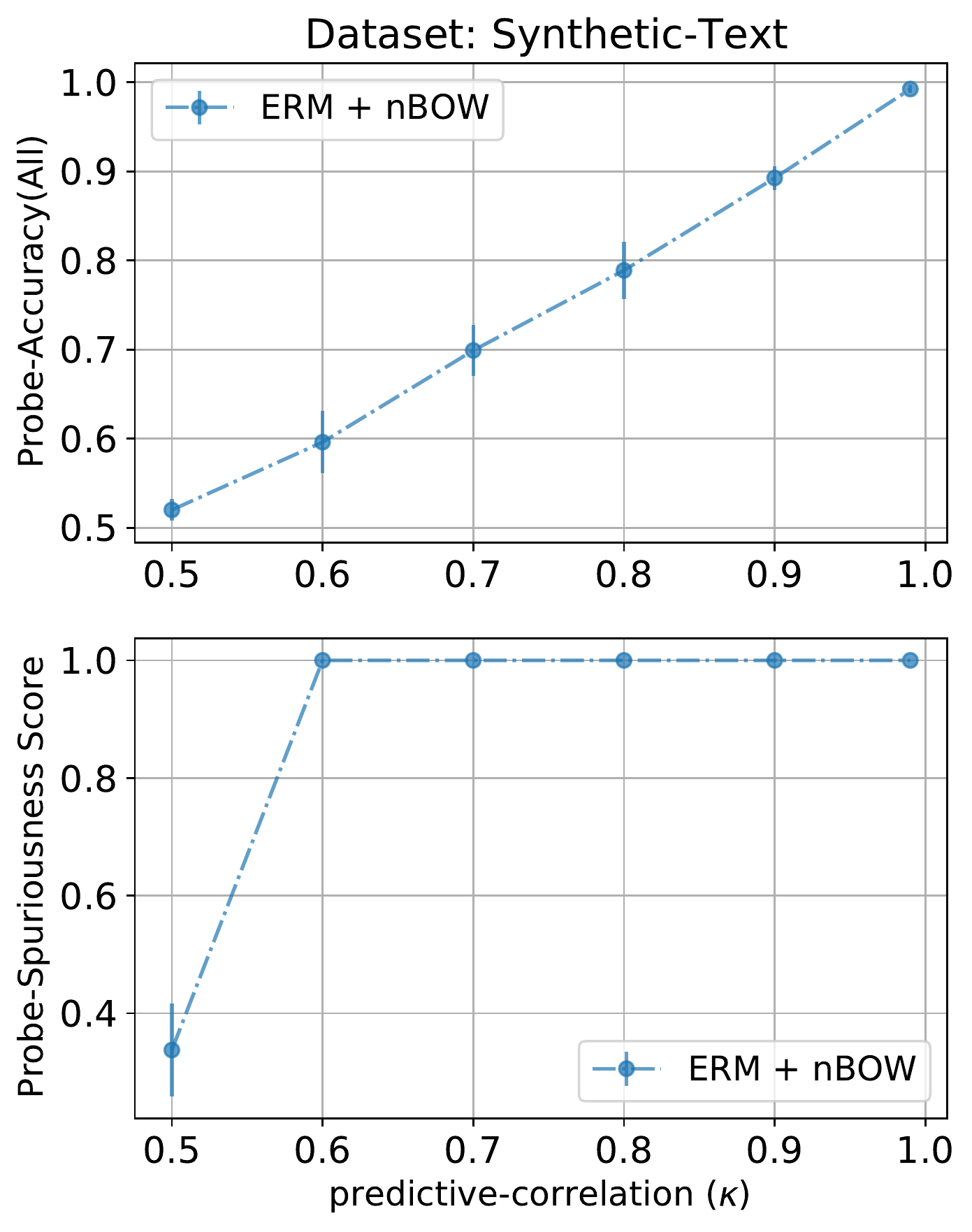}
    	\caption[]%
        {{\small  \syn + Probe Feature Absent}}    
    	\label{fig:app_prob_syn_causal_same} 
    \end{subfigure}
\hfill 
\begin{subfigure}[h]{.32\textwidth}
    \centering
    \captionsetup{justification=centering}
    \includegraphics[width=\linewidth,height=1.0\linewidth]{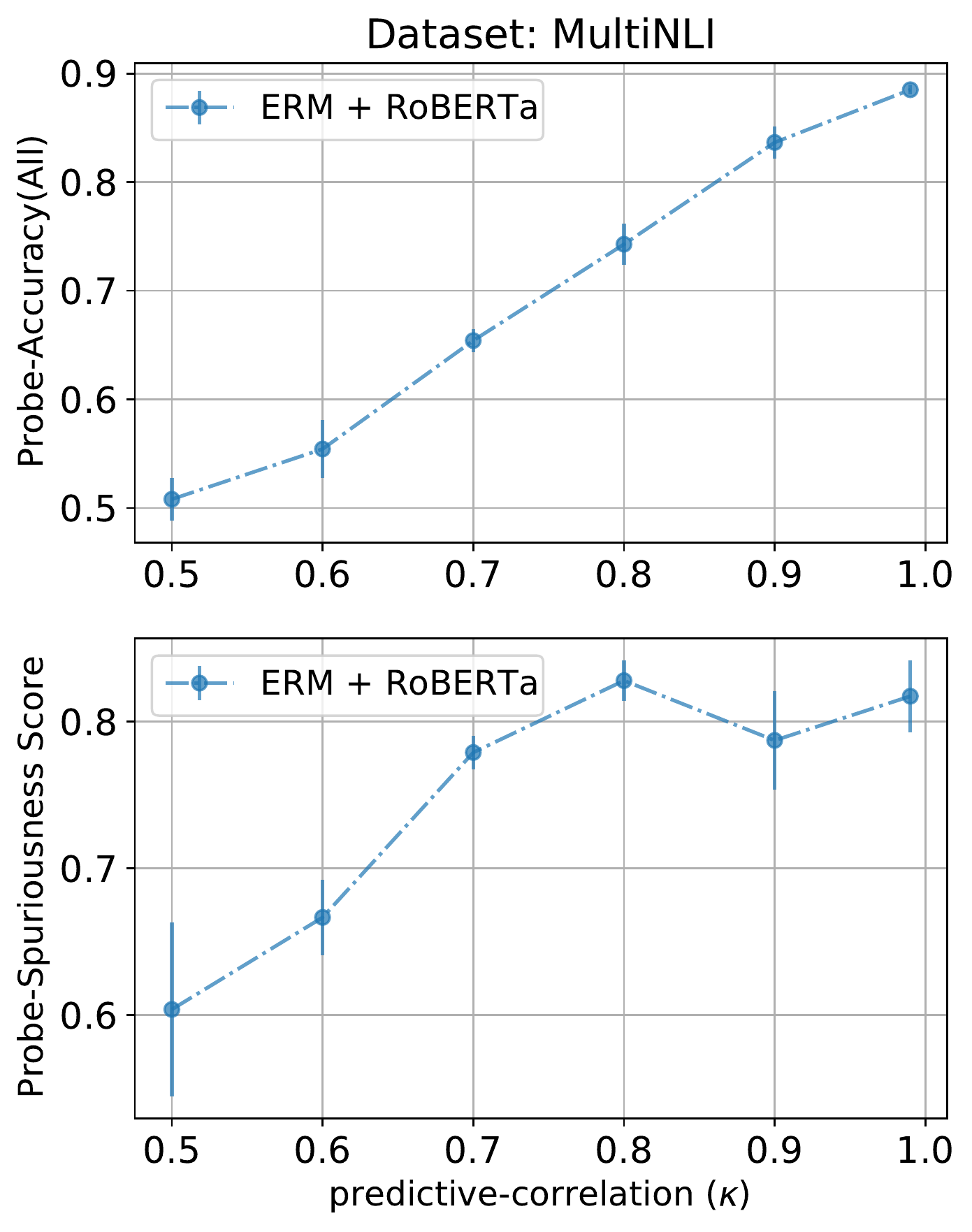}
    	\caption[]%
        {{\small \mnli + Probe Feature Absent}}    
    	\label{fig:app_prob_mnli_causal_same} 
    \end{subfigure}
\hfill 
\begin{subfigure}[h]{.32\textwidth}
    \centering
    \captionsetup{justification=centering}
    \includegraphics[width=\linewidth,height=1.0\linewidth]{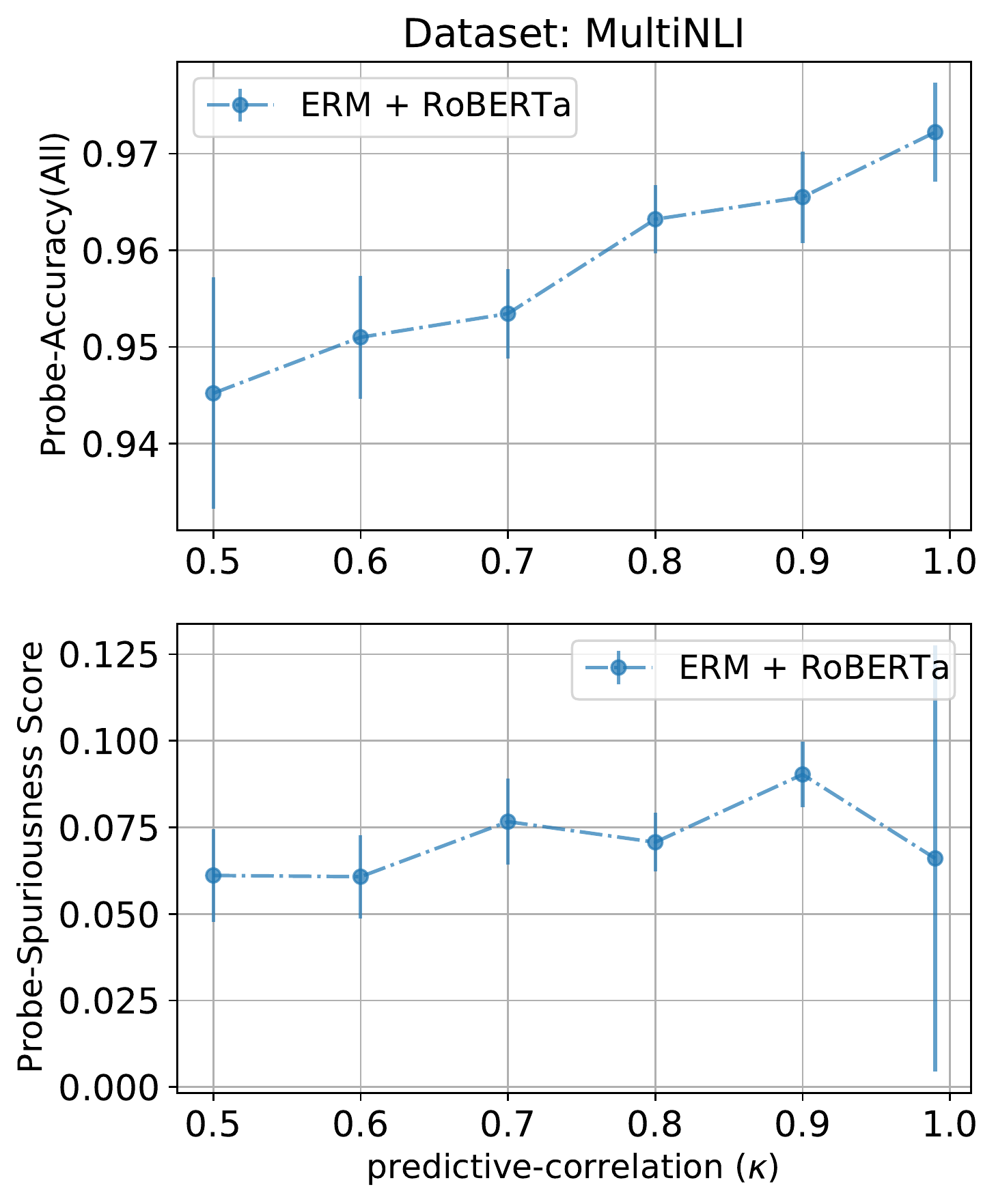}
    	\caption[]%
        {{\small \mnli + Probe Feature Present}}    
    	\label{fig:app_prob_mnli_causal_rebalance} 
    \end{subfigure}

\caption{\textbf{Failure Modes of Probing classifier:} The first row in \reffig{fig:app_prob_syn_causal_same} and \ref{fig:app_prob_mnli_causal_same} shows that even when the latent representation doesn't contain the probing \prop-causal feature, the probing classifier is still has >50\% accuracy when other correlated feature is present. The accuracy increases as the correlation $\kappa$ between the probing \prop-causal feature and other correlated features increases. The first row \reffig{fig:app_prob_mnli_causal_rebalance} shows  that presence of correlated features could increase the probing classifier's accuracy thus increasing the confidence in the presence of \prop-causal feature in latent representation. The second row of all the figures shows that the probing classifier is getting more spurious as the $\kappa$ increases thus implying that the probing classifier is using some other correlated feature than \prop-causal feature. For more discussion see \refappendix{subsec:extended_probing_result}.}

\label{fig:app_prob_all}
\end{figure*}

\subsection{Extended Null-Space Removal Results}
\label{subsec:app_extended_null_space_results}

\begin{figure*}[h]
\centering
\includegraphics[width=\textwidth,height=
\textwidth]{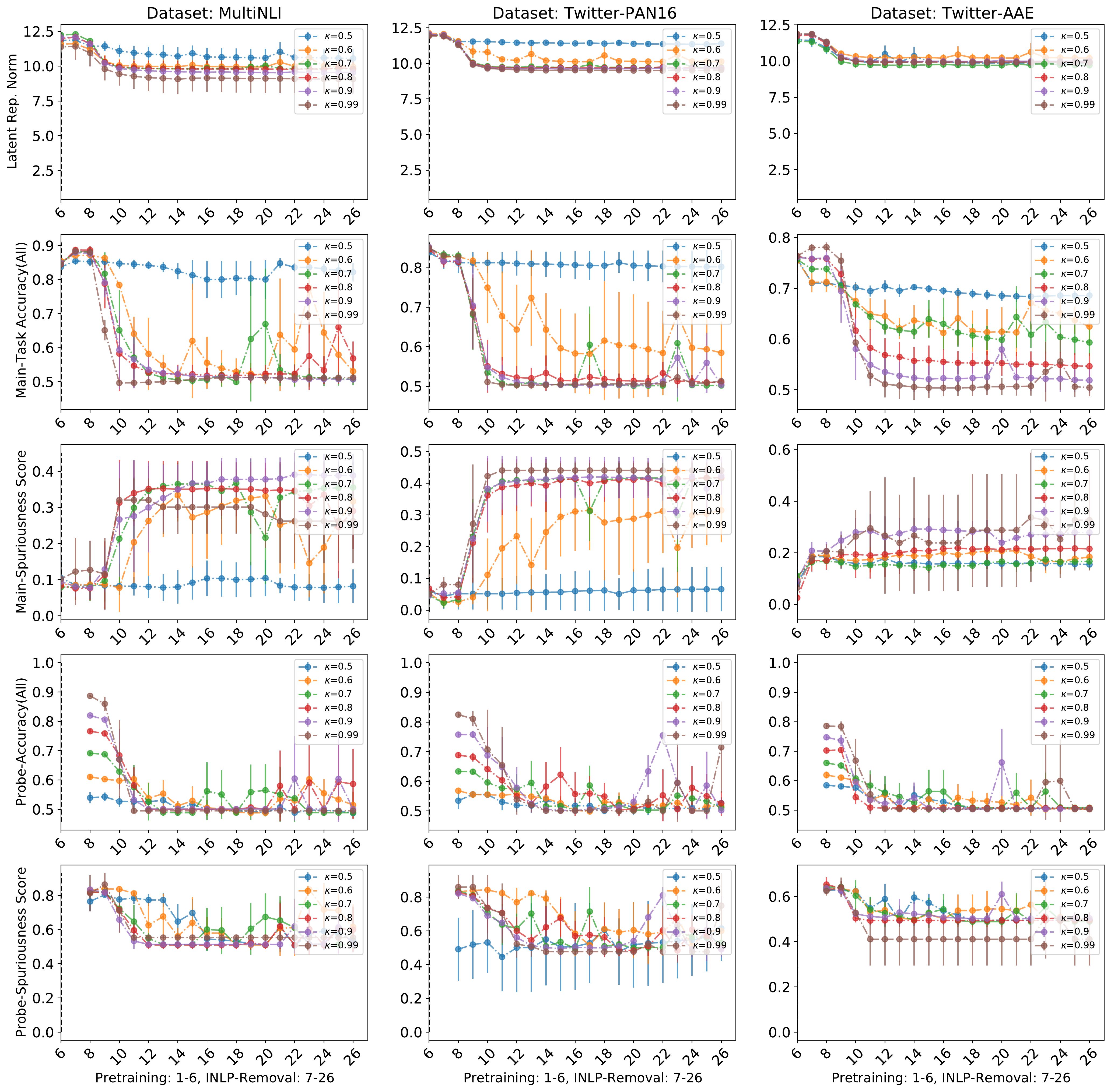}

\caption{\textbf{Failure of Null Space Removal when using RoBERTa as encoder:} Different columns of the figure are for three different real datasets --- \mnli, \pan, and \aae respectively. The x-axis from steps 8-26 is different \INLP removal steps. The y-axis shows different metrics to evaluate the main task and probing classifier. Different colored lines show the spurious correlation ($\kappa$) in the probing dataset used by \INLP for removal of spurious-\prop. The pretrained classifier is clean i.e. doesn't use the spurious \prop-causal feature, hence \INLP shouldn't have any effect on main classifier when removing \prop-causal feature from the latent space. Against our expectation, the second row shows that the main-task classifier's accuracy is decreasing even when it is not using the \prop-feature. The main reason for this failure to learn a clean \prop-probing classifier. This can be verified from the last row  which shows that the \prop-probing classifier has a high spuriousness score thus implying that it is using the main-task feature for \prop label prediction and hence during the removal step, wrongly removing the main-task feature which leads to a drop in main-task accuracy. For more discussion see \refappendix{subsec:app_extended_null_space_results}.\abhinav{At iter 7, pre projection, the probe acc is random guess. Why? Tried increasing the number of topic iterations. No effect. Removed currently.}}
\label{fig:app_expt_null_space_failure_roberta_all}
\end{figure*}

\begin{figure*}[h]
\centering
\includegraphics[width=\textwidth,height=
\textwidth]{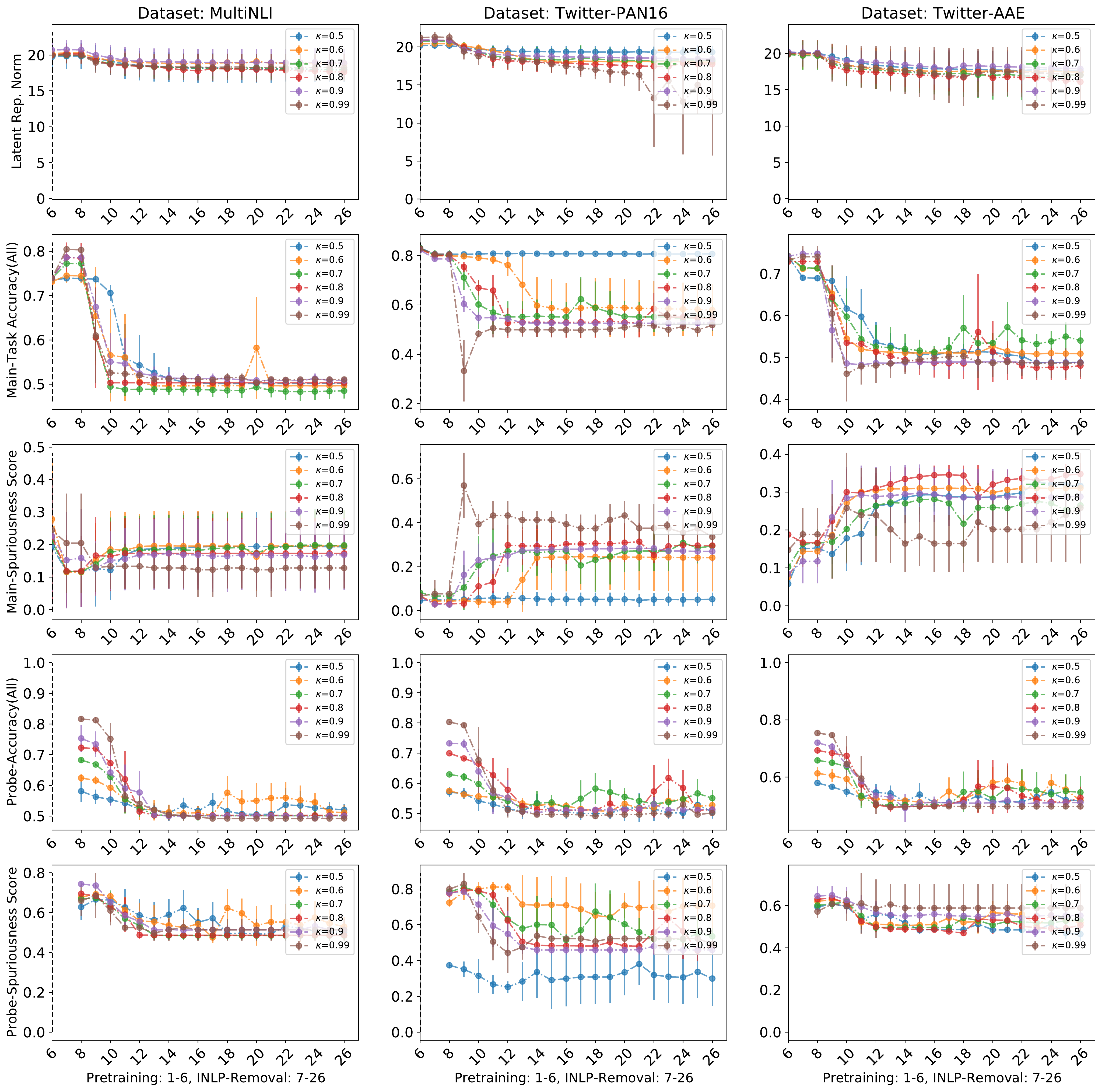}
\caption{\textbf{Failure of Null Space Removal when using BERT as encoder:} The observation is similar to the case when RoBERTa was used as encoder (see \reffig{fig:app_expt_null_space_failure_roberta_all}) . Different columns of the figure are for three different real datasets --- \mnli, \pan, and \aae respectively. The x-axis from steps 8-26 is different \INLP removal steps. The y-axis shows different metrics to evaluate the main task and probing classifier. Different colored lines show the spurious correlation ($\kappa$) in the probing dataset used by \INLP for removal of spurious-\prop. The pretrained classifier is clean i.e. doesn't use the spurious \prop-causal feature, hence \INLP shouldn't have any effect on main-classifier when removing \prop-causal feature from the latent space. Against our expectation, the second row shows that the main-task classifier's accuracy is decreasing even when it is not using the \prop-feature. The main reason for this failure to learn a clean \prop-probing classifier. This can be verified from the last row  which shows that the \prop-probing classifier has high spuriousness score thus implying that it is using the main-task feature for \prop label prediction and hence during the removal step, wrongly removing the main-task feature which leads to a drop in main-task accuracy. For more discussion see \refappendix{subsec:app_extended_null_space_results}.}
\label{fig:app_expt_null_space_failure_bert_all}
\end{figure*}

\reffig{fig:app_expt_null_space_failure_roberta_all} and \ref{fig:app_expt_null_space_failure_bert_all}, shows the failure mode of null-space removal (\INLP) in the real dataset when using RoBERTa and BERT as encoders respectively. Different columns of the figure are for three different real datasets --- \mnli, \pan, and \aae respectively. The x-axis from steps 8-26 is different \INLP removal steps. The y-axis shows different metrics to evaluate the main task and probing classifier. Different colored lines show the spurious correlation ($\kappa$) in the probing dataset used by \INLP for the removal of spurious-\prop. The pretrained classifier is clean, i.e., does not use the spurious \prop-causal feature; hence \INLP shouldn't have any effect on main-classifier when removing \prop-causal feature from the latent space. The first row shows that as the \INLP iteration progresses, the norm of latent representation, which is being \emph{cleaned} of \prop-causal feature, decreases. This indicates that some features are being removed. However, the results are against our expectation from the second statement of \reftheorem{theorem:null_space_failure}, which states that the norm of the classifier will tend to zero as the \INLP removal progresses. The possible reason is that from \reftheorem{theorem:null_space_failure} the norm of latent representation will go to zero when the latent representation only contains the spurious \prop-causal feature and the other features correlated to it. But, the encoder representation could have other features which are not correlated with \prop-label and hence not removed. Since, the pretrained classifier given for \INLP was \emph{clean} (using  method described in \refsec{subsec:app_clean_dataset}), we do not expect the \INLP to have any effect on the main-task classifier. 

The second row in \reffig{fig:app_expt_null_space_failure_roberta_all} and \ref{fig:app_expt_null_space_failure_bert_all} shows that the main classifier accuracy drops to random guess i.e 50\% except for the case when probing dataset have $\kappa=0.5$ i.e no correlation between the main and \prop label. Thus \INLP method corrupted a clean classifier and made it useless. The reason behind this could be observed from the fourth and fifth rows. The fourth row shows the accuracy of the probing classifier before the projection step. We can see that at step 8 on the x-axis $\kappa=0.5$, the probing classifier correctly has an accuracy of 50\% showing that the \prop-causal feature is not present in the latent representation. But for other values of $\kappa$, the probing classifier accuracy is proportional to the value of $\kappa$ implying that the probing classifier is using the main-task feature for its prediction. Hence at the time of removal, it removes the main-task feature which leads drop in the main-task accuracy. This can also be verified from the last row of \reffig{fig:app_expt_null_space_failure_roberta_all} and \ref{fig:app_expt_null_space_failure_bert_all}, which shows that the spuriousness score of probing classifier is high; thus it is using the main-task feature for its prediction. We observe similar results for \syn dataset when using \INLP in \reffig{fig:app_inlp-toy_all}. For all the \INLP experiment on \syn dataset, there were no hidden layers after the nBOW encoder (see \refsec{subsec:app_syn_encoder}).

\new{So far, we have kept the main-task classifier frozen when performing \INLP removal. Note that, we also experiment with the setting when the main task classifier is trained after every projection step of \INLP (see \refsec{subsec:app_inlp_expt_setup} for experimental setup and  \reffig{fig:app_inlp-head_retrain_mix} for a result description). We observe a similar drop in the main-task accuracy with prolonged removal using \INLP and early stopping leads to an even higher reliance on the spurious \prop-causal feature than it had at the beginning of \INLP. The rest of the experimental configurations were kept the same as the other \INLP experiments described above.}

\begin{figure*}[h]
\begin{subfigure}[h]{0.32\textwidth}
    \centering
    \includegraphics[width=\linewidth,height=2.9\linewidth]{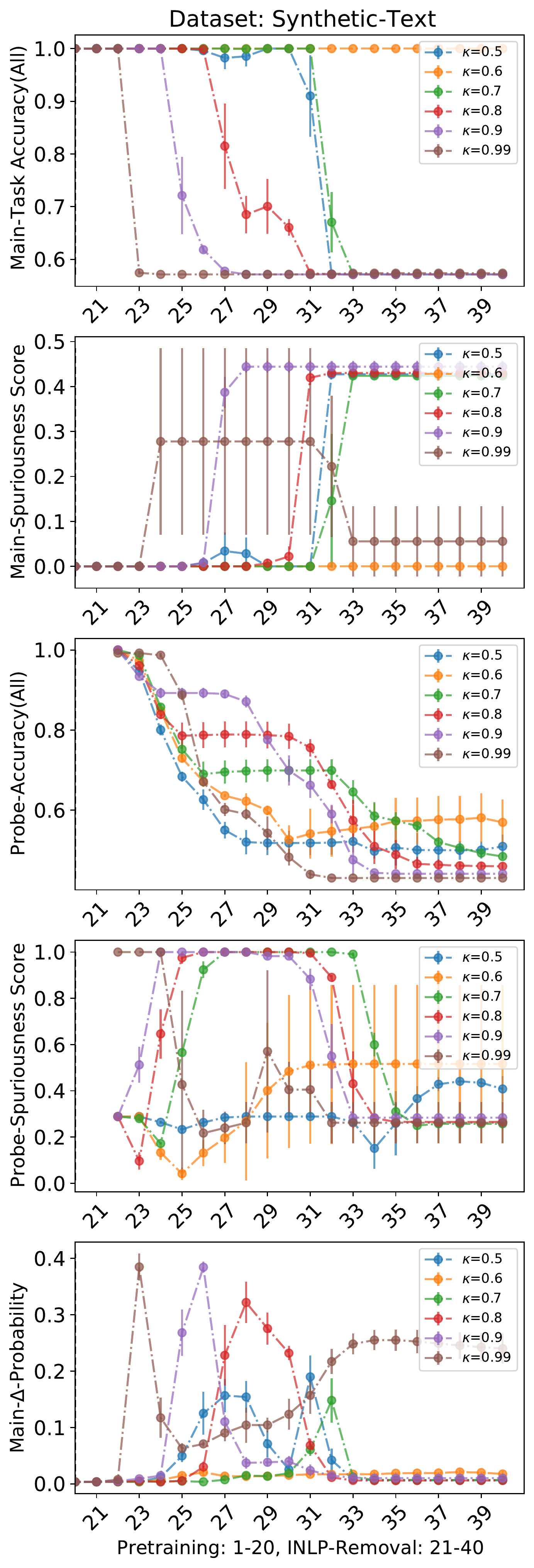}
    	\caption[]%
        {{\small  \syn + n=0.0}}    
    	\label{fig:app_inlp_syn_n0.0} 
    \end{subfigure}
\hfill 
\begin{subfigure}[h]{.32\textwidth}
    \centering
    \includegraphics[width=\linewidth,height=2.9\linewidth]{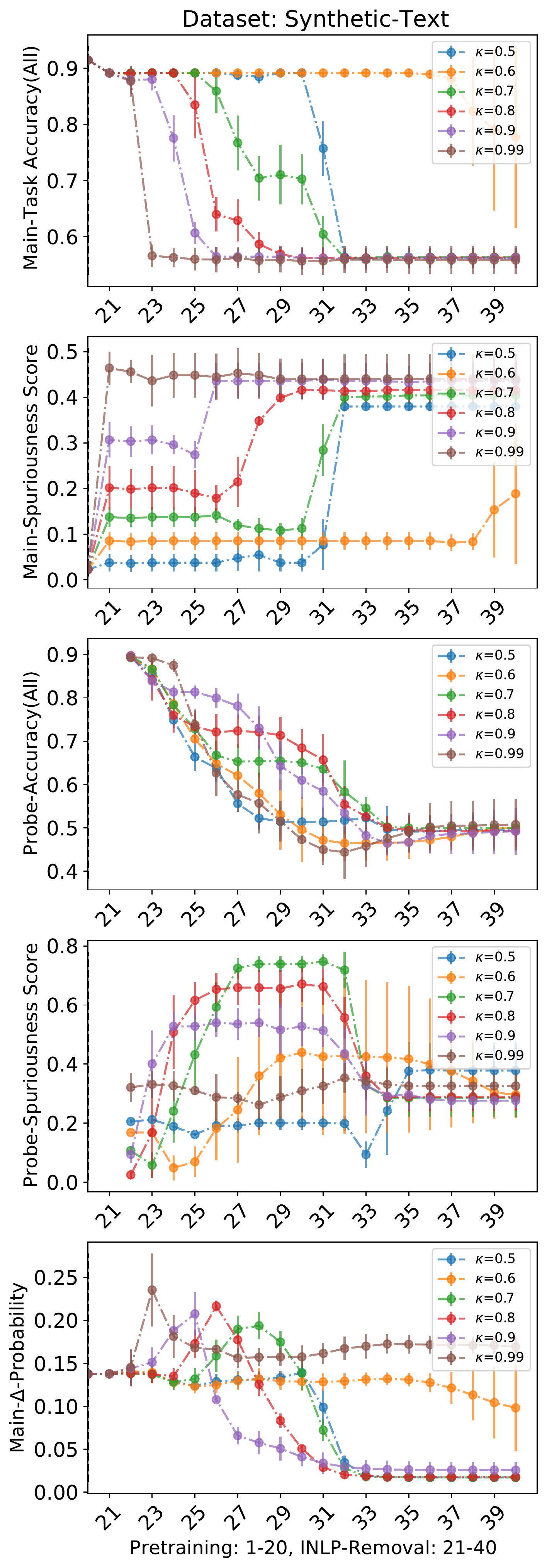}
    	\caption[]%
        {{\small \syn + n=0.1}}    
    	\label{fig:app_inlp_syn_n0.1} 
    \end{subfigure}
\hfill 
\begin{subfigure}[h]{.32\textwidth}
    \centering
    \includegraphics[width=\linewidth,height=2.9\linewidth]{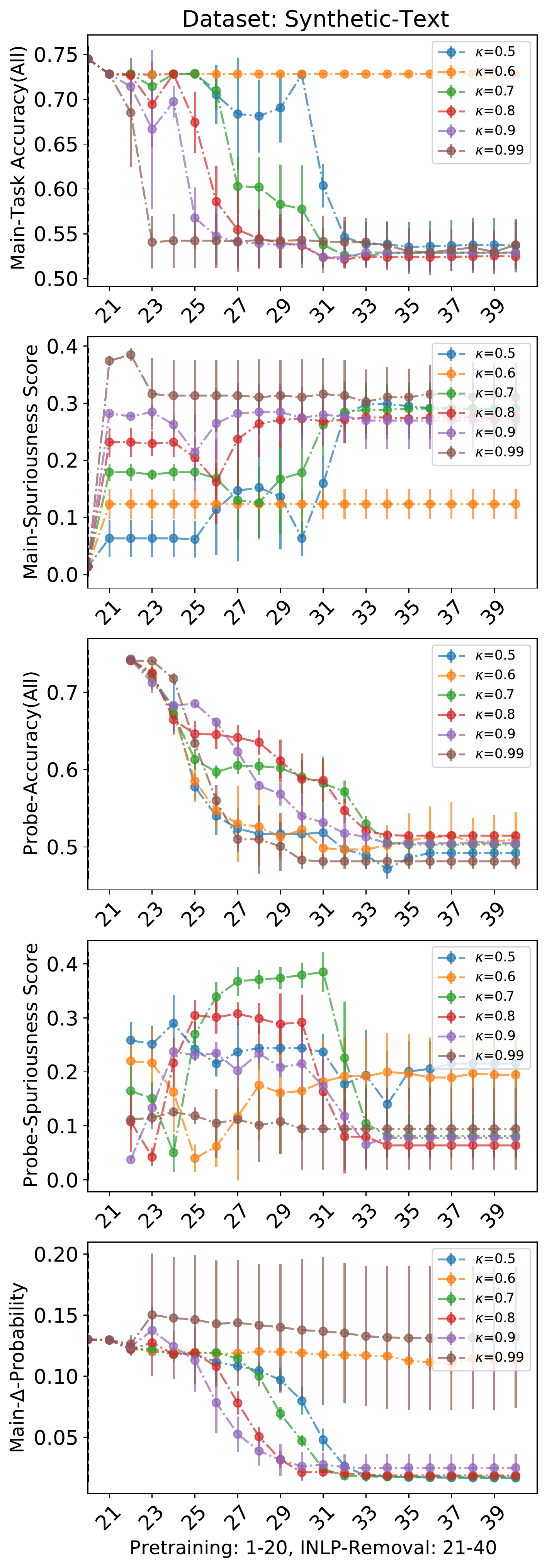}
    	\caption[]%
        {{\small \syn + n=0.3}}    
    	\label{fig:app_inlp_syn_n0.3} 
    \end{subfigure}
\caption{\textbf{Failure Mode of \INLP in \syn dataset}: Different columns of the figure are \syn dataset with different levels of noise in the main task and probing task label. \new{Here, n=0.0 means there is 0\% noise and n=0.3 means there is 30\% noise in the labels}. The x-axis from steps 22-40 is different \INLP removal steps. The y-axis shows different metrics to evaluate the main task and probing classifier. Different colored lines show the spurious correlation ($\kappa$) in the probing dataset used by \INLP for the removal of spurious-\prop. The pretrained classifier is clean i.e. doesn't use the spurious \prop-causal feature, hence \INLP shouldn't have any effect on main classifier when removing \prop-causal feature from the latent space. Contrary to our expectation, the first row shows main-task classifier accuracy drops as the \INLP progresses. Higher the correlation between the main-task and \prop label, faster the drop in the main task accuracy. The last row shows the change in prediction probability ($\Delta$-Prob) of main-task classifier when we change the input corresponding to \prop-label. This shows, how much sensitive the main task classifier is wrt. to \prop feature. We observe that the $\Delta$-Prob increases in the middle of \INLP showing that the main-classifier which was not using the \prop initially (as in iteration 21), started using the sensitive \prop because of \INLP removal. Thus stopping \INLP prematurely could lead to a more \emph{unclean} classifier than before whereas running \INLP longer removes all the correlated features and could make the classifier useless. For more discussion see \refappendix{subsec:app_extended_null_space_results}.}
\label{fig:app_inlp-toy_all}
\end{figure*}

\begin{figure*}[h]
\begin{subfigure}[h]{0.32\textwidth}
    \centering
    \includegraphics[width=\linewidth,height=2.9\linewidth]{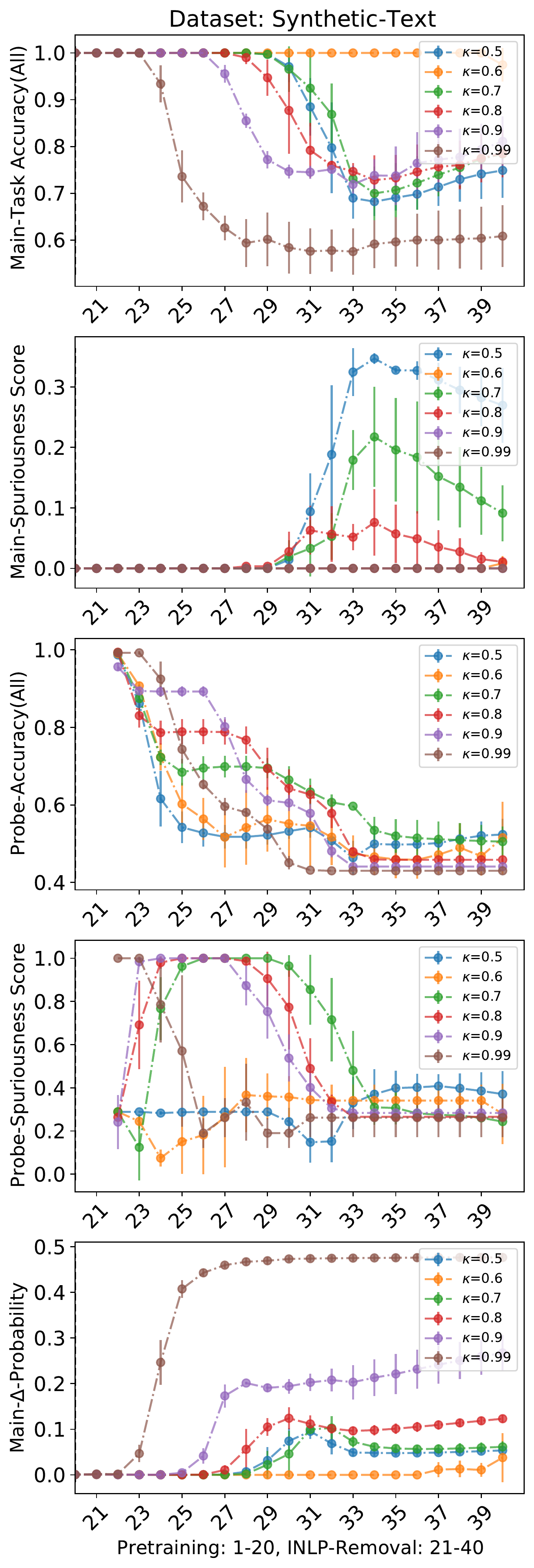}
    	\caption[]%
        {{\small  \syn + n=0.0}}    
    	\label{fig:app_inlp_hretrain_syn_n0.0} 
    \end{subfigure}
\hfill 
\begin{subfigure}[h]{.32\textwidth}
    \centering
    \includegraphics[width=\linewidth,height=2.9\linewidth]{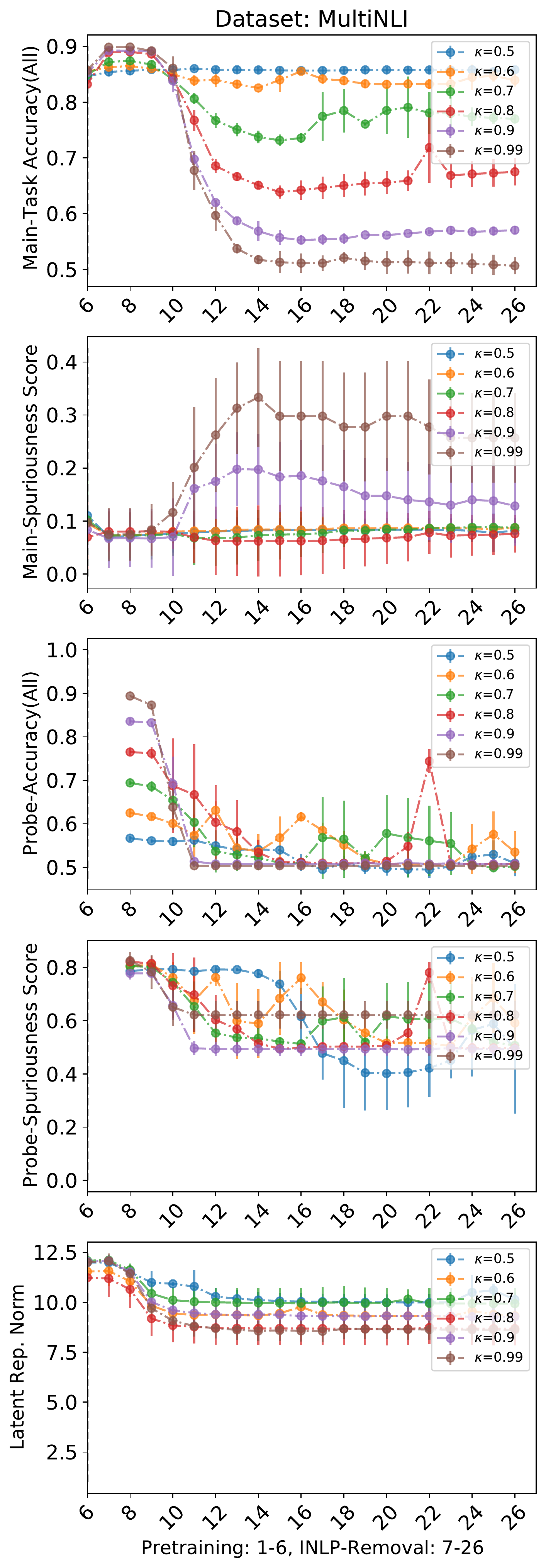}
    	\caption[]%
        {{\small \mnli + RoBERTa}}    
    	\label{fig:app_inlp_hretrain_mnli} 
    \end{subfigure}
\hfill 
\begin{subfigure}[h]{.32\textwidth}
    \centering
    \includegraphics[width=\linewidth,height=2.9\linewidth]{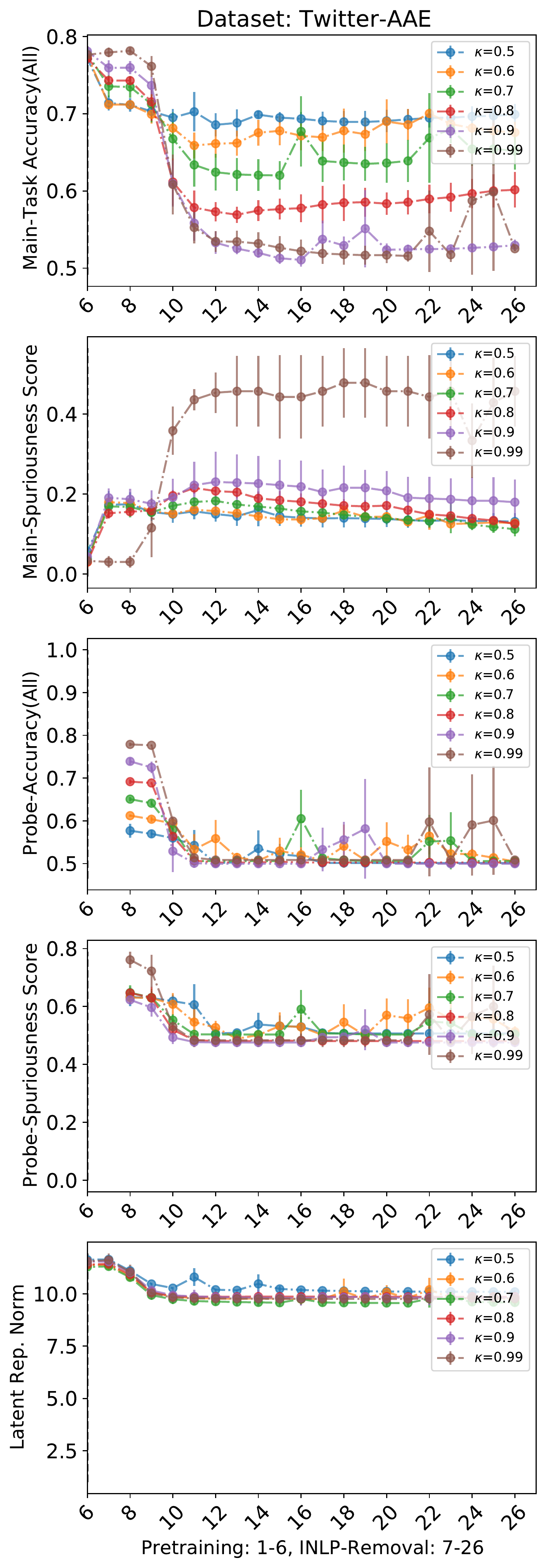}
    	\caption[]%
        {{\small \aae + RoBERTa}}    
    	\label{fig:app_inlp_hretrain_aae} 
    \end{subfigure}
\caption{\new{\textbf{Failure Mode of \INLP + Main Task Classifier head retraining}: Given a pretrained encoder and the main task classifier as input to \INLP for spurious \prop removal, in this experiment, we retrain the main task classifier after every step of null-space projection by \INLP. All the other experiment configurations for these experiments are kept the same as the case when we don't retrain the main-task classifier. The first, second, and third columns show the results for \syn, \mnli, and \aae datasets respectively. We observe a similar trend as the case when the main task classifier was not trained after each projection step (see \reffig{fig:app_expt_null_space_failure_roberta_all}, \ref{fig:app_expt_null_space_failure_bert_all} and \ref{fig:app_inlp-toy_all}). The main task classifier's accuracy drops as the null-space removal proceeds (iteration 21-40 for \syn and iteration 7-26 for \mnli and \aae datasets). Though the drop is not as severe as in the previous setting (when we didn't train the main task classifier), it is significant enough to impact the practical utility of the model (greater than 20\% drop in the accuracy when $\kappa>0.8$ for all the datasets above). Similar to previous setting, early-stopping of \INLP removal may lead to a classifier that has a higher reliance on the spurious \prop than it had before the \INLP removal. For example, for $\kappa=0.8$ in \syn dataset, the main-task classifier's performance drops for the first time at iteration 29 (a valid heuristic for early stopping), but it has high $\Delta$Prob $\approx 10\%$ as shown in the last row of the \syn dataset column of this figure. For discussion of the case when the main task classifier is not trained after every projection step, see  \refsec{subsec:app_extended_null_space_results} and \refsec{subsec:expt_null_space_failure}.}}
\label{fig:app_inlp-head_retrain_mix}
\end{figure*}

\subsection{Extended Adversarial Removal Results}
\label{subsec:app_extended_adversarial_removal_results}

\begin{figure*}[]
\centering
\includegraphics[width=\textwidth,height=0.8
\textwidth]{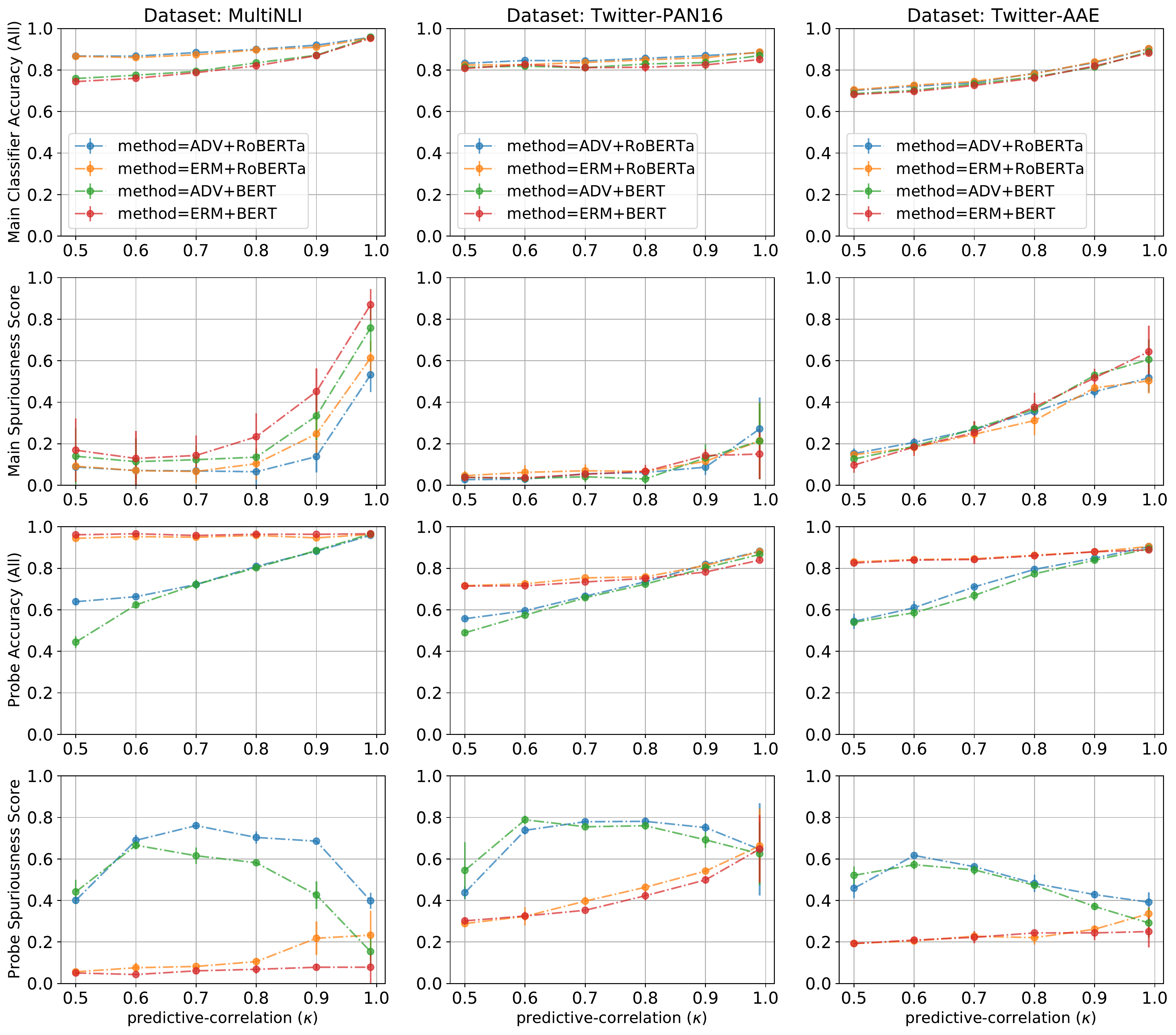}
\caption{\textbf{Failure Mode of Adversarial removal on real-dataset:} Different column shows the result on three different real datasets ---\mnli, \pan, and \aae respectively. The second row shows the accuracy of spuriousness score of the main-task classifier after \ar when the dataset contains different levels of spurious correlation between the main-task and unwanted-\prop label, denoted by $\kappa$ in the x-axis. When using RoBERTa as the encoder, the orange curve in second row shows the spuriousness score of the main-task classifier when trained using the ERM loss. The spuriousness score describes how much unwanted \prop-causal feature the main task classifier is using. The blue curve shows that the \ar method reduces the spuriousness of main-task though cannot completely remove it. When using BERT as encoder, the observation is same i.e green curve in second row shows \ar is able to reduce the spuriousness of main classifier than the red curve which is trained using ERM, but is not able to completely remove the spurious feature.  For more discussion see \refsec{subsec:app_extended_adversarial_removal_results}.}
\label{fig:app_adv_real}
\end{figure*}

\paragraph{Adversarial removal failure in real-world datasets.}
\reffig{fig:app_adv_real} shows the failure mode of adversarial removal \ar on real-world datasets. In the x-axis we vary the predictive correlation $\kappa$ between the main and the \prop-label in different datasets and measure the performance of \ar on different metrics on the y-axis. The second row shows the spuriousness score of the main-task classifier after \ar as we vary $\kappa$ on the x-axis. When using RoBERTa as the encoder, the orange curve in second row shows the spuriousness score of the main-task classifier when trained using the ERM loss. The spuriousness score describes how much unwanted \prop-causal feature the main-task classifier is using. The blue curve shows that the \ar method reduces the spuriousness of main-task though cannot completely remove it. The reason for this failure can be attributed to probing classifier. Even when \ar has successfully removed the unwanted \prop feature, the accuracy of \prop-probing classifier will be proportion to $\kappa$ due to presence of correlated main-task feature in the latent space. This can be seen in the third row of \reffig{fig:app_adv_real}. Thus we cannot be sure if the unwanted \prop-causal feature has been completely removed from the latent space or just became noisy enough to have accuracy proportional to $\kappa$ after \ar converges. In \reffig{fig:app_adv_real}, for each dataset and encoder, we manually choose the hyperparameter described $\lambda$ described in \refsec{subsec:app_adv_expt_setup_desc} which reduces the spuriousness score most for the main-task classifier while not hampering the main-task classifier accuracy. In \reffig{fig:app_adv_real_lambda_var}, we show the trend in spuriousness score is similar for all choices of hyperparameter $\lambda$ in our search. No value of $\lambda$ is able to completely reduce the spuriousness score to zero.

\begin{figure*}[]
\centering
\includegraphics[width=\textwidth,height=0.5
\textwidth]{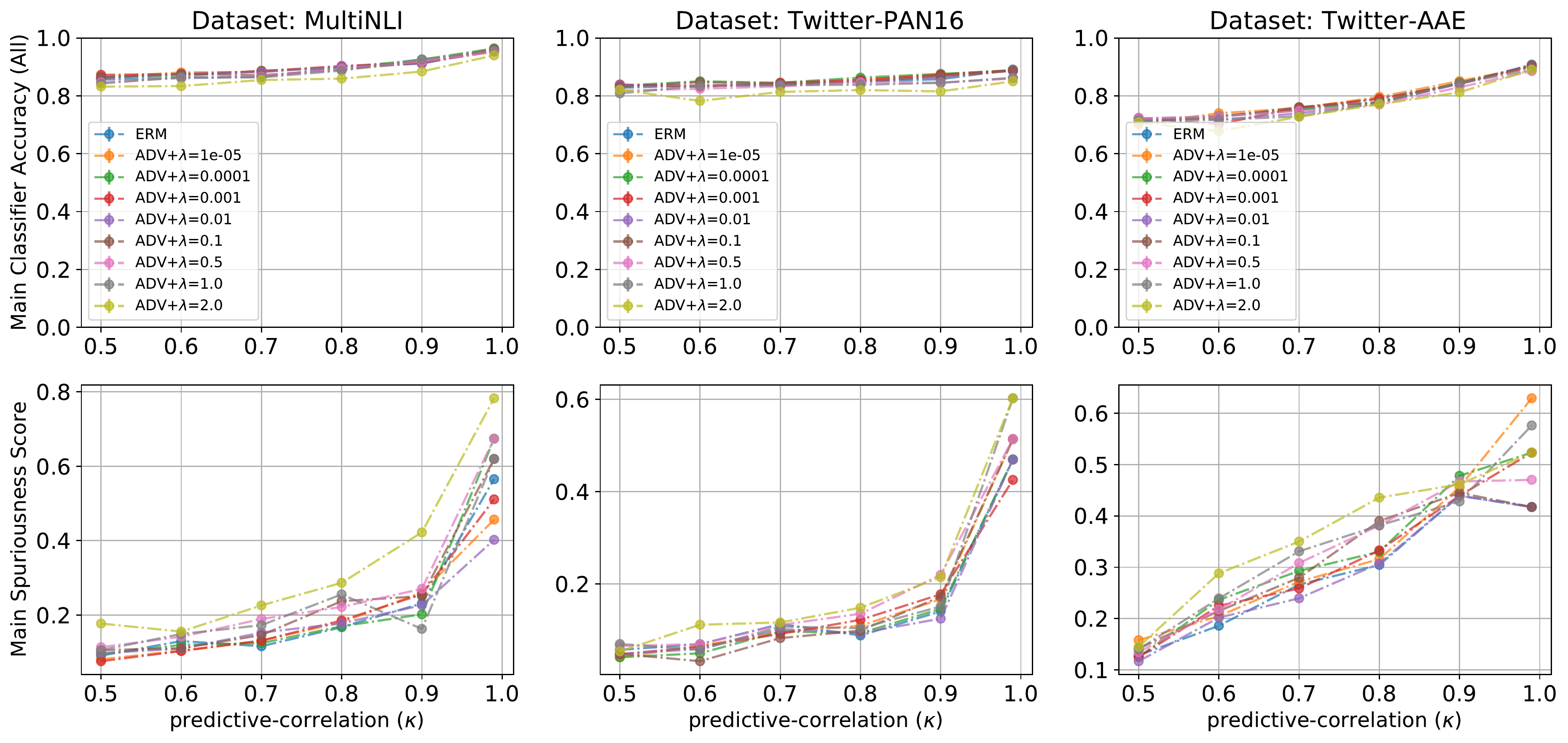}
\caption{\textbf{Choice of Adversarial Strength Parameter $\lambda$}: The second plot shows that trend in spuriousness score after \ar is similar for all the choices of hyperparameter $\lambda$ we have taken in our search. None of the settings of $\lambda$ is able to completely reduce the spuriousness score to zero. For more discussion see \refsec{subsec:app_extended_adversarial_removal_results}.}
\label{fig:app_adv_real_lambda_var}
\end{figure*}

\paragraph{Adversarial removal makes a classifier clean.} \reffig{fig:app_adv_unfair} shows that when the adversarial classifier is initialized with a clean main-task classifier that doesn't use unwanted-\prop causal features, it makes matters worse by making the main-task classifier use the unwanted-\prop feature. For the \syn dataset, since the word embeddings are non-trainable, one single hidden layer is applied after the nBOW encoder so that \ar methods could remove the unwanted-\prop feature from the new latent representation. We create a \emph{clean} \syn dataset by training a classifier (iteration 1-20) on dataset with predictive correlation $\kappa=0.5$ between the main-task and \prop label. $\kappa=0.5$ which implies there is no correlation between the labels thus  we can expect the main-task classifier to not use the \prop-causal feature. This can be seen from Main classifier spuriousness score in \reffig{fig:app_adv_syn_unfair} (2nd row) which is close to 0. We chose this method to create a clean classifier since this allows us to measure the spuriousness score for the main-task classifier. If we would have followed method described in \cite{ravichander-etal-2021-probing}, then we would have had only a single value of \prop label ($y_{p}$) in the dataset and couldn't have defined the majority and minority group required for calculation of spuriousness score (see \refdef{def:spurriousness_score}). For all our experiments on \syn dataset we use noise =0.3 and trained the main-task and probing classifier with 1 hidden layer.
Similarly for training a clean classifier for \mnli dataset (iteration 1-6) we again use a dataset with predictive correlation $\kappa=0.5$. Post training the clean classifier the \ar method is initialized with these clean classifiers for removal of \prop-causal features. Since \ar is initialized with clean classifier which doesn't use \prop-causal feature, we expect \ar to have no effect on the classifier. In contrast we observe that the spuriousness score of main-classifier for both \syn and \mnli dataset increases (2nd row in \reffig{fig:app_adv_syn_unfair} and \ref{fig:app_adv_mnli_unfair}) which shows that \ar when initialized \emph{clean}/fair classifier could make them unclean/unfair.

\begin{figure*}[]
\centering
% \includegraphics[width=\textwidth,height=0.5
% \textwidth]{figs/adv-real-roberta-base-both-final.png}

\begin{subfigure}[h]{0.48\textwidth}
    \centering
    \includegraphics[width=\linewidth,height=1.9\linewidth]{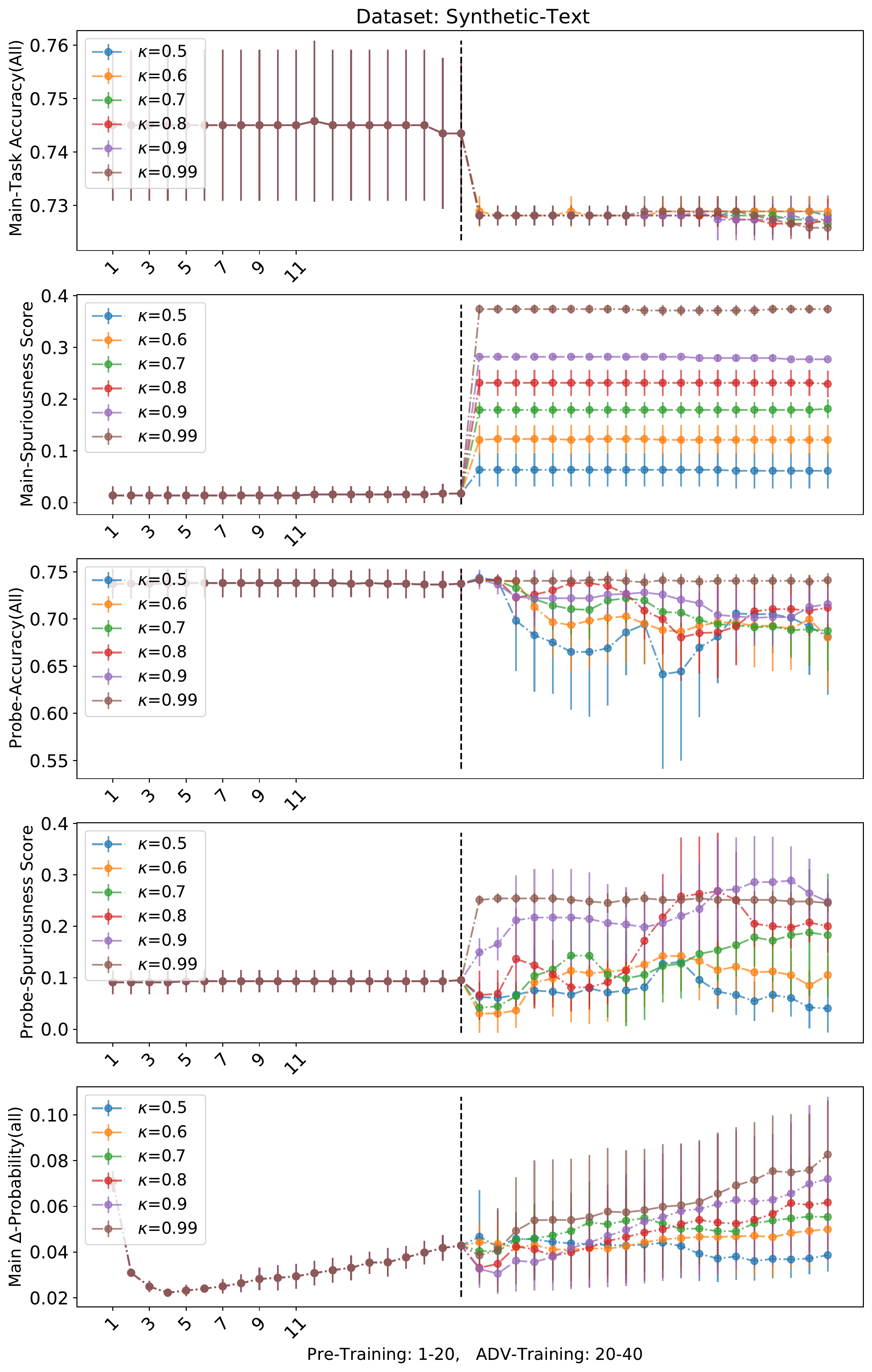}
    	\caption[]%
        {{\small  \syn + \new{n=0.3} + nBOW}}    
    	\label{fig:app_adv_syn_unfair} 
    \end{subfigure}
\hfill 
\begin{subfigure}[h]{.48\textwidth}
    \centering
    \includegraphics[width=\linewidth,height=1.9\linewidth]{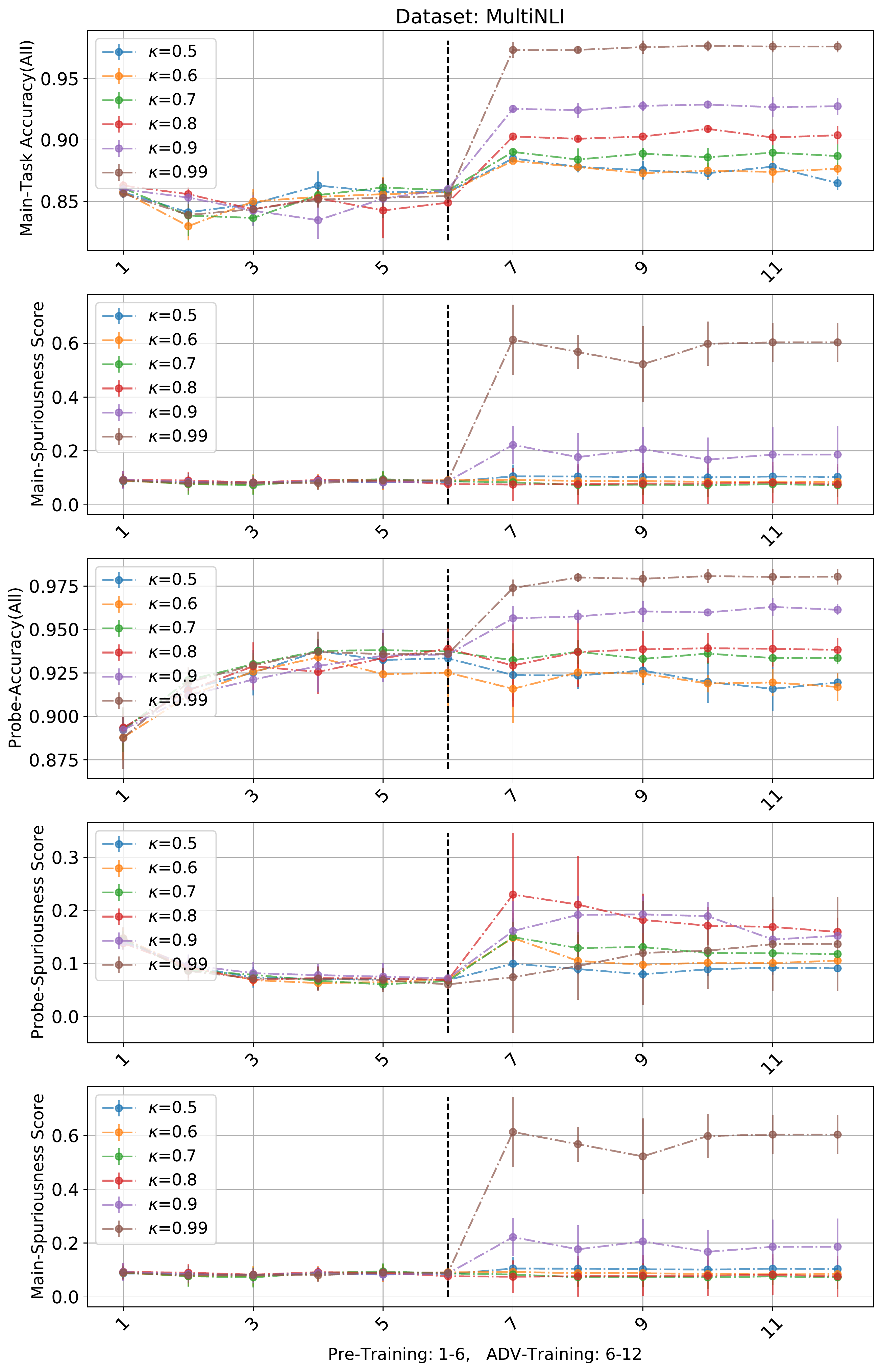}
    	\caption[]%
        {{\small \mnli + RoBERTa}}    
    	\label{fig:app_adv_mnli_unfair} 
    \end{subfigure}

\caption{\textbf{Adversarial Removal Makes a classifier unclean}: We test if the \ar method increases the spuriousness of a main-task classifier if initialized with a \emph{clean} classifier. In \ref{fig:app_adv_syn_unfair}, from iteration 1-20 in x-axis, a clean classifier is trained on \syn dataset (\new{with 30\% noise i.e n=0.3 in main-task and probing labels})  such that it doesn't uses the unwanted \prop-causal feature by training on a dataset with $\kappa=0.5$ (see \refsec{subsec:app_extended_adversarial_removal_results} for details). Then the classifier is given to \ar method for removing the unwanted \prop feature which makes the initially clean classifier unclean. This can be seen from the second row of the \ref{fig:app_adv_syn_unfair} which shows the spuriousness score of main-classifier is $0$ during 1-20 iteration but increases after the \ar start from 21-40. Also, the last row shows the $\Delta$-Prob of the main-task classifier on changing the unwanted-\prop in input which increases for datasets which have large $\kappa$ i.e correlation between the main and \prop label. A similar result can be seen for the \mnli dataset where a clean classifier is trained in iterations 1-6 (using a dataset with $\kappa=0.5$) which is made unclean by \ar. Second row again shows that spuriousness score of main-task classifier increases after \ar starts in iteration 7-12. For more discussion see \refsec{subsec:app_extended_adversarial_removal_results}.}
\label{fig:app_adv_unfair}
\end{figure*}

\subsection{\syn dataset Ablations}
\label{subsec:app_extended_synthetic_results}

\paragraph{Adversarial Removal Failure in \syn dataset:} Figure~\ref{fig:app_adv-toy_all} shows the failure of \ar on the synthetic dataset as we vary the noise in the main-task label and unwanted \prop-label. For the experiment, since the word embeddings are non-trainable, one single hidden layer is applied after the nBOW encoder so that \ar methods could remove the unwanted-\prop feature from the new latent representation. 

\begin{figure*}[]
\begin{subfigure}[h]{0.32\textwidth}
    \centering
    \includegraphics[width=\linewidth,height=2.9\linewidth]{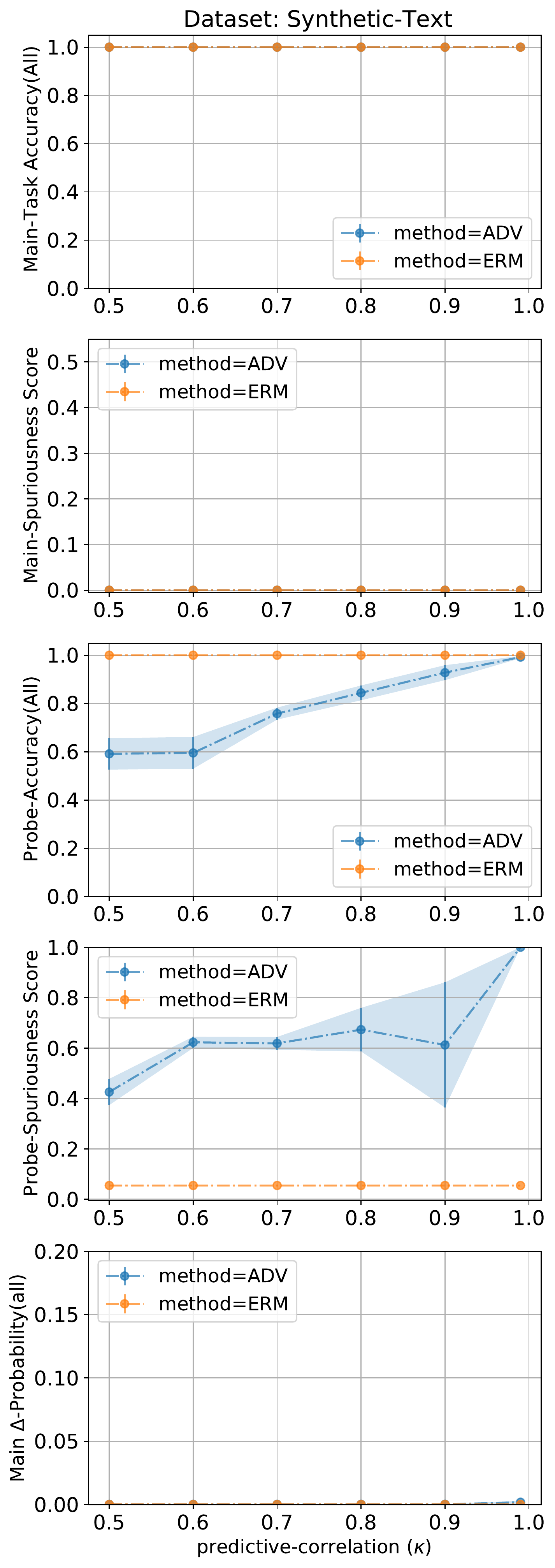}
    	\caption[]%
        {{\small  \syn + n=0.0}}    
    	\label{fig:app_adv_syn_n0.0} 
    \end{subfigure}
\hfill 
\begin{subfigure}[h]{.32\textwidth}
    \centering
    \includegraphics[width=\linewidth,height=2.9\linewidth]{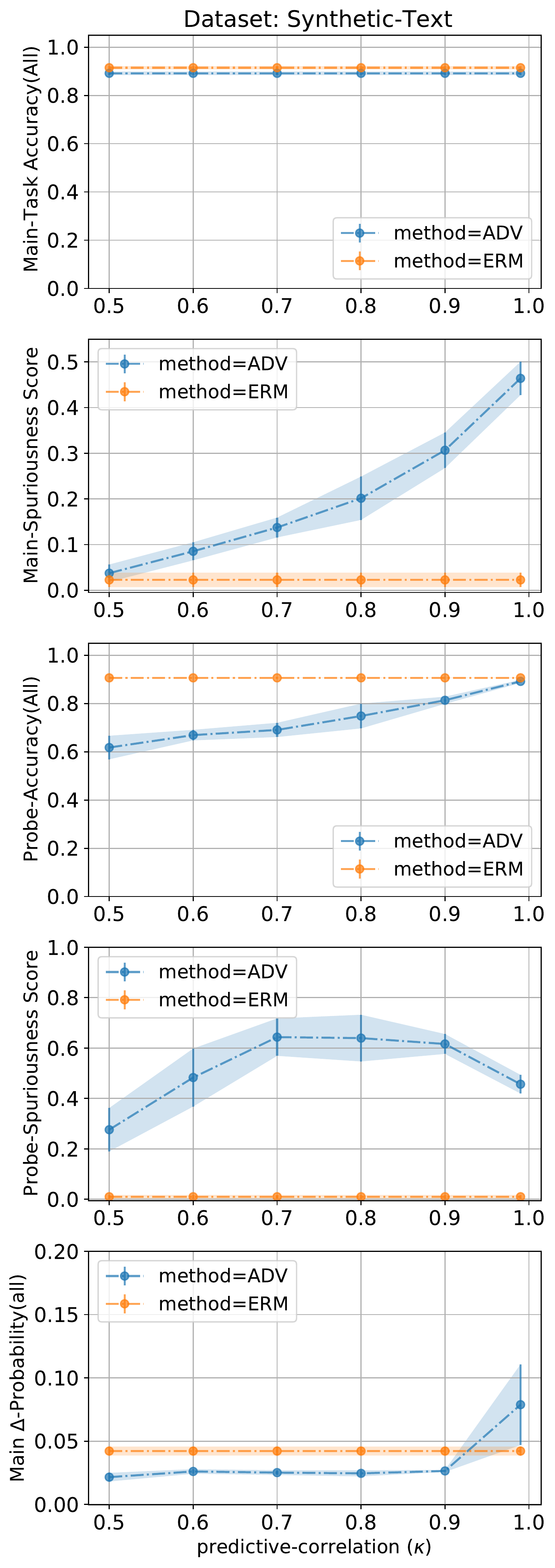}
    	\caption[]%
        {{\small \syn + n=0.1}}    
    	\label{fig:app_adv_syn_n0.1} 
    \end{subfigure}
\hfill 
\begin{subfigure}[h]{.32\textwidth}
    \centering
    \includegraphics[width=\linewidth,height=2.9\linewidth]{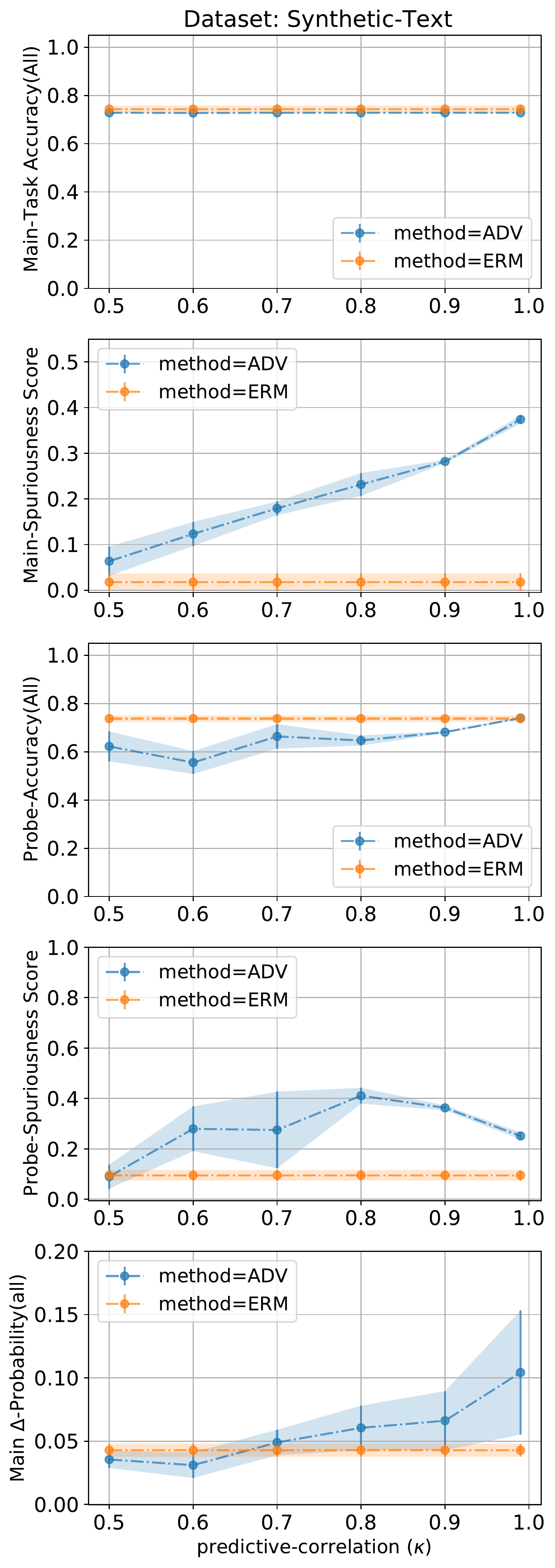}
    	\caption[]%
        {{\small \syn + n=0.3}}    
    	\label{fig:app_adv_syn_n0.3} 
    \end{subfigure}
\caption{\textbf{Failure of Adversarial Removal method on \syn dataset:} Different columns show the adversarial removal method on \syn dataset with different levels of noise in the main-task and \prop label. When there is no noise, from the second row in \reffig{fig:app_adv_syn_n0.0}, we see that both the classifier trained by ERM and \ar has zero-spuriousness score. But as we increase the noise to 10\% in \reffig{fig:app_adv_syn_n0.1}, we observe that the spuriousness score increases when \ar is applied in contrast to classifier trained by ERM which stays at 0. Also, higher the predictive correlation $\kappa$, higher the increase in spuriousness. This observation augments the observation in \reffig{fig:app_adv_unfair} which shows that using \ar makes a clean classifier unclean. Similarly in \reffig{fig:app_adv_syn_n0.3} when we increase the noise to 30\% we observe in second row, \ar is increased the spuriousness, unlike ERM which is at 0. For discussion see \refappendix{subsec:app_extended_synthetic_results}}
\label{fig:app_adv-toy_all}
\end{figure*}

\paragraph{Dropout Regularization Helps \ar method:}
Continuing on observation from  \reffig{fig:app_adv_syn_n0.3-drate0.0}, \ref{fig:app_adv_syn_n0.3-drate0.5} and \ref{fig:app_adv_syn_n0.3-drate0.9} shows the $\Delta$-Prob of the main-task classifier after we apply the \ar on \syn dataset (with noise=0.3) and how they changes as we increase the dropout regularization. As we increase the dropout (drate in the figure), the $\Delta$-Prob of the main classifier decreases showing that the regularization methods could help improve the removal methods.

\begin{figure*}[]
\begin{subfigure}[h]{0.32\textwidth}
    \centering
    \includegraphics[width=\linewidth,height=0.9\linewidth]{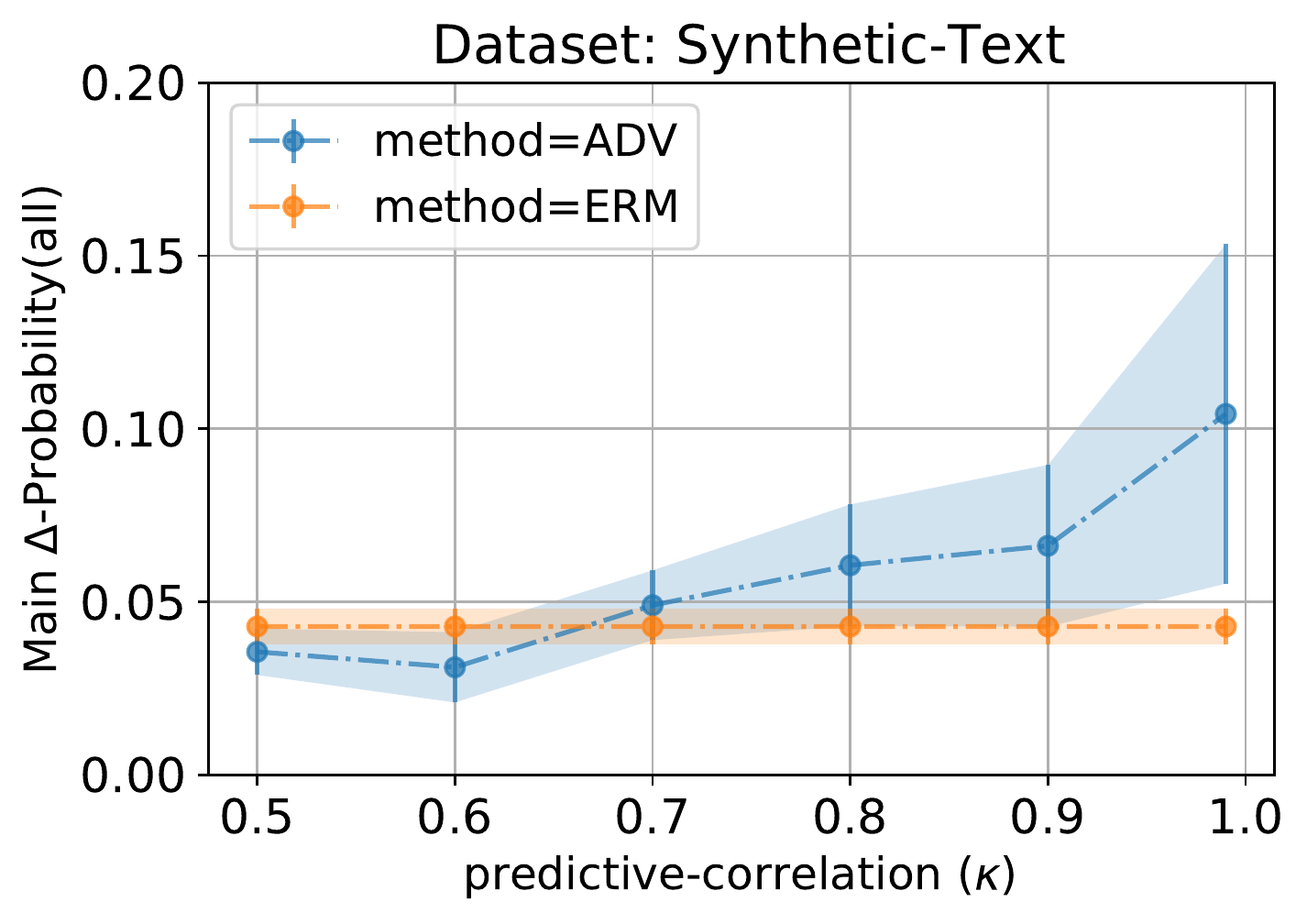}
    	\caption[]%
        {{\small  \syn + drate=0.0}}    
    	\label{fig:app_adv_syn_n0.3-drate0.0} 
    \end{subfigure}
\hfill 
\begin{subfigure}[h]{.32\textwidth}
    \centering
    \includegraphics[width=\linewidth,height=0.9\linewidth]{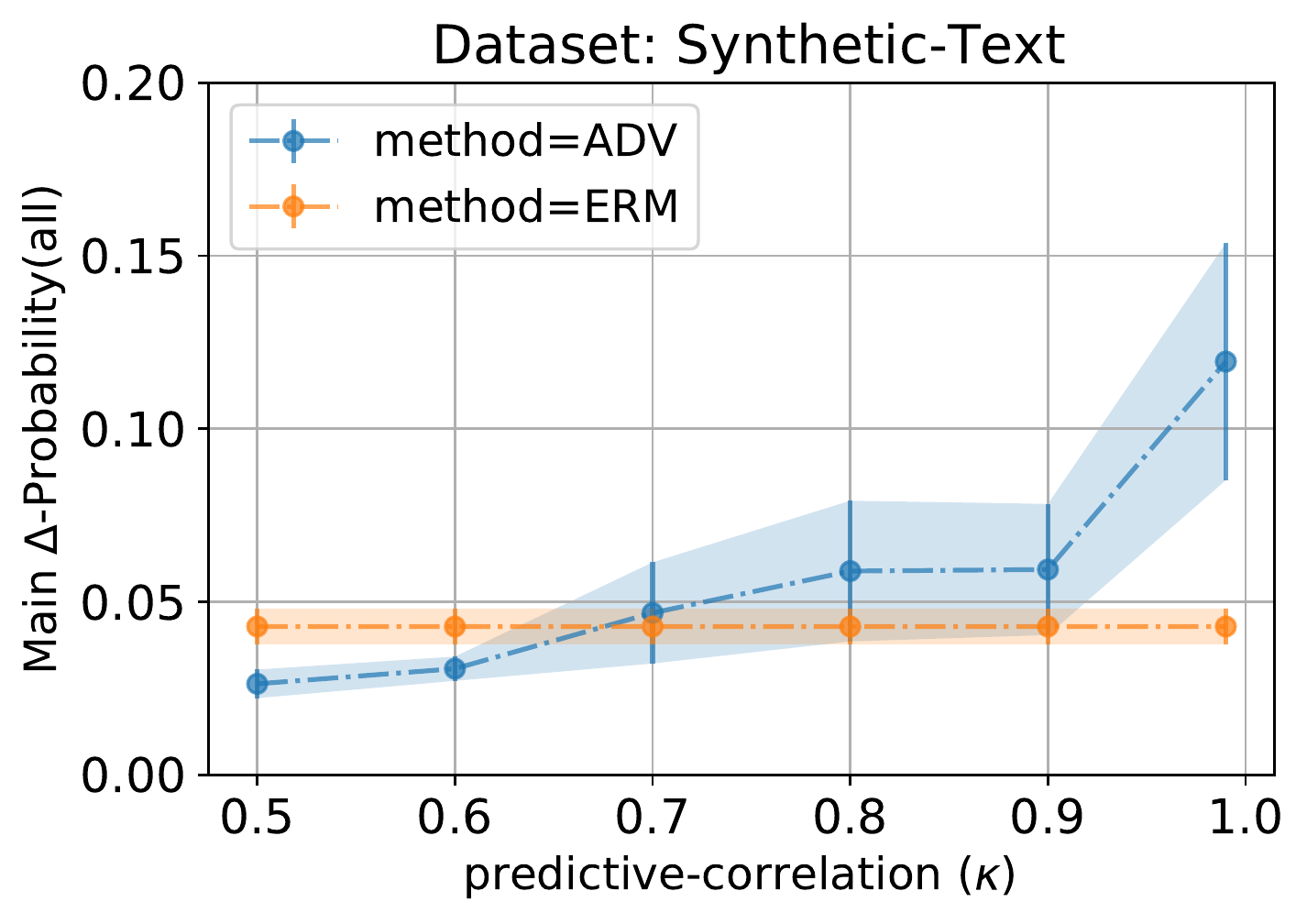}
    	\caption[]%
        {{\small \syn + drate=0.5}}    
    	\label{fig:app_adv_syn_n0.3-drate0.5} 
    \end{subfigure}
\hfill 
\begin{subfigure}[h]{.32\textwidth}
    \centering
    \includegraphics[width=\linewidth,height=0.9\linewidth]{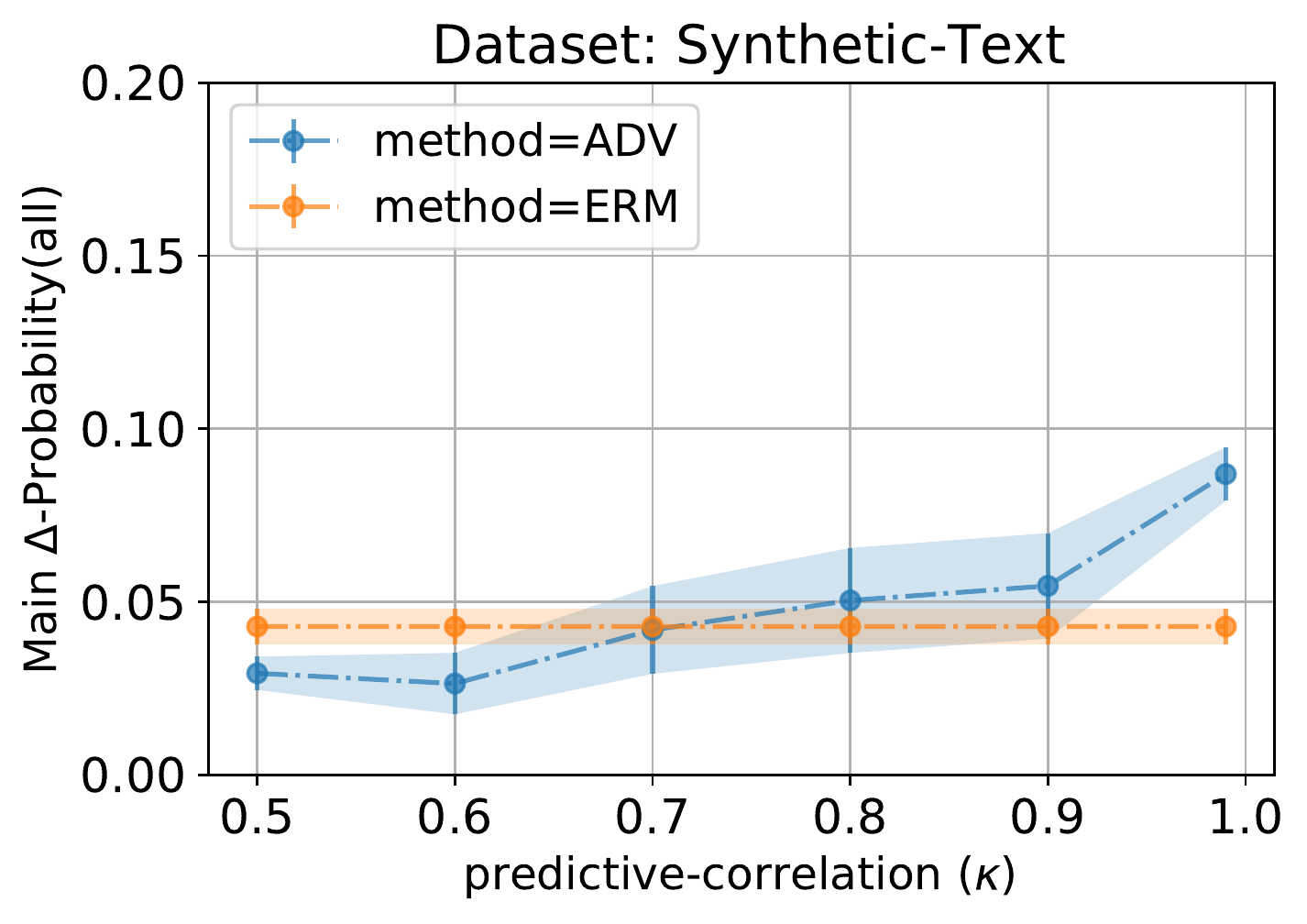}
    	\caption[]%
        {{\small \syn + drate=0.9}}    
    	\label{fig:app_adv_syn_n0.3-drate0.9} 
    \end{subfigure}
\caption{\textbf{Dropout Regularization helps in Adversarial Removal:} $\Delta$-Prob of the main-task classifier after we apply the \ar on \syn dataset (with noise=0.3) decreases as we increase the dropout regularization from 0.0 to 0.9. For discussion see \refappendix{subsec:app_extended_synthetic_results}.}
\label{fig:app_adv-toy_dropout}
\end{figure*}

\section{Comparison between Spuriousness Score and \texorpdfstring{$\Delta$} \texorpdfstring{}Prob }
\label{sec:app_comp_sp_score}

% \begin{figure*}[]
% \begin{subfigure}[h]{0.32\textwidth}
%     \centering
%     \includegraphics[width=\linewidth,height=1.8\linewidth]{figs/adv-syn-all-n0.0_sp.pdf}
%     	\caption[]%
%         {{\small  \syn + n=0.0}}    
%     	\label{fig:app_adv_syn_n0.0_sp} 
%     \end{subfigure}
% \hfill 
% \begin{subfigure}[h]{.32\textwidth}
%     \centering
%     \includegraphics[width=\linewidth,height=1.8\linewidth]{figs/adv-syn-all-n0.1_sp.pdf}
%     	\caption[]%
%         {{\small \syn + n=0.1}}    
%     	\label{fig:app_adv_syn_n0.1_sp} 
%     \end{subfigure}
% \hfill 
% \begin{subfigure}[h]{.32\textwidth}
%     \centering
%     \includegraphics[width=\linewidth,height=1.8\linewidth]{figs/adv-syn-all-n0.3_sp.pdf}
%     	\caption[]%
%         {{\small \syn + n=0.3}}    
%     	\label{fig:app_adv_syn_n0.3_sp} 
%     \end{subfigure}
% \caption{\rbl{
% \textbf{Comparison between Spuriousness Score and $\Delta-$Prob on \syn dataset:}
% This figure is a subset taken from \reffig{fig:app_adv-toy_all} comparing the spuriousness score of main-task classifier with $\Delta-$Prob. For all settings, we observe that our proposed Spuriousness score (in the first row)  is correlated with the $\Delta$-Prob metric (shown in the bottom row). For precise quantification of correlation see \reftbl{tbl:app_sp_pdelta}  and for more discussion see \refappendix{sec:app_comp_sp_score}.}}
% \label{fig:app_comp_sp_score}
% \end{figure*}

In this section we compare the Spuriousness Score proposed in \refsec{subsec:spurriousness_score} for measuring a classifier's use of a binary spurious feature with the ideal, ground-truth metric, $\Delta$Probability ($\Delta$Prob for short) defined in \refsec{subsec:app_metric_desc}. $\Delta$Prob measures the reliance on a spurious feature by changing the spurious feature in the input space (when possible) and measuring the change in the prediction probability of the given classifier. Hence $\Delta$Prob is a direct and intuitive measure of spuriousness in a given classifier. But changing the spurious feature is difficult in the input space for real-world data, thus we only evaluate this metric on the \syn dataset.  

To do so, we use the result from \reffig{fig:app_adv-toy_all} that showed failure of the adversarial removal method on the \syn dataset under various noise settings (refer \refappendix{subsec:app_extended_synthetic_results} for details). For the setting with noise $n=0.0$, both Spuriousness Score and $\Delta$Prob curve for Adversarial Removal (marked as ADV in \reffig{fig:app_adv-toy_all}) are identical (close to $0$ for all values of $\kappa$ with mean $=0.0$ and standard-deviation $=0.0$). For the other settings with non-zero noise, we compute the Pearson correlation between the Spuriousness score and $\Delta$Prob for the ADV curve. As \reftbl{tbl:app_sp_pdelta} shows,  we observe high Pearson correlation of $0.83$ and $0.95$ for the noise setting, $n=0.1$ and $n=0.3$ respectively. The third column in the table shows p-value ($<0.05$) assuming a null hypothesis that the two metrics are uncorrelated.   These results suggest that Spuriousness-Score can be a good approximation for the ideal $\Delta$Prob metric.

\begin{table}[]
\caption{
% \rbl{
\textbf{Correlation between Spuriousness Score and $\Delta$Prob on \syn dataset:} Pearson-correlation between Spuriousness score and $\Delta$Prob; the two metrics for quantifying the dependence of a classifier on a spurious feature. We measure the correlation for adversarial-removal experiment over two different noise setting on \syn dataset. For more details,  see \refappendix{subsec:app_extended_synthetic_results}. The first column shows different experimental settings and the second column shows the Pearson correlation between the two metrics. The third column shows the p-value under the null hypothesis that the two metrics are uncorrelated. Both correlations are statistically significant since p-value for both the case is < 0.05.
% }}
}

\centering
\begin{tabular}{ c c c } 
 \hline
  & Pearson Correlation & p-value \\
 \hline
 \syn + n=0.1 & 0.83   & 0.0403  \\ 
 \syn + n=0.3 & 0.95 & 0.0033 \\ 
 \hline
\end{tabular}
\label{tbl:app_sp_pdelta}
\end{table}

\end{document}